\documentclass{article}

\pdfoutput=1

\usepackage{times}
\usepackage[left=1.125in, top=1in, bottom=1in, right=1.125in]{geometry}
\usepackage{ifthen}

\newboolean{withSupp}
\setboolean{withSupp}{true}

\ifthenelse{\boolean{withSupp}}{\newcommand{\Supp}[2]{#1}}{\newcommand{\Supp}[2]{#2}}

\usepackage{upgreek,morenotations,rotating}
\usepackage{graphicx} 
\usepackage{subfigure} 
\usepackage{multirow}

\graphicspath{{Figs//}{Plots//}{Supp-Plots//}}

\usepackage{natbib}

\usepackage{algorithm}
\usepackage{algorithmic}
\usepackage{makecell}
\usepackage{tabularx}
\usepackage{enumitem}
\usepackage[usenames,dvipsnames]{xcolor}

\usepackage[utf8]{inputenc} 
\usepackage[T1]{fontenc}    
\usepackage[colorlinks=true,citecolor=red]{hyperref} 
\usepackage{url}            
\usepackage{booktabs}       
\usepackage{amsfonts}       
\usepackage{nicefrac}       
\usepackage{microtype}      

\usepackage{cleveref}

\title{A scaled Bregman theorem with applications}

\author{
  Richard Nock \qquad Aditya Krishna Menon \qquad Cheng Soon Ong\\
  Data61 and the Australian National University \\
  \texttt{\{richard.nock, aditya.menon, chengsoon.ong\}@data61.csiro.au} \\
}

\newcommand{\sign}{\mathrm{sign}}

\newcommand{\ie}{i.e.~}

\newenvironment{remark}{\par\noindent{\bf Remark.\ }}{\hfill\BlackBox}
\newenvironment{example}{\par\noindent{\bf Example.\ }}{\hfill\BlackBox}

\newcommand{\XCal}{\mathcal{X}}
\newcommand{\Real}{\mathbb{R}}

\newcommand{\SSf}{\mathsf{S}}

\date{\empty}

\begin{document}

\maketitle

\begin{abstract}

Bregman divergences play a central role in the design and analysis of a range of machine learning algorithms.
This paper explores the use of Bregman divergences to establish \emph{reductions} between such algorithms and their analyses.
We present a new \emph{scaled isodistortion theorem involving Bregman
divergences} (scaled Bregman theorem for short) which shows that
certain ``Bregman distortions'' (employing a potentially \emph{non-convex} generator)
may be \textit{exactly} re-written as a scaled Bregman divergence computed over transformed data.
Admissible distortions include {geodesic distances} on curved manifolds and projections or gauge-normalisation,
while admissible data include scalars, vectors and matrices.

Our theorem allows one to leverage to the wealth and convenience of Bregman divergences when analysing algorithms relying on the aforementioned Bregman distortions. 
We illustrate this with three novel applications of our theorem:
a reduction from multi-class density ratio to class-probability estimation,
a new adaptive
projection free yet norm-enforcing 
dual norm mirror descent
algorithm, 
and
a reduction from clustering on flat manifolds to clustering on curved manifolds.
Experiments on each of these domains validate the analyses and suggest that the scaled Bregman theorem might be a worthy addition to the popular handful of Bregman divergence properties that have been pervasive in machine learning.
\end{abstract}

\section{Introduction: Bregman divergences as a reduction tool}

Bregman divergences play a central role in the design and analysis of a range of machine learning algorithms.
In recent years, Bregman divergences have arisen in procedures for
convex optimisation \citep{Beck:2003},
online learning \citep[Chapter 11]{Cesa-Bianchi:2006}
clustering \citep{bmdgCWj},
matrix approximation \citep{Dhillon:2008},
class-probability estimation \citep{Buja:2005,nnBD,Reid:2010,Reid:2011},
density ratio estimation \citep{Sugiyama:2012},
boosting \citep{cssLRj},
variational inference \citep{hlrhbtBB},
and
computational geometry \citep{bnnBV}.
Despite these being very different applications, 
many of these algorithms and their analyses basically
rely on
three beautiful analytic properties of Bregman divergences,
properties that we summarize for differentiable scalar convex functions $\varphi$ with derivative $\varphi'$, conjugate $\varphi^\star$, and divergence $D_\varphi$:
\vspace{-0.05in}
\begin{itemize}[leftmargin=0.25in]
\item the triangle equality: $D_\varphi(x\|y) + D_\varphi(y\|z) - D_\varphi(x\|z) = ( \varphi'(z) -  \varphi'(y) ) (x-y)$;
\item the dual symmetry property: $D_\varphi(x\|y) = D_{\varphi^\star}(\varphi'(y)\| \varphi'(x))$;
\item the right-centroid (population minimizer) is the average: $\arg\min_\mu \expect [D_\varphi(\X\| \mu)] = \expect[\X]$.
\end{itemize}
\vspace{-0.05in}
Casting a problem as a Bregman minimisation
allows one to
employ these properties to simplify analysis;
for example, by interpreting mirror descent as applying a particular Bregman regulariser, \citet{Beck:2003} relied on the triangle equality above to simplify its proof of convergence.

Another intriguing possibility 
is that one may
derive \emph{reductions} amongst learning problems by connecting their underlying Bregman minimisations.
\citet{Menon:2016} recently established how 
(binary) density ratio estimation (DRE) can be exactly reduced to class-probability estimation (CPE).
This was facilitated by interpreting CPE as a Bregman minimisation \citep[Section 19]{Buja:2005}, and a new property of Bregman divergences ---
\citet[Lemma 2]{Menon:2016} showed that for \emph{any} twice differentiable scalar convex $\varphi$, 
for $g(x) = 1 + x$ and $\varphi^\dagger(x) \defeq g(x) \cdot \varphi(x/g(x))$,
\begin{equation}
	\label{eqn:menon-ong-lemma}
 	g(x) \cdot D_\varphi( x/g(x) \| y/g(y) ) = D_{\varphi^\dagger}( x \| y ).
\end{equation} 
Since the binary class-probability function $\eta(\ve{x}) = \Pr(\Y=1 | \X=\ve{x})$ is related to the class-conditional density ratio 
$r(\ve{x}) = \Pr(\X=\ve{x} | \Y=1)/\Pr(\X=\ve{x} | \Y=-1)$
via Bayes' rule as $\eta(\ve{x}) = r(\ve{x})/g(r(\ve{x}))$,
any $\hat{\eta}$ with small $D_{\varphi}( \eta\| \hat{\eta} )$ implicitly produces an $\hat{r}$ with low $D_{\varphi^\dagger}( r\| \hat{r} )$
\ie a good estimate of the density ratio.
The Bregman property of Equation \ref{eqn:menon-ong-lemma} thus establishes a reduction from DRE to CPE.
Two natural questions arise from this analysis:
can we generalise Equation \ref{eqn:menon-ong-lemma} to other $g(\cdot)$,
and if so,
can we similarly relate \emph{other} problems to each other?

This paper presents a new Bregman identity (Theorem \ref{th00}), the
\emph{scaled Bregman theorem}, a significant
generalisation of 
\citet[Lemma 2]{Menon:2016}.
It shows that general \emph{distortions} $D_{\varphi^\dagger}$ 
-- which are not necessarily convex, positive, bounded or symmetric --
may be re-expressed as a Bregman divergence $D_\varphi$
computed over transformed data, 
where this transformation can be as simple as a projection or normalisation by a gauge,
or more involved like the exponential map on lifted coordinates for a curved manifold.
Interestingly, candidate distortions include {geodesic distances} on curved manifolds. 
Equivalently, Theorem \ref{th00} shows various distortions
can be ``reverse engineered'' as Bregman divergences
(despite appearing \emph{prima facie} to be a very different object),
and thus inherit their good properties.
Hence, Bregman divergences can embed several distances in a different --- and arguably less involved --- way than the transformations known to date \citep{abbBD}. 

As with the aforementioned key properties of Bregman divergences,
Theorem \ref{th00} has potentially wide applicability.
We present three such novel applications 
(see Table \ref{t-red})
to vastly different problems:
\vspace{-\topsep}
\begin{itemize}[leftmargin=0.25in]
	\item a reduction of multiple density ratio estimation to multiclass-probability estimation (\S\ref{sec:multi-class-dr}), generalising the results of \cite{Menon:2016} for the binary label case,
	\item a 
	\emph{projection-free} yet norm-enforcing mirror gradient
        algorithm (enforced norms are those of mirrored vectors
        \textit{and} of the offset) with 
	guarantees for adaptive filtering (\S\ref{sec:pLMS}), and
      \item  a seeding approach for clustering on positively or
        negatively (constant) curved manifolds based on a popular
        seeding for flat manifolds and with the same approximation
        guarantees (\S\ref{sec:clustering}).
\end{itemize}
Experiments on each of these domains (\S\ref{sec:experiments}) validate our analysis.
The Supplementary Material details the proofs of all results, provides the experimental results \textit{in extenso} and some additional (nascent) applications of the scaled Bregman theorem to exponential families and computational geometry.

\begin{table}[!t]
\begin{center}
\renewcommand{\arraystretch}{1.25}
\scalebox{0.9}{
\begin{tabular}{lll}
\hline
\hline
 \textbf{Problem A} & \textbf{Problem B that Theorem \ref{th00} reduces A to} & \textbf{Reference} \\ \hline
Multiclass density-ratio estimation & Multiclass class-probability estimation & \S\ref{sec:multi-class-dr}, Lemma \ref{lemm:multiclass-dr} \\
Online optimisation on $L_q$ ball & Convex unconstrained online learning & \S\ref{sec:pLMS}, Lemma \ref{lemlplq} \\
Clustering on curved manifolds & Clustering on flat manifolds & \S\ref{sec:clustering}, Lemma \ref{lemm:kmeanspp} \\
\hline
\hline
\end{tabular}
}
\end{center}
\caption{Applications of our scaled Bregman Theorem (Theorem \ref{th00})
  --- ``Reduction'' encompasses shortcuts on algorithms \textit{and} on
  analyses (algorithm/proof
  A uses algorithm/proof B as subroutine).\label{t-red}}
\vspace{-0.175in}
\end{table}


\section{Main result: the scaled Bregman theorem}

In the remaining, $[k] \defeq \{0, 1, ..., k\}$ and $[k]_* \defeq \{1, 2, ..., k\}$ for $k\in \mathbb{N}$.
For any differentiable (but not necessarily convex) $\varphi : {\XCal} \rightarrow {\Real}$, we define the Bregman ``distortion'' $D_\varphi$ as
\begin{eqnarray}
D_{\varphi}(
  \ve{x}\|\ve{y}) & \defeq & \varphi(\ve{x}) - \varphi(\ve{y}) -
  (\ve{x}-\ve{y})^\top \nabla \varphi(\ve{y})\:\:. \label{eqBreg}
\end{eqnarray}
When $\varphi$ is convex, $D_\varphi$ is the familiar Bregman divergence with generator $\varphi$. 

Without further ado, we present our main result.
\begin{theorem}\label{th00}
Let, $\varphi :
{\XCal} \rightarrow {\Real}$ be 
convex differentiable, and $g :
{\XCal} \rightarrow {\Real}_{*}$ be differentiable. Then,
\begin{eqnarray}
g(\ve{x}) \cdot D_{\varphi}\left(\ve{x}/g(\ve{x}) \bigm\| \ve{y}/g(\ve{y}) \right) & = & D_{\varphi^\dagger}\left( \ve{x} \bigm\| \ve{y}\right)\:\:, \forall \ve{x}, \ve{y}\in \XCal\:\:,\label{eq11} \\
\text{ where }\varphi^\dagger(\ve{x}) & \defeq & g(\ve{x})\cdot \varphi\left( \ve{x}/{g(\ve{x})} \right)\:\:,\label{defdagger}
\end{eqnarray}
if and only if
(i) $g$ is affine on $\XCal$,
and/or
(ii) for every $\ve{z} \in \XCal_g \defeq\{ (1/g(\ve{x}))\cdot \ve{x} : \ve{x}\in \XCal\}$,
\begin{equation}
  \label{eqn:gradient-condition}
  \varphi\left(\ve{z}\right) = \ve{z}^\top \nabla\varphi(\ve{z})\:\:.
\end{equation}
\end{theorem}

Table \ref{t-ex-short2} presents some examples
of (sometimes involved) triplets $(D_\varphi,  D_{\varphi^\dagger}, g )$ for which Equation \ref{eq11} holds;
related proofs are in \Supp{Appendix \ref{sec-app-der}}{Appendix C}. If we fold $g$ into
$D_\varphi$ in the left hand-side of eq. (\ref{eq11}), then 
Theorem \ref{th00} states a scaled isodistortion
(\textit{sometimes} it turns out to be equivalently an
\textit{adaptive} scaled isometry, see \Supp{Appendix \ref{app:scaled-iso}}{Appendix I}) property
between $\XCal$ and $\XCal_g$. Because $D_\varphi$ is such an
important object, we do not perform this folding and refer to Theorem
\ref{th00} as the \emph{scaled Bregman theorem} for short.

\begin{remark}
If $\XCal_g$ is a vector space, 
$\varphi$ satisfies Equation \ref{eqn:gradient-condition} if and only if it is positive homogeneous of degree 1 on $\XCal_g$ (i.e.\ $\varphi( \alpha \ve{z} ) = \alpha \cdot \varphi( \ve{z} )$ for any $\alpha > 0$) from Euler's homogenous function theorem.
When $\XCal_g$ is not a vector space,
this only holds for $\alpha$ such that $\alpha \ve{z} \in \XCal_g$ as well.
We thus call the gradient condition of Equation \ref{eqn:gradient-condition} ``restricted positive homogeneity'' for simplicity.
\end{remark}

\begin{remark}
\Supp{Appendix \ref{app:deep-bregman}}{Appendix D} gives a ``deep composition'' extension of Theorem \ref{th00}.
\end{remark}

For the special case where $\XCal = \Real$, and $g( x ) = 1 + x$, Theorem \ref{th00} is exactly \citet[Lemma 2]{Menon:2016} (c.f.\ Equation \ref{eqn:menon-ong-lemma}).
We wish to highlight a few points with regard to our more general result.
First, the ``distortion'' generator $\varphi^\dagger$ may be\footnote{Evidently, $\varphi^\dagger$ is convex iff $g$ is non-negative, by Equation (\ref{eq11}) and the fact that a function is convex iff its Bregman ``distortion'' is nonnegative \citep[Section 3.1.3]{Boyd:2004}.}
\emph{non-convex}, as the following illustrates.

\begin{example}
Suppose $\varphi( \ve{x} ) = \frac{1}{2} \| \ve{x} \|_2^2$ corresponds to the generator for squared Euclidean distance.
Then, for $g( \ve{x} ) = 1 + \ve{1}^\top \ve{x}$, we have
$ \varphi^{\dagger}( \ve{x} ) = \frac{1}{2} \cdot \frac{\| \ve{x} \|_2^2}{1 + \ve{1}^\top \ve{x}}, $
which is non-convex on $\XCal = \Real^d$.
\end{example}

When $\varphi^\dagger$ is non-convex, the right hand side in Equation \ref{eq11} is an object
that ostensibly bears only a superficial similarity to a Bregman divergence;
it is somewhat remarkable that Theorem \ref{th00} shows this general ``distortion'' between a pair $(\ve{x}, \ve{y})$
to be entirely equivalent to a (scaling of a) Bregman divergence between some {transformation} of the points.
Second, when $g$ is linear, Equation \ref{eq11} holds for \emph{any} convex $\varphi$.
(This was the case considered in \citet{Menon:2016}.)
When $g$ is non-linear, however, $\varphi$ must be chosen carefully so that $( \varphi, g )$ satisfies the restricted homogeneity conditon\footnote{We stress that this condition only needs to hold on $\XCal_g \subseteq \XCal$.; it would not be really interesting in general for $\varphi$ to be homogeneous \textit{everywhere} in its domain, since we would basically have $\varphi^\dagger = \varphi$.} of Equation \ref{eqn:gradient-condition}.
In general, given a convex $\varphi$, one can ``reverse engineer'' a suitable $g$ to guarantee this conditon,
as illustrated by the following example.

\begin{example}
Suppose\footnote{The constant $1/2$ added in $\varphi$ does \textit{not} change $D_\varphi$, since a Bregman divergence is invariant to affine terms; removing this however would make the divergences $D_\varphi$ and $D_{\varphi^\dagger}$ differ by a constant.}
$\varphi( \ve{x} ) = (1 + \| \ve{x} \|_2^2)/2$.
Then, Equation \ref{eqn:gradient-condition} requires that
$\|\ve{x}\|_2^2 = 1$ for every $\ve{x} \in \XCal_g$, i.e.\ $\XCal_g$ is (a subset of) the unit sphere.
This is afforded by the choice $g( \ve{x} ) = \| \ve{x} \|_2$.
\end{example}

\begin{table}[!t]
\renewcommand{\arraystretch}{1.25}
\begin{center}
\scalebox{0.775}{
\begin{tabular}{l|l|l|l}
\hline\hline
 $\XCal$ & $D_{\varphi}\left(\ve{x}\|\ve{y}\right)$ & $D_{\varphi^\dagger}\left(\ve{x}\|\ve{y}\right)$ & $g(\ve{x})$ \\
 \hline 
 $\Real^d$ & $\frac{1}{2}\cdot \|\ve{x}-\ve{y}\|_2^2$ & $\|\ve{x}\|_2 \cdot \left( 1 - \cos \angle \ve{x}, \ve{y}\right)$ & $\| \ve{x} \|_2$ \\
 $\Real^d$ & $\frac{1}{2}\cdot(\|\ve{x}\|_q^2-\|\ve{y}\|_q^2) -\sum_i \frac{(x_i - y_i) \cdot \mathrm{sign}(y_i) \cdot |y_i|^{q-1}}{\|\ve{y}\|_q^{q-2}}$ &  $W \cdot \|\ve{x}\|_q - W \cdot \sum_i \frac{x_i \cdot \mathrm{sign}(y_i) \cdot |y_i|^{q-1}}{\|\ve{y}\|_q^{q-1}}$ & $\| \ve{x} \|_q/W$ \\
 $\Real^d\times \Real$ & $\frac{1}{2}\cdot \|\ve{x}^S-\ve{y}^S\|_2^2$ & $\frac{\|\ve{x}\|_2}{\sin \|\ve{x}\|_2} \cdot \left( 1 - \cos D_G(\ve{x}, \ve{y})\right)$ & $\| \ve{x} \|_2/\sin \| \ve{x} \|_2$ \\
 $\Real^d\times \mathbb{C}$ & $\frac{1}{2}\cdot \|\ve{x}^H-\ve{y}^H\|_2^2$ & $-\frac{\|\ve{x}\|_2}{\sinh \|\ve{x}\|_2} \cdot \left( \cosh D_G(\ve{x}, \ve{y}) - 1\right)$ & $-\| \ve{x} \|_2/\sinh \| \ve{x} \|_2$ \\
 $\Real_+^d$ & $\sum_i x_i \log \frac{x_i}{y_i}- \ve{1}^\top(\ve{x}-\ve{y})$ & $\sum_i x_i\log \frac{x_i}{y_i} -d\cdot \expect[\X] \cdot \log \frac{\expect[\X]}{\expect[\Y]}$ & $\ve{1}^\top \ve{x}$ \\ 
 $\Real_+^d$ & $\sum_i\frac{x_i}{y_i} -\sum_i\log\frac{x_i}{y_i} - d$ & $\sum_i \frac{x_i (\prod_j y_j)^{1/d}}{y_i}  - d (\prod_j x_j)^{1/d}$ & $\prod_i x_i^{1/d}$ \\ 
 $\textbf{S}(d)$ & $\trace{\matrice{x}\log \matrice{x} - \matrice{x}\log \matrice{y}}-\trace{\matrice{x}} + \trace{\matrice{y}}$ & $\trace{\matrice{x}\log \matrice{x} - \matrice{x}\log \matrice{y}}-\trace{\matrice{x}} \cdot \log\frac{\trace{\matrice{x}}}{\trace{\matrice{y}}}$ & $\trace{\matrice{x}}$ \\
 $\textbf{S}(d)$ & $\trace{\matrice{x}\matrice{y}^{-1}}-\log\det (\matrice{x}\matrice{y}^{-1}) - d$ & $\det (\matrice{y}^{1/d}) \trace{\matrice{x}\matrice{y}^{-1}} -d\cdot \det (\matrice{x}^{1/d})$ & $\det(\matrice{x}^{1/d})$ \\
\bottomrule
\end{tabular}
}
\end{center}
\caption{Examples of $( D_\varphi, D_{\varphi^\dagger}, g )$ for which Equation \ref{eq11} holds.
Function $\ve{x}^S \defeq f(\ve{x}) : {\Real}^d \rightarrow
{\Real}^{d+1}$ and $\ve{x}^H \defeq f(\ve{x}) : {\Real}^d \rightarrow
{\Real}^{d}\times \mathbb{C}$ are the Sphere and Hyperbolic lifting
maps defined in \Supp{Equation \ref{defSmap}, \ref{defHymap}}{Equation 67, 78}. $W>0$ is a constant. $D_G$ denotes the $G$eodesic
distance on the sphere (for $\ve{x}^S$) or the hyperboloid (for $\ve{x}^H$).
$\textbf{S}(d)$ is the set of symmetric real matrices.
Related proofs are in \Supp{Section \ref{sec-app-der}}{Section C}.\label{t-ex-short2}}
\vspace{-0.25in}
\end{table}

Third, Theorem \ref{th00} is not merely a mathematical curiosity:
we now show that it facilitates novel results in three very different domains,
namely
estimating multiclass density ratios,
constrained online optimisation,
and
clustering data on a manifold with non-zero  
curvature.
We discuss nascent applications to exponential families and computational geometry in \Supp{Appendices \ref{app:exp-family} and \ref{app:comp-geom}}{Appendices E and F}.

\section{Multiclass density-ratio estimation via class-probability estimation}

Given samples from a number of densities, density ratio estimation concerns estimating the ratio between each density and some reference density.
This has applications in the covariate shift problem wherein the train and test distributions over instances differ \citep{Shimodaira:2000}.
Our first application of Theorem \ref{th00} is to show how density ratio estimation can be reduced to class-probability estimation \citep{Buja:2005,Reid:2010}.

To proceed, we fix notation.
For some integer $C \geq 1$, consider a distribution $\pr(\X, \Y)$ over an (instance, label) space $\XCal \times [C]$.
Let $(\{P_c\}_{c=1}^C,\ve{\pi})$ be densities giving $\pr(\X | \Y = c)$ and $\pr(\Y = c)$ respectively, and $M$ giving $\pr(\X)$ accordingly.
Fix $c^* \in [C]$ a reference class, and suppose for simplicity that $c^* = C$.
Let $\tilde{\ve{\pi}} \in \bigtriangleup^{C-1}$ such that
$\tilde{\pi}_c \defeq \pi_c / (1-\pi_C)$.
	\emph{Density ratio estimation} \citep{Sugiyama:2012} concerns inferring the vector $\ve{r}(\ve{x}) \in {\Real}^{C-1}$ of density ratios relative to $C$, with
$ r_c (\ve{x}) \defeq {\pr(\X = \ve{x} | \Y = c)}/{\pr(\X = \ve{x} | \Y = C)}\:\:, $
while
	\emph{class-probability estimation} \citep{Buja:2005} concerns inferring the vector $\ve{\eta}(\ve{x}) \in {\Real}^{C-1}$ of class-probabilities, with
$ \eta_c (\ve{x}) \defeq \pr(\Y = c | \X = \ve{x})/\tilde{\pi}_c \:\:. $
In both cases, we estimate the respective quantities given an iid sample $\SSf \sim \pr(\X, \Y)^N$.

\label{sec:multi-class-dr}




The genesis of the reduction from density ratio to class-probability estimation is
the fact that $\ve{r}( \ve{x} ) = (\pi_C/(1 - \pi_C)) \cdot \ve{\eta}(\ve{x})/\eta_C(\ve{x})$.
In practice one will only have an estimate $\hat{\ve{\eta}}$, typically derived by minimising a suitable loss on the given $\SSf$ \citep{Williamson:2014},
with a canonical example being multiclass logistic regression.
Given $\hat{\ve{\eta}}$, it is natural to estimate the density ratio via:
\begin{eqnarray}
 \hat{\ve{r}}(\ve{x}) & = & \frac{\hat{\ve{\eta}}(\ve{x})}{\hat{{\eta}}_C( \ve{x} )} \label{eqn:dr-from-cpe}\:\:. 
\end{eqnarray}
While this estimate is intuitive, to establish a formal reduction we must relate the quality of $\hat{\ve{r}}$ to that of $\hat{\ve{\eta}}$.
Since the minimisation of a suitable loss for class-probability estimation is equivalent to a Bregman minimisation \citep[Section 19]{Buja:2005}, \citep[Proposition 7]{Williamson:2014},
this is however immediate by Theorem \ref{th00}, as shown below.

\begin{lemma}
\label{lemm:multiclass-dr}
Given a class-probability estimator $\hat{\eta} \colon \XCal \to [0, 1]^{C - 1}$, let the density ratio estimator $\hat{r}$ be
as per Equation \ref{eqn:dr-from-cpe}.
Then for any convex differentiable $\varphi \colon [0, 1]^{C - 1} \to \Real$,
\begin{eqnarray}
\label{eqn:nock-menon-ong}
\expect_{\X \sim M}[D_\varphi(\ve{\eta}(\X) \| \hat{\ve{\eta}}(\X))] & = & (1-\pi_C) \cdot \expect_{\X \sim P_C}\left[ D_{\varphi^\dagger}(\ve{r}(\X) \| \hat{\ve{r}}(\X))\right]\:\:
\end{eqnarray}
where $\varphi^\dagger$ is as per Equation \ref{defdagger} with
$ g(\ve{x}) \defeq {\pi_C}/({1-\pi_C}) + \tilde{\ve{\pi}}^\top \ve{x} \:\:.$ 
\end{lemma}

Lemma \ref{lemm:multiclass-dr} generalises \citet[Proposition 3]{Menon:2016}, which focussed on the binary case with $\pi = \nicefrac[]{1}{2}$.
(See \Supp{Appendix \ref{app:binary-dr}}{Appendix G} for a review of that result.)
Unpacking the Lemma, the LHS in Equation \ref{eqn:nock-menon-ong} represents the object minimised by some suitable loss for class-probability estimation.
Since $g$ is affine, we can use \emph{any} convex, differentiable $\varphi$,
and so can use \emph{any} suitable class-probability loss to estimate $\hat{\ve{\eta}}$. 
Lemma \ref{lemm:multiclass-dr} thus implies that
producing $\hat{\ve{\eta}}$ by minimising any class-probability loss \emph{equivalently}
produces an $\hat{\ve{r}}$ as per Equation \ref{eqn:dr-from-cpe} that
minimises a Bregman divergence to the true $\ve{r}$.
Thus, Theorem \ref{th00} provides a reduction from density ratio to multiclass probability estimation.

We now detail two applications where $g(\cdot)$ is no longer affine, and $\varphi$ must be chosen more carefully.

\section{Dual norm mirror descent: projection-free online learning on $L_p$ balls}
\label{sec:pLMS}

A substantial amount of work in the intersection of machine learning and convex optimisation has focused on constrained optimisation 
within
a ball \citep{Shalev-Shwartz:2007,Duchi:2008}.
This optimisation is typically via projection operators that can be expensive to compute \citep{Hazan:2012,jRF}.
We now show that \emph{gauge functions} can be used as an inexpensive alternative,
and that Theorem \ref{th00} easily yields guarantees for this procedure in online learning.

%

We consider the adaptive filtering problem, closely related to the online least squares problem with linear predictors \citep[Chapter 11]{Cesa-Bianchi:2006}.
Here, over a sequence of $T$ rounds, we observe some $\ve{x}_t \in \XCal$.
We must the predict a target value $\hat{y}_t = \ve{w}_{t-1}^\top \ve{x}_t$ using our current weight vector $\ve{w}_{t-1}$.
The true target $y_t = \ve{u}^\top \ve{x}_t + \epsilon_t$ is then revealed, where $\epsilon_t$ is some unknown noise, and we may update our weight to $\ve{w}_t$.
Our goal is to minimise the regret of the sequence $\{ \ve{w}_t \}_{t = 0}^{T}$,
\begin{equation}
  \label{eqn:regret}
   R( \ve{w}_{1:T} | \ve{u} )  \defeq \sum_{t=1}^T \left( \ve{u}^\top \ve{x}_t - \ve{w}_{t-1}^\top \ve{x}_t \right)^2 - \sum_{t=1}^T \left( \ve{u}^\top \ve{x}_t - y_t \right)^2\:\:.
\end{equation}
Let $q \in (1, 2]$ and $p$ be such that $1/p + 1/q = 1$.
For $\varphi \defeq \frac{1}{2} \cdot \| \ve{x} \|_q^2$
and loss $\ell_t( \ve{w} ) = \frac{1}{2} \cdot (y_t - \ve{w}^\top \ve{x}_t)^2$,
the $p$-LMS algorithm \citep{kwhTP} employs the stochastic mirror gradient updates
\begin{eqnarray}
  \label{eqn:plms-implicit}
  \ve{w}_t & \defeq & \underset{\ve{w}}{\operatorname{argmin}} \, {\eta_t} \cdot \ell_t( \ve{w} ) + D_{\varphi}( \ve{w} \| \ve{w}_{t - 1} )
  \label{eqn:plms-explicit}
  = (\nabla\varphi)^{-1}\left( \nabla\varphi(\ve{w}_{t-1}) - \eta_t \cdot \nabla\ell_t \right),
\end{eqnarray}
where 
$\eta_t$ is a learning rate to be specified by the user.
\citet[Theorem 2]{kwhTP} shows that 
for appropriate $\eta_t$, one has $R( \ve{w}_{1:T} | \ve{u} ) \leq (p - 1) \cdot \max_{\ve{x} \in \XCal} \| \ve{x} \|_p^2 \cdot \| \ve{u} \|_q^2$.

%

The $p$-LMS updates do not provide any explicit control on $\|
\ve{w}_t \|$, \ie there is no regularisation. Experiments ($\S$\ref{sec:experiments}) suggest that leaving $\|
\ve{w}_t \|$ uncontrolled may not be a good idea as the increase
of the norm sometimes prevents (significant) updates
(\ref{eqn:plms-explicit}). Also, the wide success of regularisation in
machine learning calls for regularised variants that \emph{retain} the regret guarantees and computational efficiency of $p$-LMS.
(Adding a projection step to Equation \ref{eqn:plms-explicit} would
not achieve both.)
We now do just this.
For fixed $W > 0$, let $\varphi \defeq (1/2)(W^2 + \|\ve{x}\|_q^2)$, a translation of that used in $p$-LMS.
Invoking Theorem \ref{th00} with the admissible $g_q(\ve{x}) = ||\ve{x}||_q/W$
yields $\varphi^\dagger \defeq \varphi_q^\dagger = W \|\ve{x}\|_q$ (see Table \ref{t-ex-short2}).
Using the fact that $L_p$ and $L_q$ norms are dual of each other,
we replace Equation \ref{eqn:plms-implicit} by:
\begin{eqnarray}
\ve{w}_t & \defeq & \nabla{\varphi_p^\dagger}\left( \nabla\varphi_q^\dagger(\ve{w}_{t-1}) - \eta_t \cdot \nabla\ell_t \right)\:\:.\label{eqn:plms-dagger-explicit}
\end{eqnarray}
See \Supp{Lemma \ref{lemm:varphi-properties}}{Lemma 6} of the Appendix for the simple forms of $\nabla\varphi_{\{p,q\}}^\dagger$.
We call update (\ref{eqn:plms-dagger-explicit}) the 
\emph{dual norm $p$-LMS (DN-pLMS) algorithm},
noting that the dual refers to the polar transform
of the norm,
and $g$ stems from a gauge normalization for ${\mathcal{B}}_q(W)$, the
closed $L_q$ ball with radius $W > 0$. Namely, we have
$\gamma_{\textsc{gau}}(\ve{x}) = W /\|\ve{x}\|_q = g(\ve{x})^{-1}$ for
the gauge $\gamma_{\textsc{gau}}(\ve{x}) \defeq \sup\{z\geq 0 : z \cdot \ve{x} \in {\mathcal{B}}_q(W)\}$,
so that $\varphi_q^{\dagger}$ implicitly performs gauge normalisation of the data.
This update is no more computationally expensive than Equation
\ref{eqn:plms-explicit} --- we simply need to compute the $p$- and
$q$-norms of appropriate terms --- but, crucially, automatically
constrains the norms of $\ve{w}_t$ and its image by $\nabla{\varphi_q^{\dagger}}$.

\begin{lemma}
\label{lemm:plms-update-norm}
For the update in Equation \ref{eqn:plms-dagger-explicit},
$\|\ve{w}_t\|_q = \| \nabla{\varphi_q^{\dagger}}(\ve{w}_t) \|_p = W, \forall t > 0$.
\end{lemma}



Lemma \ref{lemm:plms-update-norm} is remarkable, since \emph{nowhere in Equation \ref{eqn:plms-dagger-explicit} do we project onto the $L_q$ ball}.
Nonetheless, for the DN-pLMS updates to be principled, we need a similar regret guarantee to the original $p$-LMS.
Fortunately, this may be done using Theorem \ref{th00} to exploit the original proof of \citet{kwhTP}.
For any $\ve{u}\in {\Real}^d$,
define the \emph{$q$-normalised regret} of $\{ \ve{w}_t \}_{t = 0}^{T}$ by
\begin{eqnarray}
R_q(\ve{w}_{1:T} | \ve{u})  & \defeq & \sum_{t=1}^T \left((1/{g_{q}(\ve{u})}) \cdot \ve{u}^\top \ve{x}_t - \ve{w}_{t-1}^\top \ve{x}_t\right)^2 - \sum_{t=1}^T \left((1/{g_{q}(\ve{u})})\cdot \ve{u}^\top \ve{x}_t - y_t\right)^2\:\:.
\end{eqnarray}
We have the following bound on $R_q$ for the DN-pLMS updates.
(We cannot expect a bound on the unnormalised $R(\cdot)$ of
Equation \ref{eqn:regret}, since by Lemma \ref{lemm:plms-update-norm} we
can only compete against norm $W$ vectors.)

\begin{lemma}\label{lemlplq}
Pick any $\ve{u}\in {\Real}^d$, $p, q$ satisfying $1/p + 1/q = 1$
and $p>2$, and $W > 0$.
Suppose $\|\ve{x}_t\|_p \leq X_p$ and $|y_t| \leq Y, \forall t \leq T$.
Let $\{ \ve{w}_t \}$ be as per Equation
\ref{eqn:plms-dagger-explicit}, using learning rate
\begin{eqnarray}
\eta_t & \defeq& \upgamma_t \cdot \frac{W}{4(p-1)\max\{W,X_p\}X_pW+|y_t -
  \ve{w}_{t-1}^\top \ve{x}_t|X_p}\:\:,
\end{eqnarray}
for \textbf{any} desired $\upgamma_t \in [1/2, 1]$. Then,
\begin{eqnarray}
R_q(\ve{w}_{1:T} | \ve{u}) & \leq & 4(p-1)X_p^2W^2 + (16p-8) \max\{W,X_p\}X_p^2W + 8YX^2_p\:\:.\label{eqb1}
\end{eqnarray}
\end{lemma}
Several remarks can be made. First, the bound depends on the maximal
signal value $Y$, but this is the maximal signal in the observed
sequence, so it may not be very large in practice; if
it is comparable to $W$, then our bound is looser than \cite{kwhTP} by
just a constant factor. Second, the learning rate is adaptive in the
sense that its choice depends on the last mistake made. There is a
nice way to represent the ``offset'' vector $\eta_t \cdot \nabla\ell_t $ in
eq. (\ref{eqn:plms-dagger-explicit}), since we have, for $Q'' \defeq 4(p-1)\max\{W,X_p\}X_pW$,
\begin{eqnarray}
\eta_t \cdot \nabla\ell_t & = & W \cdot \frac{|y_t - \ve{w}_{t-1}^\top \ve{x}_t|X_p}{Q''+|y_t - \ve{w}_{t-1}^\top \ve{x}_t|X_p} \cdot \mathrm{sign}(y_t - \ve{w}_{t-1}^\top \ve{x}_t)\cdot \left(\frac{1}{X_p}\cdot \ve{x}\right)\:\:,
\end{eqnarray}
so the $L_p$ norm of the offset is actually equal to $W \cdot Q$, where $Q\in
[0,1]$ is all the smaller as the vector $\ve{w}_.$ gets better. Hence,
the
update in eq. (\ref{eqn:plms-dagger-explicit}) controls in fact
\textit{all} norms (that of $\ve{w}_.$, its image by
$\nabla{\varphi_q^{\dagger}}$ and the offset). Third, because of the normalisation of $\ve{u}$, the bound actually does not depend on $\ve{u}$, but on the radius $W$ chosen for the $L_q$ ball.

\section{Clustering on a manifold via data transformation}
\label{sec:clustering}

\begin{figure}[t]
\centering
\scalebox{0.925}{
\begin{tabular}{c|c||c}\hline\hline
\multicolumn{2}{c||}{Sphere} & Hyperboloid\\ \hline
\includegraphics[trim=70bp 250bp 0bp 0bp,clip,width=.40\linewidth]{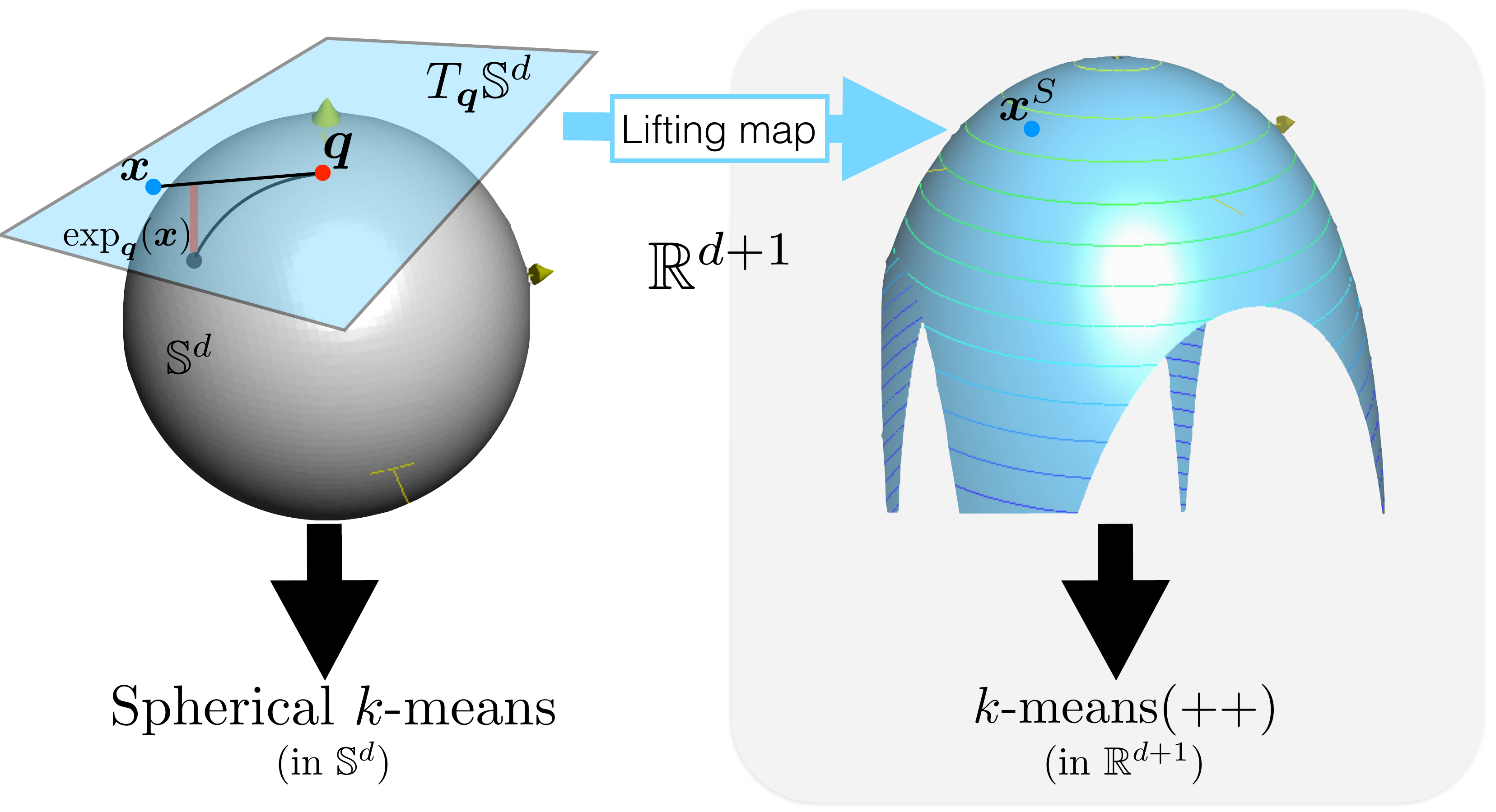} & \includegraphics[trim=40bp 0bp 40bp 0bp,clip,width=.21\linewidth]{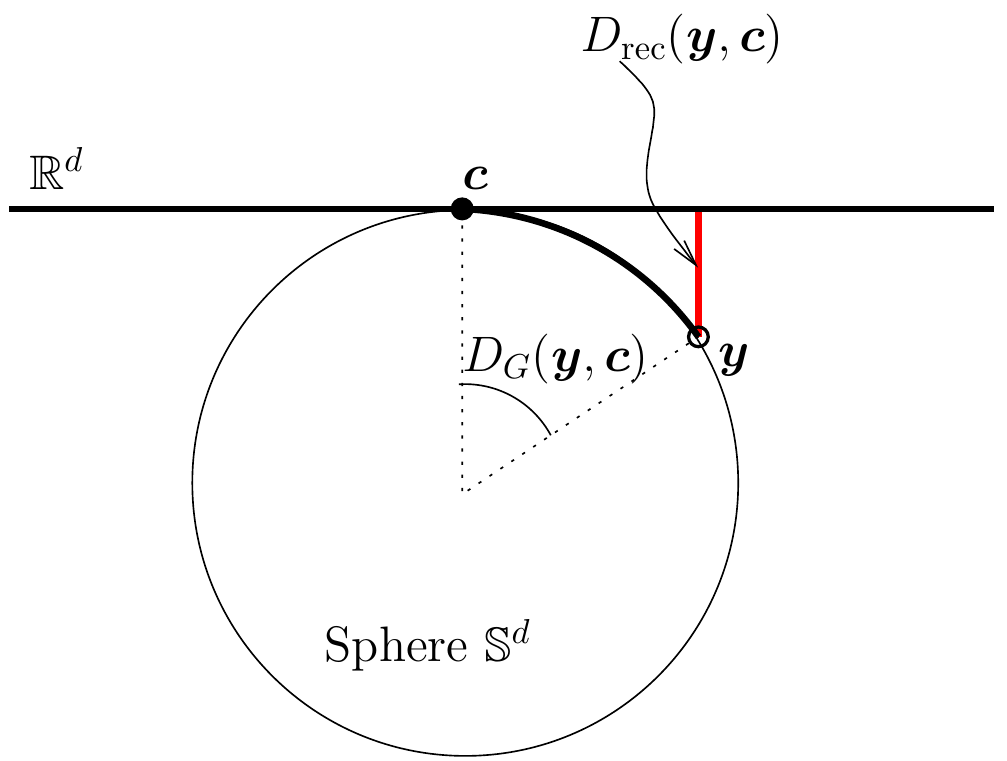} & \includegraphics[trim=470bp 250bp 0bp 0bp,clip,width=.23\linewidth]{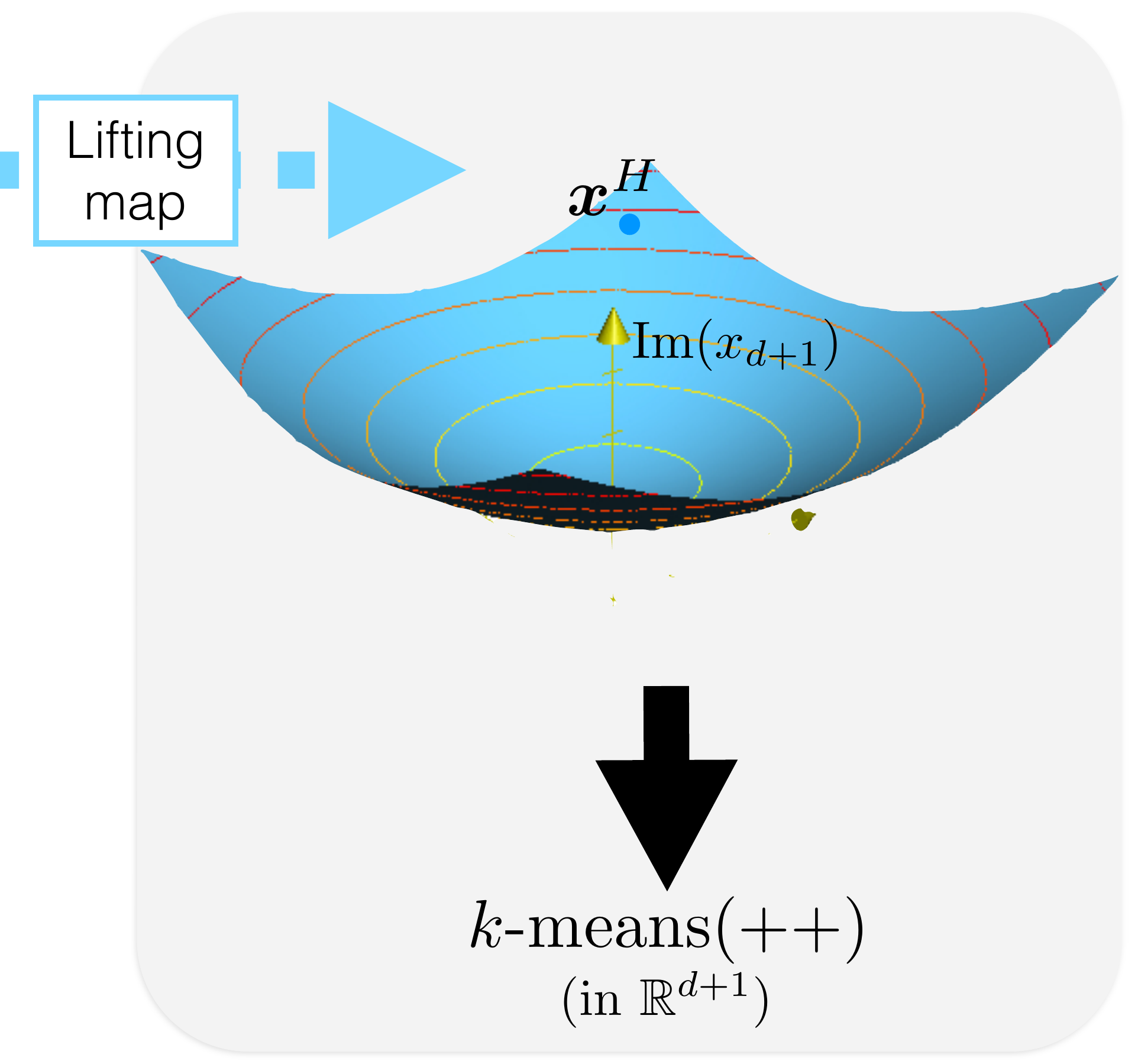}\\ \hline\hline
\end{tabular}
}

\caption{(L) Lifting map into $\Real^d \times \mathbb{\Real}$ for clustering on the sphere with k-means++.
(M) $D_\rec$ in Eq. (\ref{cpsol}) in vertical thick red line.
(R) Lifting map into $\Real^d \times \mathbb{C}$ for the hyperboloid.
	\label{fig:lift}}
\vspace{-0.125in}
\end{figure}

Our final application can be related to two problems that have received a
steadily growing interest over the past decade in unsupervised machine learning: clustering on a
non-linear manifold \citep{dmCD}, and subspace custering
\citep{vSC}. We consider two fundamental manifolds investigated by \cite{gAC} to compute
centers of mass from relativistic theory: the sphere $\mathbb{S}^d$ and the
hyperboloid $\mathbb{H}^d$, the former being of positive curvature, and the latter of
negative curvature. 
Applications involving these specific manifolds are numerous in
text processing, computer vision, geometric modelling, computer
graphics, to name a few
\citep{bfSA,dmCD,emSK,kypSN,rjgHC,sgrgfvTD,sblfRT,schfSV,scffAD,srflfAM}. We
emphasize the fact that the clustering problem has significant practical impact for $d$ as
small as 2 in computer vision \citep{srflfAM}.

The problem is non-trivial for two separate reasons.
First, the ambient space, \textit{i.e.} the space of registration of the input data, is often implicitly Euclidean and therefore \textit{not} the manifold \citep{dmCD}: if the mapping to the manifold is not carefully done, then geodesic distances measured on the manifold may be inconsistent with respect to the ambient space.
Second, the fact
that the manifold has non-zero curvature essentially prevents the direct use of Euclidean optimization algorithms \citep{zsFO} --- put simply, the average of two points that belong to a manifold does not necessarily belong to the manifold, so we have to be careful on how to compute centroids for hard clustering \citep{gAC,nnaOC,rjgHC,snLM}.

What we show now is that Riemannian manifolds with constant
sectional curvature may be clustered with the $k$-means++
seeding for flat manifolds \citep{avKM}, \emph{without even touching a line of the algorithm}.
To formalise the problem, we need three key components of Riemannian
geometry: tangent planes, exponential map and geodesics
\citep{anMO}. We assume that the ambient space is a tangent plane to
the manifold $\mathcal{M}$, which conveniently makes it look Euclidean (see Figure \ref{fig:lift}).
The point of tangency is called $\ve{q}$, and the tangent plane
$T_{\ve{q}}\mathcal{M}$. The exponential map,
$\exp_{\ve{q}}: T_{\ve{q}}\mathcal{M} \rightarrow \mathcal{M}$,
performs a distance preserving mapping: the geodesic length between
$\ve{q}$ and $\exp_{\ve{q}}(\ve{x})$ in $\mathcal{M}$ is the same as the Euclidean
length between $\ve{q}$ and $\ve{x}$ in $T_{\ve{q}}\mathcal{M}$. Our
clustering objective is to find $\mathcal{C} \defeq \{\ve{c}_1,
\ve{c}_2, ... \ve{c}_k\} \subset \mathcal{M}$ such that
$D_\rec(\mathcal{S} : \mathcal{C}) = \inf_{\mathcal{C}' \subset
  \mathcal{M}, |\mathcal{C}'|=k} D_\rec(\mathcal{S}, \mathcal{C}')$, with
\begin{eqnarray}
D_\rec(\mathcal{S}, \mathcal{C}) & \defeq \sum_{i\in [m]_*} \min_{j\in [k]_*}
D_\rec(\exp_{\ve{q}}(\ve{x}_i), \ve{c}_j)\:\:,\label{cpsol}
\end{eqnarray}
where $D_\rec$ is a \textit{rec}onstruction loss, a function
of the geodesic distance between
$\exp_{\ve{q}}(\ve{x}_i)$ and $\ve{c}_j$. We use two loss functions
defined from \cite{gAC} and used in machine learning for more
than a decade \citep{dmCD}:
\begin{eqnarray}
{\mathbb{R}}_+ \ni D_\rec(\ve{y}, \ve{c}) & \defeq &
\left\{\begin{array}{rcl}
1 - \cos D_G(\ve{y}, \ve{c}) &
\mbox{for } & \mathcal{M} = \mathbb{S}^d\\
\cosh D_G(\ve{y}, \ve{c}) - 1&
\mbox{for } & \mathcal{M} = \mathbb{H}^d
\end{array}\right.\:\:.\label{defdist}
\end{eqnarray}
Here, $D_G(\ve{y}, \ve{c})$ is the corresponding geodesic distance of
$\mathcal{M}$ between $\ve{y}$ and $\ve{c}$. Figure \ref{fig:lift}
shows that $D_\rec(\ve{y}, \ve{c})$ is the orthogonal distance between $T_{\ve{c}}\mathcal{M}$
and $\ve{y}$ when $\mathcal{M} = \mathbb{S}^d$. The solution to the
clustering problem in eq. (\ref{cpsol}) is therefore the one that
minimizes the error between tangent planes defined at the centroids,
and points on the manifold.

\begin{table}[t]
 \centering
\scalebox{0.9}
{
\begin{tabular}{c|c}\hline\hline
(Sphere) \GKM(${\mathcal{S}}, k$) & (Hyperboloid) \HKM(${\mathcal{S}}, k$) \\ \hline
\begin{minipage}{.49\linewidth}
\begin{algorithmic}
\STATE  \textbf{Input:} dataset ${\mathcal{S}} \subset T_{\ve{q}}{\mathbb{S}^d}, k\in {\mathbb{N}}_*$;
\STATE Step 1: ${\mathcal{S}}^+ \leftarrow \{g^{-1}_{S}(\ve{x}^S) \cdot \ve{x}^S : \ve{x}^S\in \mathrm{lift}({\mathcal{S}})\}$; 
\STATE Step 2: ${\mathcal{C}}^+\leftarrow k\mbox{-means++$\_$seeding}({\mathcal{S}}^+, k)$;
\STATE Step 3: ${\mathcal{C}}\leftarrow \exp^{-1}_{\ve{q}}({\mathcal{C}}^+)$;
\STATE \textbf{Output:} Cluster centers ${\mathcal{C}} \in T_{\ve{q}}{\mathbb{S}^d}$;
\end{algorithmic}
\end{minipage}
&
\begin{minipage}{.49\linewidth}
\begin{algorithmic}
\STATE  \textbf{Input:} dataset ${\mathcal{S}} \subset T_{\ve{q}}{\mathbb{H}^d}, k\in {\mathbb{N}}_*$;
\STATE Step 1: ${\mathcal{S}}^+ \leftarrow \{g^{-1}_{H}(\ve{x}^H) \cdot \ve{x}^H : \ve{x}^H\in \mathrm{lift}({\mathcal{S}})\}$; 
\STATE Step 2: ${\mathcal{C}}^+\leftarrow k\mbox{-means++$\_$seeding}({\mathcal{S}}^+, k)$;
\STATE Step 3: ${\mathcal{C}}\leftarrow \exp^{-1}_{\ve{q}}({\mathcal{C}}^+)$;
\STATE \textbf{Output:} Cluster centers ${\mathcal{C}} \in T_{\ve{q}}{\mathbb{H}^d}$;
\end{algorithmic}
\end{minipage}
\\ \hline
$\ve{x}^S  \defeq  [x_1 \quad x_2 \quad \cdots \quad x_d \quad \|\ve{x}\|_2\cot \|\ve{x}\|_2]$ & $\ve{x}^H  \defeq  [x_1 \quad x_2 \quad \cdots \quad x_d \quad
i \|\ve{x}\|_2\coth \|\ve{x}\|_2]$ \\
$g_{S}(\ve{x}^S) \defeq \|\ve{x}\|_2/\sin \|\ve{x}\|_2$& $g_{H}(\ve{x}^H) \defeq -\|\ve{x}\|_2/\sinh\|\ve{x}\|_2$\\ \hline\hline
\end{tabular}
}
\caption{How to use $k$-means++ to cluster points on the sphere (left) or the hyperboloid (right).\label{tab:algos}}
\vspace{-0.2in}
\end{table}

It turns out that both distances in \ref{defdist} can be engineered as
Bregman divergences via Theorem \ref{th00}, as seen in Table \ref{t-ex-short2}. Furthermore, they imply the
same $\varphi$, which is just the generator of Mahalanobis distortion,
but a different $g$. The construction
involves a third party, a \textit{lifting map} ($\mathrm{lift}(.)$) that increases
the dimension by one.
The \textit{Sphere} lifting map $\Real^{d}\ni \ve{x} \mapsto
\ve{x}^S \in \Real^{d+1}$ is indicated in Table \ref{tab:algos}
(left). The new coordinate depends on the norm of $\ve{x}$. The
\textit{Hyperbolic} lifting map, $\Real^{d}\ni \ve{x} \mapsto
\ve{x}^H \in \Real^{d}\times \mathbb{C}$, involves a pure imaginary additional
coordinate, is indicated in in Table \ref{tab:algos} (right, with a slight abuse of notation) and Figure \ref{fig:lift}. Both $\ve{x}^S$ and
$\ve{x}^H$ live on a $d$-dimensional manifold, depicted in Figure \ref{fig:lift}. When they are scaled by
the corresponding $g_.(.)$, they happen to be mapped to $\mathbb{S}^d$
or $\mathbb{H}^d$, respectively, by what happens to be the manifold's exponential
map for the original $\ve{x}$ (see \Supp{Appendix \ref{sec-app-der}}{Appendix C}).

Theorem \ref{th00} is interesting in this case because $\varphi$ corresponds to a Mahalanobis distortion: this shows that $k$-means++ seeding \citep{avKM,nlkMB} can be used directly on the scaled coordinates ($g^{-1}_{\{S,H\}}(\ve{x}^{\{S,H\}})\cdot \ve{x}^{\{S,H\}}$) to pick centroids that yield an approximation of the global optimum for the clustering problem on the manifold which is just \textit{as good as} the original Euclidean approximation bound \citep{avKM}.

\begin{lemma}
\label{lemm:kmeanspp}
The expected potential 
of
\GKM~seeding over the random choices of $\mathcal{C}^+$ satisfies:
\begin{eqnarray}
\expect[D_\rec(\mathcal{S} : \mathcal{C})] & \leq & 8(2+\log k) \cdot \inf_{\mathcal{C}'\in \mathbb{S}^d} D_\rec(\mathcal{S} : \mathcal{C}')\:\:.
\end{eqnarray}
The same approximation bounds holds for \HKM~seeding on the
hyperboloid ($\mathcal{C}', \mathcal{C}^+\in \mathbb{H}^d$).
\end{lemma}

Lemma \ref{lemm:kmeanspp} is notable since
it was only recently shown that such a bound is possible for the
sphere \citep{emSK}, and to our knowledge, no such approximation
quality is known for clustering on the hyperboloid
\citep{rjgHC,snLM}. Notice that Lloyd iterations on non-linear
manifolds would require repetitive renormalizations to keep centers on
the manifold \citep{dmCD}, an additional disadvantage compared to
clustering on flat
manifolds that $\{\mathrm{G,K}\}$-means++ seedings do not bear.

\def\permille{\ensuremath{{}^\text{o}\mkern-5mu/\mkern-3mu_\text{oo}}}

\section{Experimental validation}
\label{sec:experiments}

We present some experiments validating our theoretical analysis for the applications above.

\textbf{Multiple density ratio estimation}.
See \Supp{Appendix \ref{app:multiclass-dr}}{Appendix H.1} for experiments in this domain.

\textbf{Dual norm $p$-LMS (DN-$p$-LMS)}.
We ran $p$-LMS and 
the DN-pLMS of \S\ref{sec:pLMS}
on the experimental setting of \citet{kwhTP}.
We refer to that paper for an exhaustive description of the experimental setting, which we briefly summarize: it is a noisy signal processing setting, involving a dense or a sparse target. We compute,
over the signal received, the error of our predictor on the signal. We
keep all parameters as they are in \citep{kwhTP}, except for one: we
make sure that data are scaled to fit in a $L_p$ ball of prescribed radius,
to test the assumption related in \citep{kwhTP} that fixing
the learning rate $\eta_t$ is not straightforward in
$p$-LMS. Knowing the true value of $X_p$, we then scale it by a
misestimation factor $\rho$, typically in $[0.1, 1.7]$. We use the
same misestimation in DN-$p$-LMS. Thus, both algorithms
suffer the same source of uncertainty. Also, we periodically change
the signal (each 1000 iterations), to assess the performances of the
algorithms in tracking changes in the signal.

\begin{table}[!t]
\centering
\scalebox{0.525}{
\begin{tabular}{ccc||ccc}\hline \hline
$(p,q) = (1.17,6.9)$ & $(p,q) = (2.0,2.0)$ & $(p,q) = (6.9,1.17)$ & $(p,q) = (1.17,6.9)$ & $(p,q) = (2.0,2.0)$ & $(p,q) = (6.9,1.17)$ \\ \hline
\includegraphics[trim=10bp 10bp 30bp 10bp,clip,width=.28\linewidth]{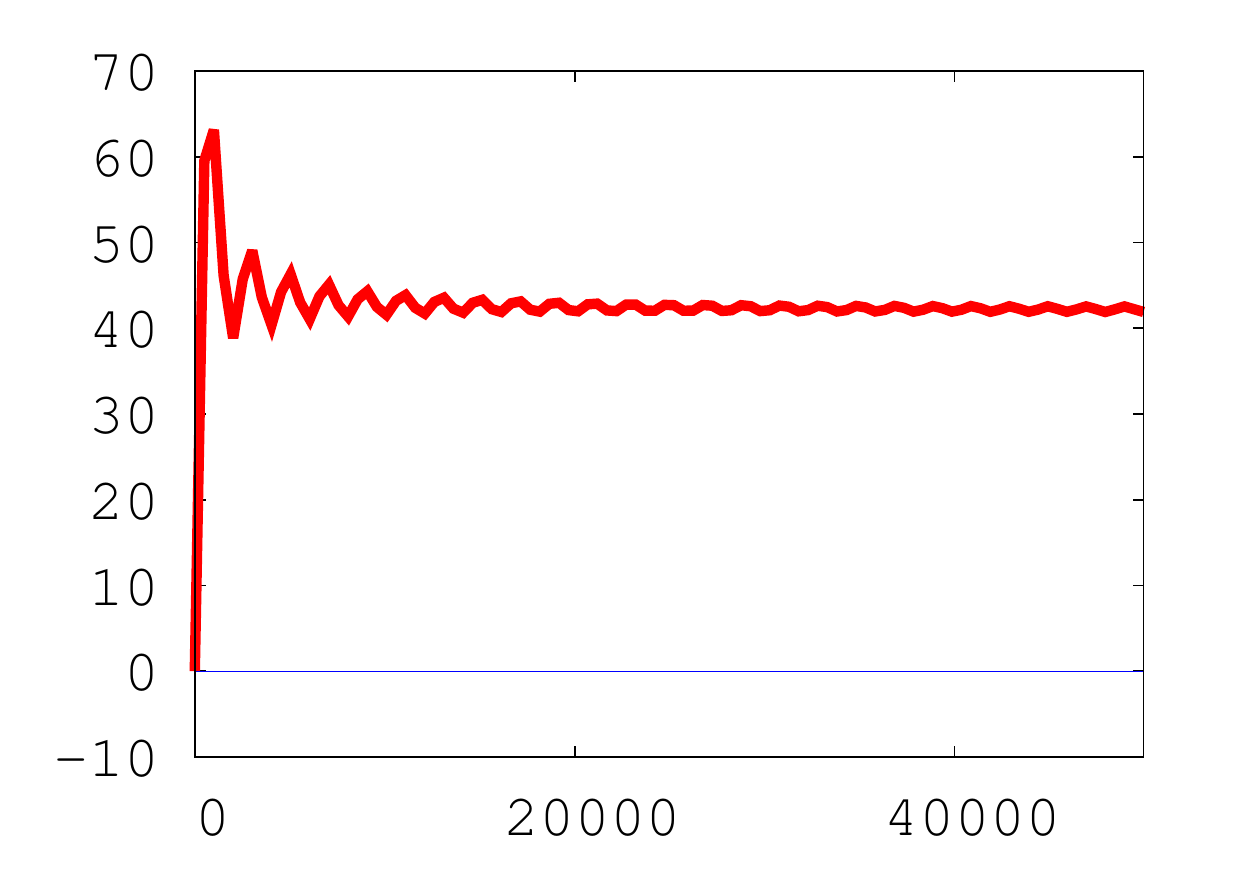} &  \includegraphics[trim=10bp 10bp 30bp 10bp,clip,width=.28\linewidth]{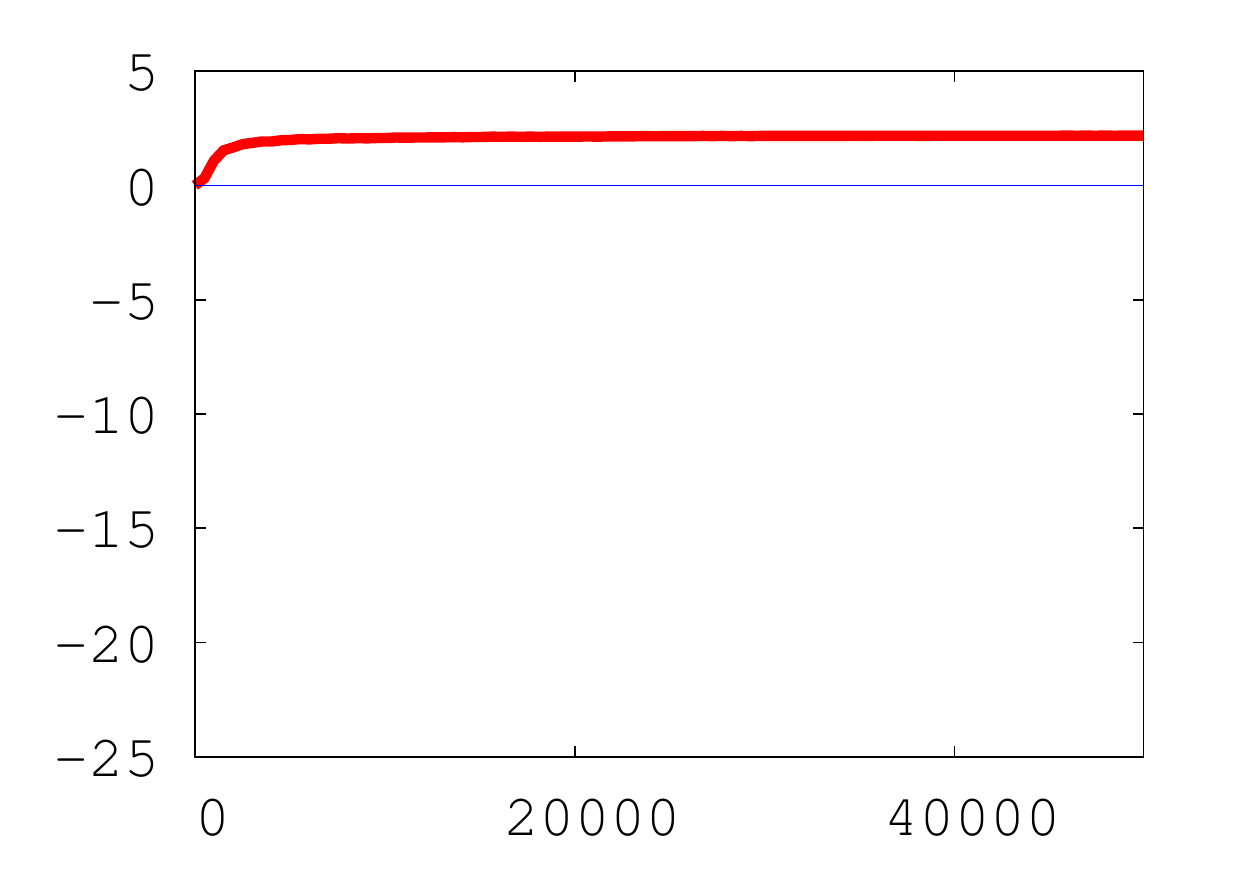}
& \includegraphics[trim=10bp 10bp 30bp 10bp,clip,width=.28\linewidth]{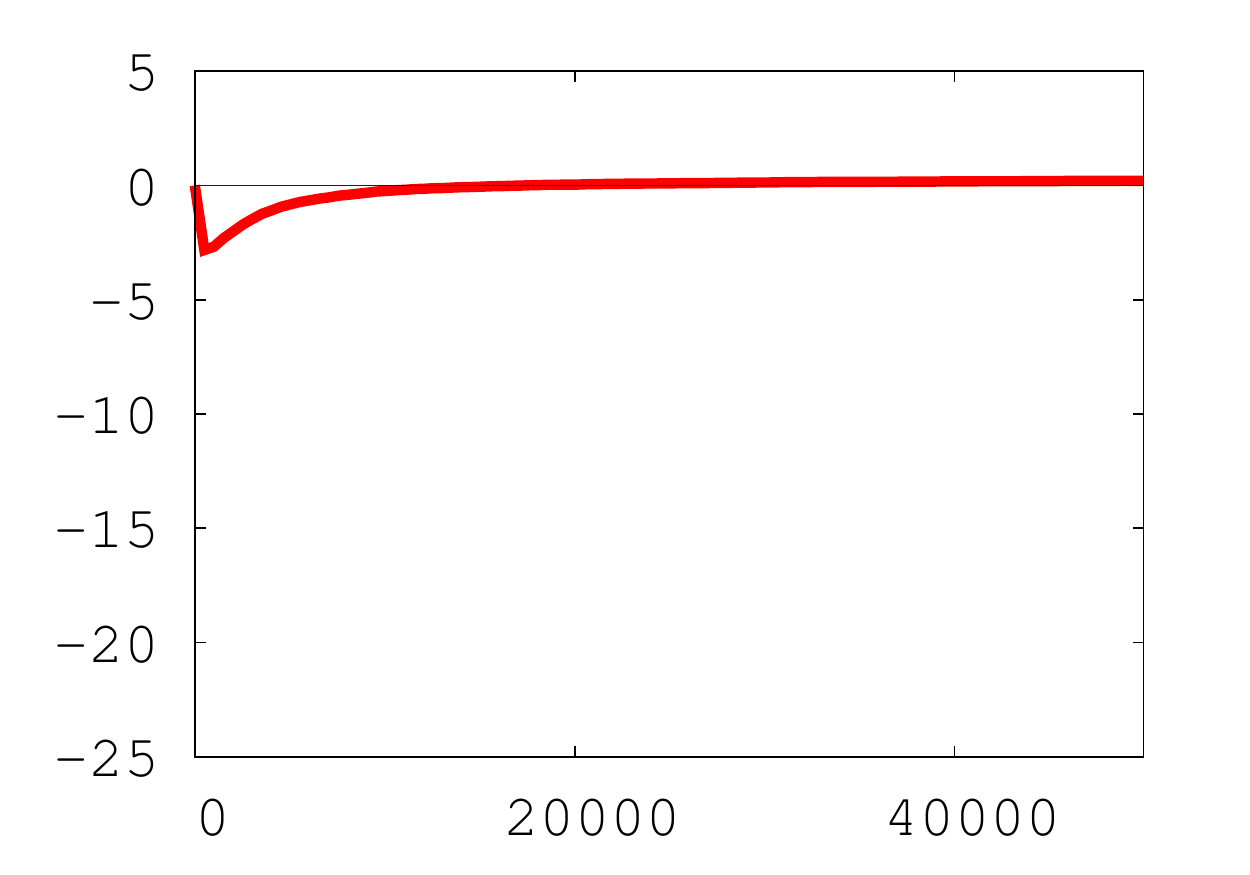}
& \includegraphics[trim=10bp 10bp 30bp 10bp,clip,width=.28\linewidth]{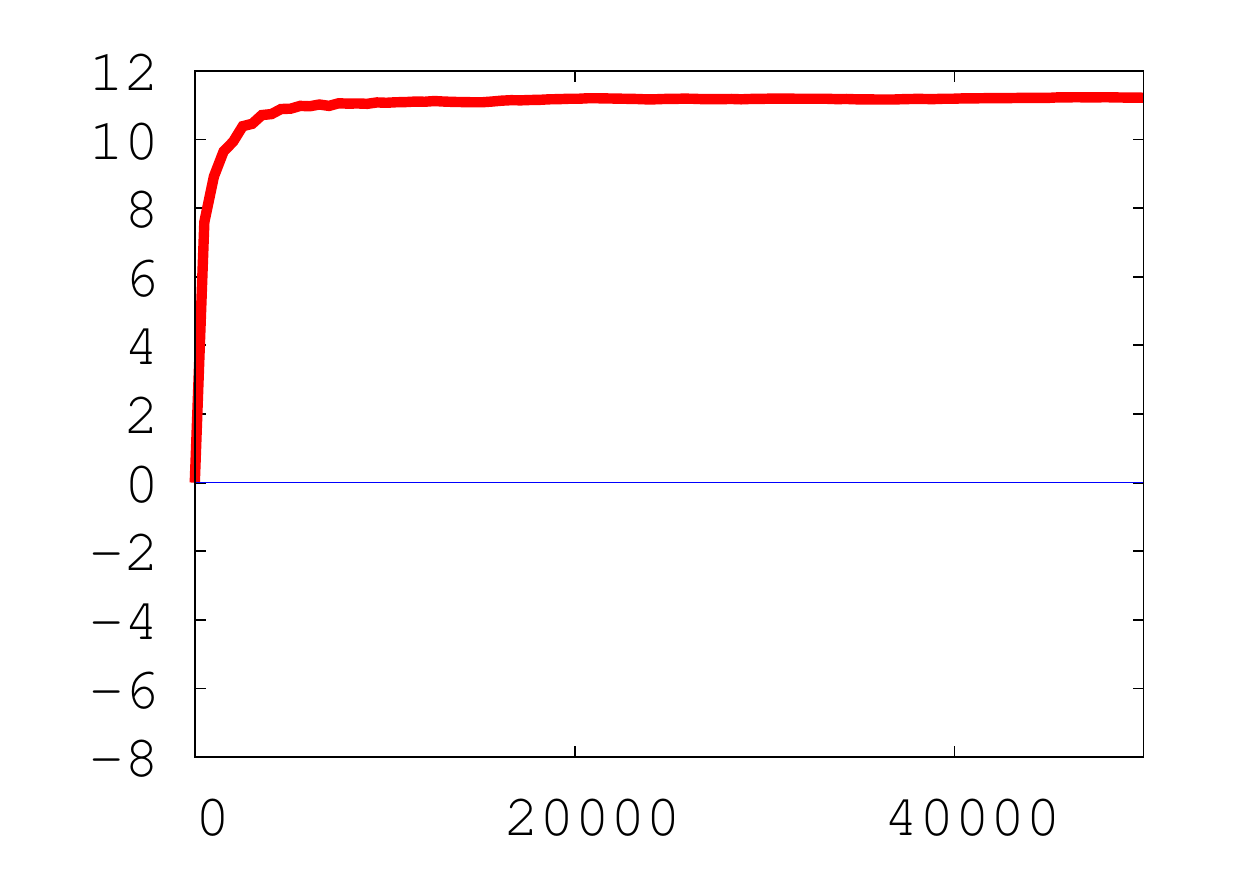}
&  \includegraphics[trim=10bp 10bp 30bp 10bp,clip,width=.28\linewidth]{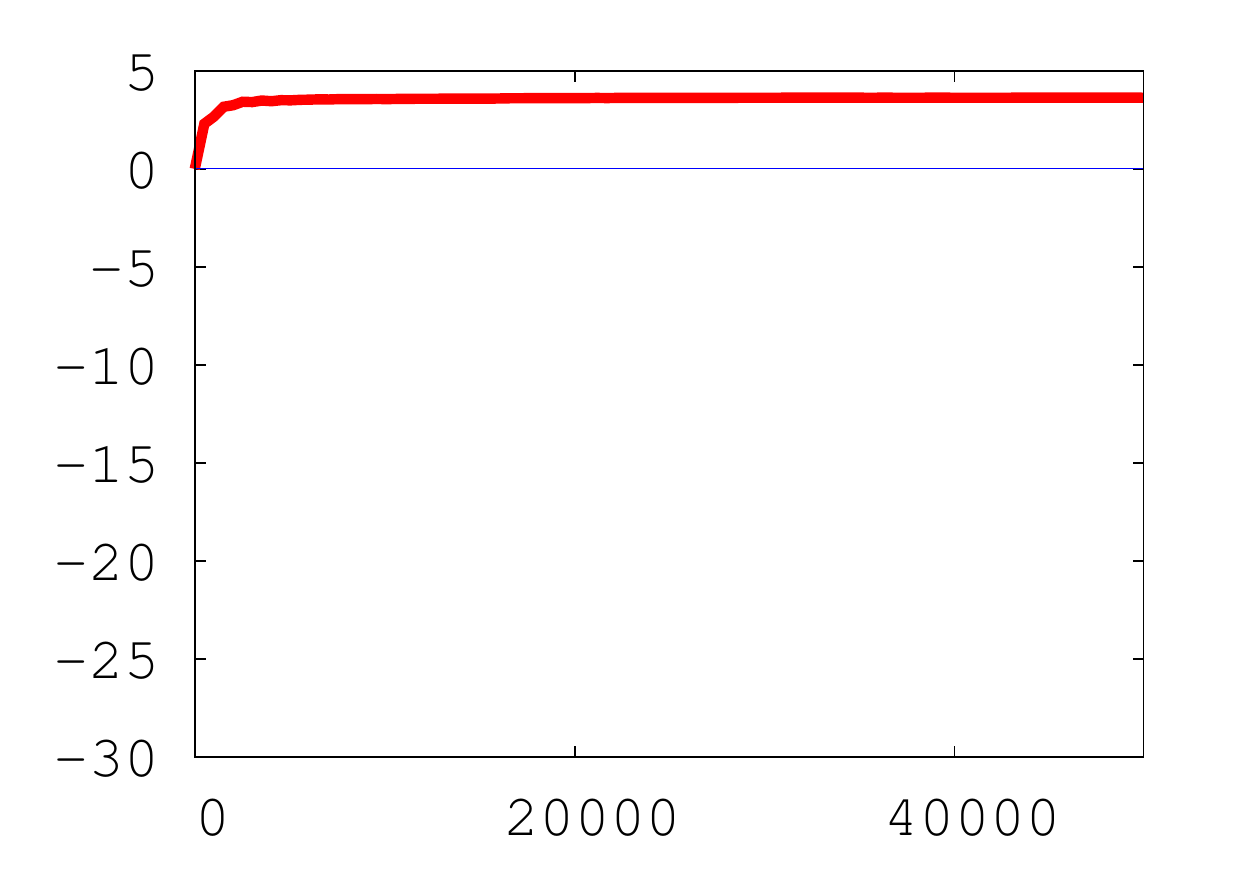}
& \includegraphics[trim=10bp 10bp 30bp 10bp,clip,width=.28\linewidth]{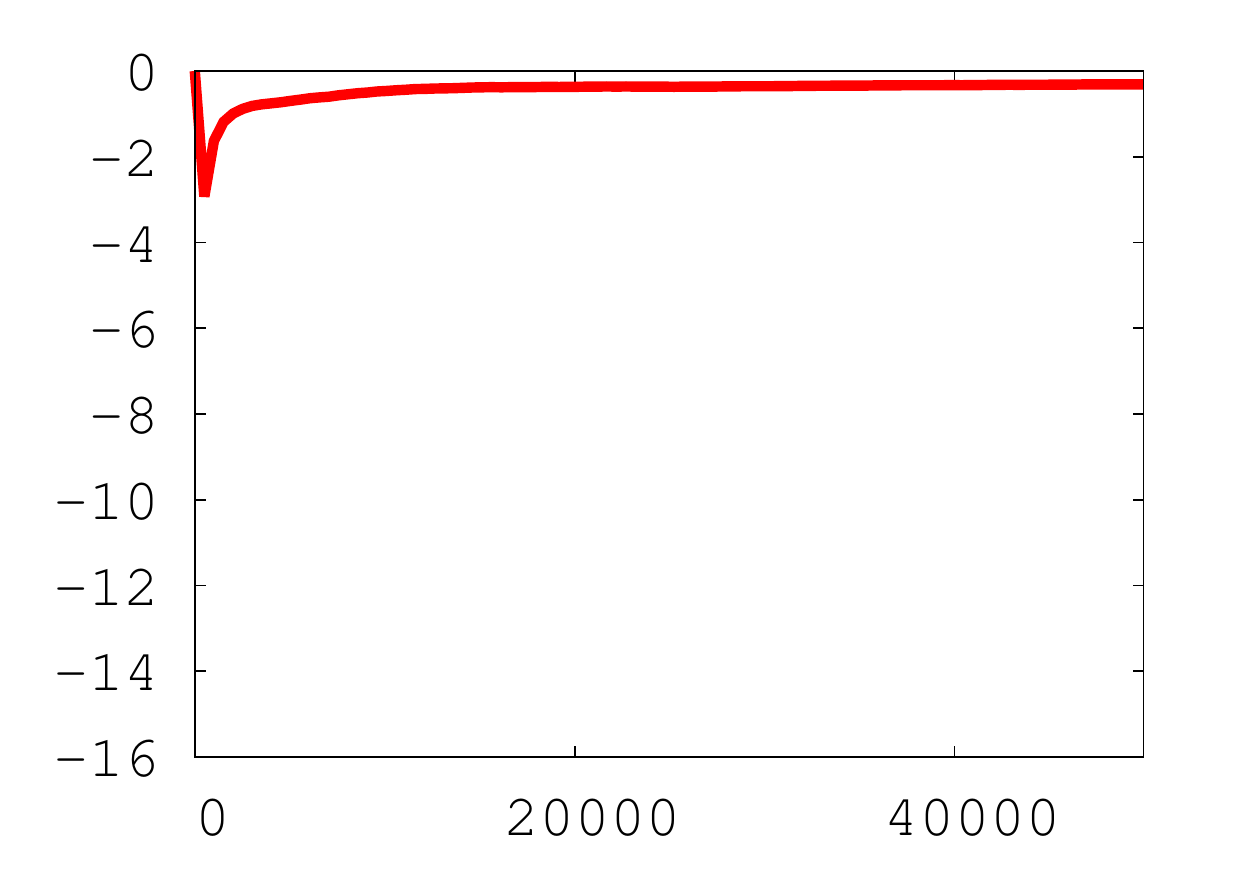}\\
$\rho = 1.0$ & $\rho = 1.0$ & $\rho = 1.0$ & $\rho = 0.5$ & $\rho =1.3$ & $\rho =0.2$ \\ \hline \hline
\end{tabular}
}
\caption{Summary of the experiments displaying ($y$) the error of $p$-LMS minus error of DN-$p$-LMS (when $>0$, DN-$p$-LMS beats $p$-LMS) as a function of $t$, in the setting of \cite{kwhTP}, for various values of $(p,q)$ (columns).
Left panel: (D)ense target; Right panel: (S)parse target.
\label{tab:exp-plms}}
  \vspace{-0.175in}
\end{table}

\begin{table}[!t]
\centering
\scalebox{0.725}{
\begin{tabular}{c|c}\hline \hline
\includegraphics[trim=0bp 0bp 0bp
0bp,clip,width=.40\linewidth]{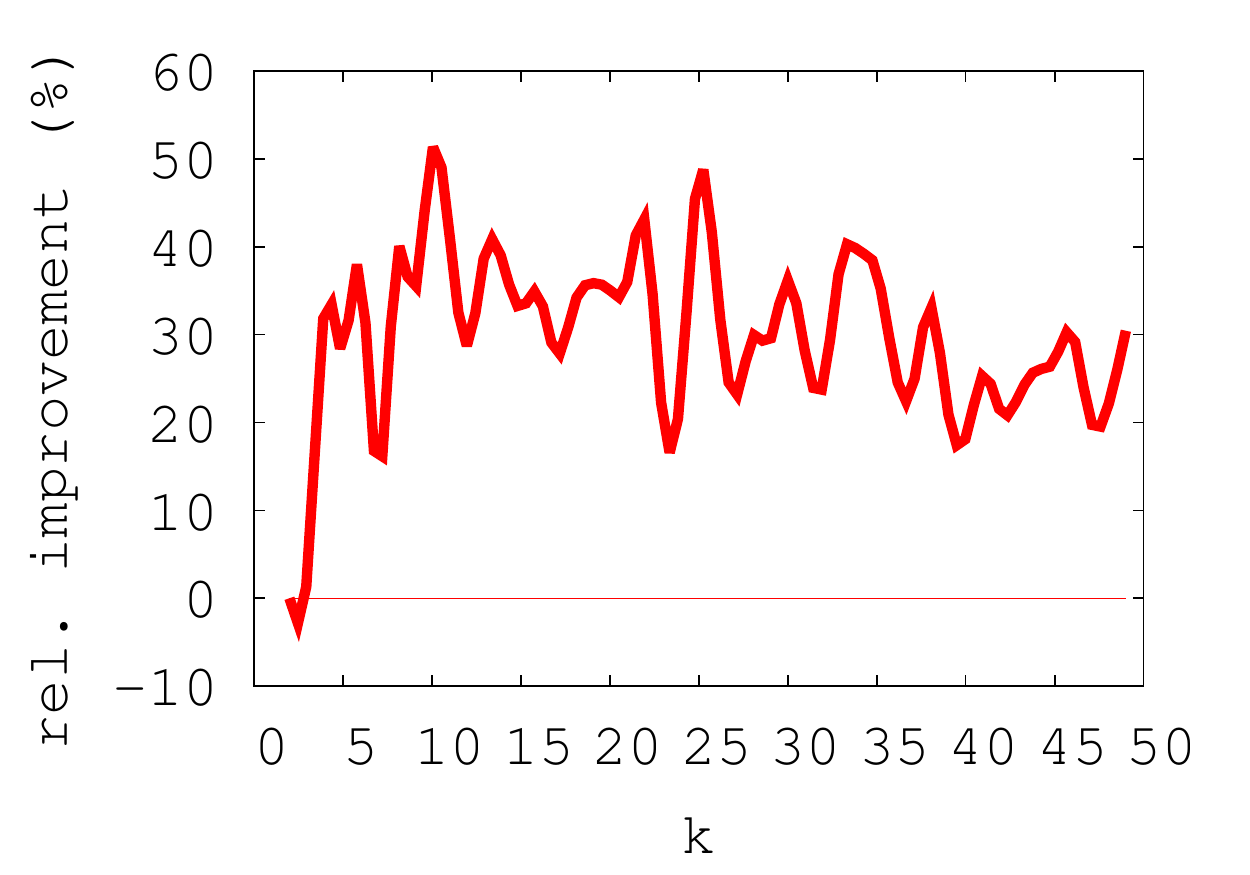}
& \includegraphics[trim=0bp 0bp 0bp
0bp,clip,width=.40\linewidth]{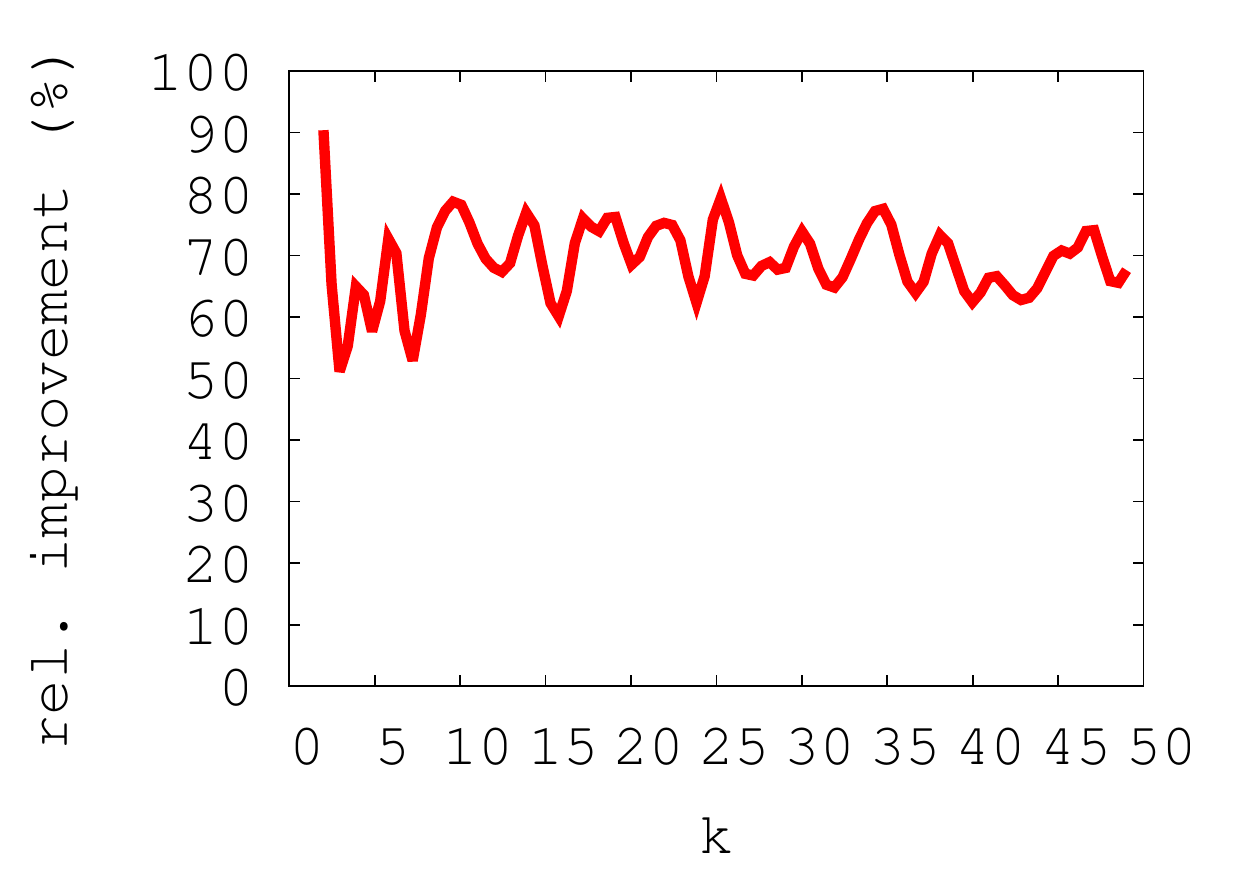}\\ \hline \hline
\end{tabular}
}
\caption{(L) Relative improvement (decrease) in $k$-means potential of
  SKM$\circ$\GKM~compared to SKM alone.
(R) Relative improvement of \GKM~over Forgy initialization on the sphere.\label{tab:kmeansDM1}}
\vspace{-0.0825in}
\end{table}

\begin{table}[!t]
\centering
\scalebox{0.825}{
\begin{tabular}{c|c||c}\hline \hline
\includegraphics[trim=0bp 0bp 0bp
0bp,clip,width=.30\linewidth]{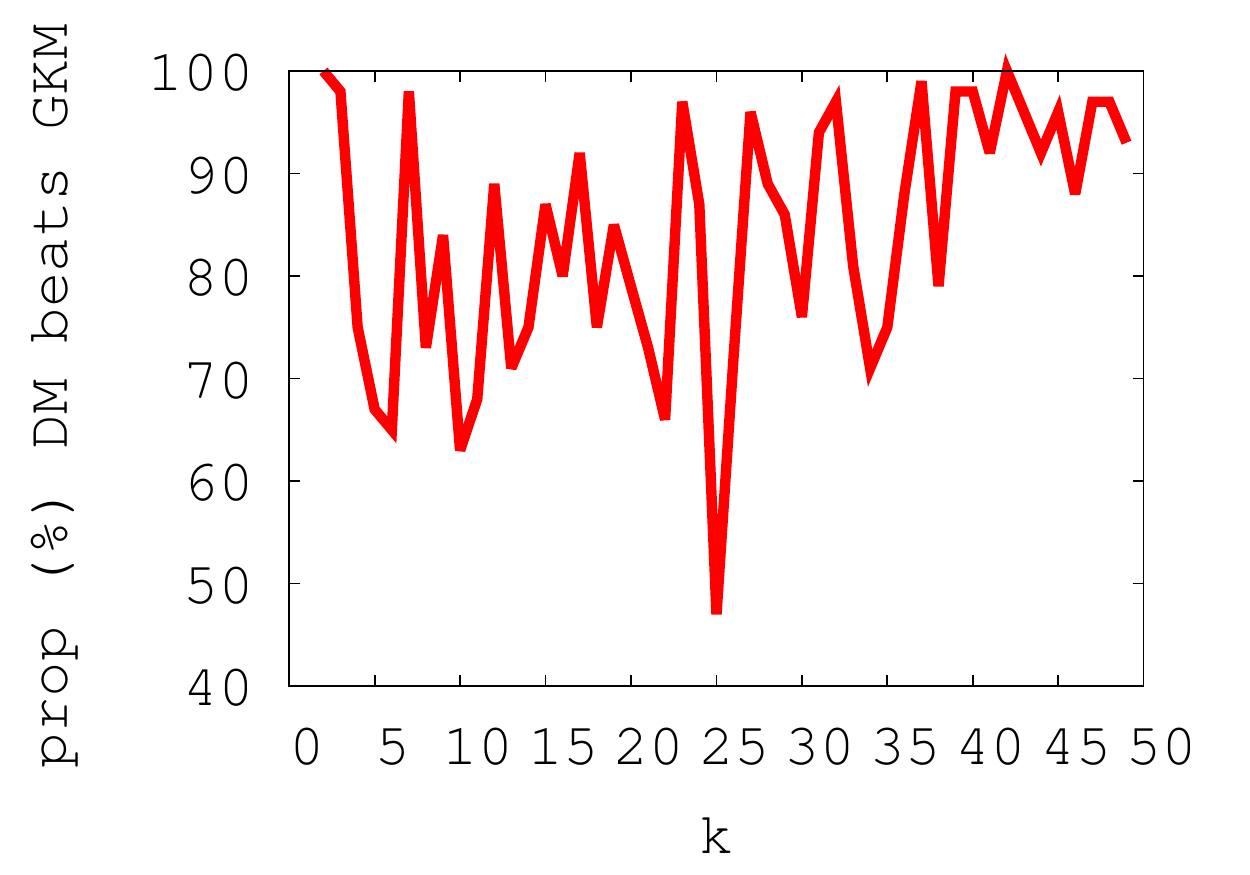}
& \includegraphics[trim=0bp 0bp 0bp
0bp,clip,width=.30\linewidth]{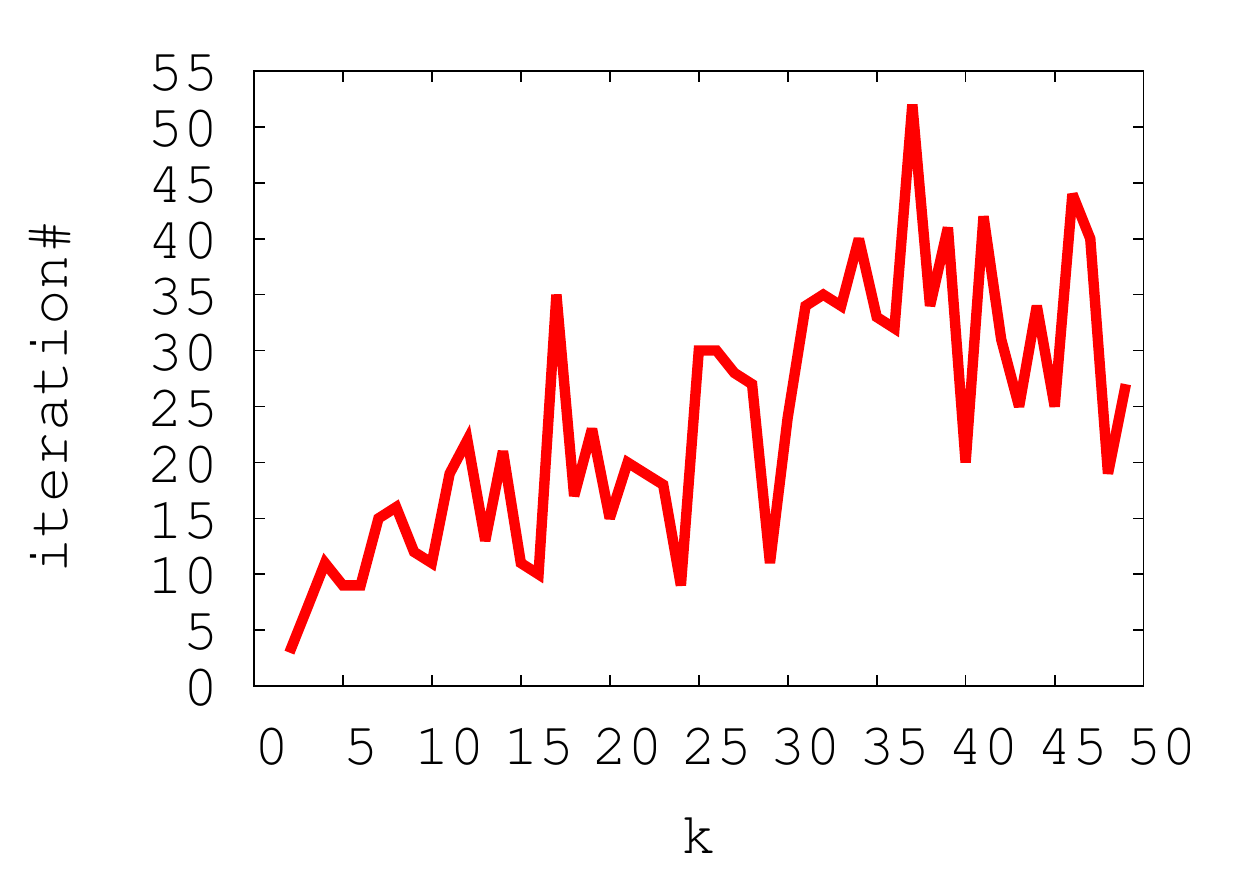}& \includegraphics[trim=0bp 0bp 0bp
0bp,clip,width=.30\linewidth]{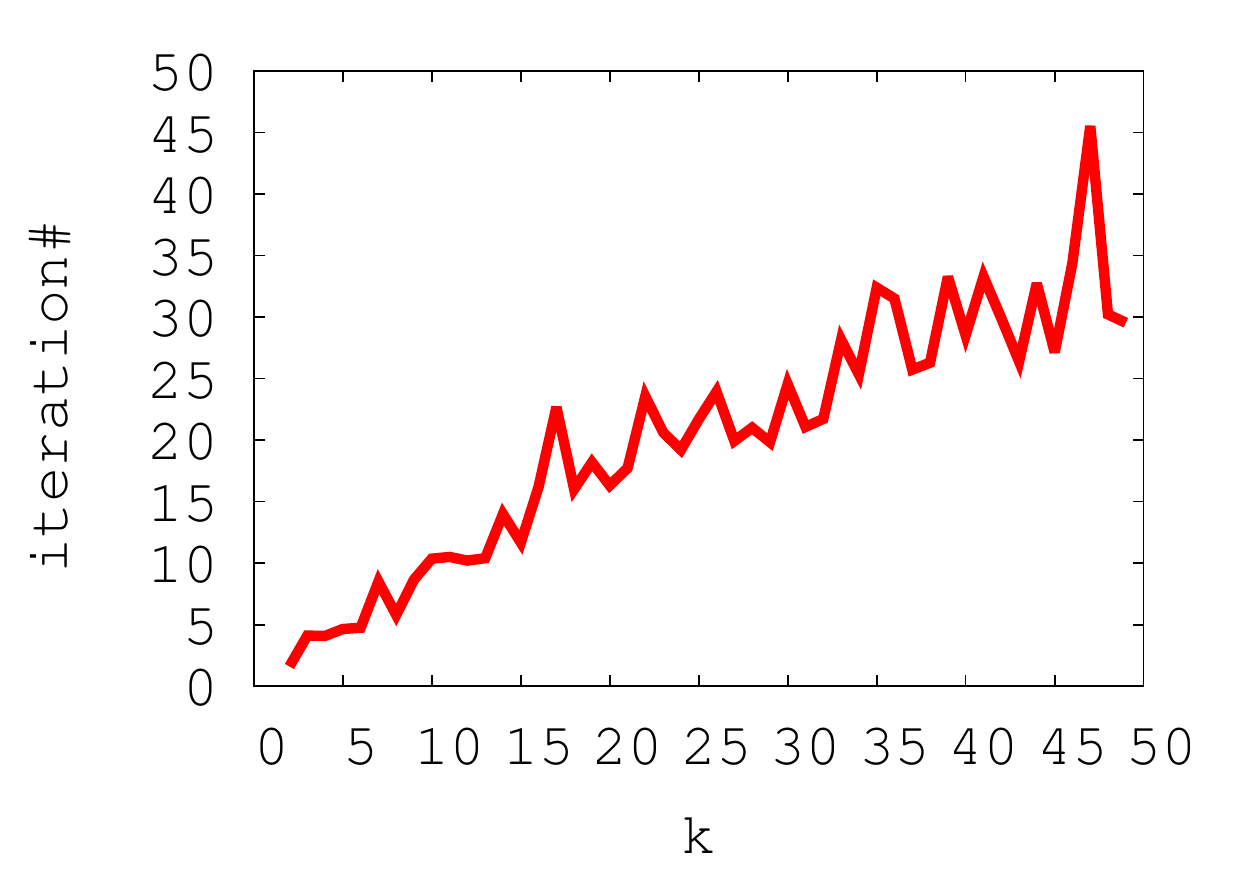}\\ \hline \hline
\end{tabular}
}
\caption{(L) \% of the number of runs of SKM whose output (\textit{when it has converged}) is better than \GKM.
(C) Maximal \# of iterations for SKM after which it beats \GKM~(ignoring runs of SKM that do not beat \GKM).
(R) Average \# of iterations for SKM to converge.\label{tab:kmeansDM2}}
  \vspace{-0.0825in}
\end{table}

Experiments, given \textit{in extenso} in
\Supp{Appendix \ref{app:exp-full}}{Appendix H.2}, are sumarized in Table \ref{tab:exp-plms}.
The following trends emerge: in the mid to long run, DN-$p$-LMS is
never beaten by $p$-LMS by more than a fraction of percent. On the
other hand, DN-$p$-LM can beat $p$-LMS by very significant
differences (exceeding 40$\%$), in particular when $p<2$,
\textit{i.e.} when we are outside the
regime of the proof of \cite{kwhTP}. This indicates that
significantly stronger and more general results than the one of Lemma \ref{lemlplq} may
be expected. Also, it seems that the problem of $p$-LMS lies in an ``exploding'' norm problem: in various cases, we observe that $\|\ve{w}_t\|$ (in any norm) blows up with $t$, and this correlates with a very significant degradation of its performances.
Clearly, DN-$p$-LMS does not have this problem since all relevant
norms are under tight control. Finally, even when the norm does not explode, DN-$p$-LMS can still beat $p$-LMS, by less important differences though.
Of course, the output of $p$-LMS can repeatedly be normalised, but the normalisation would escape the theory of \cite{kwhTP} and it is not clear which kind of normalisation would bring the best results.

\textbf{Clustering on the sphere}.
For $k \in [50]_*$, we simulate on
$T_{\ve{0}}{\mathbb{S}}^2$ a mixture of spherical Gaussian and uniform
  densities with $2k$ components. We run three algorithms: (i)
  SKM \citep{dmCD} on the data embedded on ${\mathbb{S}}^2$
  with random (Forgy) initialization, (ii), \GKM~and (iii)
  SKM with \GKM~initialisation.
  Results are averaged over the algorithms' runs.

  Table \ref{tab:kmeansDM1} (left)
  displays that using \GKM~as
  initialization for SKM brings a very significant leverage over SKM
  alone, since we almost divide the $k$-means potential by a factor 2
  on some runs. The right plot of Table
  \ref{tab:kmeansDM1} shows that \SKM~consistently reduces the
  $k$-means potential by at least a factor 2 over Forgy.
The left plot in Table
  \ref{tab:kmeansDM2} displays that even when it has converged, SKM
  does \textit{not} necessarily beat \GKM. Finally, the center+right plots in Table
  \ref{tab:kmeansDM2} display that even when it does beat \GKM~when it
  has converged, the iteration number after which SKM beats
  \GKM~increases with $k$, and in the worst case may \textit{exceed}
  the average number of iterations needed for SKM to converge (we
  stopped SKM if relative improvement is not above $1\permille$).

\section{Conclusion}
\label{sec:conclusion}

We presented a new scaled Bregman identity (Theorem \ref{th00}),
and used it to derive novel results
in multiple density ratio estimation,
adaptive filtering,
and clustering on curved manifolds.
We believe that, like other established properties of Bregman divergences,
there is potential for several other applications of the result; \Supp{Appendix \ref{app:exp-family}, \ref{app:comp-geom}}{the Appendix} present preliminary thoughts in this direction.

\newpage

\ifthenelse{\boolean{withSupp}}{
\appendix
\let\stdsection\section
\renewcommand\section{\newpage\stdsection}

\section{Additional helper lemmas}

We begin with some helper lemmas that will be used in some of the proofs.
In what follows, let
\begin{align*}
  \varphi_q(\ve{w}) &= (1/2)(W^2 + \|\ve{w}\|_q^2) \\
  \varphi^\dagger_q(\ve{w}) &= W \cdot \|\ve{w}\|_q
\end{align*}
for some $W > 0$
and
$p, q \in (1, \infty)$ such that $1/p + 1/q = 1$.

\subsection{Properties of $\varphi_q$ and $\varphi_q^\dagger$}

We use the following properties of $\varphi, \varphi^\dagger$.

\begin{lemma}
\label{lemm:varphi-properties}
For any $\ve{w}$,
\begin{align*}
  \nabla\varphi_q( \ve{w} ) &= \| \ve{w} \|_q^{2 - q} \cdot \sign( \ve{w} ) \otimes |\ve{w}|^{q - 1} \\
  \nabla\varphi_q^\dagger( \ve{w} ) &= W \cdot \| \ve{w} \|_q^{1 - q} \cdot \sign( \ve{w} ) \otimes |\ve{w}|^{q - 1} ,
\end{align*}
where $\otimes$ denotes Hadamard product.
\end{lemma}

\begin{proof}
The first identity was shown in \citet[Example 1]{kwhTP}.
The second identity follows from a simple calculation.
\end{proof}

This implies the follows useful relation between the gradients of $\varphi_q$ and $\varphi_q^\dagger$.

\begin{corollary}
\label{corr:phi-gradient-norms}
For any $\ve{w}$,
\begin{align*}
  \nabla\varphi_q(\ve{w}) &= (\|\ve{w}\|_q/W)\cdot \nabla\varphi^\dagger_q(\ve{w}) \\
  \| \nabla\varphi^\dagger_q(\ve{w}) \|_p &= W \\
  \| \nabla\varphi_q(\ve{w}) \|_p &= \| \ve{w} \|_q.
\end{align*}
\end{corollary} 

\begin{proof}[Proof of Corollary \ref{corr:phi-gradient-norms}]
The proof follows by direct application of Lemma \ref{lemm:varphi-properties} and the definition of $p, q$.
Note the third identity was shown in \citet[Appendix I]{kwhTP}.
\end{proof}

As a consequence, we conclude that the gradients of $\varphi$ and $\varphi^\dagger$ coincide when considering vectors on the $W$-sphere.

\begin{lemma}
\label{lemm:gradient-same}
For any $\|\ve{w}\|_q = W$,
$$ \nabla\varphi_q\left(\ve{w}\right) = \nabla\varphi_q^{\dagger}\left( \ve{w}\right). $$
\end{lemma}

\begin{proof}
This follows from the relation between $\nabla\varphi_q$ and $\nabla\varphi_q^\dagger$ from Lemma \ref{lemm:varphi-properties}.
\end{proof}

Finally, we have the following result about the composition of gradients.

\begin{lemma}
\label{lemm:nabla-phi-composition}
For any $\ve{w}$,
\begin{align*}
  \nabla\varphi_q\circ {\nabla\varphi_p^\dagger}(\ve{w}) &= \nabla\varphi^\dagger_q\circ {\nabla\varphi_p^\dagger}(\ve{w}) = \frac{W}{\|\ve{w}\|_p} \cdot \ve{w}.
\end{align*}
\end{lemma}

\begin{proof}
For the first identity, applying Lemma \ref{lemm:varphi-properties} twice,
\begin{eqnarray}
\nabla\varphi_q\circ {\nabla\varphi_p^\dagger}(\ve{w}) & = &
\frac{1}{\|{\nabla\varphi_p^\dagger}(\ve{w})\|_q^{q-2}}\cdot
\mathrm{sign}({\nabla\varphi_p^\dagger}(\ve{w}))\otimes
|{\nabla\varphi_p^\dagger}(\ve{w})|^{q-1}\nonumber\\
 & = & \frac{1}{W^{q-2}_q} \cdot \mathrm{sign}(\ve{w})\otimes
 \frac{W^{q-1}}{\|\ve{w}\|_p^{(p-1)(q-1)}}\cdot |\ve{w}|^{(p-1)(q-1)}\nonumber\\
 & = & \frac{W}{\|\ve{w}\|_p} \cdot \ve{w}\:\:.
\end{eqnarray}

For the second identity,
use Corollary \ref{corr:phi-gradient-norms} to conclude that
\begin{align*}
  \nabla\varphi^\dagger_q\circ {\nabla\varphi_p^\dagger}(\ve{w}) &= \frac{W}{\| \nabla\varphi_p^\dagger( \ve{w} ) \|_q} \cdot \nabla\varphi_q( \nabla\varphi_p^\dagger( \ve{w} ) ) \\
  &= W \cdot \frac{\ve{w}}{\| \ve{w} \|_p}.
\end{align*}
\end{proof}

\subsection{Bound on successive iterate divergence}

The following Lemma extends \citep[Appendix I]{kwhTP} to $\varphi^\dagger$.

\begin{lemma}\label{lemupdate}
For any $\ve{w}$ and $\ve{\delta}$,
\begin{eqnarray}
\lefteqn{D_{\varphi^\dagger_q}\left(\ve{w}\|
  {\nabla\varphi_p^\dagger}\left(\nabla\varphi_q^\dagger(\ve{w}) +
    \ve{\delta}\right)\right)}\nonumber\\
 & \leq &\frac{(p-1)
  \|\ve{w}\|_qW}{2}\cdot\left\| \frac{1}{ \|\nabla\varphi_q^\dagger(\ve{w}) +
    \ve{\delta}\|_p}\cdot \left(\nabla\varphi_q^\dagger(\ve{w}) +
    \ve{\delta}\right) -
    \frac{1}{W}\cdot \nabla\varphi^\dagger_q\left(\ve{w}\right)\right\|_p^2\:\:.
\end{eqnarray}
\end{lemma}
\begin{proof}[Proof of Lemma \ref{lemupdate}]
In this proof, $\circ$ denotes composition and $\otimes$ is Hadamard
product. The key step in the proof is the use of Theorem \ref{th00} to
``branch'' on the proof of \citep[Appendix I]{kwhTP} on the first
following identity (letting $\varphi_q(\ve{w})
\defeq (1/2) \cdot (W^2 + \|\ve{w}\|_q^2)$). We also make use
of the dual symmetry of
Bregman divergences and we obtain third identity of:
\begin{eqnarray}
\lefteqn{D_{\varphi^\dagger_q}\left(\ve{w}\|
  {\nabla\varphi_p^\dagger}\left(\nabla\varphi_q^\dagger(\ve{w}) +
    \ve{\delta}\right)\right)} \nonumber\\
& = & \frac{\|\ve{w}\|_q}{W}\cdot
D_{\varphi_q}\left(\frac{W}{\|\ve{w}\|_q}\cdot \ve{w}\left\|
  \frac{W}{\|\nabla\varphi_p^\dagger(\nabla\varphi_q^\dagger(\ve{w}) +
    \ve{\delta})\|_q}\cdot {\nabla\varphi_p^\dagger}\left(\nabla\varphi_q^\dagger(\ve{w}) +
    \ve{\delta}\right.\right)\right) \nonumber\\
& = & \frac{\|\ve{w}\|_q}{W}\cdot
D_{\varphi_q}\left(\frac{W}{\|\ve{w}\|_q}\cdot \ve{w}\left\|
  {\nabla\varphi_p^\dagger}\left(\nabla\varphi_q^\dagger(\ve{w}) +
    \ve{\delta}\right.\right)\right) \label{simpl1}\\
& = & \frac{\|\ve{w}\|_q}{W}\cdot
D_{\varphi_p}\left(
  \nabla\varphi_q\circ{\nabla\varphi_p^\dagger}\left(\nabla\varphi_q^\dagger(\ve{w}) +
    \ve{\delta}\right)  \left\|\nabla\varphi_q\left(\frac{W}{\|\ve{w}\|_q}\cdot \ve{w}\right)\right.\right) \nonumber \text{ by dual symmetry } \\
 & = & \frac{\|\ve{w}\|_q}{W}\cdot
D_{\varphi_p}\left(
  \frac{W}{ \|\nabla\varphi_q^\dagger(\ve{w}) +
    \ve{\delta}\|_p}\cdot \left(\nabla\varphi_q^\dagger(\ve{w}) +
    \ve{\delta}\right) \left\|\frac{W}{\|\ve{w}\|_q}\cdot
    \nabla\varphi_q\left(\ve{w}\right)\right.\right) \label{simpnorm}\\
& = & \frac{\|\ve{w}\|_q}{W}\cdot
D_{\varphi_p}\left(
  \frac{W}{ \|\nabla\varphi_q^\dagger(\ve{w}) +
    \ve{\delta}\|_p}\cdot \left(\nabla\varphi_q^\dagger(\ve{w}) +
    \ve{\delta}\right) \left\|
    \nabla\varphi^\dagger_q\left(\ve{w}\right)\right.\right) \label{simpnorm2}\\
& = & \|\ve{w}\|_qW\cdot
D_{\varphi_p}\left(
  \frac{1}{ \|\nabla\varphi_q^\dagger(\ve{w}) +
    \ve{\delta}\|_p}\cdot \left(\nabla\varphi_q^\dagger(\ve{w}) +
    \ve{\delta}\right) \left\|
    \frac{1}{W}\cdot \nabla\varphi^\dagger_q\left(\ve{w}\right)\right.\right) \:\:.\label{simpnorm3}
\end{eqnarray}
Equations (\ref{simpl1}) -- (\ref{simpnorm2}) hold because of Corollary \ref{corr:phi-gradient-norms}.
We now use
Appendix I\footnote{This result is stated as a bound on $D_{\varphi_q}( \ve{w} \| (\nabla\varphi_q)^{-1}( \nabla\varphi_q( \ve{w} ) + \ve{\delta} ) )$, which by the Bregman dual symmetry property is equivalent to a bound on $D_{\varphi_p}( \nabla\varphi_q( \ve{w} ) + \ve{\delta} \| \nabla\varphi_q( \ve{w} ) )$.} in \cite{kwhTP} on Equation (\ref{simpnorm3}) and obtain
\begin{eqnarray}
\lefteqn{D_{\varphi^\dagger_q}\left(\ve{w}\|
  {\nabla\varphi_p^\dagger}\left(\nabla\varphi_q^\dagger(\ve{w}) +
    \ve{\delta}\right)\right)}\nonumber\\
 & \leq & \frac{(p-1)
  \|\ve{w}\|_qW}{2}\cdot\left\| \frac{1}{ \|\nabla\varphi_q^\dagger(\ve{w}) +
    \ve{\delta}\|_p}\cdot \left(\nabla\varphi_q^\dagger(\ve{w}) +
    \ve{\delta}\right) -
    \frac{1}{W}\cdot \nabla\varphi^\dagger_q\left(\ve{w}\right)\right\|_p^2\:\:, \nonumber
\end{eqnarray}
as claimed.
\end{proof}

\subsection{Bound on successive iterate divergence to target}

In what follows, we write the DN-pLMS updates as
$\ve{w}_{t} = \nabla \varphi_p^\dagger( \ve{\theta}_t )$, where
$$ \ve{\theta}_t \defeq \nabla \varphi_q^\dagger( \ve{w}_{t - 1} ) - \Delta_t $$
for
$\Delta_t = \eta_t \cdot ( \ve{w}_{t - 1}^\top \ve{x}_t - y_t ) \cdot \ve{x}_t$.
Further, for notational ease, we write
$$ \bar{\ve{u}} \defeq \frac{\ve{u}}{g_{q}(\ve{u})} $$
and
$$ \bar{\ve{\theta}}_t \defeq \frac{\ve{\theta}_t}{\| \ve{\theta}_t \|_p}. $$

We have the following preliminary bound on the distance from iterates of DN-pLMS to the (normalised) target.

\begin{lemma}
\label{lemm:iterate-bound-1}
Fix any learning rate sequence $\{ \eta_t \}_{t = 1}^T$.
Pick any $\ve{u}$, and consider iterates $\{ \ve{w}_t \}_{t = 0}^{T}$ as per Equation \ref{eqn:plms-dagger-explicit}.
Denote $s_t \defeq (\bar{\ve{u}} - \ve{w}_{t-1})^\top\ve{x}_t, r_t \defeq \bar{\ve{u}}^\top\ve{x}_t - y_t$, and $\upalpha_t \defeq \frac{W}{ \| \ve{\theta}_t \|_p}$.
Suppose $\| \ve{x}_t \|_p \leq X_p$.
Then,
\begin{eqnarray}
D_{\varphi_q}\left(\bar{\ve{u}} \left\| \ve{w}_{t-1}\right.\right) - D_{\varphi_q}\left(\bar{\ve{u}} \left\| \ve{w}_{t}\right.\right) & \geq &  Q + R + S + T\:\:, \nonumber
\end{eqnarray}
with 
\begin{eqnarray}
Q & \defeq & \frac{\upalpha_t}{2}\eta_t(s^2_t - r^2_t), \nonumber \\
R & \defeq & (1-\upalpha_t)\cdot \underbrace{\left(W^2 -  \bar{\ve{u}}^\top \nabla\varphi_q^\dagger(\ve{w}_{t-1})\right)}_{\in [0, 2W^2]} \:\:, \nonumber \\
S & \defeq & \frac{p-1}{2}\cdot\underbrace{\left(2 \upalpha^2_t\eta^2_t (s_t-r_t)^2X_p^2 - \left\| (s_t-r_t)\eta_t \upalpha_t\cdot \ve{x}_t - (1-\upalpha_t) \cdot\nabla\varphi_q^\dagger(\ve{w}_{t-1})\right\|_p^2 \right)}_{\geq -2 (1-\upalpha_t)^2 W^2} \:\:, \nonumber\\
T & \defeq & \frac{\upalpha_t}{2} \eta_t(s_t - r_t)^2 \left(1-2(p-1) \eta_t\upalpha_tX_p^2\right) \:\: \nonumber.
\end{eqnarray}
\end{lemma}

\begin{proof}[Proof of Lemma \ref{lemm:iterate-bound-1}]
The Bregman triangle equality (also called the three points property) \citep[Property 5]{bnnBV}, \citep[Lemma 11.1]{Cesa-Bianchi:2006} brings:
\begin{eqnarray}
\lefteqn{D_{\varphi_q}\left(\bar{\ve{u}} \left\| \ve{w}_{t-1}\right.\right) - D_{\varphi_q}\left(\bar{\ve{u}} \left\| \ve{w}_{t}\right.\right)} \nonumber \\
& = & \left( \bar{\ve{u}} - \ve{w}_{t-1}\right)^\top \left(\nabla\varphi_q\left( \ve{w}_{t}\right)-\nabla\varphi_q\left( \ve{w}_{t-1}\right)\right) - D_{\varphi_q}\left(  \ve{w}_{t-1}\left\| \ve{w}_{t}\right.\right) \nonumber\\
 & = & \left( \bar{\ve{u}} - \ve{w}_{t-1}\right)^\top \left( \nabla\varphi_q^{\dagger}\left( \ve{w}_{t}\right)- \nabla\varphi_q^\dagger\left( \ve{w}_{t-1}\right)\right) -  D_{\varphi_q^\dagger}\left( \ve{w}_{t-1}\left\| \ve{w}_{t}\right.\right) \text{ by Lemmas \ref{lemm:plms-update-norm}, \ref{lemm:gradient-same} } \nonumber \:\:.
\end{eqnarray}

We now have
\begin{eqnarray}
\nabla\varphi_q^{\dagger}\left( \ve{w}_{t}\right) = \nabla\varphi_q^{\dagger}\circ {\nabla\varphi_p^\dagger}(\ve{\theta}_t) 
 = W \cdot \bar{\ve{\theta}}_t \:\: \nonumber
\end{eqnarray}
by Corollary \ref{corr:phi-gradient-norms}.
We get
\begin{eqnarray}
\lefteqn{D_{\varphi_q}\left(\bar{\ve{u}} \left\| \ve{w}_{t-1}\right.\right) - D_{\varphi_q}\left(\bar{\ve{u}} \left\| \ve{w}_{t}\right.\right)} \nonumber \\
& \geq & \left( \bar{\ve{u}} - \ve{w}_{t-1}\right)^\top \left( W \cdot \bar{\ve{\theta}}_t -\nabla\varphi_q^\dagger(\ve{w}_{t-1})\right) - \frac{(p-1)
W^2}{2}\cdot\left\| \bar{\ve{\theta}}_t - \frac{1}{W}\cdot \nabla\varphi^\dagger_q\left(\ve{w}_{t-1}\right)\right\|_p^2 \nonumber \\
& = & \left( \bar{\ve{u}} - \ve{w}_{t-1}\right)^\top \left( W \cdot \bar{\ve{\theta}}_t -\nabla\varphi_q^\dagger(\ve{w}_{t-1})\right)\nonumber - \frac{p-1}{2}\cdot\left\| W \cdot \bar{\ve{\theta}}_t - \nabla\varphi^\dagger_q\left(\ve{w}_{t-1}\right)\right\|_p^2.
\end{eqnarray}

Now, note that
$$ \ve{\theta}_t = \nabla\varphi_q^\dagger(\ve{w}_{t-1}) + \eta_t \cdot (s_t - r_t) \cdot \ve{x}_t. $$
We can thus rewrite the above as
\begin{eqnarray}
\lefteqn{D_{\varphi_q}\left(\bar{\ve{u}} \left\| \ve{w}_{t-1}\right.\right) - D_{\varphi_q}\left(\bar{\ve{u}} \left\| \ve{w}_{t}\right.\right)} \nonumber \\
& \geq & s_t(s_t-r_t)\eta_t \upalpha_t +
(1-\upalpha_t)\left(\ve{w}_{t-1}^\top \nabla\varphi_q^\dagger(\ve{w}_{t-1})
-  \bar{\ve{u}}^\top
\nabla\varphi_q^\dagger(\ve{w}_{t-1})\right)\nonumber\\
  & & - \frac{p-1}{2}\cdot\left\| (s_t-r_t)\eta_t \upalpha_t\cdot
    \ve{x}_t - (1-\upalpha_t)
    \cdot\nabla\varphi_q^\dagger(\ve{w}_{t-1})\right\|_p^2 \nonumber
  \\
& = & s_t(s_t-r_t)\eta_t \upalpha_t +
(1-\upalpha_t)\left(W^2
-  \bar{\ve{u}}^\top
\nabla\varphi_q^\dagger(\ve{w}_{t-1})\right)\nonumber\\
  & & - \frac{p-1}{2}\cdot\left\| (s_t-r_t)\eta_t \upalpha_t\cdot
    \ve{x}_t - (1-\upalpha_t)
    \cdot\nabla\varphi_q^\dagger(\ve{w}_{t-1})\right\|_p^2 
  \nonumber \text{ by definition of } \nabla\varphi_q^\dagger \\
 & = &  Q + R + S + T\:\:.\nonumber
\end{eqnarray}
\end{proof}

We can show that the sum $R + S + T \geq 0$.
This proof involves chaining together multiple simple inequalities.
We give a high level overview in Figure \ref{fig:dplms-proof-schematic}.

\begin{figure}[!t]
  \centering
  \resizebox{\linewidth}{!}{
  \includegraphics[scale=0.95]{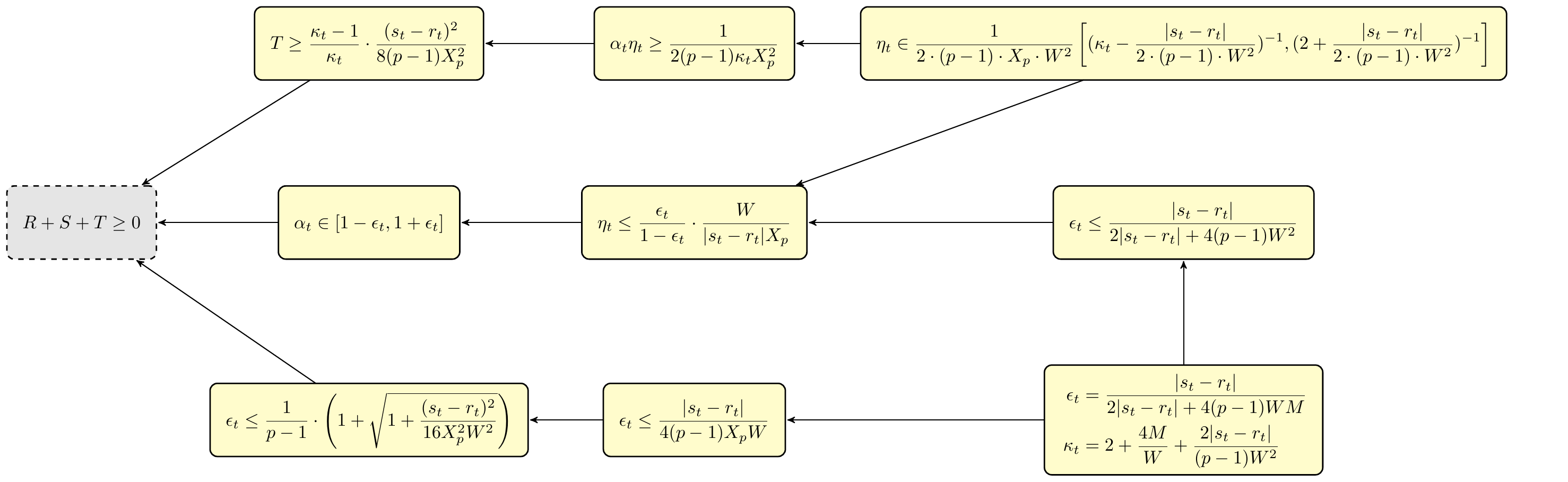}
  }
  \caption{Schematic of proof of Lemma \ref{lemm:rst}. Arrows from equation $A$ to $B$ indicate that $A \implies B$.}
  \label{fig:dplms-proof-schematic}
\end{figure}

\begin{lemma}
\label{lemm:rst}
Let $R, S, T$ be as per Lemma \ref{lemm:iterate-bound-1}.
Suppose we fix
\begin{eqnarray}
\eta_t & = & \upgamma \cdot \frac{W}{4(p-1)MX_pW+|y_t - \ve{w}_{t-1}^\top \ve{x}_t|X_p} \label{eqn:final-eta} \:\:,
\end{eqnarray}
for any $\upgamma \in [1/2, 1]$, and $M \defeq \max\{ W, X_p \}$. 
Then,
$T + R + S \geq 0$.
\end{lemma}

\begin{proof}
The triangle inequality and the fact that
$\|\nabla\varphi_q^\dagger(\ve{w}_{t-1})\|_p = W$ brings
\begin{eqnarray}
\upalpha_t & \in  & \left[ \frac{W}{\|\nabla\varphi_q^\dagger(\ve{w}_{t-1})\|_p +
    \eta_t |s_t-r_t|\cdot \|\ve{x}_t\|_p},  \frac{W}{\|\nabla\varphi_q^\dagger(\ve{w}_{t-1})\|_p -
    \eta_t |s_t-r_t|\cdot \|\ve{x}_t\|_p}\right]\nonumber\\
 & & \subseteq \left[ \frac{W}{W +
    \eta_t |s_t-r_t|X_p},  \frac{W}{W -
    \eta_t |s_t-r_t|\cdot X_p}\right],\label{intalpha}
\end{eqnarray}
\textit{assuming} that $\eta_t$ is chosen so that
\begin{eqnarray}
\eta_t & \leq & \frac{W}{|s_t-r_t|\cdot X_p}\:\:,\label{consteta0}
\end{eqnarray}
so that the right bound is non negative. To indeed ensure this, suppose that for some $0<\epsilon_t\leq 1/2$, we fix 
\begin{eqnarray}
\eta_t & \leq & \frac{\epsilon_t}{1-\epsilon_t}\cdot \frac{W}{|s_t - r_t|X_p}\:\:.\label{alphaconst1}
\end{eqnarray}
We would in addition obtain from Equation (\ref{intalpha}) that $\upalpha_t \in
[1-\epsilon_t, 1+\epsilon_t]$. Suppose $\eta_t$ is also fixed to ensure 
\begin{eqnarray}
\eta_t & \in & \left[\frac{W}{2(p-1)\upkappa_tX_pW^2 - |s_t-r_t|X_p}, \frac{W}{4(p-1)X_pW^2 + |s_t-r_t|X_p}\right]\:\:, \label{etaconst1}
\end{eqnarray}
for some $\upkappa_t $ such that
\begin{eqnarray}
\upkappa_t & \geq & 2 + \frac{|s_t-r_t|}{(p-1)W^2}\:\:.
\end{eqnarray}
Notice that constraint on $\upkappa_t$ makes the interval non empty
and its left bound strictly positive. 
Assuming (\ref{etaconst1}) holds, we would have 
\begin{eqnarray}
\upalpha_t \eta_t & \in & \left[\frac{1}{2(p-1)\upkappa_t X_p^2}, \frac{1}{4(p-1) X_p^2} \right]\:\:.\label{intetalpha}
\end{eqnarray}
The left bound of (\ref{intetalpha}) holds because 
\begin{eqnarray}
\upalpha_t \eta_t & \geq & \eta_t \cdot \frac{W}{W + \eta_t |s_t-r_t|X_p}\nonumber\\
& \geq & \frac{1}{2(p-1)\upkappa_t X_p^2}\:\:.\label{binfea}
\end{eqnarray}
The first inequality holds because of (\ref{intalpha}) and the second one holds because of
(\ref{etaconst1}). The right bound of (\ref{intetalpha}) holds because of (\ref{intalpha}), and so
\begin{eqnarray}
\upalpha_t \eta_t  & \leq & \eta_t \cdot \frac{W}{W - \eta_t
  |s_t-r_t|X_p}\nonumber\\ 
 & \leq & \frac{1}{4(p-1) X_p^2}\:\:,
\end{eqnarray}
where the last inequality is due to
(\ref{etaconst1}). 

Equation (\ref{intetalpha}) makes that
$T(\eta_t\upalpha_t)$ is at least its value when $\upalpha_t \eta_t $ attains the lower bound of (\ref{binfea}), that is,
\begin{eqnarray}
T(\eta_t\upalpha_t) & \geq & \frac{\upkappa_t - 1}{\upkappa_t}\cdot \frac{(s_t - r_t)^2}{8(p-1)X_p^2}\:\:.\label{condT}
\end{eqnarray}
Now, to guarantee $\upalpha_t \in
[1-\epsilon_t, 1+\epsilon_t]$, it is sufficient that the right-hand side
of inequality (\ref{alphaconst1}) belongs to interval (\ref{etaconst1})
\textit{and} we pick $\eta_t$ within the interval [left bound
(\ref{etaconst1}), right-hand side (\ref{alphaconst1})]. To guarantee
that the right-hand side
of inequality (\ref{alphaconst1}) falls in interval (\ref{etaconst1}), we
need first, 
\begin{eqnarray}
\frac{W}{2(p-1)\upkappa_tX_pW^2 -
    |s_t-r_t|X_p} & \leq & \frac{\epsilon_t}{1-\epsilon_t}\cdot \frac{W}{|s_t - r_t|X_p}\:\:,
\end{eqnarray}
that is,
\begin{eqnarray}
\upkappa_t & \geq & \frac{1}{\epsilon_t}\cdot \frac{|s_t - r_t|}{2(p-1)W^2}\:\:.\label{constkappa}
\end{eqnarray}
To guarantee
that the right-hand side of inequality (\ref{alphaconst1}) falls in interval (\ref{etaconst1}) we need then
\begin{eqnarray}
\frac{W}{4(p-1)X_pW^2 + |s_t-r_t|X_p} & \geq & \frac{\epsilon_t}{1-\epsilon_t}\cdot \frac{W}{|s_t - r_t|X_p}\:\:,
\end{eqnarray}
that is,
\begin{eqnarray}
\epsilon_t & \leq & \frac{|s_t - r_t|}{2|s_t - r_t|+4(p-1)W^2} \:\:.\label{constepsilon}
\end{eqnarray}
To summarize, if we pick any strictly positive $\epsilon_t$ following
inequality (\ref{constepsilon}) (note $\epsilon_t < 1$) and 
\begin{eqnarray}
\upkappa_t & \defeq & 2 + \frac{1}{\epsilon_t}\cdot\frac{|s_t-r_t|}{(p-1)W^2}\:\:,\label{defkappa}
\end{eqnarray}
then we shall have both $\upalpha_t \in
[1-\epsilon_t, 1+\epsilon_t]$ and inequality (\ref{condT}) holds as well. In
this case, we shall have
\begin{eqnarray}
T + R + S & \geq & \left(1-\frac{1}{2 + \frac{1}{\epsilon_t}\cdot\frac{|s_t-r_t|}{(p-1)W^2}}\right)\cdot \frac{(s_t -
r_t)^2}{8(p-1)X_p^2} - 2\epsilon_t W^2 -
(p-1)\epsilon_t^2W^2\nonumber\\
 & \geq & \left(1-\frac{1}{2}\right)\cdot \frac{(s_t -
r_t)^2}{8(p-1)X_p^2} - 2\epsilon_t W^2 -
(p-1)\epsilon_t^2W^2\nonumber\\
 &  & = \frac{(s_t -
r_t)^2}{16(p-1)X_p^2} - 2\epsilon_t W^2 -
(p-1)\epsilon_t^2W^2\:\:.\label{constepsilon2}
\end{eqnarray}
To finish up, we want to solve for $\epsilon_t$ the right-hand side such
that it is non negative, and we find that $\epsilon_t$ has to satisfy
\begin{eqnarray}
\epsilon_t & \leq & \frac{1}{p-1}\cdot \left(1 + \sqrt{1+\frac{(s_t-r_t)^2}{16X_p^2W^2}}\right) \:\:.
\end{eqnarray}
Since $\sqrt{1+x}\geq \sqrt{x}$, a sufficient condition is
\begin{eqnarray}
\epsilon_t & \leq & \frac{|s_t-r_t|}{4(p-1)X_pW}\:\:.
\end{eqnarray}
To ensure this and inequality (\ref{constepsilon}), it is sufficient that
we fix
\begin{eqnarray}
\epsilon_t & \defeq & \frac{|s_t - r_t|}{2|s_t - r_t|+4(p-1)WM} \:\:,\label{valepsilon2}
\end{eqnarray}
where $M\defeq \max\{W, X_p\}$. With this expression for $\epsilon_t$,
we get from (\ref{defkappa}),
\begin{eqnarray}
\upkappa_t & \defeq & 2 + \frac{4M}{W} + \frac{2|s_t -r_t|}{(p-1)W^2} \:\:.\label{defkappa2}
\end{eqnarray}
For these choices, Lemma \ref{lemm:eta-allowed-range} implies that the given $\eta_t$ is feasible.
\end{proof}

\begin{lemma}
\label{lemm:eta-allowed-range}
Suppose $\epsilon_t$ satisfies (\ref{valepsilon2})
and $\upkappa_t$ satisfies (\ref{defkappa2}).
Then, a sufficient condition for $\eta_t$ to satisfy both (\ref{alphaconst1}) and (\ref{etaconst1}) is
\begin{eqnarray}
\eta_t & = & \upgamma \cdot \frac{W}{4(p-1)MX_pW+|y_t - \ve{w}_{t-1}^\top \ve{x}_t|X_p}\:\:, \nonumber
\end{eqnarray}
for any $\upgamma \in [1/2, 1]$. 
\end{lemma}

\begin{proof}[Proof of Lemma \ref{lemm:eta-allowed-range}]
Notice the range of values authorized for $\eta_t$:
\begin{eqnarray}
\eta_t & \in & \left[\frac{W}{2(p-1)\upkappa_tX_pW^2 - |s_t-r_t|X_p}, \frac{\epsilon_t}{1-\epsilon_t}\cdot \frac{W}{|s_t - r_t|X_p}\right]\nonumber\\
 &  & = \left[\frac{W}{2(p-1)\left( 2 + \frac{4M}{W} + \frac{2|s_t - r_t|}{(p-1)W^2}\right)X_pW^2 - |s_t-r_t|X_p}, \frac{W}{4(p-1)MX_pW+|s_t - r_t|X_p}\right]\nonumber\\
 &  & = \left[\frac{W}{2(2(p-1)W^2+4M(p-1)W+2|s_t-r_t|)X_p - |s_t-r_t|X_p}, \frac{W}{4(p-1)MX_pW+|s_t - r_t|X_p}\right]\nonumber\\
 &  & = \left[\frac{W}{4(p-1)X_pW^2+8(p-1)MX_pW+3|s_t-r_t|X_p}, \frac{W}{4(p-1)MX_pW+|s_t - r_t|X_p}\right]\nonumber\\
 & \supset & \left[\frac{W}{8(p-1)MX_pW+2|s_t-r_t|X_p}, \frac{W}{4(p-1)MX_pW+|s_t - r_t|X_p}\right]\:\:. \label{etaconst2}
\end{eqnarray}
A sufficient condition for $\eta_t$ to fall in interval
(\ref{etaconst2}) is
\begin{eqnarray}
\eta_t & = & \upgamma \cdot \frac{W}{4(p-1)MX_pW+|y_t - \ve{w}_{t-1}^\top \ve{x}_t|X_p}\:\:\nonumber,
\end{eqnarray}
for any $\upgamma \in [1/2, 1]$. 
\end{proof}

\begin{lemma}
\label{lemm:iterate-divergence-bound}
Suppose we fix the learning rate as per (\ref{eqn:final-eta}).
Pick any $\ve{u}$, and consider iterates $\{ \ve{w}_t \}_{t = 0}^{T}$ as per Equation \ref{eqn:plms-dagger-explicit}.
Suppose $\|\ve{x}_t\|_p \leq X_p$ and $|y_t| \leq Y, \forall t \leq T$.
Then, for any $t$,
$$ D_{\varphi_q}\left(\bar{\ve{u}} \left\| \ve{w}_{t-1}\right.\right) - D_{\varphi_q}\left(\bar{\ve{u}} \left\| \ve{w}_{t}\right.\right) \geq \frac{1}{4(p-1)\left(2 + \frac{4M}{W} + \frac{2(Y + X_pW)}{(p-1)W^2} \right) X_p^2} \cdot (s^2_t
- r^2_t) $$
where $s_t \defeq (\bar{\ve{u}} - \ve{w}_{t-1})^\top\ve{x}_t$, $r_t \defeq \bar{\ve{u}}^\top\ve{x}_t - y_t$.
\end{lemma}

\begin{proof}[Proof of Lemma \ref{lemm:iterate-divergence-bound}]
We start from the bound of Lemma \ref{lemm:iterate-bound-1}:
\begin{eqnarray}
 D_{\varphi_q}\left(\bar{\ve{u}} \left\| \ve{w}_{t-1}\right.\right) - D_{\varphi_q}\left(\bar{\ve{u}} \left\| \ve{w}_{t}\right.\right) & \geq & Q + R + S + T \nonumber \\
 & \geq & Q \text{ by Lemma \ref{lemm:rst} } \nonumber  \\
 & = & \frac{\upalpha_t}{2}\eta_t(s^2_t - r^2_t) \text{ by definition } \nonumber\\
 & \geq & \frac{1}{4(p-1)\upkappa_t X_p^2} \cdot (s^2_t - r^2_t) \nonumber\\
 & \geq & \frac{1}{4(p-1)\left(2 + \frac{4M}{W} + \frac{2\max_t |y_t - \ve{w}_{t-1}^\top \ve{x}_t|}{(p-1)W^2} \right) X_p^2} \cdot (s^2_t - r^2_t) \nonumber\\
 & \geq & \frac{1}{4(p-1)\left(2 + \frac{4M}{W} + \frac{2(Y + X_pW)}{(p-1)W^2} \right) X_p^2} \cdot (s^2_t - r^2_t)\:\:.
\end{eqnarray}

The last constraint to check for this bound to be valid is our $\epsilon_t$ in
(\ref{valepsilon2}) has
to be $<1/2$ from inequality (\ref{consteta0}), which trivially holds since
$4(p-1)WM\geq 0$.

We conclude by noting Lemma \ref{lemm:eta-allowed-range} provides a feasible value of $\eta_t$.
\end{proof}

\subsection{Gauge normalisation}

The following lemma about the gauge of $\ve{x}$ will be useful.

\begin{lemma}
\label{lemm:gq-constant}
Let $g_q( \ve{x} ) = \| \ve{x} \|_q/W$ for some $W > 0$.
Then, for the iterates $\{ \ve{w}_t \}$ as per Equation \ref{eqn:plms-explicit}, $g_{q}(\ve{w}_t) = 1$.
\end{lemma}

\begin{proof}
We have
\begin{eqnarray}
g_{q}(\ve{w}_t) & = & \frac{\|\ve{w}_t\|_q}{W}\nonumber \\
 & = & \frac{W}{W} \text{ by Lemma \ref{lemm:plms-update-norm} } \nonumber \\
 & = & 1\:\:, \forall t\geq 1\:\:.
\end{eqnarray}
\end{proof}

\section{Proofs of results in main body}
\label{app:proofs}

We present proofs of all results in the main body.

\begin{proof}[Proof of Theorem \ref{th00}]
Let $\matrice{J} : \XCal\rightarrow \XCal_g$ denote the Jacobian of $h \colon \ve{x} \mapsto (1/g(\ve{x}))\cdot \ve{x}$.
By an elementary calculation,
$$ g(\ve{x}) \cdot \matrice{J} = \matrice{I}_d - (1/g(\ve{x}))\cdot \ve{x}  \nabla g(\ve{x})^\top, $$
which by the chain rule brings the following expression for the gradient of $\varphi^\dagger( \ve{y} ) = g(\ve{y}) \cdot (\varphi \circ h)( \ve{y} )$:
\begin{align}
  \nonumber \nabla \varphi^\dagger \left(\ve{y}\right) &= \nabla g(\ve{y}) \cdot (\varphi \circ h)( \ve{y} ) + g(\ve{y}) \cdot \nabla (\varphi \circ h)( \ve{y} ) \\
  \nonumber &= \nabla g(\ve{y}) \cdot (\varphi \circ h)( \ve{y} ) + g(\ve{y}) \cdot \matrice{J}^\top \nabla \varphi( h( \ve{y} ) ) \\
  \nonumber &= \nabla g(\ve{y}) \cdot (\varphi \circ h)( \ve{y} ) + \nabla \varphi( h( \ve{y} ) ) - (1/g(\ve{y}))\cdot \nabla g(\ve{y}) \ve{y}^\top \nabla \varphi( h( \ve{y} ) ) \\
  &= \nabla\varphi\left( \frac{1}{g(\ve{y})}\cdot \ve{y}\right) + \left(\varphi\left(\frac{1}{g(\ve{y})}\cdot \ve{y}\right) - \frac{1}{g(\ve{y})}\cdot \ve{y}^\top \nabla\varphi\left( \frac{1}{g(\ve{y})}\cdot \ve{y}\right)\right)\cdot \nabla g(\ve{y})\:\:. \label{gradphidagger}
\end{align}

For simplicity, let $\ve{u} = \ve{x}/g(\ve{x})$ and $\ve{v} = \ve{y}/g(\ve{y})$,
so that $\varphi^{\dagger}(\ve{x}) = g(\ve{x}) \cdot \varphi( \ve{u} )$ and $\varphi^{\dagger}(\ve{y}) = g(\ve{y}) \cdot \varphi( \ve{v} )$.
The above then reads
\begin{equation}
   \label{eqn:grad-phidagger-2}
   \nabla \varphi^\dagger \left(\ve{y}\right) = \nabla\varphi\left( \ve{v} \right) + \left(\varphi\left( \ve{v} \right) - \ve{v}^\top \nabla\varphi\left( \ve{v} \right)\right)\cdot \nabla g(\ve{y}).
\end{equation}
Now, the LHS of Equation (\ref{eq11}) is
\begin{align*}
  g(\ve{x}) \cdot D_{\varphi}\left(\frac{1}{g(\ve{x})} \cdot \ve{x} \bigm\| \frac{1}{g(\ve{y})} \cdot \ve{y}\right) &= g(\ve{x}) \cdot D_{\varphi}\left( \ve{u} \| \ve{v} \right) \\
  &= g(\ve{x}) \cdot \varphi( \ve{u} ) - g( \ve{x} ) \cdot \varphi( \ve{v} ) - g( \ve{x} ) \cdot \nabla\varphi( \ve{v} )^\top ( \ve{u} - \ve{v} ) \\
  &= \varphi^\dagger( \ve{x} ) - g( \ve{x} ) \cdot \varphi( \ve{v} ) - \nabla\varphi( \ve{v} )^\top ( \ve{x} - g(\ve{x}) \cdot \ve{v} ) \\
  &= \varphi^\dagger( \ve{x} ) - g( \ve{x} ) \cdot ( \varphi( \ve{v} ) - \nabla\varphi( \ve{v} )^\top \ve{v} ) - \nabla\varphi( \ve{v} )^\top \ve{x},
\end{align*}
while the RHS is
\begin{align*}
  D_{\varphi^\dagger}\left( \ve{x} \bigm\| \ve{y}\right) &= \varphi^\dagger( \ve{x} ) - \varphi^\dagger( \ve{y} ) - \nabla \varphi^\dagger( \ve{y} )^\top ( \ve{x} - \ve{y} ) \\
  &= \varphi^\dagger( \ve{x} ) - g(\ve{y}) \cdot \varphi( \ve{v} ) - \nabla\varphi\left( \ve{v} \right)^\top ( \ve{x} - \ve{y} ) - \left( \varphi\left( \ve{v} \right) - \ve{v}^\top \nabla\varphi\left( \ve{v} \right)\right)\cdot \nabla g(\ve{y})^\top ( \ve{x} - \ve{y} ).
\end{align*}
Cancelling the common $\varphi^\dagger( \ve{x} )$ and $\nabla\varphi\left( \ve{v} \right)^\top \ve{y}$ terms, the difference $\Delta = \mathrm{RHS} - \mathrm{LHS}$ is

\resizebox{\textwidth}{!}{
\begin{minipage}{\textwidth}
\begin{align*}
  \Delta &= g( \ve{x} ) \cdot ( \varphi( \ve{v} ) - \nabla\varphi( \ve{v} )^\top \ve{v} ) - g(\ve{y}) \cdot \varphi( \ve{v} ) + \nabla\varphi\left( \ve{v} \right)^\top \ve{y} - \left( \varphi\left( \ve{v} \right) - \ve{v}^\top \nabla\varphi\left( \ve{v} \right)\right)\cdot \nabla g(\ve{y})^\top ( \ve{x} - \ve{y} ) \\
  &= g( \ve{x} ) \cdot ( \varphi( \ve{v} ) - \nabla\varphi( \ve{v} )^\top \ve{v} ) - g(\ve{y}) \cdot \varphi( \ve{v} ) + g(\ve{y}) \cdot \nabla\varphi\left( \ve{v} \right)^\top \ve{v} - \left( \varphi\left( \ve{v} \right) - \ve{v}^\top \nabla\varphi\left( \ve{v} \right)\right)\cdot \nabla g(\ve{y})^\top ( \ve{x} - \ve{y} ) \\
  &= g( \ve{x} ) \cdot ( \varphi( \ve{v} ) - \nabla\varphi( \ve{v} )^\top \ve{v} ) - g(\ve{y}) \cdot \left( \varphi( \ve{v} ) - \nabla\varphi\left( \ve{v} \right)^\top \ve{v} \right) - \left( \varphi\left( \ve{v} \right) - \ve{v}^\top \nabla\varphi\left( \ve{v} \right)\right)\cdot \nabla g(\ve{y})^\top ( \ve{x} - \ve{y} ) \\
  &= ( \varphi( \ve{v} ) - \nabla\varphi( \ve{v} )^\top \ve{v} ) \cdot ( g(\ve{x}) - g(\ve{y}) - \nabla g(\ve{y})^\top (\ve{x} - \ve{y}) ) \\
  &= ( \varphi( \ve{v} ) - \nabla\varphi( \ve{v} )^\top \ve{v} ) \cdot B_g( \ve{x} \| \ve{y} ).
\end{align*}
\end{minipage}
}

Thus, the identity holds, if and only if either
$\varphi( \ve{v} ) = \nabla\varphi( \ve{v} )^\top \ve{v}$ for every $\ve{v} \in \XCal_g$,
or $B_g( \ve{x} \| \ve{y} ) = 0$.
The latter is true if and only if $g$ is affine from Equation \ref{eqBreg}.
The result follows.
\end{proof}

It is easy to check that Theorem \ref{th00} in fact holds for separable (matrix) trace divergences \citep{ksdLR} of the form
\begin{eqnarray}
D_{\varphi}( \matrice{x}\|\matrice{y}) & \defeq & \varphi(\matrice{x}) - \varphi(\matrice{y}) - \trace{\nabla \varphi(\matrice{y})^\top (\matrice{x}-\matrice{y})} \:\:, \label{eqBregMat}
\end{eqnarray}
with $\varphi, g : \textbf{S}(d) \rightarrow \Real$ (for $\textbf{S}(d)$ the set of symmetric real matrices), with $\varphi$ convex.
In this case, the restricted positive homogeneity property becomes
\begin{eqnarray}
\varphi\left(\matrice{u}\right)  & = & \trace{\nabla\varphi(\matrice{u})^\top \matrice{u}}\:\:, \forall \matrice{u}\in {\XCal}_g\:\:.
\end{eqnarray}

\begin{proof}[Proof of Lemma \ref{lemm:multiclass-dr}]
Note that by construction, $g(\ve{r}(\ve{x})) = \pr(\X=\ve{x}) / ((1-\pi_C) \cdot
\pr(\X = \ve{x} | \Y = C))$, and so
\begin{eqnarray}
\left(\frac{1}{g(\ve{r}(\ve{x}))} \cdot \ve{r}(\ve{x})\right)_c & = &
\frac{(1-\pi_C) \cdot
\pr(\X = \ve{x} | \Y = C)}{\pr(\X=\ve{x})} \cdot \frac{\pr(\X = \ve{x} | \Y = c)}{\pr(\X = \ve{x} | \Y = C)}\nonumber\\
 & = &\frac{(1-\pi_C)}{\pi_c} \cdot \frac{\pi_c \pr(\X = \ve{x} | \Y = c)}{\pr(\X=\ve{x})} \nonumber\\
 & = & \eta(\ve{x})\:\:.\label{eqETAC}
\end{eqnarray}
Furthermore,
\begin{eqnarray}
\pr(\X=\ve{x}) & = & \sum_{c=1}^{C} {\pi_c \pr(\X = \ve{x} | \Y = c)} \nonumber\\
& = & (1-\pi_C)\cdot \left( \frac{\pi_C}{1-\pi_C} + \sum_{c<C} \frac{\pi_c}{1-\pi_C}\cdot \frac{\pr(\X = \ve{x} | \Y = c)}{\pr(\X = \ve{x} | \Y = C)}\right) \cdot  \pr(\X = \ve{x} | \Y = C)\nonumber\\
& = & (1-\pi_C)\cdot  g(\ve{r}(\ve{x})) \cdot \pr(\X = \ve{x} | \Y = C)\label{eq123C}\:\:.
\end{eqnarray}
Now let
$$ \hat{\ve{r}}(\ve{x}) = \frac{1}{\hat{\eta}_C(\ve{x})} \cdot\hat{\ve{\eta}}(\ve{x}). $$
It then comes
\begin{eqnarray}
\lefteqn{\expect_{M}[D_\varphi(\ve{\eta}(\X) \| \hat{\ve{\eta}}(\X))]}\nonumber\\
 & = & (1-\pi_C)\cdot \expect_{P_C}\left[ g(\ve{r}(\ve{x}))\cdot D_\varphi(\ve{\eta}(\X) \| \hat{\ve{\eta}}(\X))\right] \nonumber \\
 & = &  (1-\pi_C) \cdot \expect_{P_C}\left[ g(\ve{r}(\ve{x}))\cdot D_\varphi\left(\left.\frac{1}{g(\ve{r}(\ve{x}))} \cdot \ve{r}(\ve{x}) \right\| \hat{\ve{\eta}}(\X)\right)\right] \nonumber \\
 & = &  (1-\pi_C) \cdot \expect_{P_C}\left[ g(\ve{r}(\ve{x}))\cdot D_\varphi\left(\left.\frac{1}{g(\ve{r}(\ve{x}))} \cdot \ve{r}(\ve{x}) \right\| \frac{1}{g(\hat{\ve{r}}(\X))} \cdot \hat{\ve{r}}(\X)\right)\right] \nonumber \\
 & = & (1-\pi_C) \cdot \expect_{P_C}\left[ D_{\varphi^\dagger}(\ve{r}(\X) \| \hat{\ve{r}}(\X))\right]\:\:,\nonumber
\end{eqnarray}
as claimed.
\end{proof}

\begin{proof}[Proof of Lemma \ref{lemm:plms-update-norm}]
For any $\ve{x}$, $\| \nabla{\varphi_p^\dagger}( \ve{x} ) \|_q = W$ by Corollary \ref{corr:phi-gradient-norms}.
Since $\ve{w}_t = \nabla{\varphi_p^\dagger}( \ve{\theta}_{t - 1} )$ for suitable $\ve{\theta}_{t - 1}$, the result follows.
The result for $\| \nabla{\varphi_q^{\dagger}}(\ve{w}_t) \|_p$ follows similarly by Corollary \ref{corr:phi-gradient-norms}.

Note that while $\|\ve{w}_t\|_q = \| \nabla{\varphi}(\ve{w}_t) \|_p$ for the standard $p$-LMS update \citep[Appendix I]{kwhTP}, these norms may vary with each iteration \ie $\ve{w}_t$ may not lie in the $L_q$ ball.
\end{proof}

\begin{proof}[Proof of Lemma \ref{lemlplq}]
Similarly to the proof of Lemma \ref{lemupdate}, a key to the proof of
Lemma \ref{lemlplq} relies on branching on \citet{kwhTP} through the
use of Theorem \ref{th00}.
We first note that
$D_{\varphi_q^\dagger}(\ve{u}\| \ve{w}_0) = W\cdot \|\ve{u}\|_q$ since $\ve{w}_0 = \ve{0}$, and $D_{\varphi_q^\dagger}(\ve{u}\| \ve{w}_{T+1}) \geq 0$, and so

\begin{minipage}{\linewidth}
\begin{eqnarray}
W\cdot\|\ve{u}\|_q & \geq & D_{\varphi_q^\dagger}(\ve{u}\| \ve{w}_0) - D_{\varphi_q^\dagger}(\ve{u}\| \ve{w}_{T+1}) \nonumber\\
  & = & \sum_{t=1}^T  \left\{D_{\varphi_q^\dagger}(\ve{u}\| \ve{w}_{t-1}) - D_{\varphi_q^\dagger}(\ve{u}\| \ve{w}_{t})\right\}\nonumber \text{ by telescoping property } \\
  & = & g_{q}(\ve{u}) \cdot \sum_{t=1}^T \left\{D_{\varphi_q}\left(\frac{\ve{u}}{g_{q}(\ve{u})} \left\| \frac{\ve{w}_{t-1}}{g_{q}(\ve{w}_{t-1})} \right.\right) -
        D_{\varphi_q}\left(\frac{\ve{u}}{g_{q}(\ve{u})}\left\| \frac{\ve{w}_{t}}{g_{q}(\ve{w}_t)} \right.\right)\right\}\nonumber \text{ by Theorem \ref{th00} } \\
  & = & g_{q}(\ve{u}) \cdot \sum_{t=1}^T \left\{D_{\varphi_q}\left(\frac{\ve{u}}{g_{q}(\ve{u})} \left\| \ve{w}_{t-1}\right.\right) - D_{\varphi_q}\left(\frac{\ve{u}}{g_{q}(\ve{u})} \left\| \ve{w}_{t}\right.\right)\right\} \text{ by Lemma \ref{lemm:gq-constant} } \:\:.\label{eq111}
\end{eqnarray}
\end{minipage}

Recall from Lemma \ref{lemm:iterate-divergence-bound} that
$$ D_{\varphi_q}\left(\frac{\ve{u}}{g_{q}(\ve{u})} \left\| \ve{w}_{t-1}\right.\right) - D_{\varphi_q}\left(\frac{\ve{u}}{g_{q}(\ve{u})} \left\| \ve{w}_{t}\right.\right) \geq \frac{1}{4(p-1)\left(2 + \frac{4M}{W} + \frac{2(Y + X_pW)}{(p-1)W^2} \right) X_p^2} \cdot (s^2_t
- r^2_t) $$
where
\begin{align*}
  s_t &\defeq ((1/g_{q}(\ve{u}))\cdot \ve{u} - \ve{w}_{t-1})^\top\ve{x}_t \\
  r_t &\defeq (1/g_{q}(\ve{u}))\cdot \ve{u}^\top\ve{x}_t - y_t.
\end{align*}
Note that $R_q(\ve{w}_{1:T} | \ve{u}) = \sum_{t = 1}^T (s_t^2 - r_t^2)$ by definition.
Summing the above for $t=1, 2, ..., T$ and telescoping sums yields
\begin{eqnarray}
R_q(\ve{w}_{1:T} | \ve{u}) & \leq & 4(p-1)\left(2 + \frac{4M}{W} + \frac{2(Y + X_pW)}{(p-1)W^2}\right) X_p^2W^2 \nonumber\\
 & = & 4(p-1)X_p^2W^2 + 16(p-1)MX_p^2W + 8(Y+X_pW)X^2_p\nonumber \\
 & \leq & 4(p-1)X_p^2W^2 + (16p-8)MX_p^2W + 8YX^2_p\:\:.
\end{eqnarray}

See Figure \ref{fig:lplq} for some geometric intuition about the updates.
\end{proof}

\begin{figure}[!t]
\centering 
\includegraphics[trim=170bp 280bp 130bp 230bp,clip,width=.575\linewidth]{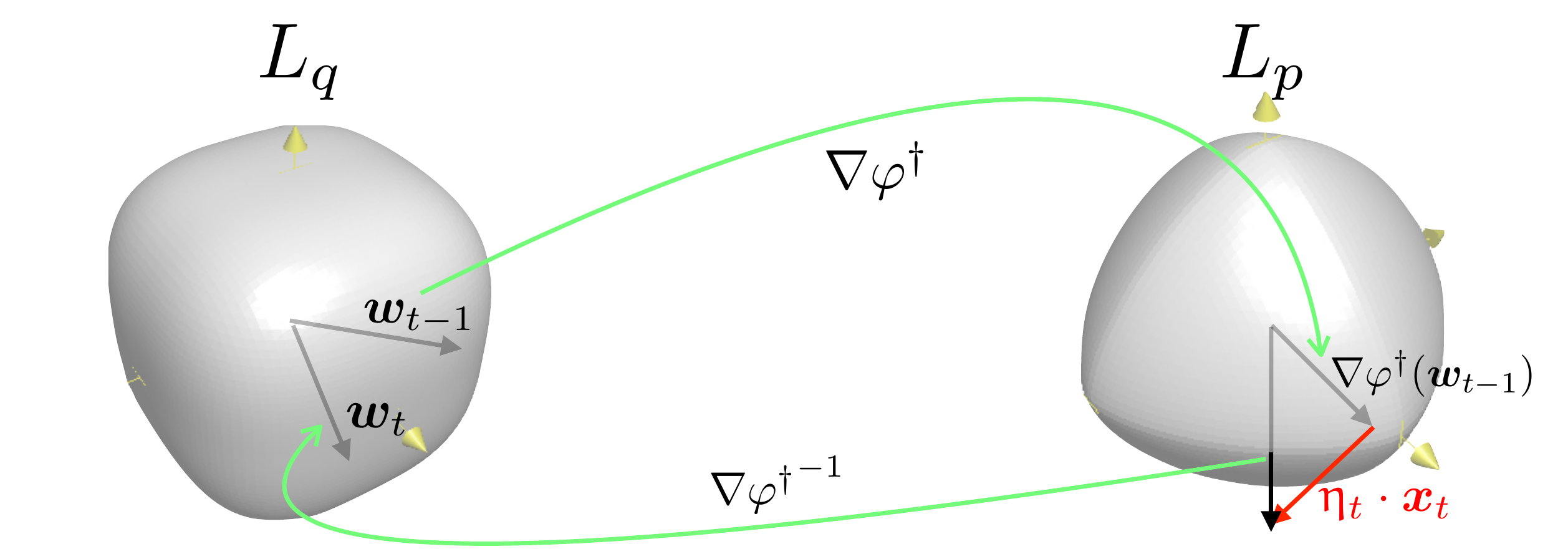} 
\caption{Illustration of the case $W=1$ for the ${\mathcal{B}}_q(W)$-update: all classifiers and image via $\nabla{\varphi^\dagger}$ belong to a ball of radius 1  (here, $q=3$, $p=3/2$).\label{fig:lplq}}
\end{figure}

\begin{proof}[Proof of Lemma \ref{lemm:kmeanspp}]
We start by the sphere. Let $\varphi(\ve{x}) \defeq (1/2) \cdot \|\ve{x}\|_2^2$. Since a
Bregman divergence is invariant to linear transformation, it comes from
Table \ref{t-ex} that
\begin{align*}
   D_\varphi\left( \frac{\ve{x}^S}{g_S(\ve{x}^S)} \bigm\| \frac{\ve{c}^S}{g_S(\ve{c}^S)} \right) = \frac{1}{g_S(\ve{c}^S)} \cdot D_{\varphi^{\dagger}}( \ve{x} \| \ve{c} )
   = 1 - \cos D_G(\ve{x}, \ve{c}), \nonumber
\end{align*}
where we recall that $D_G$ denotes the geodesic distance on the sphere
(see Figure \ref{fig:lift} and Appendix
\ref{sec-app-der}). Equivalently, 
\begin{eqnarray}
\left\| \frac{1}{g_S(\ve{x}^S)} \cdot \ve{x}^S - \frac{1}{g_S(\ve{c}^S)}\cdot  \ve{c}^S\right\|_2^2 & = & 1 - \cos D_G(\ve{x}, \ve{c})\:\:.\label{eqcostKM}
\end{eqnarray}
This equality allows us to use $k$-means++ using the LHS of
(\ref{eqcostKM}) to compute the distribution that picks a center. The
key to using the approximation property of $k$-means++ relies on the
existence of a coordinate system on the sphere for which the cluster
centroid is just the average of the cluster points (polar
coordinates), an average that eventually has to be rescaled if the coordinate
system is not that one \citep{dmCD,emSK}. The existence of this
coordinate system makes that the proof of \cite{avKM} (and in
particular the key Lemmata 3.2 and 3.3) can be carried out without
modification to yield the same approximation ratio as that of
\cite{avKM} \textit{if} the distortion at hand is the squared
Euclidean distance, which turns out to be $D_\rec(.:.)$ from eq. (\ref{eqcostKM}).\\
The case of the hyperboloid follows the exact same path, but starts
from the fact that 
Table \ref{t-ex} now brings
\begin{align*}
   D_\varphi\left( \frac{\ve{x}^H}{g_H(\ve{x}^H)} \bigm\| \frac{\ve{c}^H}{g_H(\ve{c}^H)} \right) 
   = \cosh D_G(\ve{y}, \ve{c}) - 1 
=\left\| \frac{1}{g_H(\ve{x}^H)} \cdot \ve{x}^H - \frac{1}{g_H(\ve{c}^H)}\cdot  \ve{c}^H\right\|_2^2 \:\:.\nonumber
\end{align*}
To finish, in the same way as for the Sphere, we just need the
existence of a coordinate system for which the centroid is an average
of the cluster points, which can be obtained from hyperbolic
barycentric coordinates \cite[Section 18]{Ungar2014}. 
\end{proof}


\section{Working out examples of Table \ref{t-ex}}
\label{sec-app-der}

We fill in the details justifying each of the examples of Equation \ref{eq11} provided in Table \ref{t-ex-short2}.
We also provide the form of the corresponding divergences $D_{\varphi}$ and distortions $D_{\varphi^{\dagger}}$ in the augmented Table \ref{t-ex}.

\begin{table}[!h]
\renewcommand{\arraystretch}{1.5}
\begin{center}
\scalebox{0.7}{
\begin{tabular}{l|l|l||l||l|l}\hline\hline
 & $\varphi$ & $D_{\varphi}\left(\ve{x}\|\ve{y}\right)$ & $g$ & $\varphi^\dagger$ & $D_{\varphi^\dagger}\left(\ve{x}\|\ve{y}\right)$ \\\hline
 I & $\frac{1}{2}\cdot(1 + \|\ve{x}\|_2^2)$ & $(1/2)\cdot \|\ve{x}-\ve{y}\|_2^2$ & $\|\ve{x}\|_2$ & $\|\ve{x}\|_2$ & $\|\ve{x}\|_2 \cdot \left( 1 - \cos \angle \ve{x}, \ve{y}\right)$ \\\hline
 \multirow{2}{*}{II} & \multirow{2}{*}{$ \frac{1}{2}\cdot (W + \|\ve{x}\|_q^2)$} & $(1/2)\cdot(\|\ve{x}\|_q^2-\|\ve{y}\|_q^2)$ & \multirow{2}{*}{$\frac{\|\ve{x}\|_q}{W}$} & \multirow{2}{*}{$W \cdot \|\ve{x}\|_q$} & \multirow{2}{*}{$W \cdot |\ve{x}\|_q - W \cdot \sum_i \frac{x_i \cdot \mathrm{sign}(y_i) \cdot |y_i|^{q-1}}{\|\ve{y}\|_q^{q-1}}$}\\
 & & $-\sum_i \frac{(x_i - y_i) \cdot \mathrm{sign}(y_i) \cdot |y_i|^{q-1}}{\|\ve{y}\|_q^{q-2}}$ & & & \\ \hline
 {III} & $\frac{1}{2}\cdot(u^2 + \|\ve{x}^S\|_2^2)$ & $(1/2)\cdot
 \|\ve{x}^S-\ve{y}^S\|_2^2$ & $\frac{\|\ve{x}\|_2}{\sin \|\ve{x}\|_2}$ & $\|\ve{x}^S\|_2$ & $\frac{\|\ve{x}\|_2}{\sin \|\ve{x}\|_2} \cdot \left( 1 - \cos D_G(\ve{x}, \ve{y})\right)$ \\\hline  
 {IV} & $\frac{1}{2}\cdot(u^2 + \|\ve{x}^H\|_2^2)$ & $(1/2)\cdot
 \|\ve{x}^H-\ve{y}^H\|_2^2$ & $-\frac{\|\ve{x}\|_2}{\sinh
   \|\ve{x}\|_2}$ & $\|\ve{x}^H\|_2$ & $-\frac{\|\ve{x}\|_2}{\sinh
   \|\ve{x}\|_2} \cdot \left( \cosh D_G(\ve{x}, \ve{y}) - 1\right)$ \\\hline  
 \multirow{2}{*}{V} & \multirow{2}{*}{$\sum_i x_i \log x_i - x_i$} & $\sum_i x_i \log \frac{x_i}{y_i}$ &  \multirow{2}{*}{$\ve{1}^\top \ve{x}$} & $\sum_i x_i\log x_i - \ve{1}^\top \ve{x}$ & $\sum_i x_i\log \frac{x_i}{y_i} $\\
 & & $- \ve{1}^\top(\ve{x}-\ve{y})$ & & $-(\ve{1}^\top \ve{x}) \log (\ve{1}^\top \ve{x}) $  & $-d\cdot \expect[\X] \cdot \log \frac{\expect[\X]}{\expect[\Y]}$\\\hline
 \multirow{2}{*}{VI} & \multirow{2}{*}{$- d -\sum_i \log x_i$} & $\sum_i\frac{x_i}{y_i} $ &
 \multirow{2}{*}{$\prod_i x_i^{1/d}$} & \multirow{2}{*}{$-d \cdot \prod_i x_i^{1/d}$} & \multirow{2}{*}{$\sum_i \frac{x_i (\pi_{\ve{y}})^{1/d}}{y_i}  - d (\pi_{\ve{x}})^{1/d}$}\\
 & & $-\sum_i\log\frac{x_i}{y_i} - d$& & & \\\hline
  \multirow{2}{*}{VII} & \multirow{2}{*}{$\trace{\matrice{x}\log \matrice{x} - \matrice{x}}$} & $\trace{\matrice{x}\log \matrice{x} - \matrice{x}\log \matrice{y}}$ & \multirow{2}{*}{$\trace{\matrice{x}}$} & $\trace{\matrice{x}\log \matrice{x} - \matrice{x}}$ & $\trace{\matrice{x}\log \matrice{x} - \matrice{x}\log \matrice{y}}$\\
 & & $-\trace{\matrice{x}} + \trace{\matrice{y}}$ & & $-\trace{\matrice{x}}\log \trace{\matrice{x}}$ & $-\trace{\matrice{x}} \cdot \log\frac{\trace{\matrice{x}}}{\trace{\matrice{y}}}$\\\hline
\multirow{2}{*}{VIII} & \multirow{2}{*}{$-d-\log \det (\matrice{x})$} &
$\trace{\matrice{x}\matrice{y}^{-1}}$ & \multirow{2}{*}{$\det
  (\matrice{x}^{1/d})$} & \multirow{2}{*}{$-d \cdot \det (\matrice{x}^{1/d})$} & $\det (\matrice{y}^{1/d}) \trace{\matrice{x}\matrice{y}^{-1}} -d\cdot \det
 (\matrice{x}^{1/d})$\\
 & & $-\log\det (\matrice{x}\matrice{y}^{-1}) - d$ & & & \\ \hline \hline
\end{tabular}
}
\end{center}
\caption{Example of distortions (right columns) that can be ``reverse
  engineered'' as Bregman divergences involving a particular, non
  necessary linear $g$. Function $\ve{x}^S \defeq f(\ve{x}) : {\Real}^d
  \rightarrow  {\Real}^{d+1}$ is the (S)phere lifting map defined
  in (\ref{defSmap}), and $\ve{x}^H$ is the (H)yperboloid lifting map
  defined in (\ref{defHymap}). $D_G(.,.)$ is the geodesic distance
  between the exponential map of
  $\ve{x}$ and $\ve{y}$ on their respective manifold (sphere or hyperboloid).
Related proofs are in Section
  \ref{sec-app-der}. Expectation $\expect[\X]$ is a shorthand for
  $(1/d)\cdot \sum_i x_i$. $W \in \Real_{+*}$ and $u\in \Real$ are constants.\label{t-ex}}
\end{table}

\paragraph{Row I ---} for $\XCal = \Real^d$, consider $\varphi(\ve{x}) =
(1+\|\ve{x}\|_2^2)/2$ and $g(\ve{x}) = \|\ve{x}\|_2$ (we project on the
Euclidean sphere). It comes 
\begin{eqnarray}
\varphi^\dagger (\ve{x}) & = & \|\ve{x}\|_2\cdot
\left(\frac{1+\left\|\frac{1}{\|\ve{x}\|_2}\cdot
      \ve{x}\right\|_2^2}{2}\right) = \|\ve{x}\|_2\:\:.
\end{eqnarray}
$g$ is not linear (but it is homogeneous of degree 1), but we have
\begin{eqnarray}
\varphi(\ve{x}) = 1 = \ve{x}^\top \nabla\varphi(\ve{x}) \:\:,
\forall \ve{x} : \|\ve{x}\|_2 = 1\:\:,
\end{eqnarray}
so $\varphi$ is 1-homogeneous on the Euclidean sphere, and we can apply
Theorem \ref{th00}. We have
\begin{eqnarray}
g(\ve{x}) \cdot D_{\varphi}\left(\frac{1}{g(\ve{x})}\cdot
  \ve{x}\|\frac{1}{g(\ve{y})}\cdot \ve{y}\right) & = & \frac{\|\ve{x}\|_2}{2}
\cdot \left\| \frac{1}{\|\ve{x}\|_2}\cdot
      \ve{x}-\frac{1}{\|\ve{y}\|_2}\cdot
      \ve{y}\right\|_2^2\nonumber\\
 & = & \|\ve{x}\|_2 \cdot \left(1 - \frac{\ve{x}^\top
     \ve{y}}{\|\ve{x}\|_2 \|\ve{y}\|_2}\right) = \|\ve{x}\|_2 \cdot (1
 - \cos(\ve{x}, \ve{y}))\:\:,\label{peq1}
\end{eqnarray}
and we also have
\begin{eqnarray}
D_{\varphi^\dagger}\left(
  \ve{x}\|\ve{y}\right) & = & \|\ve{x}\|_2 - \|\ve{y}\|_2 -
\frac{1}{\|\ve{y}\|_2} \cdot (\ve{x} -
\ve{y})^\top \ve{y}\nonumber\\
 & = & \|\ve{x}\|_2 - \|\ve{y}\|_2 - \frac{\ve{x}^\top
     \ve{y}}{\|\ve{y}\|_2} + \|\ve{y}\|_2\\
 & =& \|\ve{x}\|_2 \cdot \left(1 - \frac{\ve{x}^\top
     \ve{y}}{\|\ve{x}\|_2 \|\ve{y}\|_2}\right) = \|\ve{x}\|_2 \cdot (1
 - \cos(\ve{x}, \ve{y}))\:\:,
\end{eqnarray}
which is equal to Equation (\ref{peq1}), so we check that Theorem
\ref{th00} applies in this case. $D_{\varphi^\dagger}$ has some
properties. One is a weak form of triangle inequality.
\begin{lemma}
$D_{\varphi^\dagger}\left(
  \ve{x}\|\ve{y}\right)+D_{\varphi^\dagger}\left(
  \ve{y}\|\ve{z}\right) \leq D_{\varphi^\dagger}\left(
  \ve{x}\|\ve{z}\right)$, $\forall \ve{x}, \ve{y}, \ve{z}$ such that $\|\ve{y}\|_2 \leq \|\ve{x}\|_2$.
\end{lemma}
\begin{proof}
\begin{eqnarray}
\lefteqn{D_{\varphi^\dagger}\left(
  \ve{x}\|\ve{y}\right)+D_{\varphi^\dagger}\left(
  \ve{y}\|\ve{z}\right)}\nonumber\\
 & = & \|\ve{x}\|_2 \cdot (1
 - \cos(\ve{x}, \ve{y})) + \|\ve{y}\|_2 \cdot (1
 - \cos(\ve{y}, \ve{z}))\nonumber\\ 
& = & \|\ve{x}\|_2 \cdot\left( (1
 - \cos(\ve{x}, \ve{y})) + (1
 - \cos(\ve{y}, \ve{z}))\right) + (\|\ve{y}\|_2 - \|\ve{x}\|_2)\cdot (1
 - \cos(\ve{y}, \ve{z})) \nonumber\\ 
& \leq & \|\ve{x}\|_2 \cdot(1
 - \cos(\ve{x}, \ve{z})) + (\|\ve{y}\|_2 - \|\ve{x}\|_2)\cdot (1
 - \cos(\ve{y}, \ve{z})) \nonumber\\ 
& \leq & D_{\varphi^\dagger}\left(
  \ve{x}\|\ve{z}\right) + (\|\ve{y}\|_2 - \|\ve{x}\|_2)\cdot (1
 - \cos(\ve{y}, \ve{z})) \nonumber\\ 
& \leq & D_{\varphi^\dagger}\left(
  \ve{x}\|\ve{z}\right) \:\:,
\end{eqnarray}
since $\|\ve{y}\|_2 \leq \|\ve{x}\|_2$. We have used the fact that $(1
 - \cos(\ve{x}, \ve{y}))$ is half the Euclidean distance between unit-normalized vectors.
\end{proof}
It turns out that $D_{\varphi^\dagger}(\ve{x}\|\ve{\mu})$ can be
related to
the log-likelihood of a von Mises-Fisher distribution with expected
direction $\ve{\mu}$, which is useful
in text analysis \citep{rwsmST}.

\paragraph{Row II ---} Let $\varphi(\ve{x}) \defeq (1/2) \cdot
(u^2 + \|\ve{x}\|_q^2)$, for $q > 1$ \citep{kwhTP}. We have
\begin{eqnarray}
\varphi\left(\frac{1}{g(\ve{x})}\cdot
  \ve{x}\right) & = & \frac{u^2}{2} + \frac{1}{2}\cdot \left\|\frac{1}{g(\ve{x})}\cdot
  \ve{x}\right\|_q^2 = \frac{u^2}{2} + \frac{1}{2 g^2(\ve{x})} \cdot \|\ve{x}\|_q^2\:\:.
\end{eqnarray}
We also have
\begin{eqnarray}
\left(\frac{1}{g(\ve{x})}\cdot
  \ve{x}\right)^\top \nabla\varphi\left(\frac{1}{g(\ve{x})}\cdot
  \ve{x}\right) & = & \frac{1}{g(\ve{x})}\cdot \sum_i
\frac{x_i\cdot \mathrm{sign}\left(\frac{1}{g(\ve{x})}\cdot x_i\right)
  \left|\frac{1}{g(\ve{x})}\cdot x_i\right|^{q-1}}{\left\|\frac{1}{g(\ve{x})}\cdot
  \ve{x}\right\|_q^{q-2}}\nonumber\\
 & = & \sum_i
\frac{
  \left|\frac{1}{g(\ve{x})}\cdot x_i\right|^{q}}{\left\|\frac{1}{g(\ve{x})}\cdot
  \ve{x}\right\|_q^{q-2}}\nonumber\\
 & = & \frac{1}{g^2(\ve{x})} \cdot \|\ve{x}\|_q^2 \:\:.
\end{eqnarray}
To have the condition of Theorem \ref{th00} satisfied, we therefore
need
\begin{eqnarray}
\|\ve{x}\|_q & = & u g(\ve{x})\:\:,
\end{eqnarray}
So we use $g(\ve{x}) = \|\ve{x}\|_q/W$ and $u=W$, observing that $\varphi$ is 1-homogeneous on the $L_p$ sphere. We check that 
\begin{eqnarray}
\varphi^\dagger(\ve{x}) & = & W\cdot \|\ve{x}\|_q\:\:.
\end{eqnarray}
and we obtain
\begin{eqnarray}
D_{\varphi^\dagger}(\ve{w} \| \ve{w}') & = & W\cdot \|\ve{w}\|_q - W\cdot \sum_i
\frac{w_i \cdot \mathrm{sign}(w'_i) \cdot |w'_i|^{q-1}}{\|\ve{w}'\|_q^{q-1}}\:\:.
\end{eqnarray}

\paragraph{Row III ---} As in \cite{bfSA}, 
we assume $\|\ve{x}\|_2 \leq \pi$, or we renormalize or change the
radius of the ball) We first lift the data points using the
\textit{Sphere} lifting map $\Real^{d}\ni \ve{x} \mapsto
\ve{x}^S \in \Real^{d+1}$:
\begin{eqnarray}
\ve{x}^S & \defeq & [x_1 \quad x_2 \quad\cdots
\quad x_d \quad r_{\ve{x}} \cot r_{\ve{x}}]^\top\:\:,\label{defSmap}
\end{eqnarray}
where $r_{\ve{x}}\defeq \|\ve{x}\|_2$ is the Euclidean norm of
$\ve{x}$. Notice that the last
coordinate is a coordinate of the Hessian of the geodesic distance to the origin on
the sphere \citep{bfSA}. We then let $g(\ve{x}^S) \defeq r_{\ve{x}} /
\sin r_{\ve{x}}$ (notice that $g$ is computed uding the first $d$
coordinates). Finally, for $\XCal = \Real^{d+1}$ and $u>
1$, consider $\varphi(\ve{x}^S) =
(u^2+\|\ve{x}^S\|_2^2)/2$. The set of points for which
$\varphi(\ve{x}^S) = (\ve{x}^S)^\top \nabla\varphi (\ve{x}^S)$ is
equivalently the subset $\XCal_g \subseteq \Real^{d+1}$ such that
\begin{eqnarray}
\XCal_g  & \defeq & \{\ve{x}^S : g^2(\ve{x}^S) = u^2\}\:\:.
\end{eqnarray}
So $\varphi$ satisfies the restricted positive homogeneity of degree 1 on $\XCal_g$ and we can
apply Theorem \ref{th00}.
We first remark that:
\begin{eqnarray}
\|\ve{x}^S\|_2^2 & = & r^2_{\ve{x}} + r^2_{\ve{x}} \cot^2
r_{\ve{x}}\nonumber\\
 & = & \frac{r^2_{\ve{x}}}{\sin r^2_{\ve{x}}} = g^2(\ve{x}^S)\:\:,\label{ppx}
\end{eqnarray}
and
\begin{eqnarray}
\varphi^\dagger (\ve{x}^S) & = & \frac{r_{\ve{x}}}{\sin r_{\ve{x}}}
\cdot \varphi\left( \frac{\sin r_{\ve{x}}}{r_{\ve{x}}} \cdot
  \ve{x}^S\right)= \frac{r_{\ve{x}}}{\sin r_{\ve{x}}}
\cdot \left( \frac{\sin r_{\ve{x}}}{r_{\ve{x}}} \right)^2 \cdot
  \|\ve{x}^S\|_2^2 = \|\ve{x}^S\|_2\:\:,
\end{eqnarray}
and finally, because of the spherical law of cosines,
\begin{eqnarray}
\sin r_{\ve{x}} \sin r_{\ve{y}} \cos(\ve{x},\ve{y}) +
       \cos r_{\ve{x}} \cos r_{\ve{y}} & = & \cos D_G(\ve{x}, \ve{y})\:\:,\label{lawc}
\end{eqnarray}
where we recall from eq. (\ref{defdist}) that $D_G(\ve{x}, \ve{y})$ is
the geodesic distance between the image of the exponential maps of
$\ve{x}$ and $\ve{y}$ on the sphere.

We then derive
\begin{eqnarray}
\lefteqn{g(\ve{x}^S)\cdot D_{\varphi} \left(\frac{1}{g(\ve{x}^S)} \cdot \ve{x}^S\|
  \frac{1}{g(\ve{y}^S)} \cdot \ve{y}^S\right)}\nonumber\\
 & = &
\frac{r_{\ve{x}}}{2 \sin r_{\ve{x}}} 
\cdot \left\| \frac{\sin r_{\ve{x}}}{r_{\ve{x}}} \cdot 
      \ve{x}^S-\frac{\sin r_{\ve{y}}}{r_{\ve{y}}} \cdot 
      \ve{y}^S\right\|_2^2\nonumber\\
 & = & \frac{r_{\ve{x}}}{2 \sin r_{\ve{x}}} 
\cdot \left( \frac{\sin^2 r_{\ve{x}}}{\|\ve{x}\|^2_2} \cdot 
      \|\ve{x}^S\|_2^2 + \frac{\sin^2 r_{\ve{y}}}{\|\ve{y}\|^2_2} \cdot 
      \|\ve{y}^S\|_2^2 - 2\cdot \frac{\sin r_{\ve{x}}}{r_{\ve{x}}} \cdot \frac{\sin r_{\ve{y}}}{r_{\ve{y}}} \cdot 
      (\ve{x}^S)^\top \ve{y}^S\right)\nonumber\\
 & = & \frac{r_{\ve{x}}}{\sin r_{\ve{x}}} 
\cdot \left( 1 - \frac{\sin r_{\ve{x}}}{r_{\ve{x}}} \cdot \frac{\sin r_{\ve{y}}}{r_{\ve{y}}} \cdot 
      (\ve{x}^S)^\top \ve{y}^S\right)\label{pp1}\\
 & = & \frac{r_{\ve{x}}}{\sin r_{\ve{x}}} 
\cdot \left( 1 - \frac{\sin r_{\ve{x}}}{r_{\ve{x}}} \cdot \frac{\sin r_{\ve{y}}}{r_{\ve{y}}} \cdot 
     \left(\ve{x}^\top \ve{y} +r_{\ve{x}} r_{\ve{y}}\cot
       r_{\ve{x}}\cot r_{\ve{y}}\right)\right) \label{pp2}\\
 & = & \frac{r_{\ve{x}}}{\sin r_{\ve{x}}} 
\cdot \left( 1 - \sin r_{\ve{x}} \sin r_{\ve{y}} \cdot 
     \left(\cos(\ve{x},\ve{y}) + \cot r_{\ve{x}} \cot r_{\ve{y}}\right)\right) \nonumber\\
 & = & \frac{r_{\ve{x}}}{\sin r_{\ve{x}}} 
\cdot \left( 1 -
     \left(\sin r_{\ve{x}} \sin r_{\ve{y}} \cos(\ve{x},\ve{y}) +
       \cos r_{\ve{x}} \cos r_{\ve{y}}\right)\right) \nonumber\\
 & = & \frac{r_{\ve{x}}}{\sin r_{\ve{x}}} 
\cdot \left( 1 -
     \cos D_G(\ve{x}, \ve{y})\right)\:\:.\label{feqq1}
\end{eqnarray}
In Equation (\ref{pp1}), we use Equation (\ref{ppx}), and we use
Equation (\ref{lawc}) in Equation (\ref{feqq1}). We also check
\begin{eqnarray}
D_{\varphi^\dagger} (\ve{x}^S\| \ve{y}^S) & = & \|\ve{x}^S\|_2 -
\|\ve{y}^S\|_2 - \frac{1}{\|\ve{y}^S\|_2}\cdot
(\ve{x}^S-\ve{y}^S)^\top \ve{y}^S\nonumber\\
 & = & \|\ve{x}^S\|_2 - \frac{1}{\|\ve{y}^S\|_2}\cdot
(\ve{x}^S)^\top \ve{y}^S\nonumber\\
 & = & \|\ve{x}^S\|_2 \cdot \left( 1 - \frac{(\ve{x}^S)^\top
     \ve{y}^S}{\|\ve{x}^S\|_2\|\ve{y}^S\|_2}\right)\\
 & = & \frac{r_{\ve{x}}}{\sin r_{\ve{x}}} 
\cdot \left( 1 -
     \cos D_G(\ve{x}, \ve{y})\right)\:\:. \label{ll12}
\end{eqnarray}
To obtain (\ref{ll12}), we use the fact that
\begin{eqnarray}
\frac{(\ve{x}^S)^\top \ve{y}^S}{\|\ve{x}^S\|_2\|\ve{y}^S\|_2} & = & \frac{\sin r_{\ve{x}}}{r_{\ve{x}}} \cdot \frac{\sin r_{\ve{y}}}{r_{\ve{y}}} \cdot 
     \left(\ve{x}^\top \ve{y} +r_{\ve{x}} r_{\ve{y}}\cot
       r_{\ve{x}}\cot r_{\ve{y}}\right)\:\:,
\end{eqnarray}
and then plug it into Equation (\ref{ll12}), which yields the identity
between Equation (\ref{pp2}) (and thus (\ref{feqq1})) and
(\ref{ll12}). So Theorem \ref{th00} holds in this case as well.

We remark that $(1/g(\ve{x}^S)) \cdot \ve{x}^S =
\exp_{\ve{0}}(\ve{x})$ is the
exponential map for the sphere \citep{bfSA}.

\paragraph{Row IV ---} In the same way as we did for row IV, we first
create a lifting map, but this time \textit{complex} valued, the Hyperboloid
lifting map $H$: $\Real^{d}\ni \ve{x} \mapsto
\ve{x}^H \in \Real^{d} \times \mathbb{C}$. With an abuse of notation,
it is given by
\begin{eqnarray}
\ve{x}^H & \defeq & [x_1 \quad x_2 \quad\cdots
\quad x_d \quad i r_{\ve{x}} \coth r_{\ve{x}}]^\top\:\:,\label{defHymap}
\end{eqnarray}
and we let $g(\ve{x}^H) \defeq - r_{\ve{x}} / \sinh
r_{\ve{x}}$, with $\coth$ and $\sinh$ defining respectively the
hyperbolic cotangent and hyperbolic sine. We let $0\coth 0 = 0 / \sinh 0 =
1$. Notice that the complex number is pure imaginary and so
$H$ defines a $d$ dimensional manifold that lives in
$\Real^{d+1}$ assuming that the last coordinate is the imaginary
axis. Let $\exp_{\ve{q}}(\ve{x}) \defeq (1/g(\ve{x}^H)\cdot \ve{x}^H$.
Notice that 
\begin{eqnarray}
\left\|\exp_{\ve{q}}(\ve{x})\right\|_2^2 & = & \frac{\sinh^2
r_{\ve{x}}}{r^2_{\ve{x}}} \cdot \left(r^2_{\ve{x}} + i^2 r^2_{\ve{x}}\coth^2
r_{\ve{x}} \right)\nonumber\\
 & = & \sinh^2
r_{\ve{x}} + i^2 \cosh^2
r_{\ve{x}} \nonumber\\
 & = & \sinh^2
r_{\ve{x}} - \cosh^2
r_{\ve{x}} = -1\:\:,\label{equnit}
\end{eqnarray}
so $\exp_{\ve{q}}(\ve{x})$ defines a lifting map from ${\mathbb{R}}^d$ to the
hyperboloid model $\mathbb{H}^d$ of hyperbolic geomety \cite{gAC}. In fact, it
defines the exponential map for the plane $T_{\ve{q}}\mathbb{H}^d$
tangent to $\mathbb{H}^d$ in point $\ve{q}\defeq [0\quad 0\quad
\cdots \quad 0 \quad i]
= \ve{0}^H$. To see this, remark that we can express the geodesic
distance $D_G$ with the hyperbolic
metric between $\ve{x}^H$ and $\ve{y}^H$ as
\begin{eqnarray}
D_G(\ve{x}^H, \ve{y}^H) & \defeq & \cosh^{-1}(-(\ve{x}^H)^\top \ve{y}^H)\:\:,\label{defgd}
\end{eqnarray}
where $\cosh^{-1}$ is the inverse hyperbolic cosine. 
So,
for any $\ve{x}\in T_{\ve{q}}\mathbb{H}^d$, since $r_{\ve{x}} =
\|\ve{x} - \ve{0}\|_2$, we have
\begin{eqnarray}
D_G(\exp_{\ve{q}}(\ve{x}), \ve{q}) & = &
\cosh^{-1}(-(\ve{x}^H)^\top \ve{0}^H)\nonumber\\
 & = & \cosh^{-1}(-i^2 \cosh r_{\ve{x}})\nonumber\\
 & = & r_{\ve{x}} = \|\ve{x} - \ve{0}\|_2\:\:,
\end{eqnarray}
and $\exp_{\ve{q}}(\ve{x})$ is indeed the exponential map for
$T_{\ve{q}}\mathbb{H}^d$.  Now, remark that 
\begin{eqnarray}
\exp_{\ve{q}}(\ve{x})^\top \exp_{\ve{q}}(\ve{y}) & = &
\frac{\sinh r_{\ve{x}}}{r_{\ve{x}}} \cdot \frac{\sinh
  r_{\ve{y}}}{r_{\ve{y}}}  \cdot (\ve{x}^H)^\top \ve{y}^H \nonumber\\
& = & \frac{\sinh r_{\ve{x}}}{r_{\ve{x}}} \cdot \frac{\sinh
  r_{\ve{y}}}{r_{\ve{y}}} \cdot \left(\ve{x}^\top \ve{y} + i^2 r_{\ve{x}} r_{\ve{y}}\coth
       r_{\ve{x}}\coth r_{\ve{y}}\right)\nonumber\\
 & = & \sinh r_{\ve{x}} \sinh r_{\ve{y}} \cdot
 \left(\cos(\ve{x},\ve{y}) - \coth
       r_{\ve{x}}\coth r_{\ve{y}}\right)\nonumber\\
 & = & \sinh r_{\ve{x}} \sinh r_{\ve{y}} \cos(\ve{x},\ve{y}) - \cosh
       r_{\ve{x}}\cosh r_{\ve{y}}\nonumber\\
 & = & - \cosh D_G(\ve{x}^H, \ve{y}^H)\:\:.\label{eqhlc}
\end{eqnarray}
Eq. (\ref{eqhlc}) holds by the hyperbolic law of cosines.
Now, we let $\varphi(\ve{x}^H) =
(u^2+\|\ve{x}^H\|_2^2)/2$ and
\begin{eqnarray}
\XCal_g  & \defeq & \{\ve{x}^H : \|\ve{x}^H\|_2^2 = u^2\}\:\:.
\end{eqnarray}
We check that $\varphi(\ve{x}^H) = u^2 = (\ve{x}^H)^\top \nabla\varphi
(\ve{x}^H)$ for any $\ve{x}^H \in \XCal_g$, so we can
apply Theorem \ref{th00}. We then use eqs. (\ref{equnit}) and (\ref{eqhlc}) and derive
\begin{eqnarray}
\lefteqn{g(\ve{x}^H)\cdot D_{\varphi} \left(\frac{1}{g(\ve{x}^H)} \cdot \ve{x}^H\|
  \frac{1}{g(\ve{y}^H)} \cdot \ve{y}^H\right)}\nonumber\\
 & = &
-\frac{r_{\ve{x}}}{2 \sinh r_{\ve{x}}} 
\cdot \left\| \exp_{\ve{q}}(\ve{x})-\exp_{\ve{q}}(\ve{y})\right\|_2^2\nonumber\\
 & = & -\frac{r_{\ve{x}}}{2 \sinh r_{\ve{x}}} 
\cdot \left(\left\|\exp_{\ve{q}}(\ve{x})\right\|_2^2 +\left\|\exp_{\ve{q}}(\ve{y})\right\|_2^2  - 2\exp_{\ve{q}}(\ve{x})^\top \exp_{\ve{q}}(\ve{y})\right)\nonumber\\
 & = & - \frac{r_{\ve{x}}}{\sinh r_{\ve{x}}} 
\cdot \left( \cosh D_G(\ve{x}^H, \ve{y}^H) - 1\right) \:\:.\label{pp3}
\end{eqnarray}
Note that eq. (\ref{pp3}) is a negative-valued and concave distortion.

\paragraph{Row V ---} for $\XCal = \Real_{+*}^d$,
consider $\varphi(\ve{x}) = \sum_i x_i \log x_i - x_i$ and $g(\ve{x}) =
\ve{1}^\top \ve{x}$
(we normalize on the simplex). 
Since $g$ is linear, we do not need to check for the homogeneity of
$\varphi$, and we directly obtain:
\begin{eqnarray}
g(\ve{x}) \cdot D_{\varphi}\left(\frac{1}{g(\ve{x})}\cdot
  \ve{x}\|\frac{1}{g(\ve{y})}\cdot \ve{y}\right) & = & \sum_i x_i\log x_i -
(\ve{1}^\top \ve{x}) \log (\ve{1}^\top \ve{x}) - \ve{1}^\top
\ve{x}\nonumber\\ 
& & - \frac{\ve{1}^\top \ve{x}}{\ve{1}^\top \ve{y}}\cdot \sum_i
y_i\log y_i -
(\ve{1}^\top \ve{x}) \log (\ve{1}^\top \ve{y}) + \ve{1}^\top
\ve{x}\nonumber\\
 & & - (\ve{1}^\top \ve{x})\cdot \sum_i \left(\frac{x_i}{\ve{1}^\top
     \ve{x}} - \frac{y_i}{\ve{1}^\top \ve{y}}\right) \cdot
 \log\frac{y_i}{\ve{1}^\top \ve{y}}\nonumber\\
 & = & \sum_i x_i\log \frac{x_i}{y_i} -
(\ve{1}^\top \ve{x}) \cdot \log \frac{\ve{1}^\top \ve{x}}{\ve{1}^\top \ve{y}}\:\:.\label{kleq1}
\end{eqnarray}
Furthermore,
\begin{eqnarray}
\varphi^\dagger (\ve{x}) & = & \ve{1}^\top \ve{x}\cdot
\left(\sum_i \frac{x_i}{\ve{1}^\top \ve{x}} \cdot \log
  \frac{x_i}{\ve{1}^\top \ve{x}} - 1\right) = \sum_i x_i\log x_i -
(\ve{1}^\top \ve{x}) \log (\ve{1}^\top \ve{x}) - \ve{1}^\top \ve{x}\:\:.
\end{eqnarray}
Noting that $\varphi^\dagger(\ve{x})$ is the sum of
three terms, one of which is linear and can be removed for the
divergence, so the divergence is just the sum of the two divergences
with the two generators,
which is found to be Equation (\ref{kleq1}) as well. Remark that while the
KL divergence is convex in its both arguments, $D_{\varphi^\dagger}(\ve{x}\|\ve{y})$
may not be (jointly) convex. Indeed, its Hessian in $\ve{y}$ equals:
\begin{eqnarray}
\matrice{H}_{\ve{y}}(D_{\varphi^\dagger}) & = & 
\mathrm{Diag}(\{x_i/y_i^2\}_i) - \frac{\ve{1}^\top \ve{x}}{(\ve{1}^\top
\ve{y})^2}\cdot \ve{1}\ve{1}^\top\:\:,
\end{eqnarray}
which may be indefinite.

\paragraph{Row VI ---} for $\XCal = \Real_{+*}^d$,
consider $\varphi(\ve{x}) = - d -\sum_i \log x_i$ and $g(\ve{x}) =
(\pi_{\ve{x}})^{1/d}$, where we let $\pi_{\ve{x}} \defeq\prod_i x_i$
(we normalize with the geometric average). It comes 
\begin{eqnarray}
\varphi^\dagger (\ve{x}) & = & (\pi_{\ve{x}})^{1/d}\cdot
\left( - d - \sum_i \log \frac{x_i}{(\pi_{\ve{x}})^{1/d}}\right) = -d
\cdot (\pi_{\ve{x}})^{1/d} \:\:.
\end{eqnarray}
$g$ is not linear (but it is homogeneous of degree $1$), and we have
\begin{eqnarray}
\varphi(\ve{x}) = -d = \ve{x}^\top \nabla\varphi(\ve{x}) \:\:,
\forall \ve{x} : \prod_i x_i = 1\:\:,
\end{eqnarray}
so $\varphi$ is 1-homogeneous on $\XCal_g$, and we can apply
Theorem \ref{th00}. We have
\begin{eqnarray}
\lefteqn{g(\ve{x}) \cdot D_{\varphi}\left(\frac{1}{g(\ve{x})}\cdot
  \ve{x}\|\frac{1}{g(\ve{y})}\cdot \ve{y}\right)}\nonumber\\
 & = & (\pi_{\ve{x}})^{1/d}  \cdot \sum_i \left( \frac{x_i (\pi_{\ve{y}})^{1/d}}{y_i (\pi_{\ve{x}})^{1/d}} -
  \log\frac{x_i (\pi_{\ve{y}})^{1/d}}{y_i (\pi_{\ve{x}})^{1/d}}\right) - d (\pi_{\ve{x}})^{1/d} \nonumber\\
 & = & \sum_i \frac{x_i (\pi_{\ve{y}})^{1/d}}{y_i} - d(\pi_{\ve{x}})^{1/d}\log
 (\pi_{\ve{x}})^{1/d} - (\pi_{\ve{x}})^{1/d}\log
 \pi_{\ve{y}} + (\pi_{\ve{x}})^{1/d}\log
 \pi_{\ve{y}} \nonumber\\
 & & + d (\pi_{\ve{x}})^{1/d}\log
 (\pi_{\ve{x}})^{1/d} - d (\pi_{\ve{x}})^{1/d} \nonumber\\
 & = & \sum_i \frac{x_i (\pi_{\ve{y}})^{1/d}}{y_i}  - d (\pi_{\ve{x}})^{1/d} \:\:.\label{peq2}
\end{eqnarray}
We also have
\begin{eqnarray}
\frac{\partial}{\partial x_i} \varphi^{\dagger}(\ve{x}) & = & - (1/x_i)\cdot (\pi_{\ve{x}})^{1/d}\:\:,
\end{eqnarray}
and so
\begin{eqnarray}
D_{\varphi^\dagger}\left(
  \ve{x}\|\ve{y}\right) & = & -d(\pi_{\ve{x}})^{1/d} + d(\pi_{\ve{y}})^{1/d} + \cdot \sum_i (x_i - y_i) \cdot \frac{(\pi_{\ve{y}})^{1/d}}{y_i}\nonumber\\
 & = &  -d (\pi_{\ve{x}})^{1/d} + d (\pi_{\ve{y}})^{1/d} + \sum_i \frac{x_i
  (\pi_{\ve{y}})^{1/d}}{y_i} - d (\pi_{\ve{y}})^{1/d}\nonumber\\
 & =&\sum_i \frac{x_i
  (\pi_{\ve{y}})^{1/d}}{y_i}  -d (\pi_{\ve{x}})^{1/d} \:\:,
\end{eqnarray}
which is equal to Equation (\ref{peq2}), so we check that Theorem \ref{th00} applies in this case.

\paragraph{Row VII ---} We use the following fact \cite{ksdLR}. Let $\matrice{x} = \matrice{u} \matrice{l}
\matrice{u}^\top$ and $\matrice{y} = \matrice{v} \matrice{t}
\matrice{v}^\top$ be the eigendecomposition of symmetric
positive definite matrices $\matrice{x}$ and $\matrice{y}$, with
$\matrice{l} \defeq \mathrm{Diag}(\ve{l})$, $\matrice{t} \defeq
\mathrm{Diag}(\ve{t})$, and $\matrice{u} \defeq [\ve{u}_1 | \ve{u}_2 |
\cdots | \ve{u}_d], \matrice{v}\defeq [\ve{v}_1 | \ve{v}_2 |
\cdots | \ve{v}_d]$ orthonormal;
let $\varphi = \trace{\matrice{x}\log
    \matrice{x} - \matrice{x}}$. Then we have
\begin{eqnarray}
D_{\varphi}\left(
  \matrice{x}\|\matrice{y}\right) & = & \sum_{i,j} (\ve{u}_i^\top \ve{v}_j)^2
\cdot D_{\varphi_2} (l_i \| t_j) \:\:,
\end{eqnarray}
with $\varphi_2(x) = x \log x - x$. 
We pick $g(\matrice{x}) =
\trace{\matrice{x}} = \sum_i l_i$, which brings from Equation (\ref{kleq1})
\begin{eqnarray}
\lefteqn{g(\matrice{x}) \cdot D_{\varphi}\left(\frac{1}{g(\matrice{x})}\cdot
  \matrice{x}\|\frac{1}{g(\matrice{y})}\cdot
  \matrice{y}\right)}\nonumber\\
 & = &
\sum_{i,j} (\ve{u}_i^\top \ve{v}_j)^2
\cdot \trace{\matrice{x}} \cdot  D_{\varphi_2} \left(\frac{l_i}{\trace{\matrice{x}}} \|
  \frac{t_j}{\trace{\matrice{y}}}\right) \nonumber\\
 & = & \sum_{i,j} (\ve{u}_i^\top \ve{v}_j)^2 \cdot \trace{\matrice{x}} \cdot  
 \left(\frac{l_i}{\trace{\matrice{x}}} \cdot \log \frac{l_i \cdot
     \trace{\matrice{y}}}{t_j \cdot \trace{\matrice{x}}} -
   \frac{l_i}{\trace{\matrice{x}}} +
   \frac{t_j}{\trace{\matrice{y}}}\right)\nonumber\\
 & = & \sum_{i,j} (\ve{u}_i^\top \ve{v}_j)^2 \cdot\left( l_i 
   \log \frac{l_i}{t_j} - l_i + t_j\right) + \log
 \left(\frac{\trace{\matrice{y}}}{\trace{\matrice{x}}}\right)\cdot \sum_{i,j} (\ve{u}_i^\top
 \ve{v}_j)^2 \cdot l_i \nonumber \\
 & & + \frac{\trace{\matrice{x}}}{\trace{\matrice{y}}} \cdot \sum_{i,j} (\ve{u}_i^\top
 \ve{v}_j)^2 \cdot t_j - \sum_{i,j} (\ve{u}_i^\top
 \ve{v}_j)^2 \cdot t_j\label{inp1}\:\:.
\end{eqnarray}
Because $\matrice{u}, \matrice{v}$ are orthonormal, we also get $\sum_{i,j} (\ve{u}_i^\top
 \ve{v}_j)^2 \cdot l_i = \sum_i l_i \sum_j \cos^2(\ve{u}_i,
 \ve{v}_j) = \sum_i l_i = \trace{\matrice{x}}$ and $\sum_{i,j} (\ve{u}_i^\top
 \ve{v}_j)^2 \cdot t_j = \trace{y}$, and so Equation (\ref{inp1}) becomes
 \begin{eqnarray}
\lefteqn{g(\matrice{x}) \cdot D_{\varphi}\left(\frac{1}{g(\matrice{x})}\cdot
  \matrice{x}\|\frac{1}{g(\matrice{y})}\cdot
  \matrice{y}\right)}\nonumber\\
 & = & \trace{\matrice{x}\log \matrice{x}
- \matrice{x}\log \matrice{y}} - \trace{\matrice{x}} +
\trace{\matrice{y}} + \trace{\matrice{x}}\cdot \log
 \left(\frac{\trace{\matrice{y}}}{\trace{\matrice{x}}}\right) +
 \trace{\matrice{x}} - \trace{\matrice{y}}\nonumber\\
 & = & \trace{\matrice{x}\log \matrice{x}
- \matrice{x}\log \matrice{y}} - \trace{\matrice{x}}\cdot \log
 \left(\frac{\trace{\matrice{x}}}{\trace{\matrice{y}}}\right) \:\:.\label{eqgrad1}
\end{eqnarray}
We also check that
\begin{eqnarray}
\varphi^\dagger(\matrice{x}) & = & \trace{\matrice{x}} \cdot
\trace{\frac{1}{\trace{\matrice{x}}} \cdot \matrice{x}\log
    \left(\frac{1}{\trace{\matrice{x}}} \cdot \matrice{x}\right) -
    \frac{1}{\trace{\matrice{x}}} \cdot \matrice{x}}\nonumber\\
 & = & \trace{\matrice{x}\log
    \left(\frac{1}{\trace{\matrice{x}}} \cdot \matrice{x}\right)} - \trace{\matrice{x}}\:\:,
\end{eqnarray}
and
\begin{eqnarray}
\matrice{x}\log
    \left(\frac{1}{\trace{\matrice{x}}} \cdot \matrice{x}\right) & = & \matrice{u} \matrice{l}
\matrice{u}^\top \matrice{u} \log\left(\frac{1}{\ve{1}^\top \ve{l}}\cdot \matrice{l}\right)
\matrice{u}^\top\nonumber\\
 & = & \matrice{u} \matrice{l}
\log\left(\frac{1}{\ve{1}^\top \ve{l}}\cdot \matrice{l}\right)
\matrice{u}^\top\\
 & = & \matrice{u} \matrice{l}
\log\matrice{l}
\matrice{u}^\top - \log \trace{\matrice{x}}\cdot \matrice{u} \matrice{l}
\matrice{u}^\top\:\:,
\end{eqnarray}
so that $\varphi^\dagger(\matrice{x}) = \trace{\matrice{x} \log
  \matrice{x} -  \matrice{x}} -  \trace{\matrice{x}} \cdot \log
\trace{\matrice{x}}$. Let $\varphi_3(\matrice{x}) \defeq \trace{\matrice{x}} \cdot \log
\trace{\matrice{x}}$. We have $\nabla \varphi_3(\matrice{x}) = (1 + \log
\trace{\matrice{x}})\cdot
\matrice{i}$. Since a (Bregman) divergence involving a sum of generators
is the sum of (Bregman) divergences, we get
\begin{eqnarray}
D_{\varphi^\dagger}\left(
  \matrice{x}\|\matrice{y}\right) & = & \trace{\matrice{x} \log
  \matrice{x} - \matrice{x} \log
  \matrice{y} -  \matrice{x} +  \matrice{y}} - \trace{\matrice{x}} \cdot \log
\trace{\matrice{x}} + \trace{\matrice{y}} \cdot \log
\trace{\matrice{y}} \nonumber\\
 & & +  (1 + \log
\trace{\matrice{y}})\cdot \trace{\matrice{x}-\matrice{y}}\nonumber\\
 & = & \trace{\matrice{x} \log
  \matrice{x} - \matrice{x} \log
  \matrice{y} } - \trace{\matrice{x}} \cdot \log
\trace{\matrice{x}} + \trace{\matrice{x}} \cdot \log
\trace{\matrice{y}}\nonumber\\
 & = & \trace{\matrice{x}\log \matrice{x}
- \matrice{x}\log \matrice{y}} - \trace{\matrice{x}}\cdot \log
 \left(\frac{\trace{\matrice{x}}}{\trace{\matrice{y}}}\right)\:\:,
\end{eqnarray}
which is Equation (\ref{eqgrad1}).

\paragraph{Row VIII ---} We have the same property as for Row V, but
this time with $\varphi_2 = -d - \log x$ \citep{ksdLR}.
We check that whenever $\det (\matrice{x}) = 1$, we have
\begin{eqnarray}
  \varphi(\matrice{x}) = - d - \log \det(\matrice{x}) & = & -d \nonumber\\
 & = & -\det(\matrice{x}) \trace{\matrice{i}}\nonumber\\
 & = & \trace{\det(\matrice{x}) \matrice{x}^{-1} \matrice{x}}=
\trace{\nabla\varphi(\matrice{x})^\top \matrice{x}}\:\:.
\end{eqnarray}
For $g(\matrice{x}) \defeq \det{\matrice{x}^{1/d}}$, we get:
\begin{eqnarray}
\varphi^\dagger(\matrice{x}) & = & \det{\matrice{x}^{1/d}}
\cdot \left( -d - \log \det \left( \frac{1}{\det{\matrice{x}^{1/d}}}
    \cdot \matrice{x}\right)\right)\nonumber\\
 & = & \det{\matrice{x}^{1/d}}
\cdot \left( -d - \log \frac{1}{\det{\matrice{x}}} \cdot \det
  \matrice{x}\right) = -d \cdot \det{\matrice{x}^{1/d}} \:\:,
\end{eqnarray}
and furthermore
\begin{eqnarray}
\nabla\varphi^\dagger(\matrice{x}) & = & - d \cdot \nabla(
\det{\matrice{x}^{1/d}})(\matrice{x})\nonumber\\ 
 & = & - \det(\matrice{x}^{1/d}) \cdot \matrice{x}^{-1}
\end{eqnarray}
So,
\begin{eqnarray}
D_{\varphi^\dagger}\left(
  \matrice{x}\|\matrice{y}\right) & = & -d \cdot
\det{\matrice{x}^{1/d}} + d \cdot \det{\matrice{y}^{1/d}} +
\trace{\det(\matrice{y}^{1/d}) \cdot \matrice{y}^{-1}
  (\matrice{x}-\matrice{y})}\nonumber\\
 & = & -d \cdot
\det{\matrice{x}^{1/d}} + d \cdot \det{\matrice{y}^{1/d}} +
\det(\matrice{y}^{1/d}) \trace{\matrice{x}\matrice{y}^{-1}} - d \cdot
\det{\matrice{y}^{1/d}} \nonumber\\
 & = & \det(\matrice{y}^{1/d}) \trace{\matrice{x}\matrice{y}^{-1}} -d \cdot
\det{\matrice{x}^{1/d}}\:\:. \label{eqISM1}
\end{eqnarray}
We check that it is equal to:
\begin{eqnarray}
\lefteqn{g(\matrice{x}) \cdot D_{\varphi}\left(\frac{1}{g(\matrice{x})}\cdot
  \matrice{x}\|\frac{1}{g(\matrice{y})}\cdot
  \matrice{y}\right)}\nonumber\\
 & = & \det{\matrice{x}^{1/d}}\cdot \sum_{i,j} (\ve{u}_i^\top \ve{v}_j)^2
\cdot \left( \frac{l_i \det{\matrice{y}^{1/d}}}{t_j
    \det{\matrice{x}^{1/d}}} - \log \frac{l_i
    \det{\matrice{y}^{1/d}}}{t_j \det{\matrice{x}^{1/d}}} -
  d\right)\:\:.\label{eqISM2}
\end{eqnarray}
To check it, we use the fact that, since $\matrice{u}$ and
$\matrice{v}$ are orthonormal,
\begin{eqnarray}
\lefteqn{\sum_{i,j} (\ve{u}_i^\top \ve{v}_j)^2
\cdot \log \frac{l_i
    \det{\matrice{y}^{1/d}}}{t_j \det{\matrice{x}^{1/d}}}}\nonumber\\
 & = & \sum_{i,j} (\ve{u}_i^\top \ve{v}_j)^2
\cdot \log l_i - \sum_{i,j} (\ve{u}_i^\top \ve{v}_j)^2
\cdot \log\det{\matrice{x}^{1/d}} \nonumber\\
 & & + \sum_{i,j} (\ve{u}_i^\top \ve{v}_j)^2
\cdot \log\det{\matrice{y}^{1/d}} -\sum_{i,j} (\ve{u}_i^\top \ve{v}_j)^2
\cdot \log t_j\nonumber\\
 & = & \underbrace{\sum_{i} \log l_i - d \cdot \log\det{\matrice{x}^{1/d}}}_{=0} +  \underbrace{ d
 \cdot \log\det{\matrice{y}^{1/d}} - \sum_{j} \log l_j}_{=0} = 0\:\:,
\end{eqnarray}
which yields
\begin{eqnarray} 
g(\matrice{x}) \cdot D_{\varphi}\left(\frac{1}{g(\matrice{x})}\cdot
  \matrice{x}\|\frac{1}{g(\matrice{y})}\cdot
  \matrice{y}\right) & = & \det{\matrice{x}^{1/d}}\cdot \sum_{i,j} (\ve{u}_i^\top \ve{v}_j)^2
\cdot \left( \frac{l_i \det{\matrice{y}^{1/d}}}{t_j
    \det{\matrice{x}^{1/d}}} -
  d\right) \nonumber\\
 & = & \det{\matrice{y}^{1/d}}\cdot \sum_{i,j} (\ve{u}_i^\top \ve{v}_j)^2
\cdot \frac{l_i}{t_j} - d \cdot \det{\matrice{x}^{1/d}}\nonumber\\
 & = & \det{\matrice{y}^{1/d}} \cdot
   \trace{\matrice{x}\matrice{y}^{-1}} - d \cdot \det{\matrice{x}^{1/d}}\:\:,
\end{eqnarray}
which is equal to Equation (\ref{eqISM1}).

\section{Going deep: higher-order identities}
\label{app:deep-bregman}
We can generalize Theorem \ref{th00} to higher order identities. For
this, consider $k>0$ an integer, and let $g_1, g_2, ..., g_k :
{\XCal} \rightarrow {\Real}_{*}$ be a sequence of
differentiable functions. For any $\ell, \ell'\in [k]_*$ such that $\ell\leq \ell'$, we let $\tilde{g}_{\ell,\ell'}$ be defined recursively as:
\begin{eqnarray}
\tilde{g}_{\ell,\ell'}(\ve{x}) & \defeq & \left\{
\begin{array}{ccl}
\tilde{g}_{\ell-1,\ell'}(\ve{x}) \cdot
g_{\ell'-(\ell-1)}\left(\frac{1}{\tilde{g}_{\ell-1,\ell'}(\ve{x})}\cdot \ve{x}\right) &
\mbox{ if } & 1<\ell\leq \ell'\:\:,\\
g_{\ell'}(\ve{x}) & \mbox{ if } & \ell = 1\:\:,
\end{array}
\right. \label{defG}
\end{eqnarray}
and, for any $\ell\in [k]$,
\begin{eqnarray}
\varphi^{\dagger(\ell)} (\ve{x}) & \defeq & \left\{
\begin{array}{ccl} 
g_\ell(\ve{x}) \cdot
\varphi^{\dagger(\ell-1)} \left(\frac{1}{g_{\ell}(\ve{x})}\cdot
  \ve{x}\right) & \mbox{ if } & 0 < \ell \leq k\:\:,\\
\varphi(\ve{x}) & \mbox{ if } & \ell = 0\:\:.
\end{array}\right.\label{defPHI}
\end{eqnarray}
Notice that even when all $g_.$ are linear, this does not guarantee
that some $\tilde{g}_{\ell,\ell'}$ for $\ell\neq 1$ is going to be linear. However,
if for example $g_{\ell'}$ is linear and all ``preceeding'' $g_\ell$
($\ell \leq \ell'$) are homogeneous of degree 1,
then all $\tilde{g}_{\ell,\ell'}$ ($\forall \ell\leq \ell'$) are linear.

\begin{corollary}
\label{corr:deeper}
For any $k\in \mathbb{N}_*$, let $\varphi :
{\XCal} \rightarrow {\Real}$ be 
convex differentiable, and $g_\ell :
{\XCal} \rightarrow {\Real}_{*}$ ($\ell \in [k]$) a
sequence of $k$ differentiable functions. Then the following
relationship holds, for any $\ell, \ell'\in [k]_*$ with $\ell\leq \ell'$:
\begin{eqnarray}
\tilde{g}_{\ell,\ell'} (\ve{x}) \cdot D_{\varphi^{\dagger(\ell'-\ell)}}\left(\frac{1}{\tilde{g}_{\ell,\ell'} (\ve{x})}\cdot
  \ve{x}\|\frac{1}{\tilde{g}_{\ell,\ell'} (\ve{y})}\cdot \ve{y}\right) & = & 
D_{\varphi^{\dagger(\ell')}}\left(
  \ve{x}\|\ve{y}\right)\:\:, \forall \ve{x}, \ve{y}\in \XCal\:\:,\label{eq11GEN}
\end{eqnarray}
with $\tilde{g}_{\ell,\ell'}$ defined as in Equation (\ref{defG}) and $\varphi^{\dagger(\ell')}$ defined as in Equation (\ref{defPHI}), 
if and only if at least one of the two following conditions hold:
\begin{itemize}
\item [(i)] $\tilde{g}_{\ell,\ell'}$ is linear on $\XCal$;
\item [(ii)] $\varphi^{\dagger(\ell'-\ell)}$ is positive homogeneous of degree
  1 on $\XCal_{\ell, \ell'} \defeq\{ (1/\tilde{g}_{\ell,\ell'} (\ve{x}))\cdot \ve{x} : \ve{x}\in \XCal\}$.
\end{itemize}
\end{corollary}

We check that whenever $\varphi$ is convex and all $g_.$ are
non-negative, then all $\varphi^{\dagger(\ell)}$ are convex ($\forall \ell\in
[k]$). To prove this, we choose $\ell' = \ell$ and rewrite
Equation (\ref{eq11}), which brings, since $\varphi^{\dagger(\ell'-\ell)} =
\varphi^{\dagger(0)} = \varphi$,
\begin{eqnarray}
\tilde{g}_{\ell, \ell}(\ve{x}) \cdot
D_{\varphi}\left(\frac{1}{\tilde{g}_{\ell, \ell} (\ve{x})}\cdot
  \ve{x}\|\frac{1}{\tilde{g}_{\ell,\ell} (\ve{y})}\cdot \ve{y}\right) & = & 
D_{\varphi^{\dagger(\ell)}}\left(
  \ve{x}\|\ve{y}\right)\:\:, \forall \ve{x}, \ve{y}\in \XCal\:\:.\label{eq12}
\end{eqnarray}
Since $\varphi$ is convex, a sufficient condition to prove our result
is to show that $\tilde{g}_{\ell,\ell}$ is non-negative --- which will prove
that the right hand side of (\ref{eq12}) is non-negative, and
therefore $\varphi^{\dagger(\ell)}$ is convex ---. This can easily
be proven by induction from the expression of $\tilde{g}_{\ell,\ell'}$ in (\ref{defG}) and the
fact that all $g_.$ are non-negative.

One interesting candidate for simplification is when all $g_.$ are the
same linear function, say $g_\ell(\ve{x}) = \ve{a}^\top \ve{x} + b,
\forall \ell \in [k]$. In this case, we have indeed:
\begin{eqnarray}
\tilde{g}_{\ell,\ell'}(\ve{x}) & = & \ve{a}^\top
    \ve{x} + b\cdot \tilde{g}_{\ell-1,\ell'}(\ve{x}) \nonumber\\
 &= & b^{\ell} + \ve{a}^\top
    \ve{x}\cdot \sum_{j=1}^{\ell-1}b^j \:\:,\label{defB}\\
\varphi^{\dagger(\ell')} (\ve{x}) & = & \left(b^{\ell'} + \ve{a}^\top
    \ve{x}\cdot \sum_{j=1}^{\ell'-1}b^j\right) \cdot
\varphi \left(\frac{1}{b^{\ell'} + \ve{a}^\top
    \ve{x}\cdot \sum_{j=1}^{\ell'-1}b^j}\cdot
  \ve{x}\right)\:\:.\label{defVARPHI}
\end{eqnarray}

\begin{proof}[Proof of Corollary \ref{corr:deeper}]
To check eq. (\ref{defdagger}), we
first remark ($\ell'$
being fixed) that it holds for $\ell=1$ (this is
eq. (\ref{defdagger})), and then proceed by an induction from
the induction base hypothesis that, for some $\ell\leq \ell'$,
\begin{eqnarray}
\varphi^{\dagger(\ell')} (\ve{x}) & = & \tilde{g}_{\ell,\ell'}(\ve{x}) \cdot \varphi^{\dagger(\ell'-\ell)} \left(\frac{1}{\tilde{g}_{\ell,\ell'}(\ve{x})}\cdot \ve{x}\right)\label{eqind}\:\:.
\end{eqnarray}
We now have
\begin{eqnarray}
\lefteqn{\varphi^{\dagger(\ell')} (\ve{x})}\nonumber\\
 & = & 
 \frac{\tilde{g}_{\ell+1,\ell'}(\ve{x})}{g_{\ell'-\ell}\left(\frac{1}{\tilde{g}_{\ell,\ell'}(\ve{x})}\cdot
     \ve{x}\right)} \cdot \varphi^{\dagger(\ell'-\ell)} \left(\frac{1}{\tilde{g}_{\ell,\ell'}(\ve{x})}\cdot
  \ve{x}\right)\label{eq2}\\
 & = & \frac{\tilde{g}_{\ell+1,\ell'}(\ve{x})}{g_{\ell'-\ell}\left(\frac{1}{\tilde{g}_{\ell,\ell'}(\ve{x})}\cdot
     \ve{x}\right)} \cdot
 g_{\ell'-\ell}\left(\frac{1}{\tilde{g}_{\ell,\ell'}(\ve{x})}\cdot \ve{x}\right) \cdot
\varphi^{\dagger(\ell'-(\ell+1))} \left(\frac{1}{\tilde{g}_{\ell,\ell'}(\ve{x}) g_{\ell'-\ell}\left(\frac{1}{\tilde{g}_{\ell,\ell'}(\ve{x})}\cdot \ve{x}\right)}
  \cdot \ve{x}\right) \label{eq3}\\
& = & \tilde{g}_{\ell+1,\ell'}(\ve{x})  \cdot
\varphi^{\dagger(\ell'-(\ell+1))} \left(\frac{1}{\tilde{g}_{\ell,\ell'}(\ve{x}) g_{\ell'-\ell}\left(\frac{1}{\tilde{g}_{\ell,\ell'}(\ve{x})}\cdot \ve{x}\right)}
  \cdot \ve{x}\right) \nonumber\\
& = & \tilde{g}_{\ell+1,\ell'}(\ve{x})  \cdot
\varphi^{\dagger(\ell'-(\ell+1))} \left(\frac{1}{\tilde{g}_{\ell+1,\ell'}(\ve{x})}
  \cdot \ve{x}\right) \label{eq40}\:\:.
\end{eqnarray}
Eq. (\ref{eq2}) comes
from eq. (\ref{eqind}) and the definition of $\tilde{g}_\ell$ in (\ref{defG}), eq. (\ref{eq3}) comes from the
definition of $\varphi^{\dagger(\ell'-\ell)}$ in (\ref{defPHI}),
eq. (\ref{eq40}) is a second use of the definition of $\tilde{g}_\ell$ in
(\ref{defG}). 
\end{proof}

Notice the eventual high non-linearities introduced by the composition
in eqs (\ref{defG},\ref{defPHI}), which justifies the "deep" characterization.

\section{Additional application: exponential families}
\label{app:exp-family}
Let $\varphi$ be the cumulant function of a regular $\varphi$-exponential family
with pdf $p_{\varphi}(.|\ve{\theta})$, where $\ve{\theta} \in
\XCal$ is its natural parameter. Let $\Omega(.)$ be a norm on
$\XCal$. Let $\XCal_\Omega$ be the image of the application from $\XCal$ onto the $\Omega$-ball
of unit norm defined by $\ve{x} 
  \mapsto (1/\Omega(\ve{x})) \cdot \ve{x}$. Let $\ve{\theta}_\Omega$ be the
  image of $\ve{\theta} \in \XCal$. For any two $\ve{\theta},
\ve{\theta}' \in \XCal$, let
\begin{eqnarray}
\mathrm{KL}_\varphi(\ve{\theta}\|
\ve{\theta}' ) & \defeq & \int p_{\varphi}(\ve{x}|\ve{\theta})\log\frac{p_{\varphi}(\ve{x}|\ve{\theta})}{p_{\varphi}(\ve{x}|\ve{\theta}')} \mathrm{d}\ve{x}
\end{eqnarray}
be the KL divergence between the two densities
$p_{\varphi}(.|\ve{\theta})$ and $p_{\varphi}(.|\ve{\theta}')$.

\begin{lemma}\label{lem:expfam}
For any convex $\varphi$ which is restricted positive 1-homogeneous on $\XCal_\Omega$, the
KL-divergence between two members of the same $\varphi$-exponential
family satisfies:
\begin{eqnarray}
\mathrm{KL}_{\varphi}(\ve{\theta}_\Omega\|
\ve{\theta}'_\Omega) & = & \frac{1}{\Omega(\ve{\theta})} \cdot D_{\varphi^\dagger}(\ve{\theta}'\|\ve{\theta})\:\:.
\end{eqnarray}
\end{lemma}
\begin{proof}
We know that $\mathrm{KL}(\ve{\theta}\|
\ve{\theta}') = D_\varphi(\ve{\theta}'\|\ve{\theta})$
\cite{bnnBV}. Hence,
\begin{eqnarray}
D_{\varphi^\dagger}(\ve{\theta}'\|\ve{\theta}) & = & \Omega(\ve{\theta}) \cdot
 D_{\varphi}\left(\frac{1}{\Omega(\ve{\theta})}\cdot
   \ve{\theta}'\|\frac{1}{\Omega(\ve{\theta}')}\cdot
   \ve{\theta}\right)\nonumber\\
 & = & \Omega(\ve{\theta})  \cdot
 D_{\varphi}(\ve{\theta}'_\Omega\|\ve{\theta}_\Omega)\nonumber\\
 & = & \Omega(\ve{\theta})  \cdot\mathrm{KL}_{\varphi}(\ve{\theta}_\Omega\|
\ve{\theta}'_\Omega)\:\:,
\end{eqnarray}
as claimed.
\end{proof}
The interest in Lemma \ref{lem:expfam} is to provide an integral-free
expression of the KL-divergence when natural parameters are scaled
by non-trivial transformations. Furthermore, even when
$D_{\varphi^\dagger}$ may not be a Bregman divergence, it still bears
the same analytical form, which still may be useful for formal derivations.

\section{Additional application: computational geometry, nearest neighbor rules}
\label{app:comp-geom}

Two important objects of central importance in (computational)
geometry are balls and Voronoi diagrams induced by a distortion, with which we can
characterize the topological and computational aspects of major
structures (Voronoi diagrams, triangulations, nearest neighbor
topologies, etc.) \citep{bnnBV}. 

Since a Bregman divergence is not
necessarily symmetric, there are two types of (dual) balls that can be
defined, the first or second types, where the variable $\ve{x}$ is respectively
placed in the left or right position. The first type
Bregman balls are convex while the second type are not necessarily
convex. A Bregman ball of the second type (with center $\ve{c}$ and "radius" $r$) is defined as:
\begin{eqnarray}
B'(\ve{c}, r | \mathcal{X}, \varphi) & \defeq & \{\ve{x} \in \mathcal{X} :
D_{\varphi}(\ve{c}\|\ve{x}) \leq r\}\:\:.
\end{eqnarray}
It turns out that any divergence $D_{\varphi^\dagger}$ induces
a ball of the second type, which is not necessarily analytically a Bregman
ball (when $\varphi^\dagger$ is not convex), \textit{but} turns out to
define the \textit{same} ball as a Bregman ball over transformed
coordinates.  In other words and to be a little bit more specific, 
\begin{center}
\textit{"any $\ve{x}$ belongs to the ball of the second type induced
by $D_{\varphi^\dagger}$ over $\mathcal{X}$ \textbf{iff}
$(1/g(\ve{x}))\cdot \ve{x}$ ($\in {\mathcal{X}}_g$) belongs to the
Bregman ball of the second type induced by $D_{\varphi}$ over $\mathcal{X}_g$."}
\end{center}

\begin{theorem}\label{thball}
Let $(\varphi, g, \varphi^\dagger)$ satisfy the conditions of Theorem
\ref{th00}, with $g$ non negative. Then
\begin{eqnarray}
B'(\ve{c}, r | \varphi^\dagger, \mathcal{X}) & = & B'\left(\left.\ve{c}, \frac{r}{g(\ve{c})} \right\| \varphi, \mathcal{X}_g\right)\:\:.
\end{eqnarray}
\end{theorem}
\begin{proof}
From Theorem \ref{th00}, we have
\begin{eqnarray}
D_{\varphi^\dagger}(\ve{c}\|\ve{x}) & \leq & r
\end{eqnarray}
iff
\begin{eqnarray}
D_{\varphi}\left(\frac{1}{g(\ve{c})}\cdot \ve{c}\left\| \frac{1}{g(\ve{c})}\cdot \ve{x}\right.\right) & \leq & \frac{1}{g(\ve{c})}\cdot r\:\:.
\end{eqnarray}
Hence,
\begin{eqnarray}
B'(\ve{c}, r | \varphi^\dagger, \mathcal{X}) & = & \left\{\ve{x} \in
\mathcal{X}_g : D_{\varphi}(\ve{c}/g(\ve{c})\|\ve{x}/g(\ve{x})) \leq
r/g(\ve{c})\right\}\nonumber\\
 & = & B'(\ve{c}, r/g(\ve{c}) | \varphi, \mathcal{X}_g)\:\:,
\end{eqnarray}
as claimed.
\end{proof}
This property is not true for balls of the first type. What
Theorem \ref{thball} says is that the topology induced by
$D_{\varphi^\dagger}$ over $\mathcal{X}$ is just \textit{no different} from that
induced by $D_{\varphi}$ over $\mathcal{X}_g$. 

Let us now investigate Bregman Voronoi diagrams. In the same way as
there exists two types of Bregman balls, we can define two types of
Bregman Voronoi diagrams that depend on the equation of the
\textit{Bregman bisector} \citep{bnnBV}. Of particular interest is the
Bregman bisector of the \textit{first} type:
\begin{eqnarray}
BB_{\varphi}(\ve{x}, \ve{y}| \mathcal{X}) & = & \{\ve{z} \in \mathcal{X} :
D_{\varphi}(\ve{z}\|\ve{x}) = D_{\varphi}(\ve{z}\|\ve{y})\}\:\:.
\end{eqnarray}
It turns out that any divergence $D_{\varphi^\dagger}$ induces a
bisector of the first type which is not necessarily analytically a Bregman
bisector (when $\varphi^\dagger$ is not convex), \textit{but} turns out to
define the \textit{same} bisector as a Bregman bisector over transformed
coordinates. Again, we get more precisely
\begin{center}
\textit{"any $\ve{x}$ belongs to a Bregman bisector of the first type induced
by $D_{\varphi^\dagger}$ over $\mathcal{X}$ \textbf{iff}
$(1/g(\ve{x}))\cdot \ve{x}$ ($\in {\mathcal{X}}_g$) belongs to the
corresponding Bregman bisector of the first type induced by $D_{\varphi}$ over $\mathcal{X}_g$."}
\end{center}
\begin{theorem}\label{thbisec}
Let $(\varphi, g, \varphi^\dagger)$ satisfy the conditions of Theorem
\ref{th00}. Then
\begin{eqnarray}
BB_{\varphi^\dagger}(\ve{x}, \ve{y}| \mathcal{X}) & = & BB_{\varphi}(\ve{x}, \ve{y}| \mathcal{X}_g)\:\:.
\end{eqnarray}
\end{theorem}
(proof similar to Theorem \ref{thball}) This property is not true for
Bregman bisectors of the second type. 
Theorems \ref{thball},
\ref{thbisec} have several important algorithmic
consequences, some of which are listed now:
\begin{itemize}
\item the Voronoi diagram (resp. Delaunay triangulation) of the first type associated to
$\varphi^\dagger$ can be constructed via the Voronoi diagram (resp. Delaunay triangulation) of the
first type associated to $\varphi$ \citep{bnnBV};
\item range search using ball trees on $D_{\varphi^\dagger}$ can be
efficiently implemented using Bregman divergence $D_{\varphi}$ on $\mathcal{X}_g$
\citep{cEB};
\item the minimum enclosing ball problem, the one-class clustering
problem (an important problem in machine learning), with balls of the second type on $D_{\varphi^\dagger}$ can be solved via the minimum Bregman
enclosing ball problem on $D_{\varphi}$ \citep{nnFT}.
\end{itemize}

\section{Review: binary density ratio estimation}
\label{app:binary-dr}

For completeness, we quickly review the central result of \citet[Proposition 3]{Menon:2016}.
Let $(P,Q,\pi)$ be densities giving $\pr(\X | \Y = 1), \pr(\X = \ve{x} | \Y = -1), \pr(\Y = 1)$ respectively, and $M$ giving $\pr(\X = \ve{x})$ accordingly.
Let
$ r(\ve{x}) \defeq {\pr(\X = \ve{x} | \Y = 1)}/{\pr(\X = \ve{x} | \Y = -1)}\:\:$
be the density ratio of the class-conditional densities,
and $\eta(\ve{x}) \defeq \pr[\Y = 1 | \X = \ve{x}]$ be the class-probability function.
Then, we have the following, which extends \cite[Proposition 6]{Menon:2016} for the case $\pi \neq \frac{1}{2}$.

\begin{lemma}
\label{lemm:density-ratio}
Given a class-probability estimator $\hat{\eta} \colon \XCal \to [0, 1]$, let the density ratio estimator $\hat{r}$ be
\begin{eqnarray}
\hat{r}(\ve{x}) & = & \frac{1-\pi}{\pi} \cdot \frac{\hat{\eta}(\ve{x})}{1-\hat{\eta}(\ve{x})}\label{defR}\:\:.
\end{eqnarray}
Then for any convex differentiable $\varphi \colon [0, 1] \to \Real$,
\begin{eqnarray}
\expect_{\X \sim M}[D_\varphi(\eta(\X) \| \hat{\eta}(\X))] & = & \pi \cdot
\expect_{\X \sim Q}\left[ D_{\varphi^\dagger}(r(\X) \| \hat{r}(\X))\right]\:\:.
\end{eqnarray}
where $\varphi^\dagger$ is as per Equation \ref{defdagger} with
$ g(z) \defeq \frac{1-\pi}{\pi} + z \:\:. $ 
\end{lemma}

\begin{proof}[Proof of Lemma \ref{lemm:density-ratio}]
Note that
\begin{eqnarray}
\frac{1}{g(r(\ve{x}))} \cdot r(\ve{x}) & = & \frac{\pi \pr(\X = \ve{x} | \Y = -1)}{\pr(\X = \ve{x})} \cdot \frac{\pr(\X = \ve{x} | \Y = 1)}{\pr(\X = \ve{x} | \Y = -1)}\nonumber\\
 & = &\frac{\pi \pr(\X = \ve{x} | \Y = 1)}{\pr(\X = \ve{x})} \nonumber\\
 & = & \eta(\ve{x})\:\:,\label{eqETA}
\end{eqnarray}
and furthermore
\begin{eqnarray}
\pr(\X = \ve{x}) & = & (1-\pi) \pr(\X = \ve{x} | \Y = -1) + \pi \pr(\X = \ve{x} | \Y = 1) \nonumber\\
& = & \pi\cdot \left(\frac{1-\pi}{\pi} + \frac{\pr(\X = \ve{x} | \Y = 1)}{\pr(\X = \ve{x} | \Y = -1)}\right) \cdot \pr(\X = \ve{x} | \Y = -1)\nonumber\\
& = & \pi \cdot g(r(\ve{x})) \cdot \pr(\X = \ve{x} | \Y = -1)\label{eq123}\:\:.
\end{eqnarray}
So,
\begin{eqnarray}
\expect_{M}[D_\varphi(\eta(\X) \| \hat{\eta}(\X))] & = & \pi \cdot \expect_{Q}\left[ g(r(\X))\cdot D_\varphi(\eta(\X) \| \hat{\eta}(\X))\right] \label{p1} \\
 & = &  \pi \cdot \expect_{Q}\left[ g(r(\X))\cdot D_\varphi\left(\left.\frac{1}{g(r(\X))} \cdot r(\X) \right\| \hat{\eta}(\X)\right)\right]\label{p3} \\
 & = &  \pi \cdot \expect_{Q}\left[ g(r(\X))\cdot D_\varphi\left(\left.\frac{1}{g(r(\X))} \cdot r(\X) \right\| \frac{1}{g(\hat{r}(\X))} \cdot \hat{r}(\X)\right)\right]\label{p4} \\
 & = & \pi \cdot \expect_{Q}\left[ D_{\varphi^\dagger}(r(\X) \| \hat{r}(\X))\right]\:\:,\label{p5}
\end{eqnarray}
as claimed.
Equation (\ref{p1}) comes from (\ref{eq123}), Equation (\ref{p3}) comes from (\ref{eqETA}), Equation (\ref{p4}) comes from (\ref{defR}) and the definition of $g$.
Equation (\ref{p5}) comes from Theorem \ref{th00}, noting that $g$ is linear.
\end{proof}

\section{Additional experiments}

\subsection{Multiclass density ratio experiments}
\label{app:multiclass-dr}

We consider a synthetic multiclass density ratio estimation problem.
We fix $\XCal = \Real^2$, and consider $C = 3$ classes.
We consider a distribution where the class-conditionals $\Pr( \X | \Y = c )$ are multivariate Gaussians with means $\ve{\mu}_c$ and covariance $\sigma_c^2 \cdot \text{Id}$.
As the class-conditionals have a closed form, we can explicitly compute $\ve{\eta}$, as well the density ratio $\ve{r}$ to the reference class $c^* = C$.

For fixed class prior $\ve{\pi} = \Pr(\Y = c)$, we draw $N_{\text{Tr}}$ samples from $\Pr( \X, \Y )$.
From this, we estimate the class-probability $\hat{\ve{\eta}}$ using multiclass logistic regression.
This can be seen as minimising $\expect_{\X \sim M}\left[ {D_{\varphi}( \ve{\eta}( \X ) \| \hat{\ve{\eta}}( \X ) )} \right]$ where $\varphi( z ) = \sum_{i} z_i \log z_i$ is the generator for the KL-divergence.

We then use Equation \ref{eqn:dr-from-cpe} to estimate the density ratios $\hat{\ve{r}}$ from $\hat{\ve{\eta}}$.
On a fresh sample of $N_{\text{te}}$ instances from $\Pr( \X, \Y )$, we estimate the right hand side of Lemma \ref{lemm:multiclass-dr}, \emph{viz.\ } $\expect_{\X \sim P_C}\left[ { D_{\varphi^\dagger}( r( \X ) \| \hat{r}( \X ) ) } \right]$,
where $\varphi^\dagger$ uses the $g$ as specified in Lemma \ref{lemm:multiclass-dr}.
From the result of Lemma \ref{lemm:multiclass-dr}, we expect this divergence to be small when $\hat{\ve{\eta}}$ is a good estimator of $\ve{\eta}$.

We perform the above for sample sizes $N \in \{ 4^4, 4^5, \ldots, 4^{10} \}$, with $N_{\text{Tr}} = 0.8 N$ and $N_{\text{Te}} = 0.2 N$.
For each sample size, we perform $T = 25$ trials, where in each trial we randomly draw $\ve{\pi}$ uniformly over $(1/C) \mathbf{1} + (1 - 1/C) \cdot [0, 1]^C$,
$\ve{\mu}_c$ from $0.1 \cdot \mathscr{N}( \mathbf{0}, 1 )$,
and $\sigma_c$ uniformly from $[0.5, 1]$.
Figure \ref{fig:multiclass-dr-results} summarises the mean divergence across the $T$ trials for each sample size.
We see that, as expected, with more training samples the divergence decreases in a monotone fashion.

\begin{figure}[!t]
	\centering
	\includegraphics[scale=0.125]{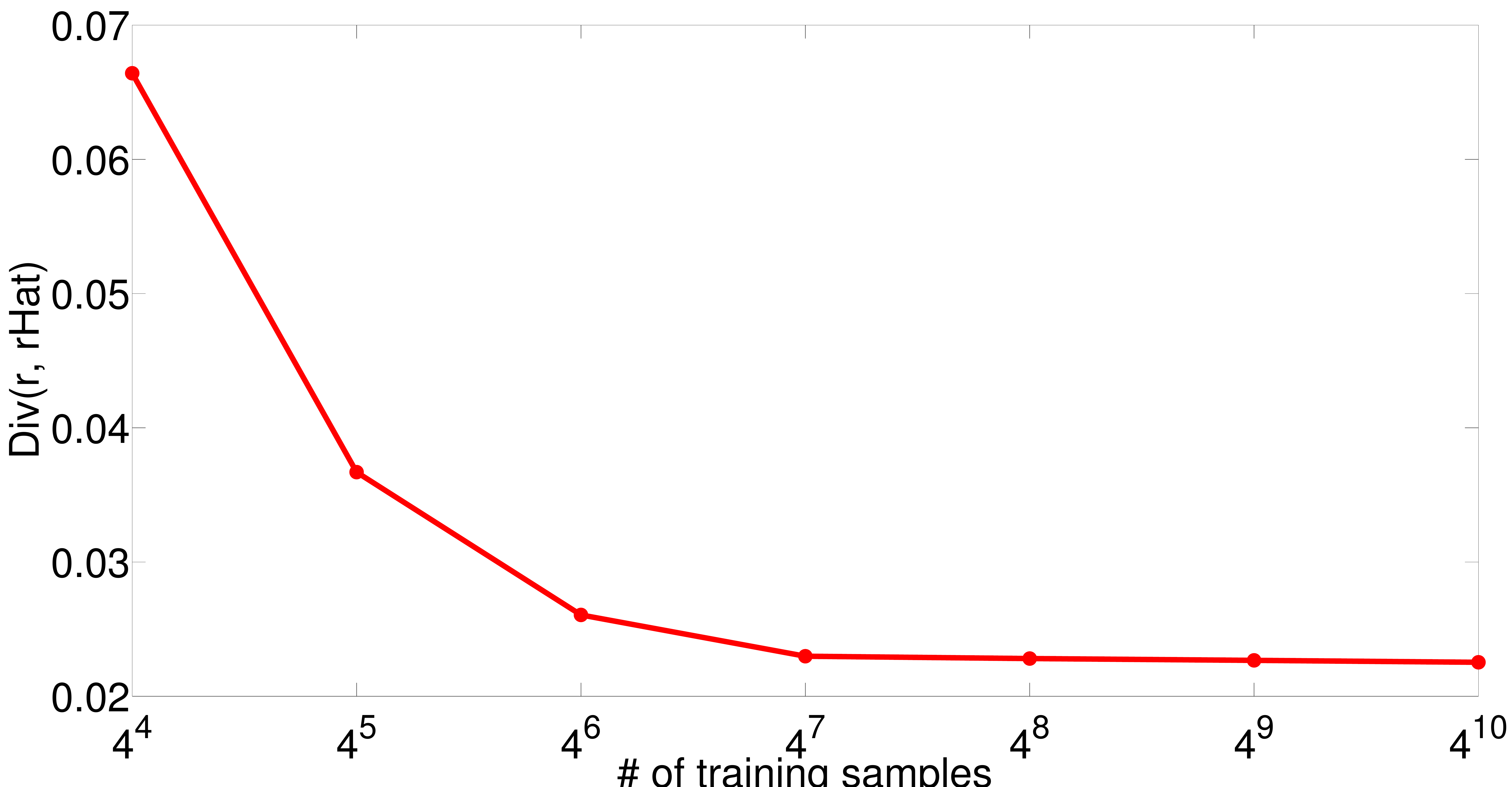}
	\caption{Density ratio estimate divergence $\expect_{\X \sim P_C}\left[{ D_{\varphi^\dagger}( r( \X ), \hat{r}( \X ) ) }\right]$ as a function of \# of training samples.}
	\label{fig:multiclass-dr-results}
\end{figure}

\subsection{Adaptive filtering experiments}
\label{app:exp-full}
Tables \ref{tc1_supp_rr1} -- \ref{tc1_supp_rr6} present \textit{in
  extenso} the experiments of $p$-LMS vs DN-$p$-LMS, as a function of
$(p,q)$, whether target $\ve{u}$ is sparse or not, and the
misestimation factor $\rho$ for $X_p$. We refer to
\cite{kwhTP} for the formal definitions used for sparse / dense
targets as well as for the experimental setting, which we have
reproduced with the sole difference that the signal changes
periodically each 1 000 iterations.

\begin{sidewaystable}[t]
\begin{center}
{\small
\begin{tabular}{cccccc}\hline \hline
\includegraphics[trim=10bp 30bp 30bp 10bp,clip,width=.15\linewidth]{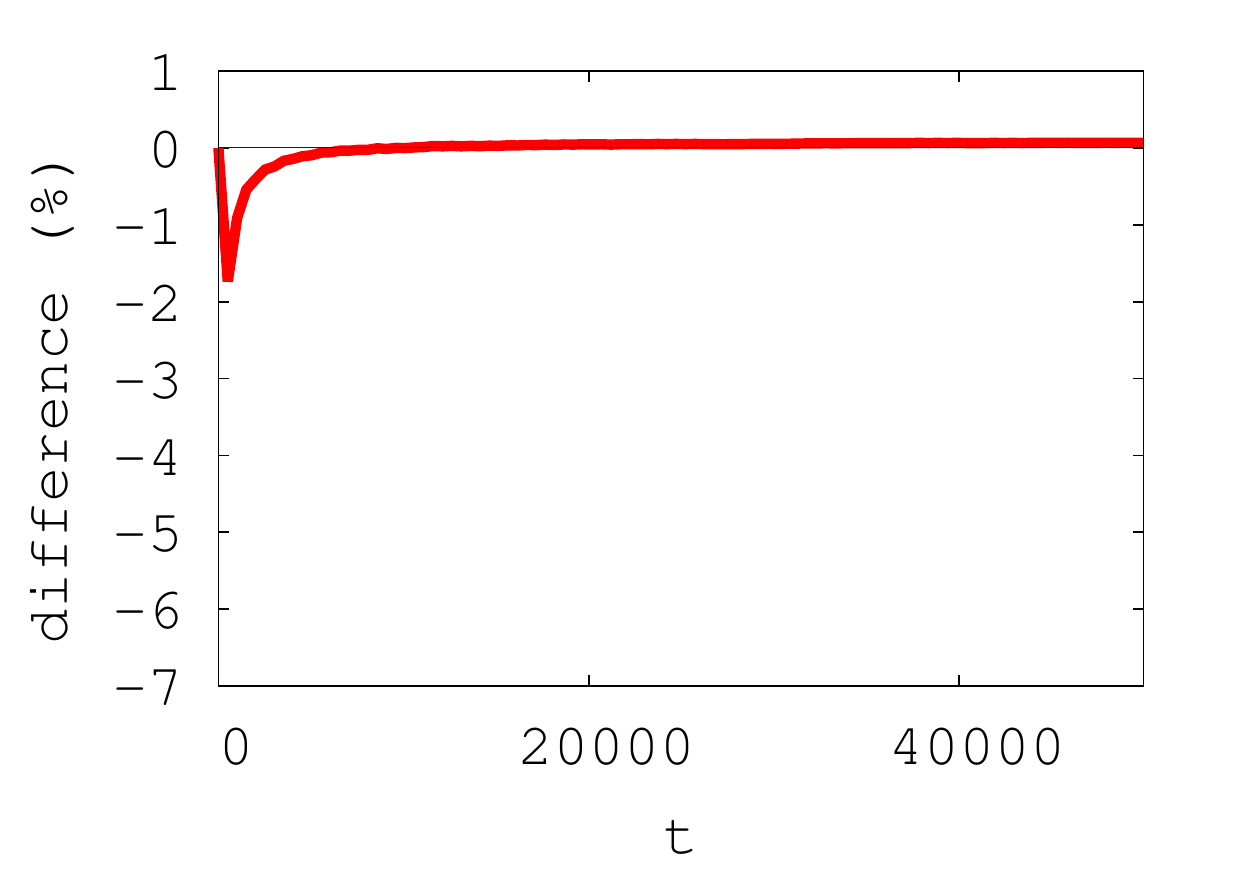}
&
\includegraphics[trim=10bp 25bp 30bp 10bp,clip,width=.15\linewidth]{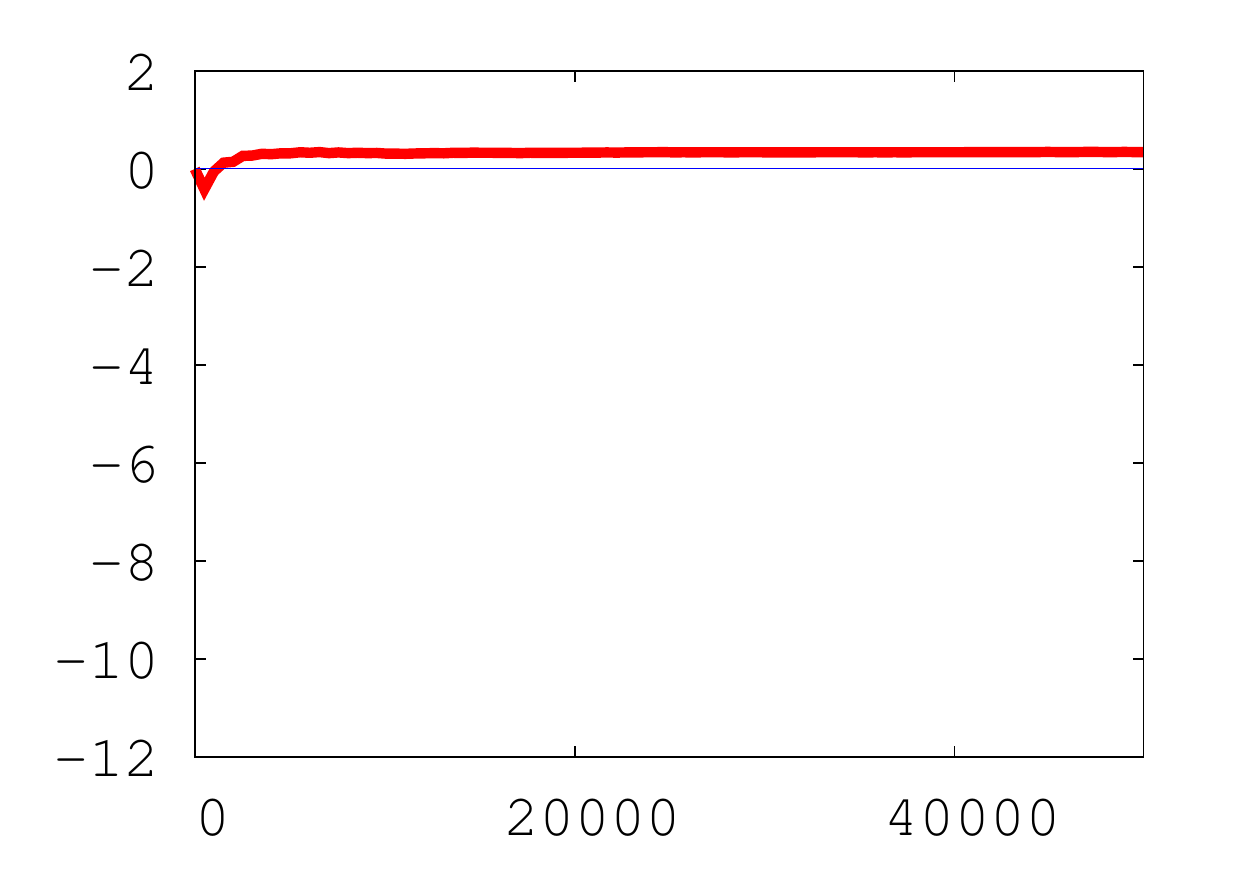}
&
\includegraphics[trim=10bp 25bp 30bp 10bp,clip,width=.15\linewidth]{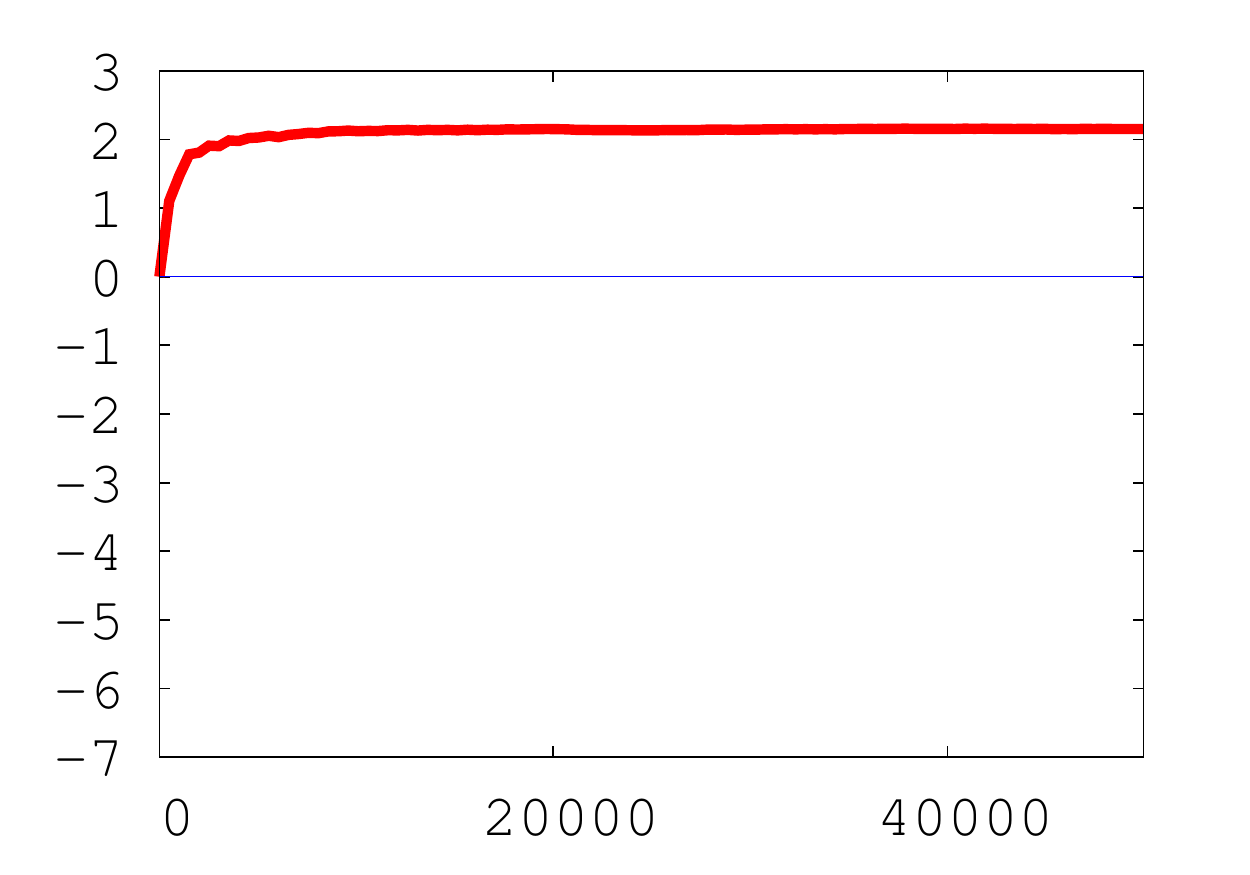}
&
\includegraphics[trim=10bp 25bp 30bp 10bp,clip,width=.15\linewidth]{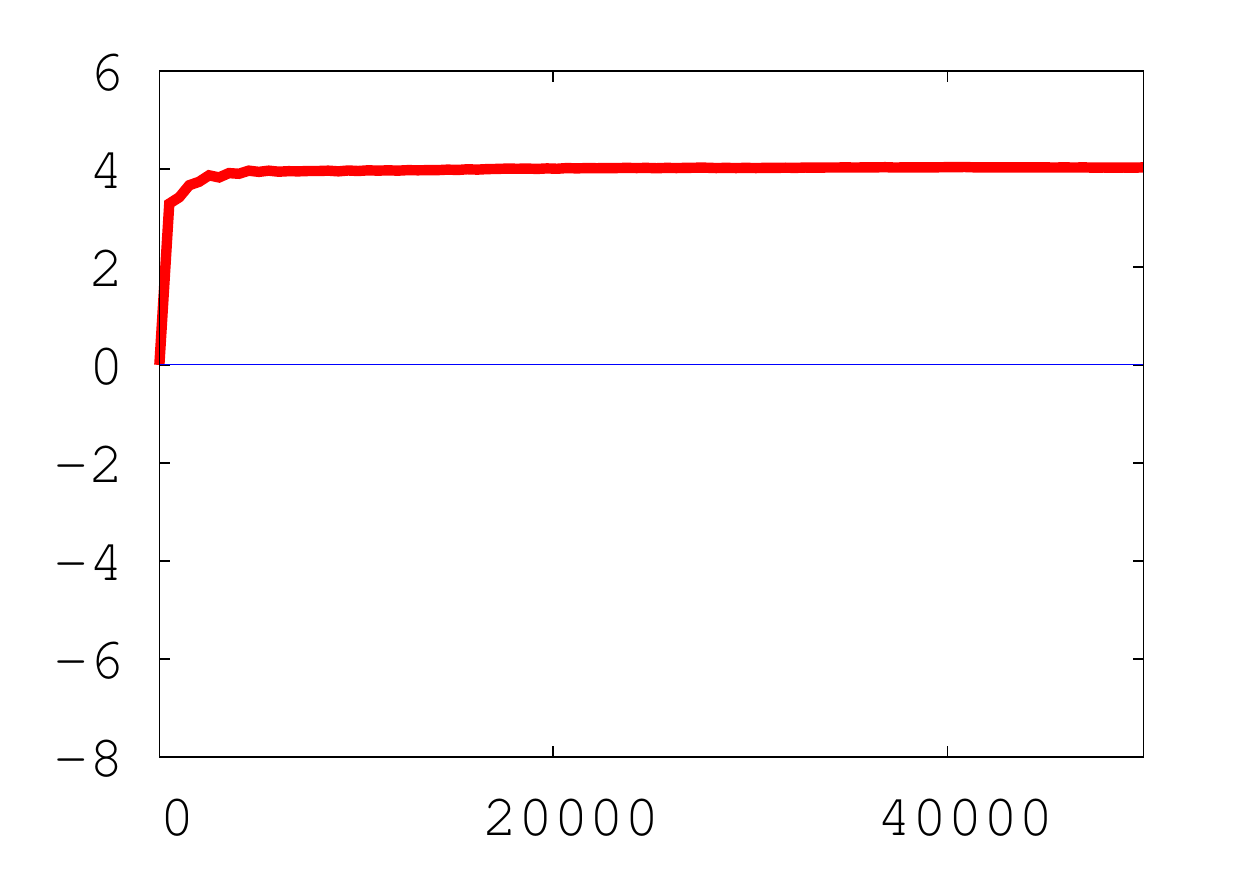}
&
\includegraphics[trim=10bp 25bp 30bp 10bp,clip,width=.15\linewidth]{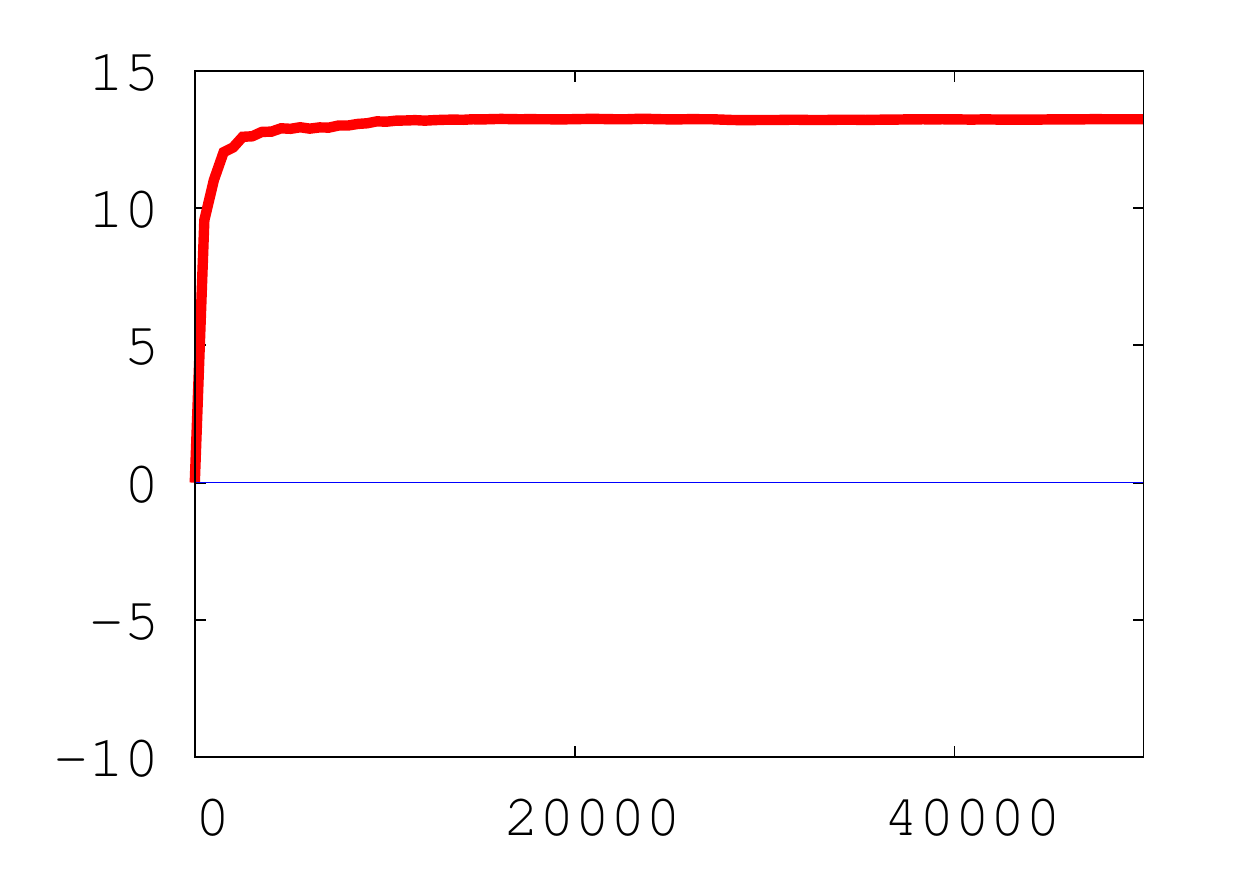}
&
\includegraphics[trim=10bp 25bp 30bp 10bp,clip,width=.15\linewidth]{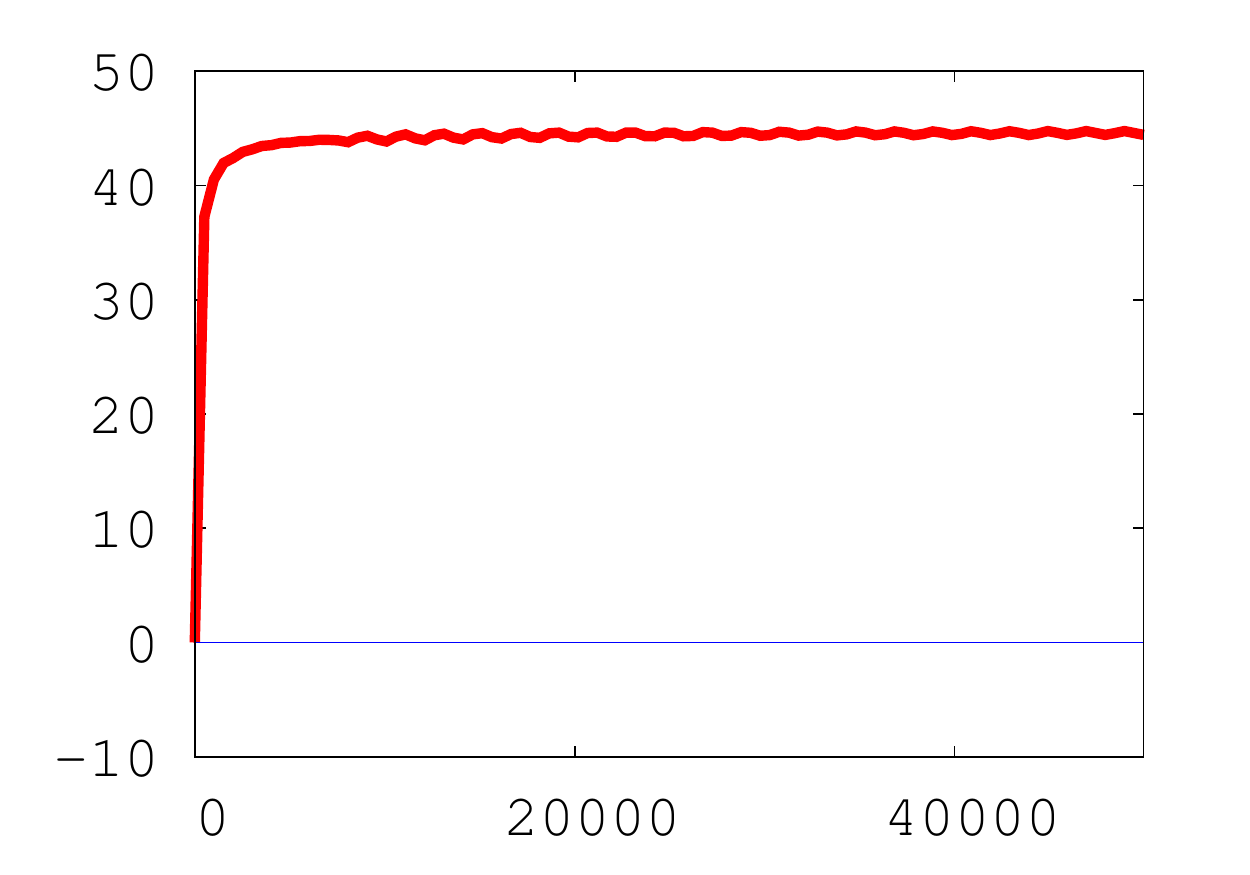}
\\
$\rho = 0.1$ & $\rho = 0.2$ & $\rho = 0.3$ & $\rho = 0.4$ & $\rho = 0.5$ & $\rho = 0.6$ \\\hline
\includegraphics[trim=10bp 25bp 30bp 10bp,clip,width=.15\linewidth]{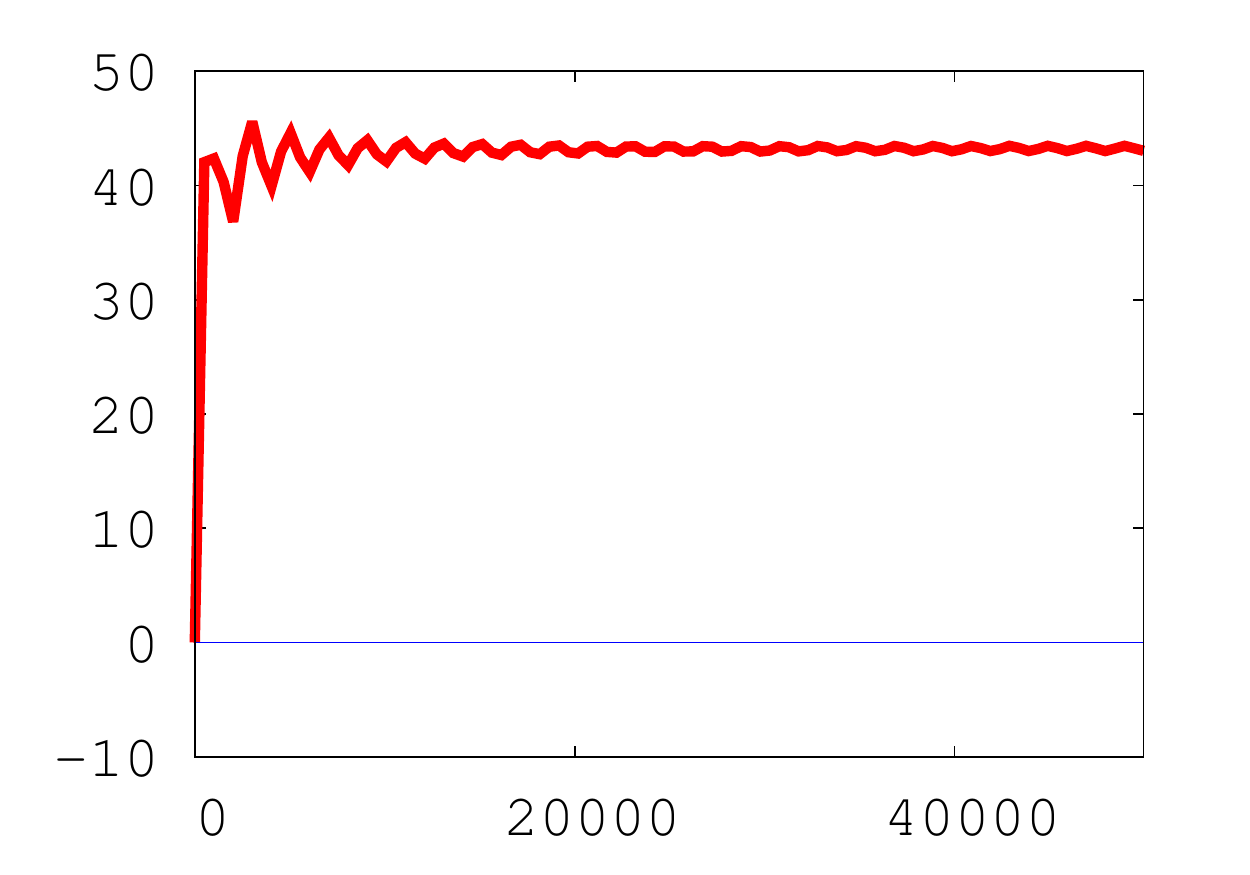}
&
\includegraphics[trim=10bp 25bp 30bp 10bp,clip,width=.15\linewidth]{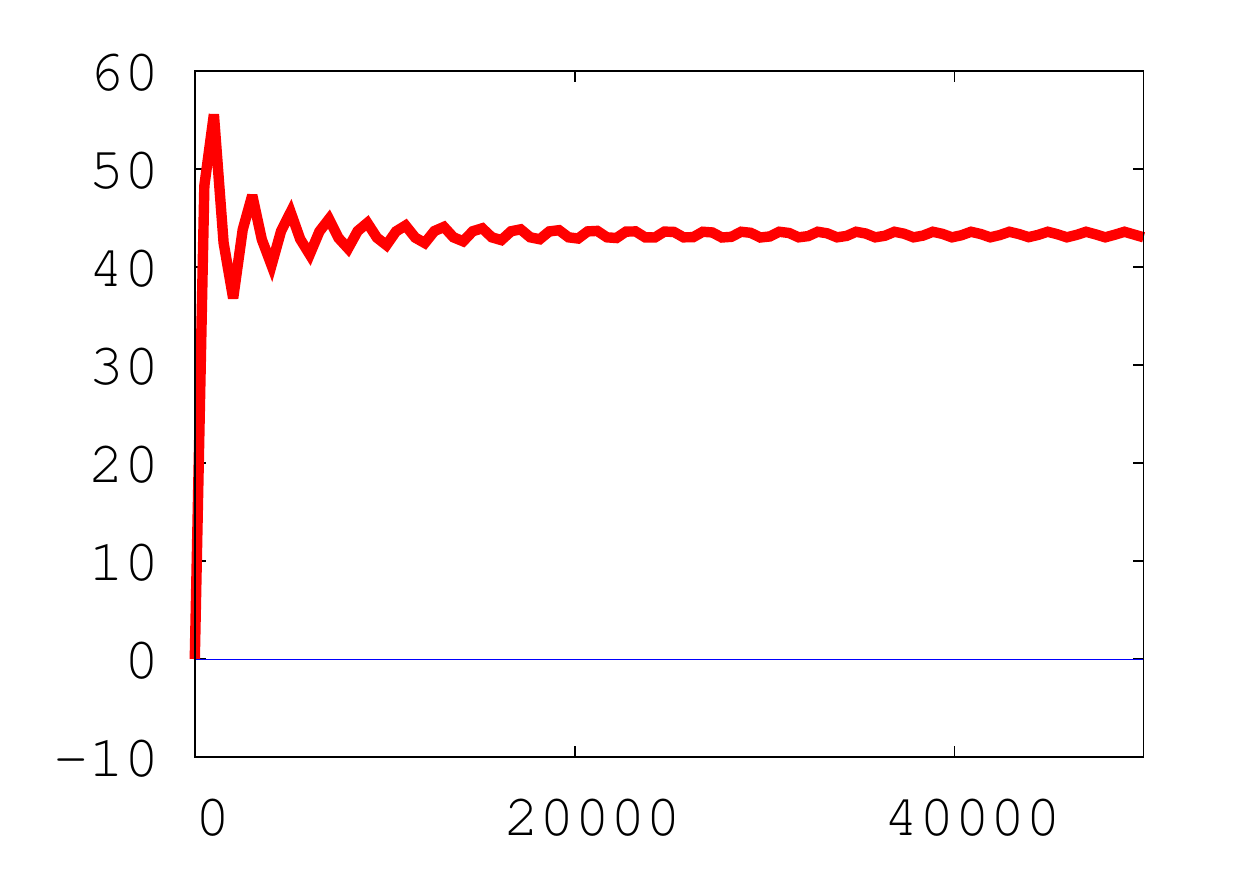}
&
\includegraphics[trim=10bp 25bp 30bp 10bp,clip,width=.15\linewidth]{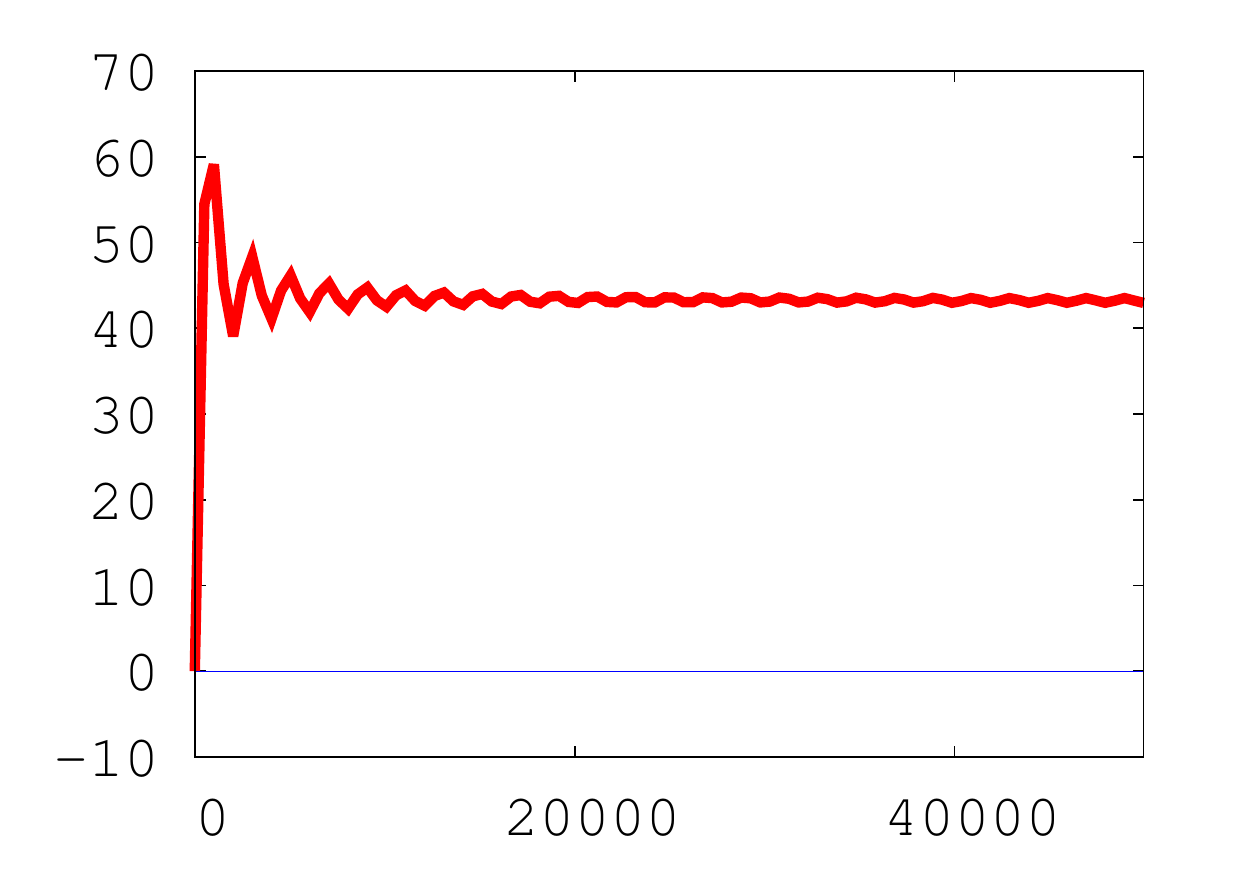}
&
\includegraphics[trim=10bp 25bp 30bp 10bp,clip,width=.15\linewidth]{results_P1_17_Q6_88_UCHOICE_GAUSS_UR_1_00_NOLABELS-eps-converted-to}
&
\includegraphics[trim=10bp 25bp 30bp 10bp,clip,width=.15\linewidth]{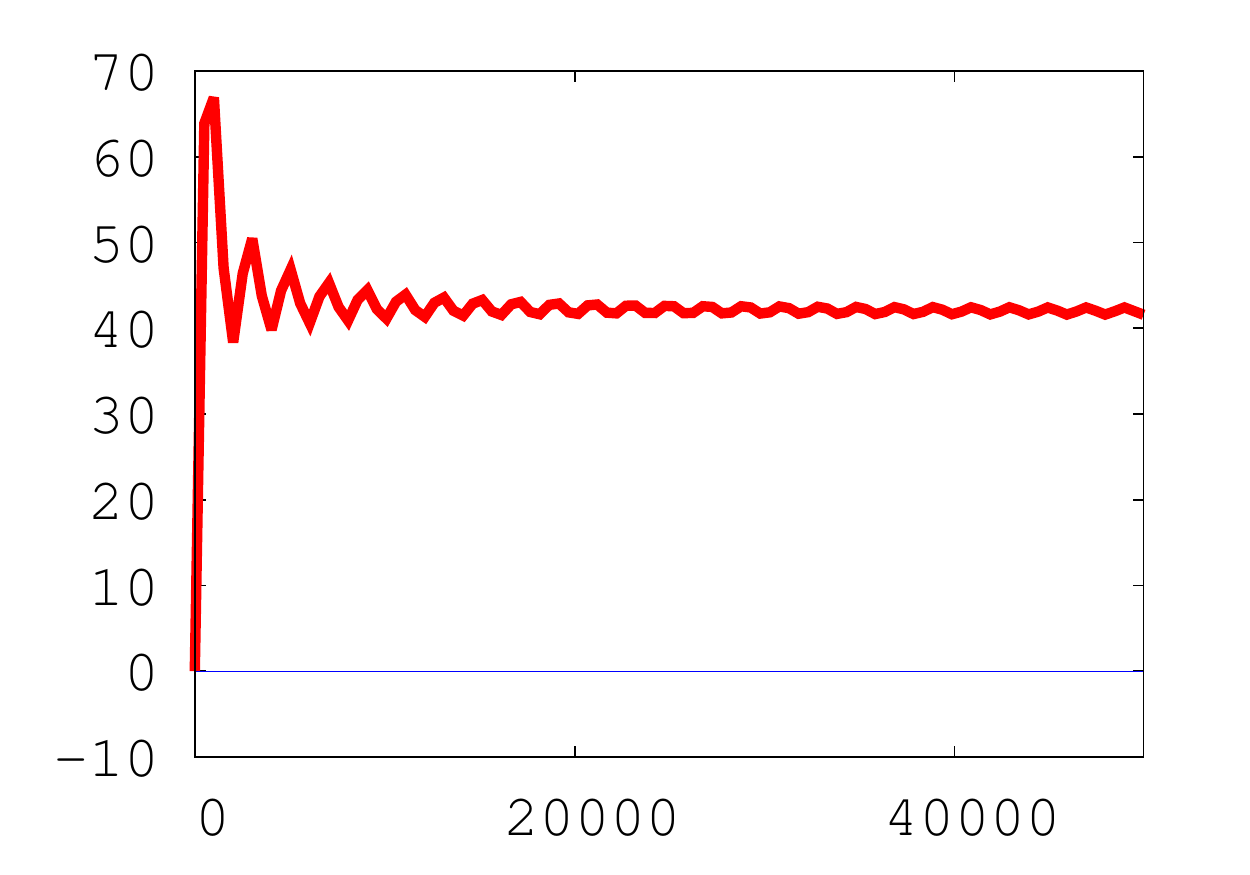}
&
\includegraphics[trim=10bp 25bp 30bp 10bp,clip,width=.15\linewidth]{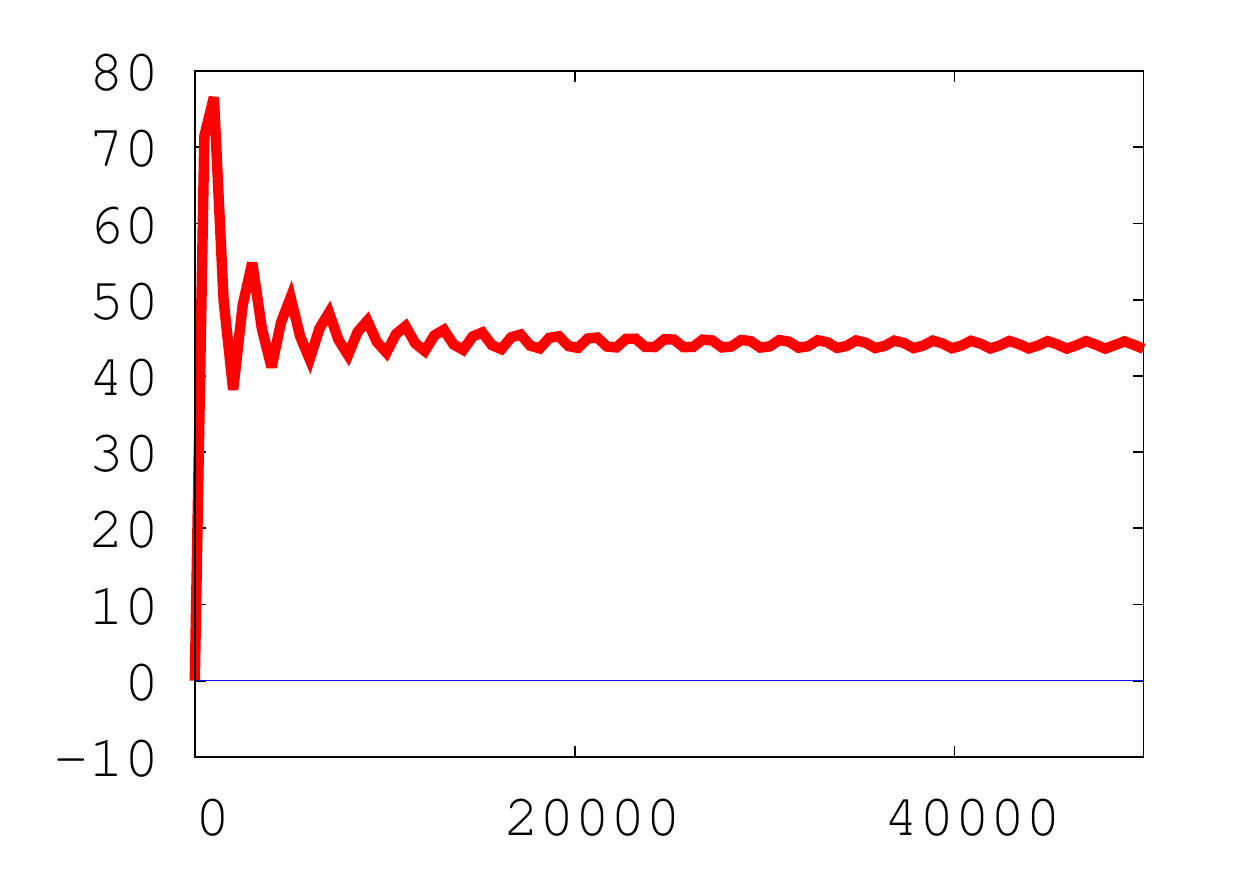}
\\
$\rho = 0.7$ & $\rho = 0.8$ & $\rho = 0.9$ & $\rho = \mbox{\textbf{1.0}}$ & $\rho = 1.1$ & $\rho = 1.2$ \\\hline
\includegraphics[trim=10bp 25bp 30bp 10bp,clip,width=.15\linewidth]{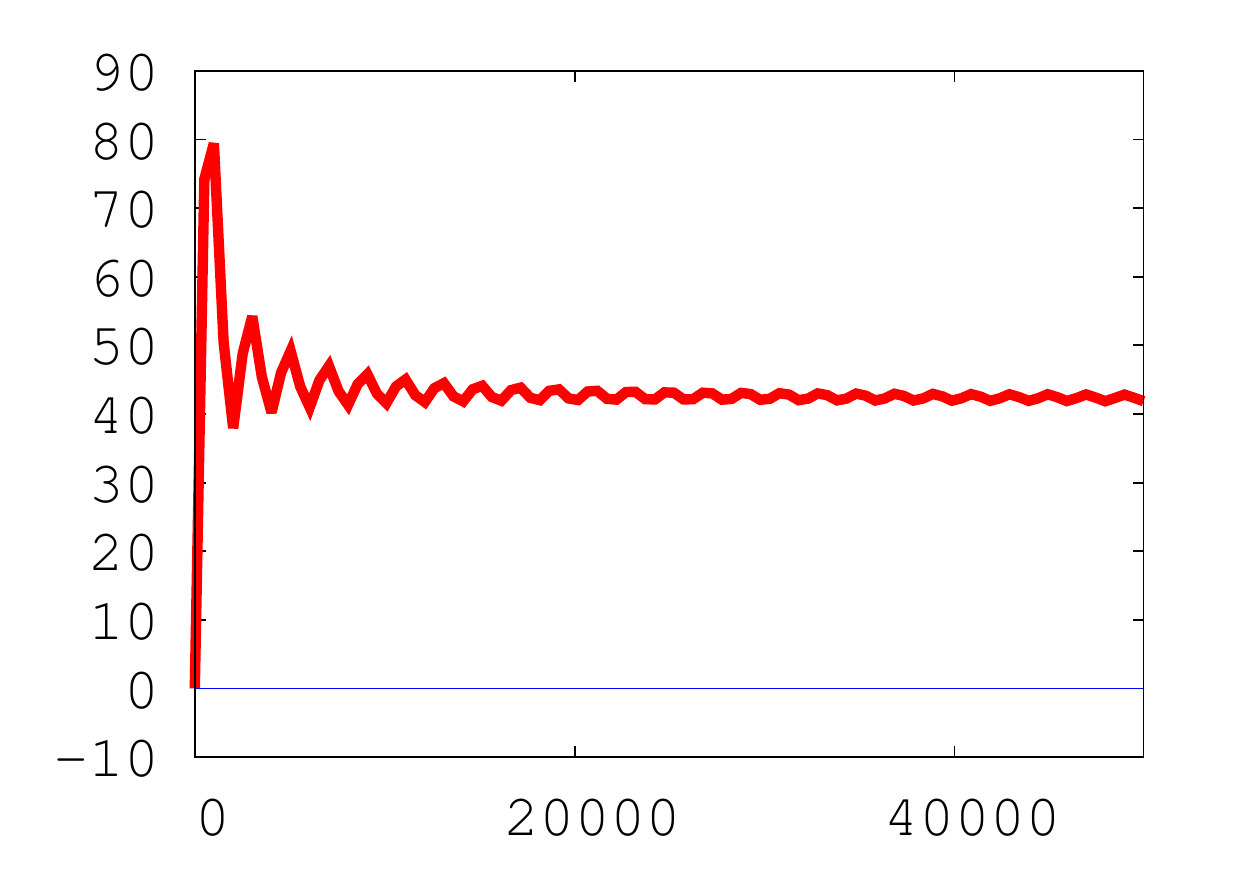}
&
\includegraphics[trim=10bp 25bp 30bp 10bp,clip,width=.15\linewidth]{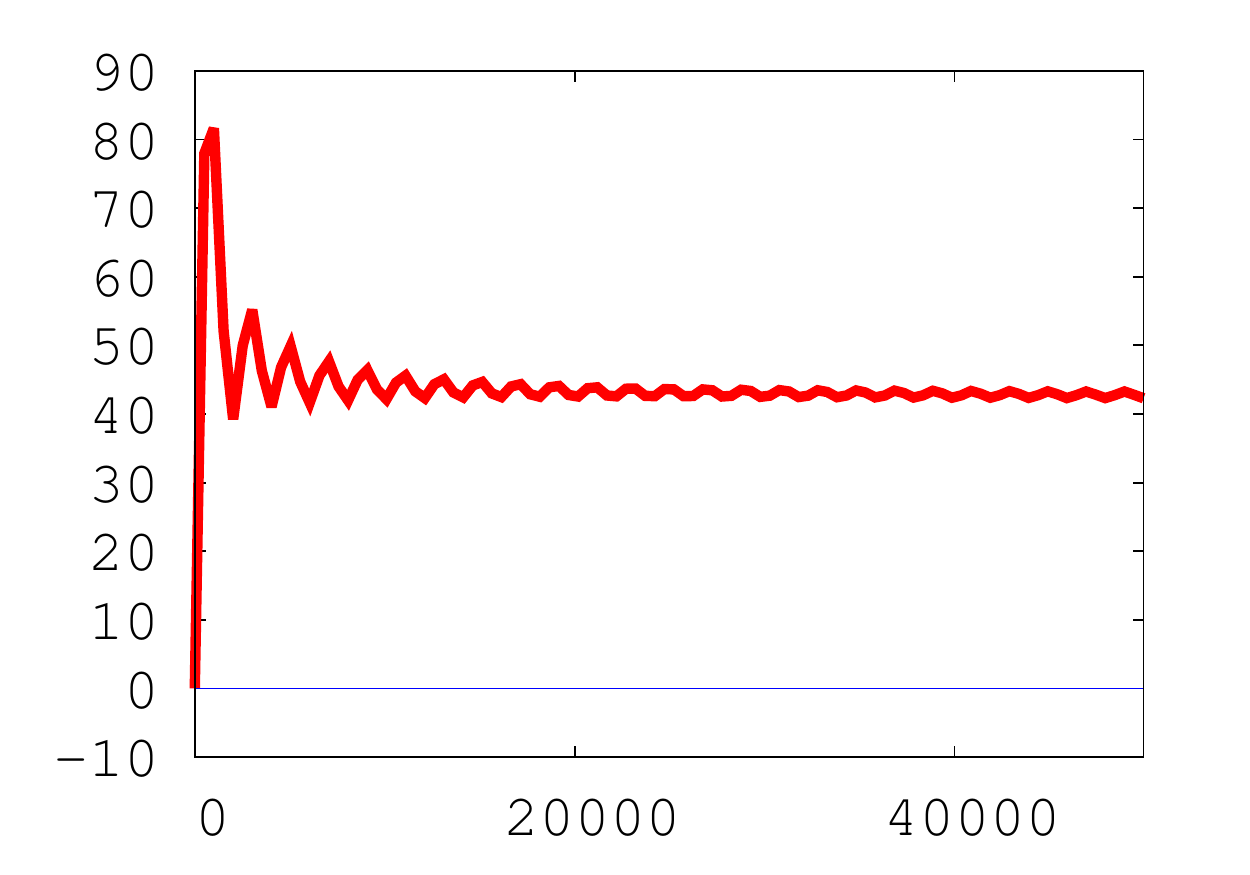}
&
\includegraphics[trim=10bp 25bp 30bp 10bp,clip,width=.15\linewidth]{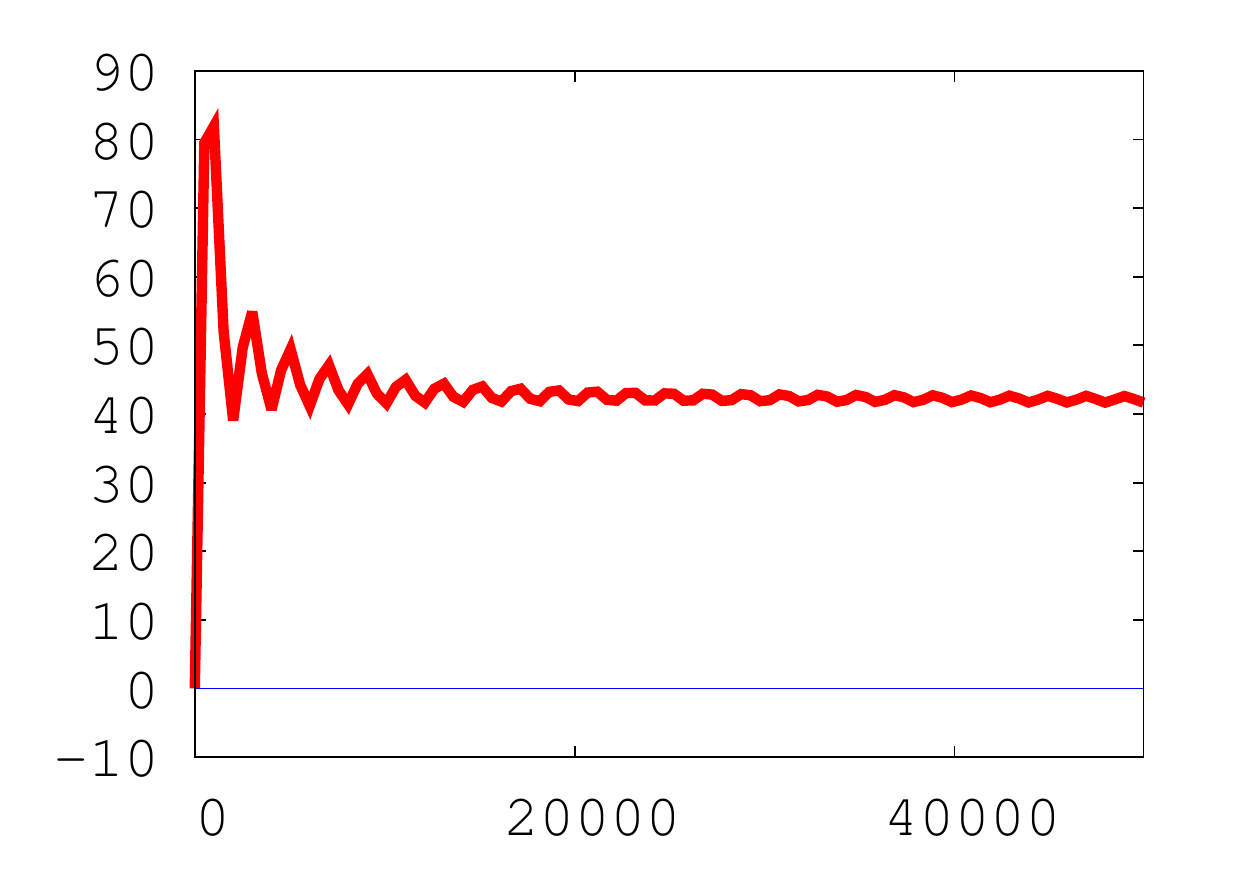}
&
\includegraphics[trim=10bp 25bp 30bp 10bp,clip,width=.15\linewidth]{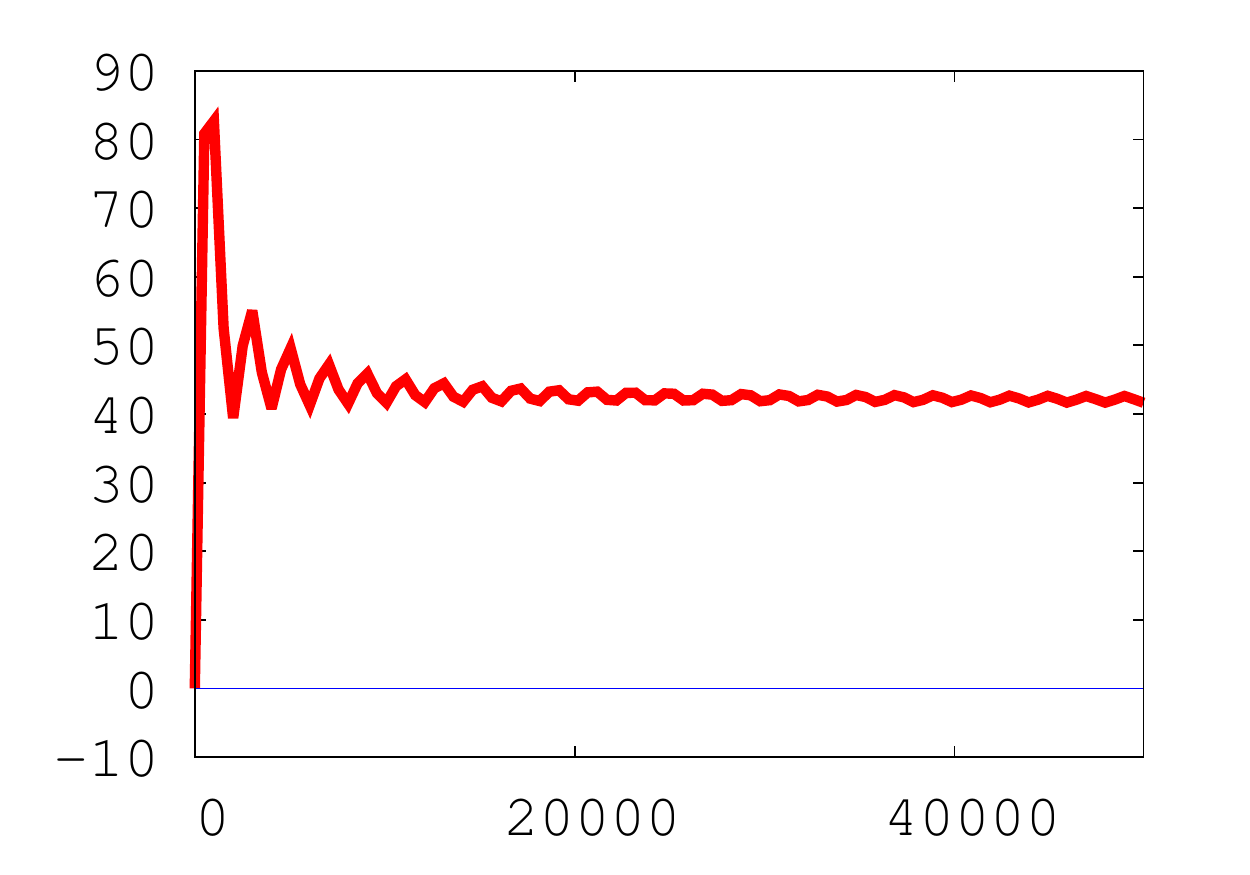}
&
\includegraphics[trim=10bp 25bp 30bp
10bp,clip,width=.15\linewidth]{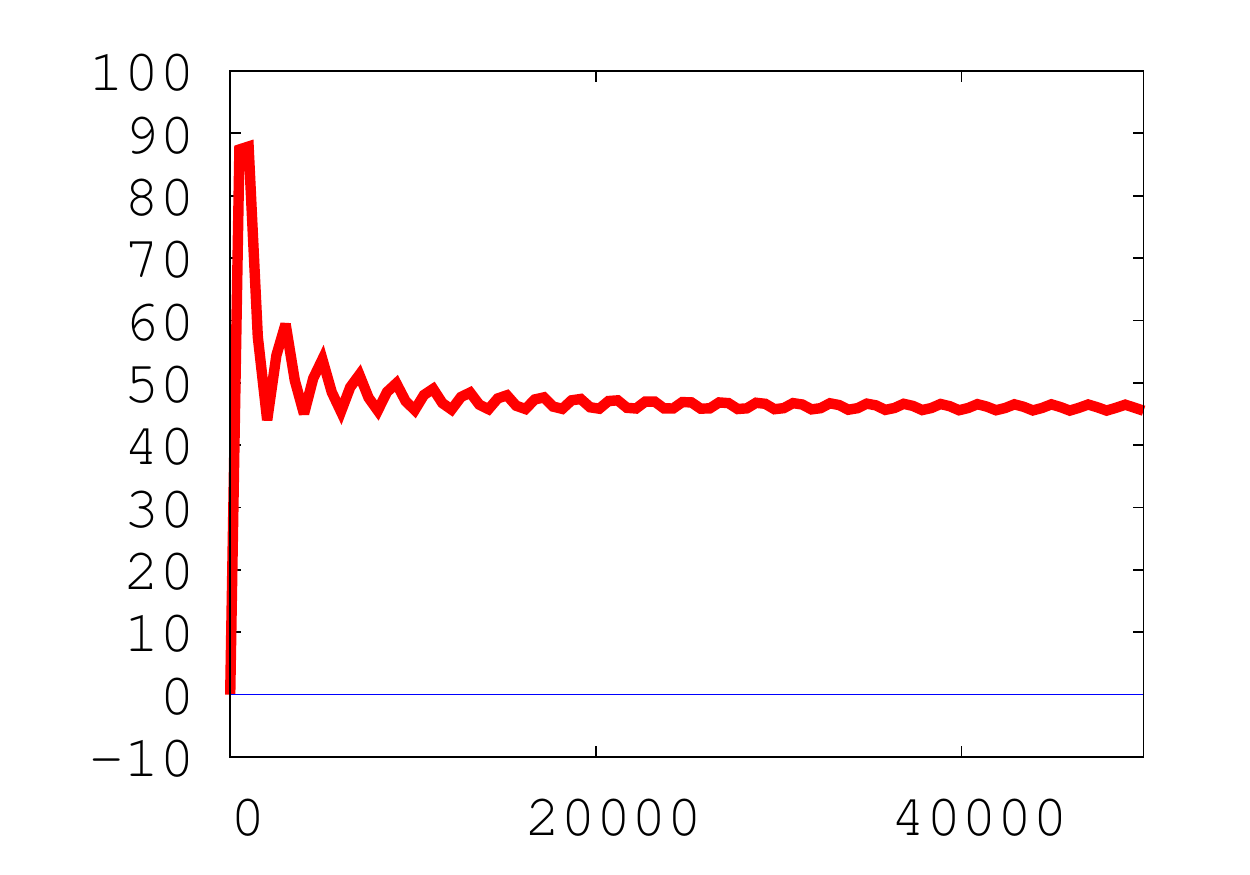}
&\\
$\rho =1.3$ & $\rho = 1.4$ & $\rho = 1.5$ & $\rho = 1.6$ & $\rho =
1.7$ &  \\\hline \hline
\end{tabular}
}
\end{center}
\caption{Error($p$-LMS) - Error(DN-$p$-LMS) as a function of $t$ ($\in \{1, 2, ..., 50 000\}$), $\bm{u}$ = dense, $(p,q) = (1.17, 6.9)$.}
  \label{tc1_supp_rr1}
\end{sidewaystable}

\begin{sidewaystable}[t]
\begin{center}
{\small
\begin{tabular}{cccccc}\hline \hline
\includegraphics[trim=10bp 30bp 30bp 10bp,clip,width=.15\linewidth]{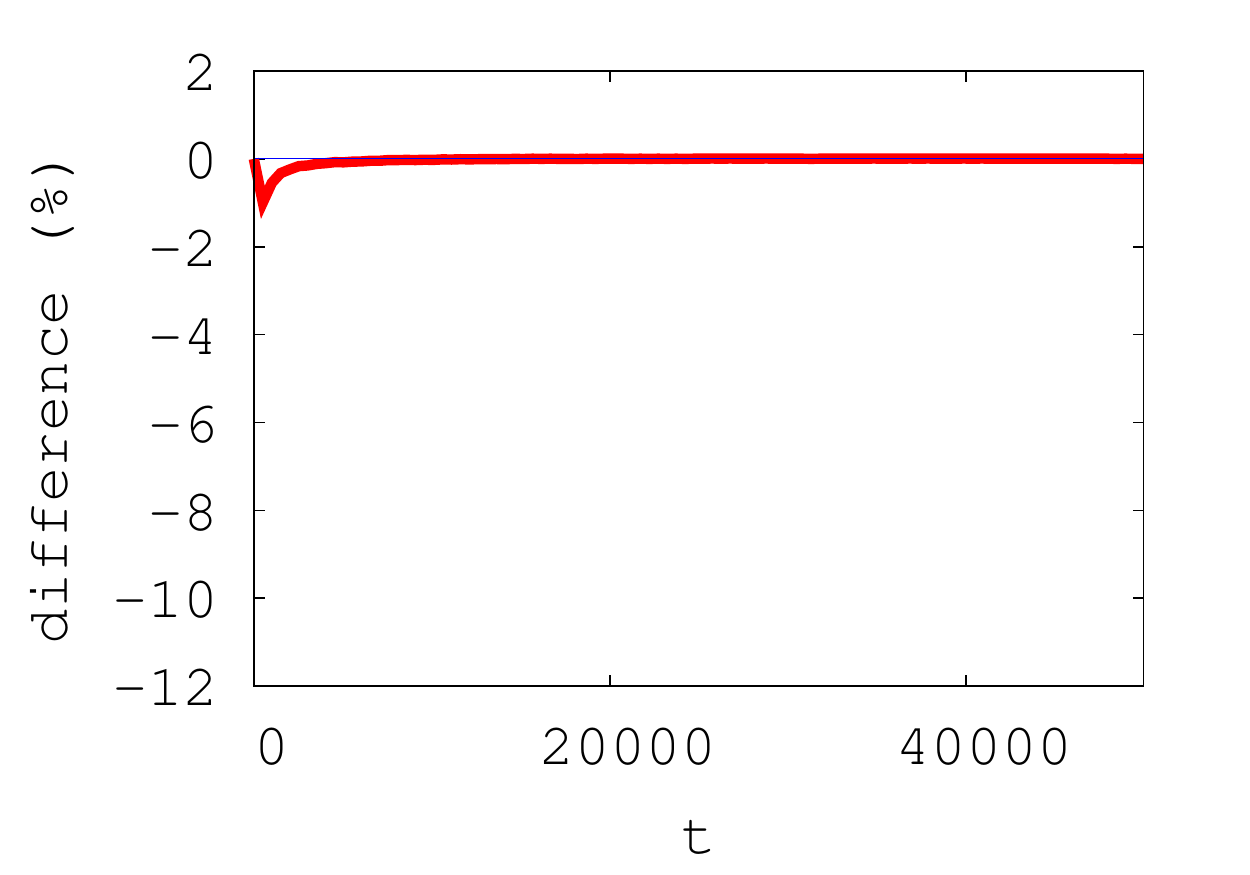}
&
\includegraphics[trim=10bp 25bp 30bp 10bp,clip,width=.15\linewidth]{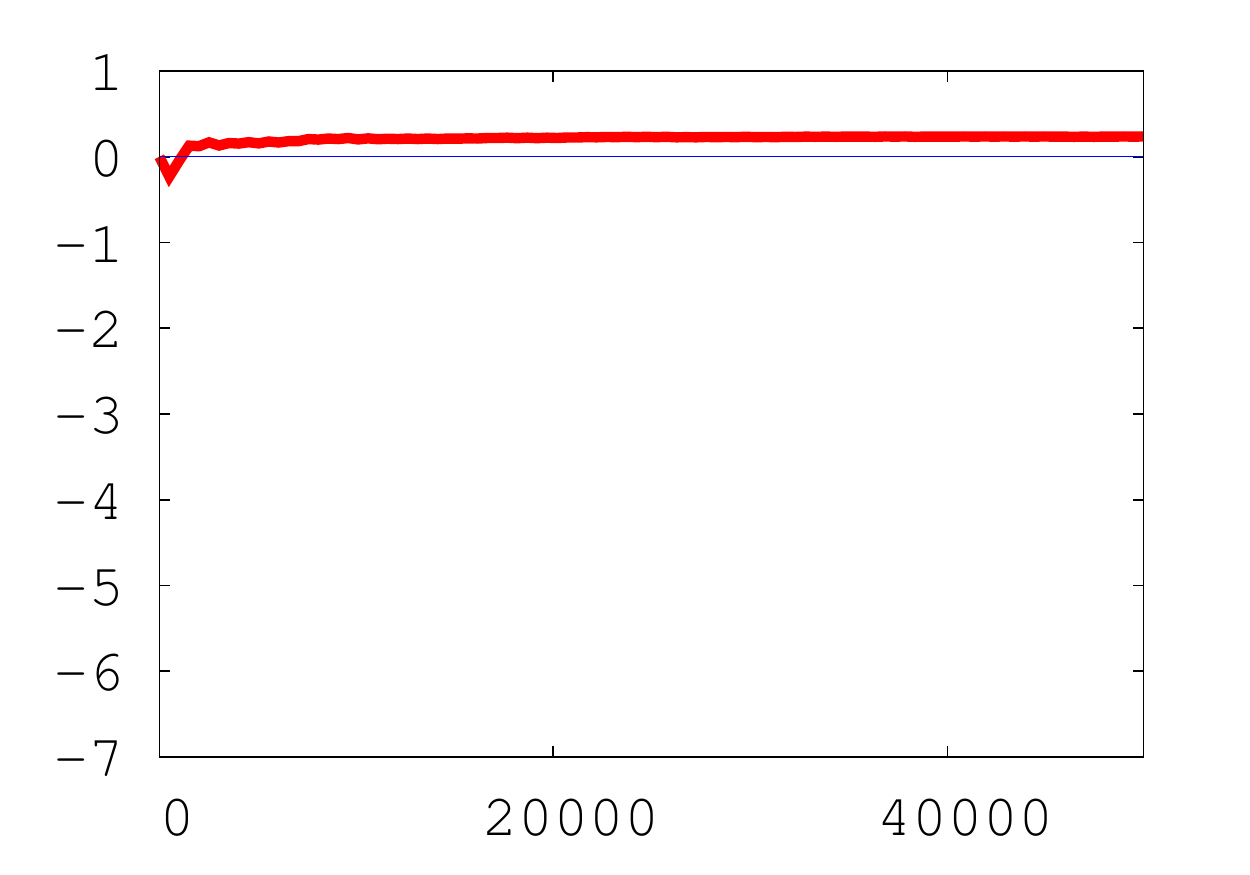}
&
\includegraphics[trim=10bp 25bp 30bp 10bp,clip,width=.15\linewidth]{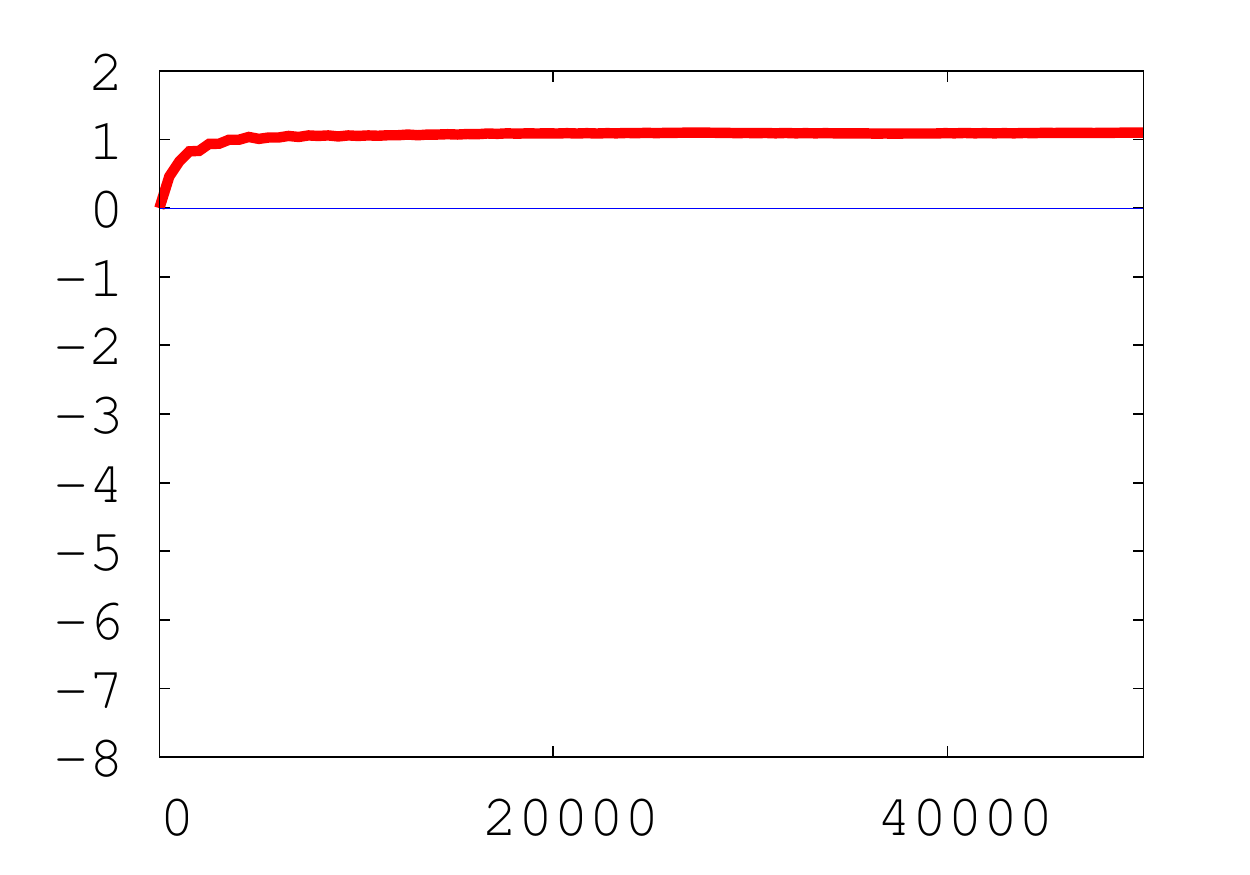}
&
\includegraphics[trim=10bp 25bp 30bp 10bp,clip,width=.15\linewidth]{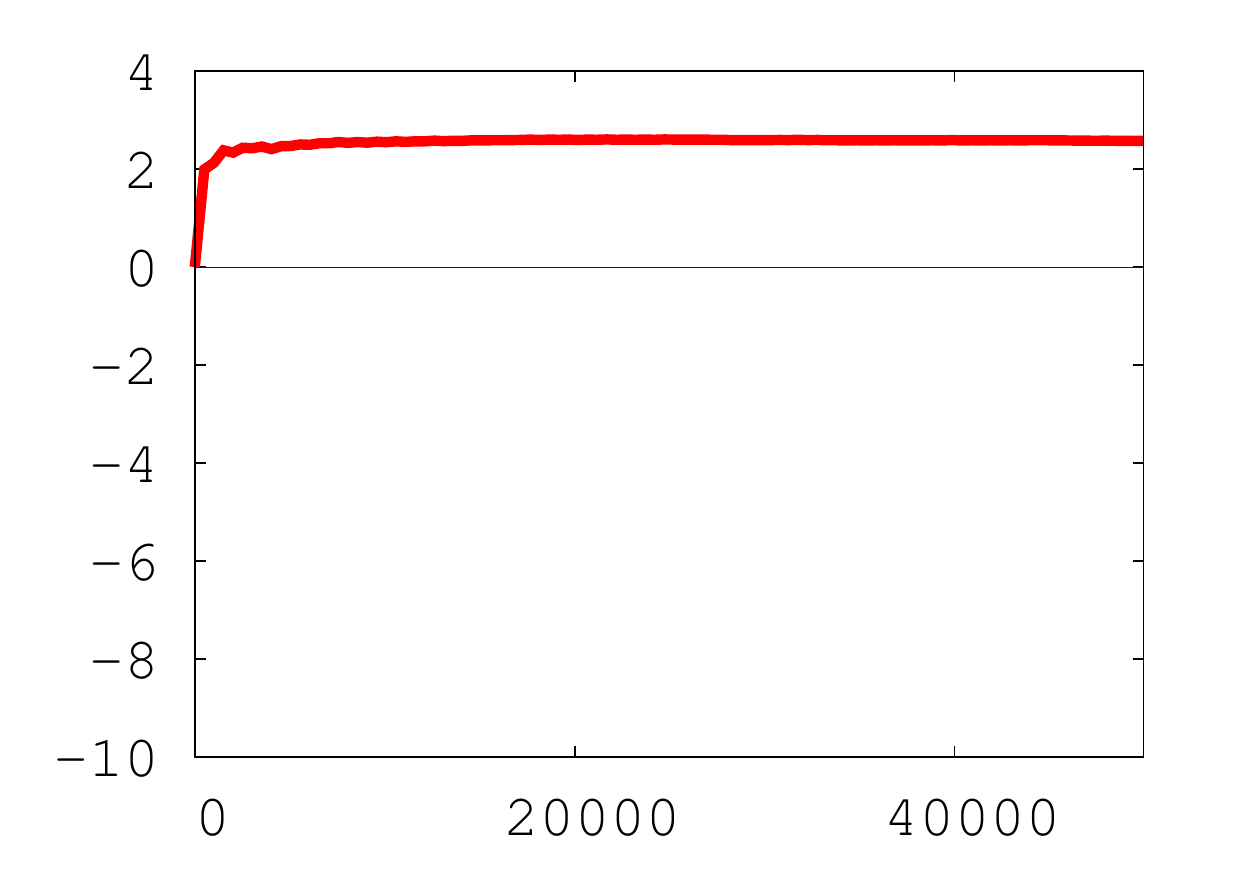}
&
\includegraphics[trim=10bp 25bp 30bp 10bp,clip,width=.15\linewidth]{results_P1_17_Q6_88_UCHOICE_S_EXP_R_UR_0_50_NOLABELS-eps-converted-to}
&
\includegraphics[trim=10bp 25bp 30bp 10bp,clip,width=.15\linewidth]{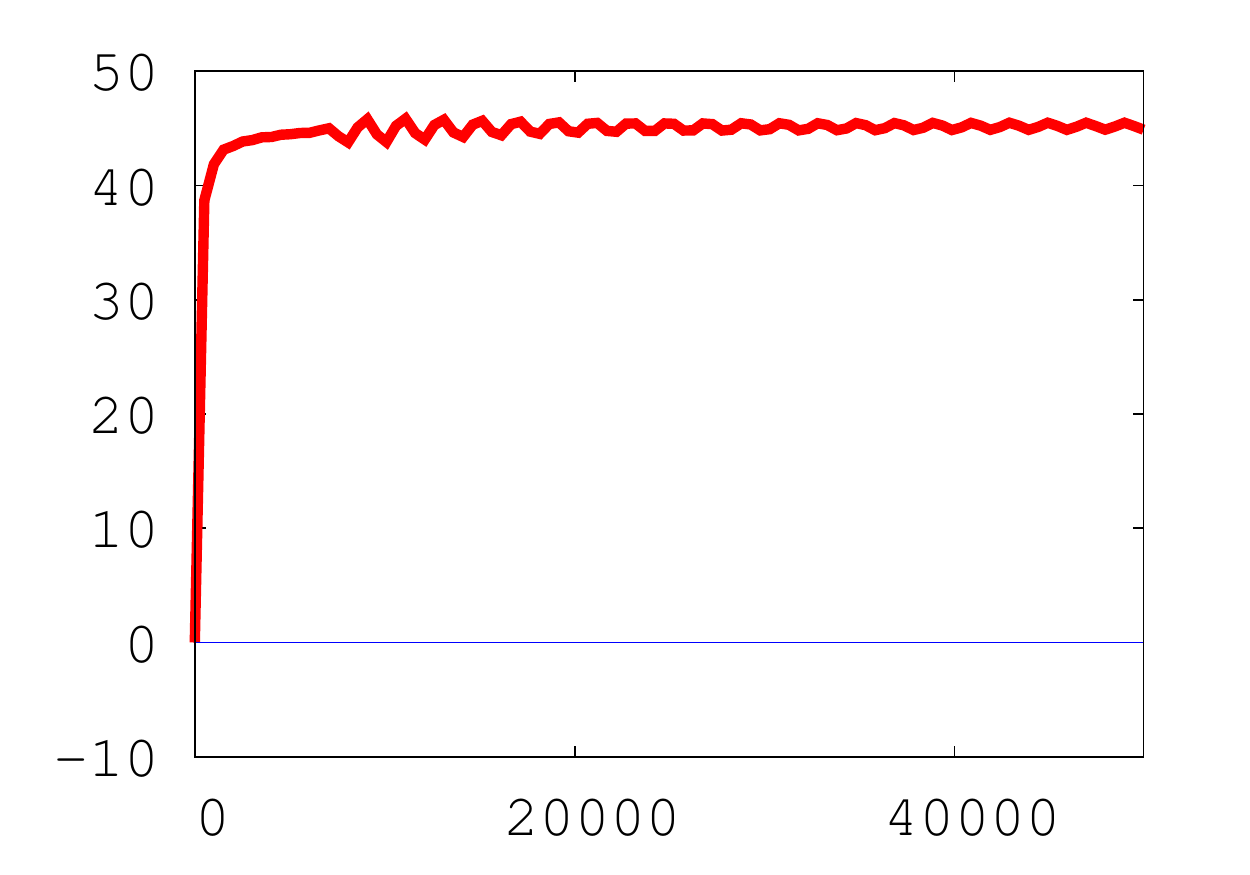}
\\
$\rho = 0.1$ & $\rho = 0.2$ & $\rho = 0.3$ & $\rho = 0.4$ & $\rho = 0.5$ & $\rho = 0.6$ \\\hline
\includegraphics[trim=10bp 25bp 30bp 10bp,clip,width=.15\linewidth]{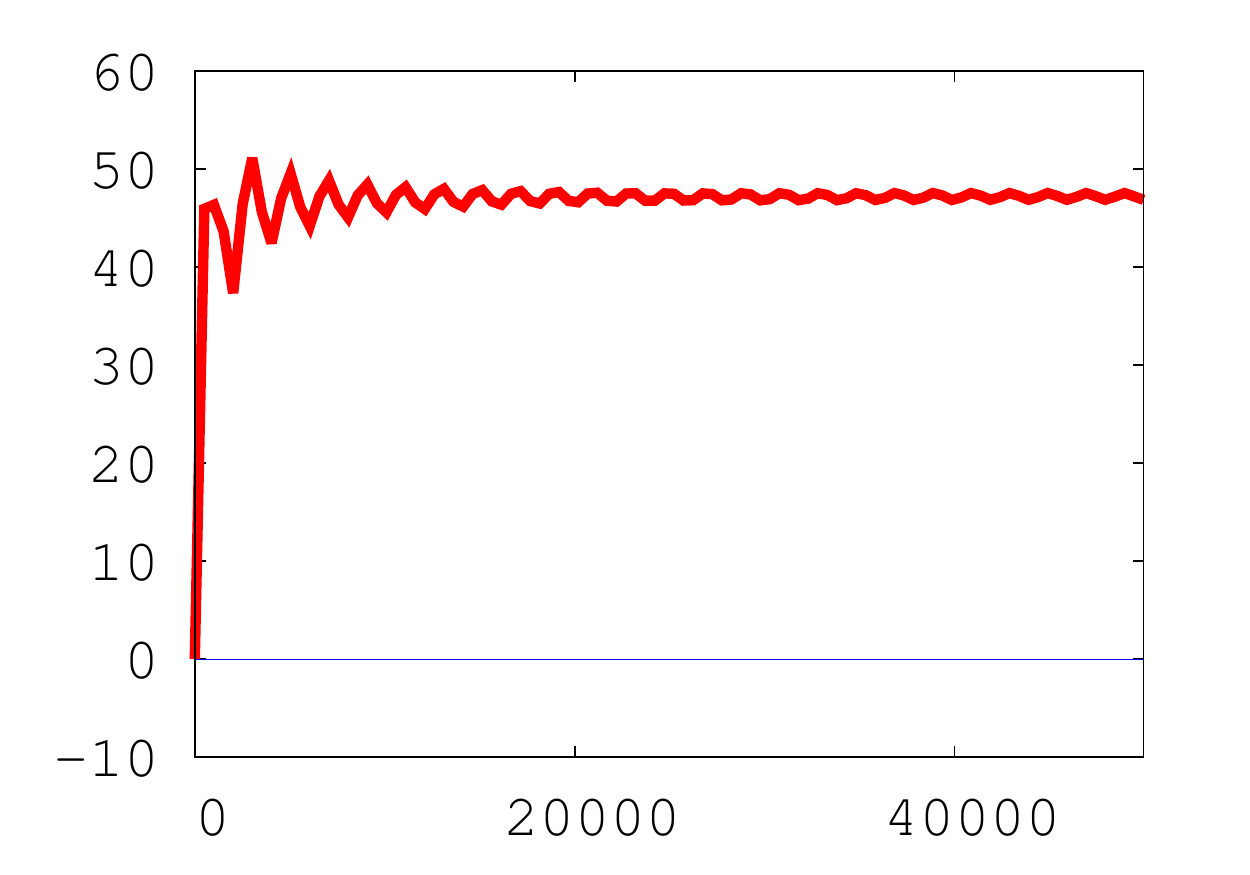}
&
\includegraphics[trim=10bp 25bp 30bp 10bp,clip,width=.15\linewidth]{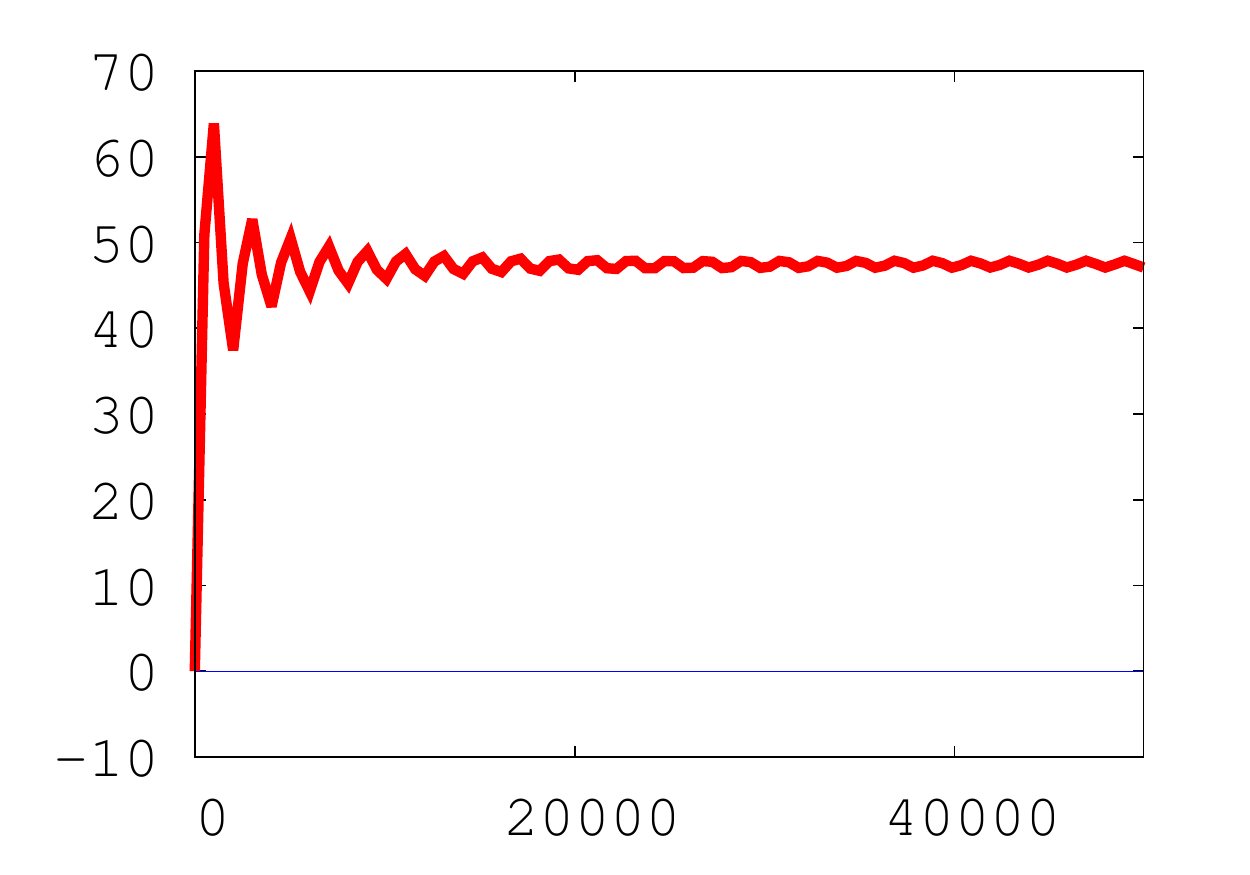}
&
\includegraphics[trim=10bp 25bp 30bp 10bp,clip,width=.15\linewidth]{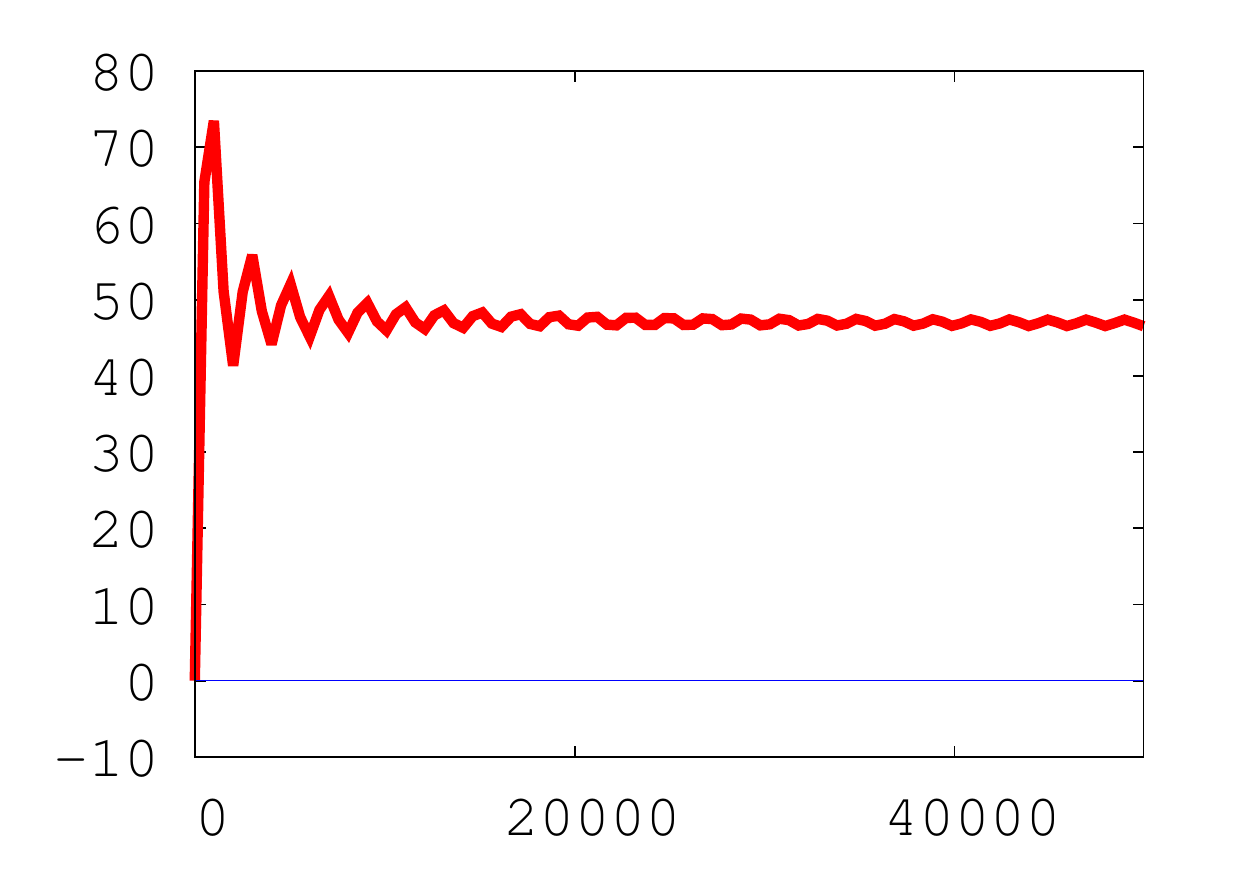}
&
\includegraphics[trim=10bp 25bp 30bp 10bp,clip,width=.15\linewidth]{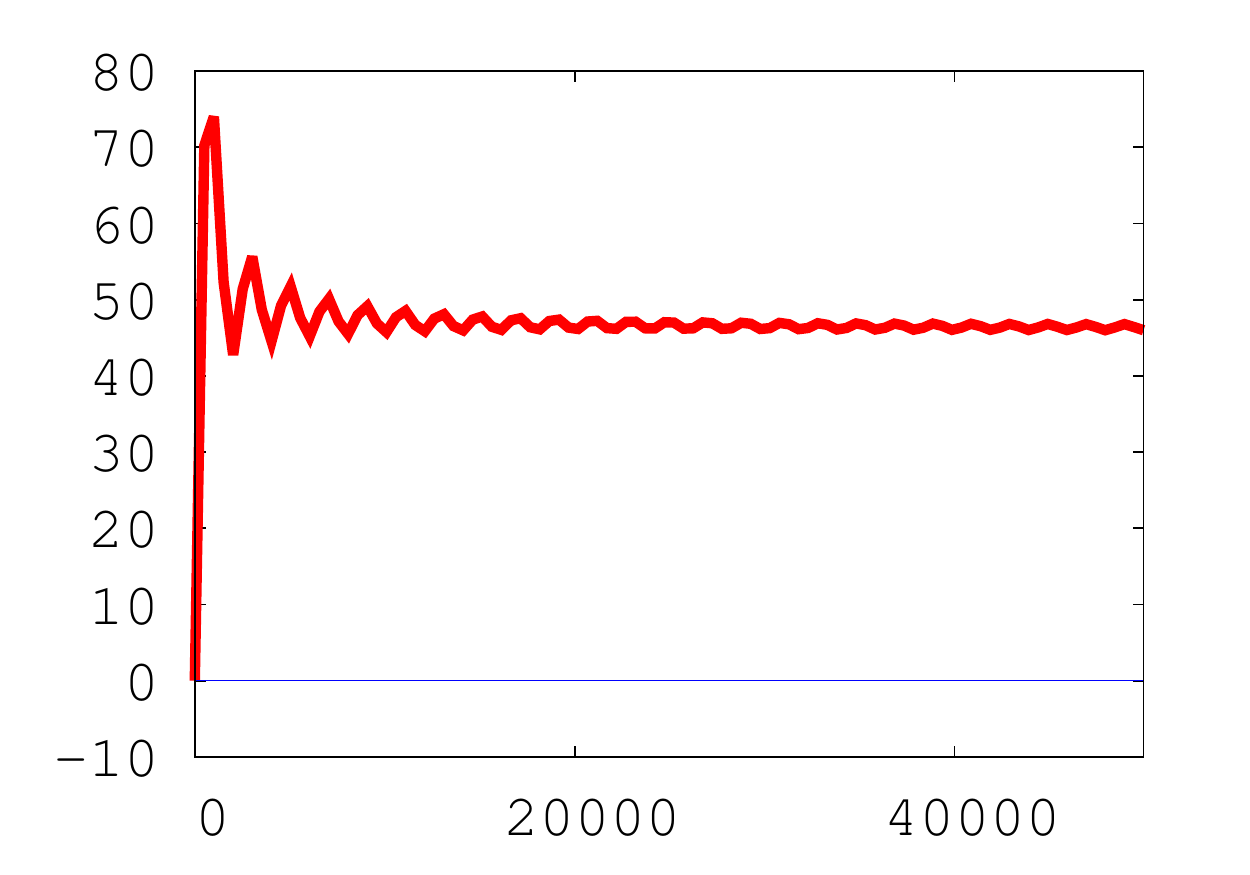}
&
\includegraphics[trim=10bp 25bp 30bp 10bp,clip,width=.15\linewidth]{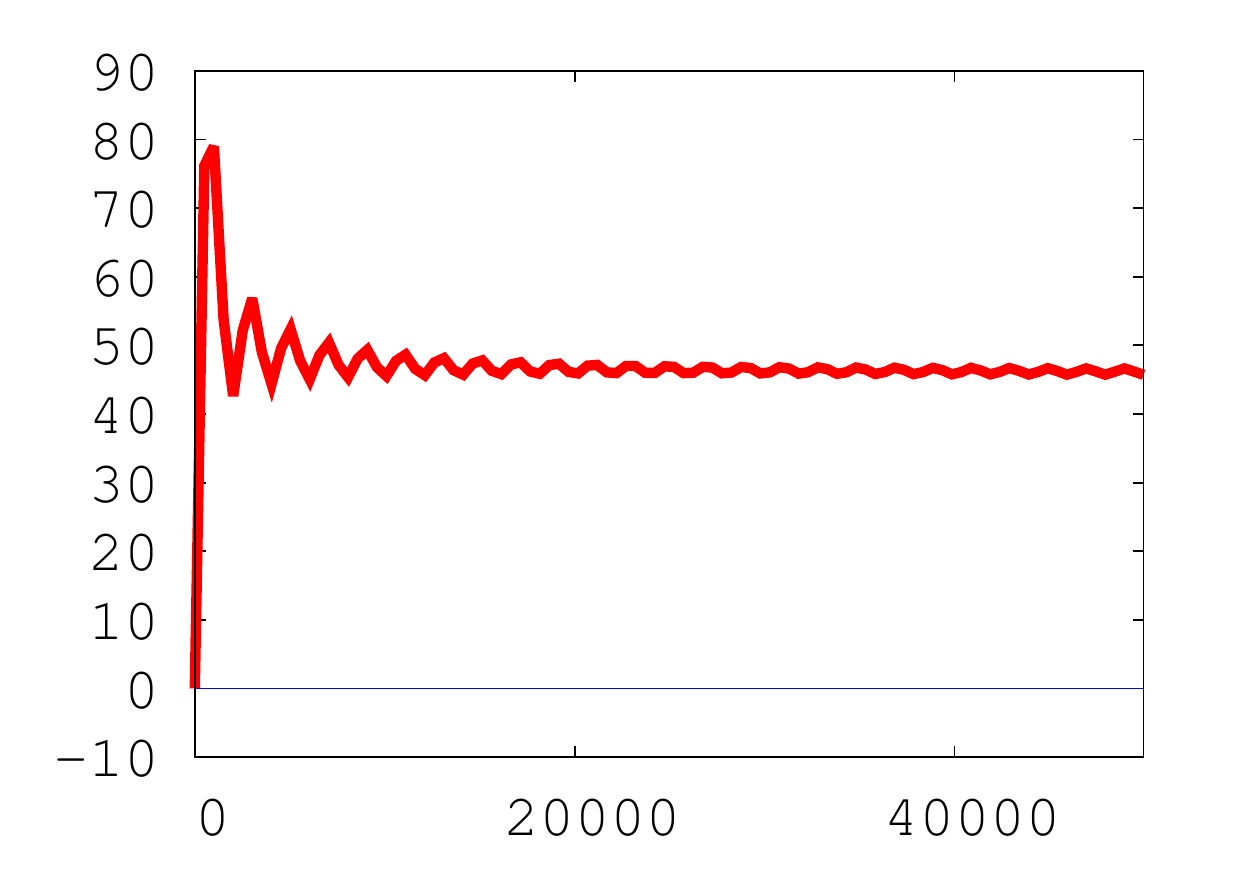}
&
\includegraphics[trim=10bp 25bp 30bp 10bp,clip,width=.15\linewidth]{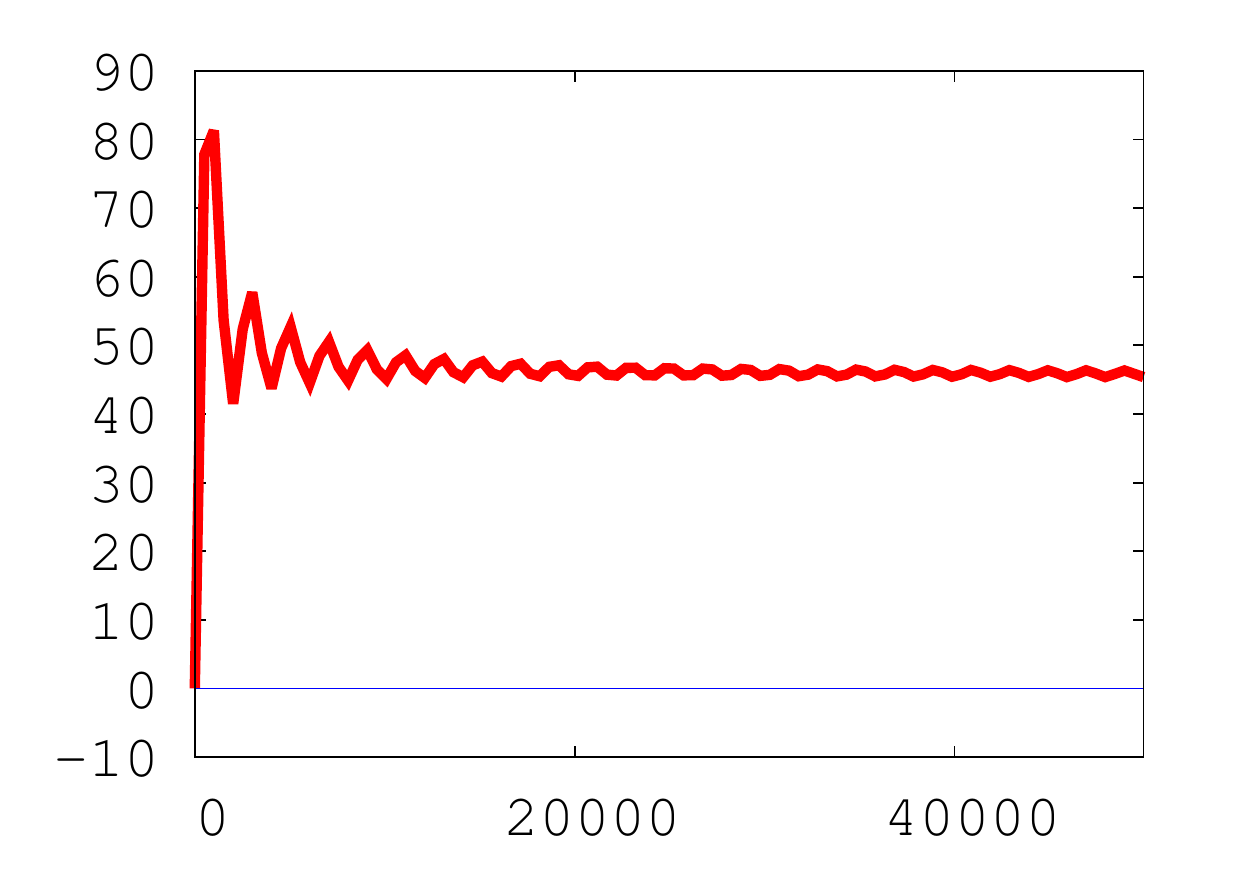}
\\
$\rho = 0.7$ & $\rho = 0.8$ & $\rho = 0.9$ & $\rho = \mbox{\textbf{1.0}}$ & $\rho = 1.1$ & $\rho = 1.2$ \\\hline
\includegraphics[trim=10bp 25bp 30bp 10bp,clip,width=.15\linewidth]{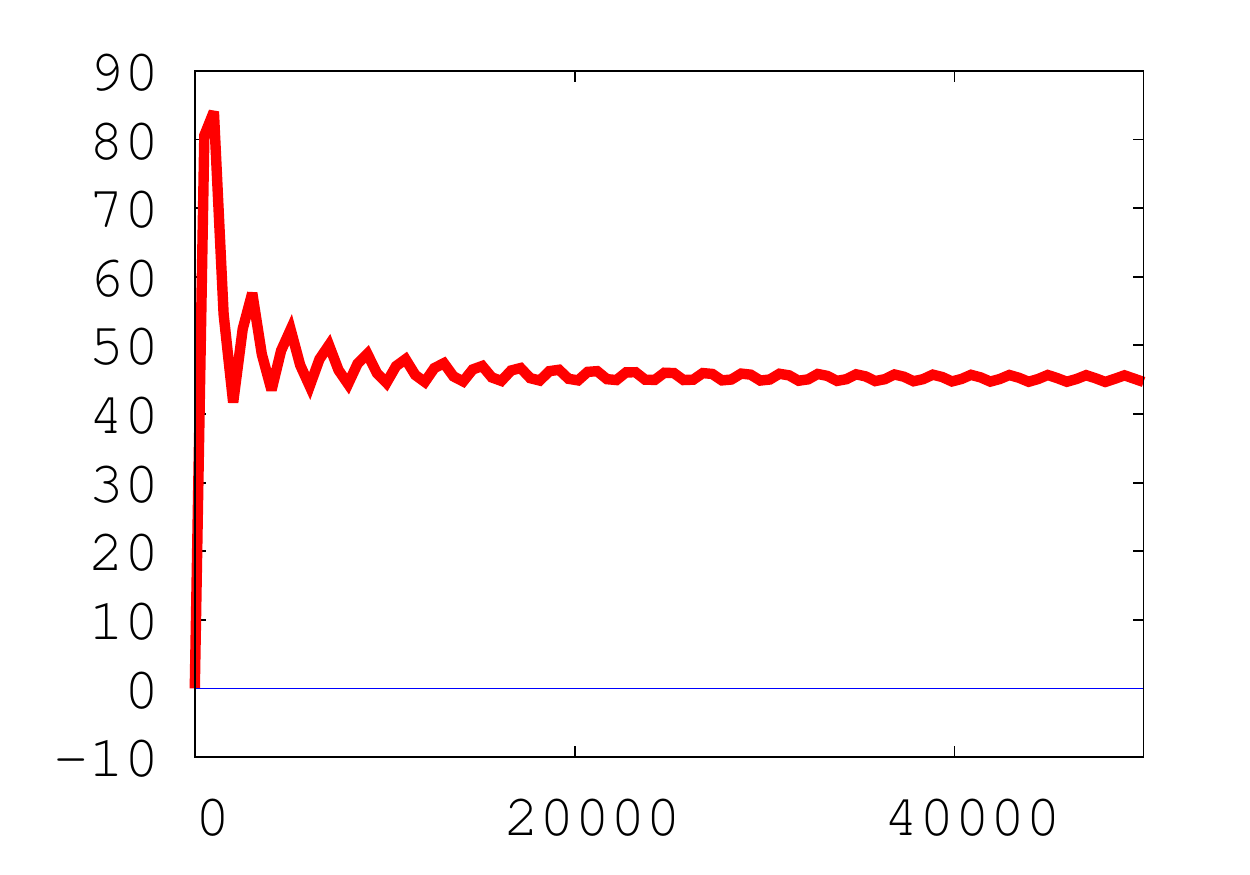}
&
\includegraphics[trim=10bp 25bp 30bp 10bp,clip,width=.15\linewidth]{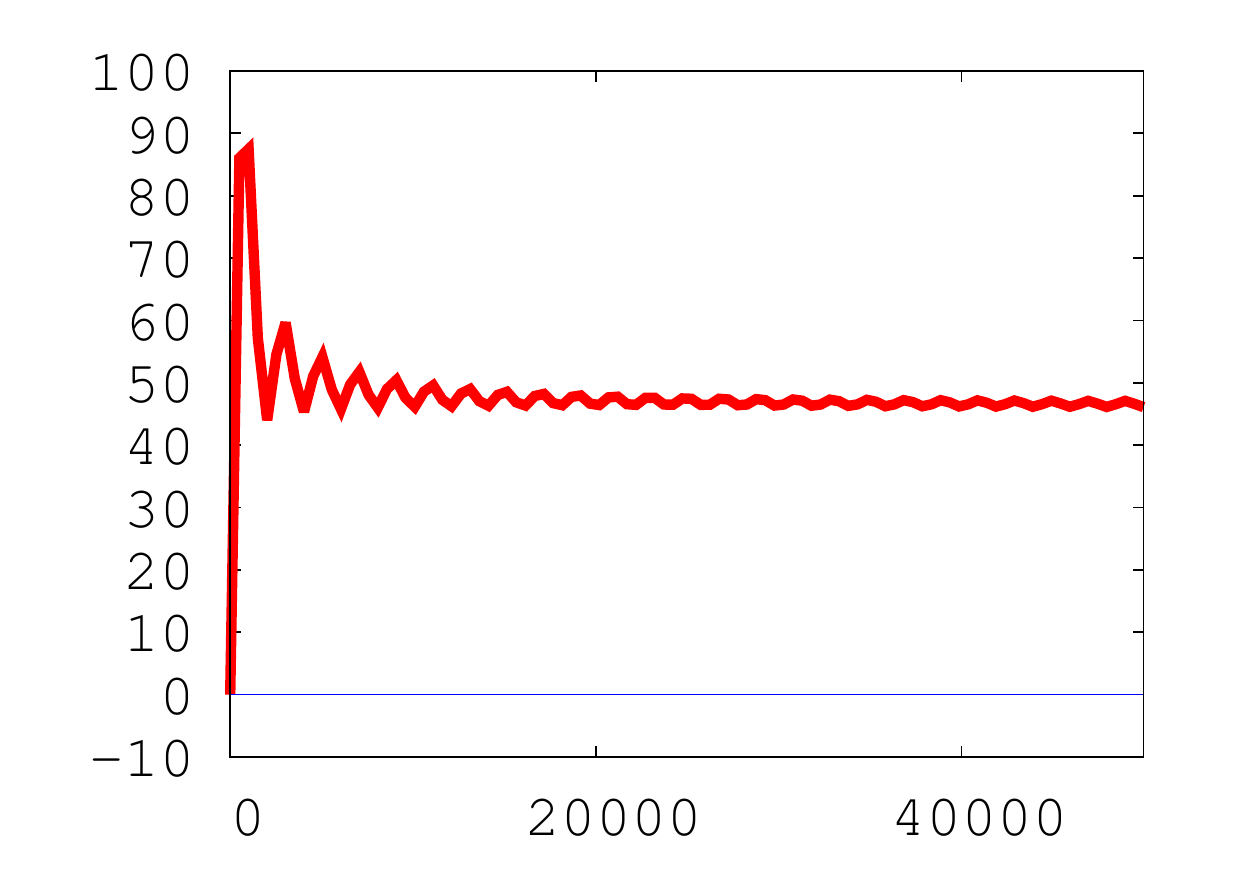}
&
\includegraphics[trim=10bp 25bp 30bp 10bp,clip,width=.15\linewidth]{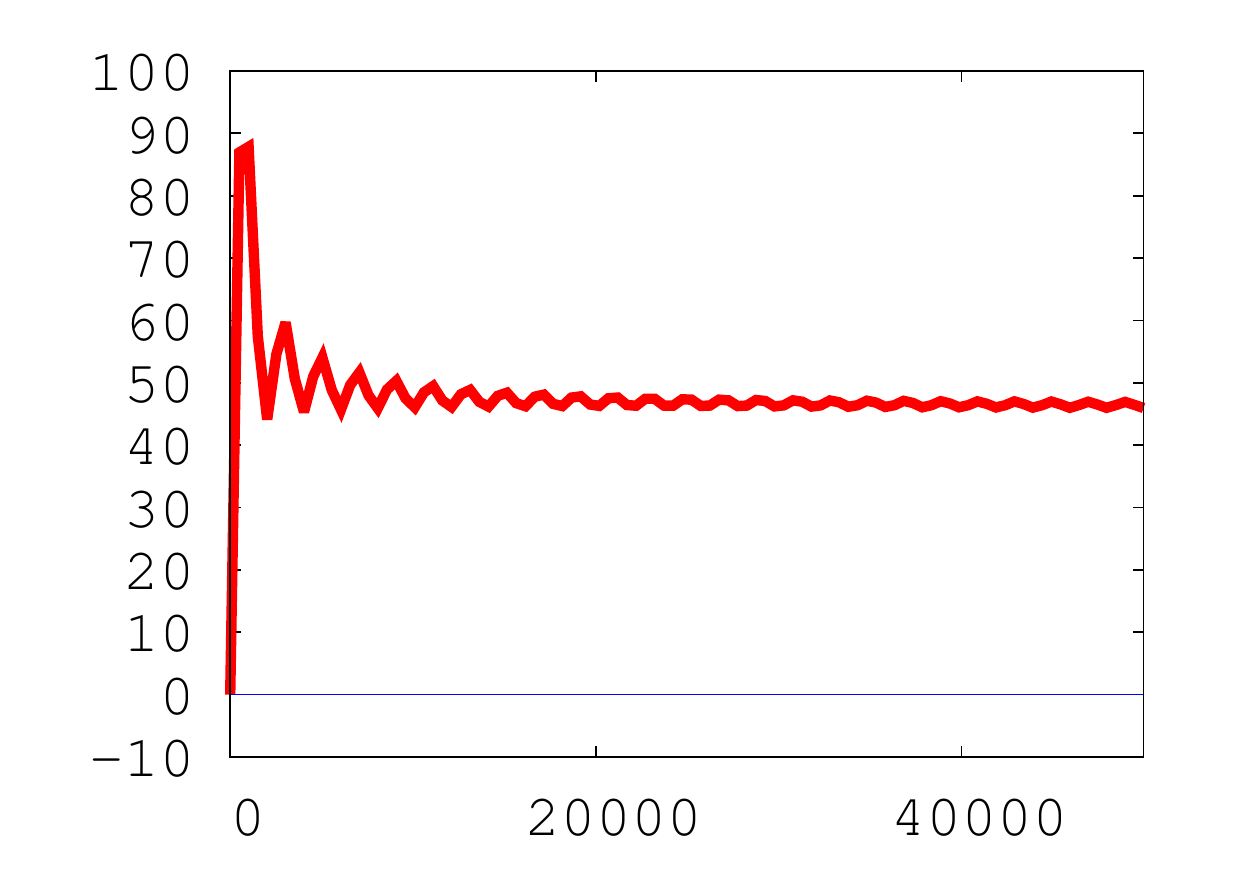}
&
\includegraphics[trim=10bp 25bp 30bp 10bp,clip,width=.15\linewidth]{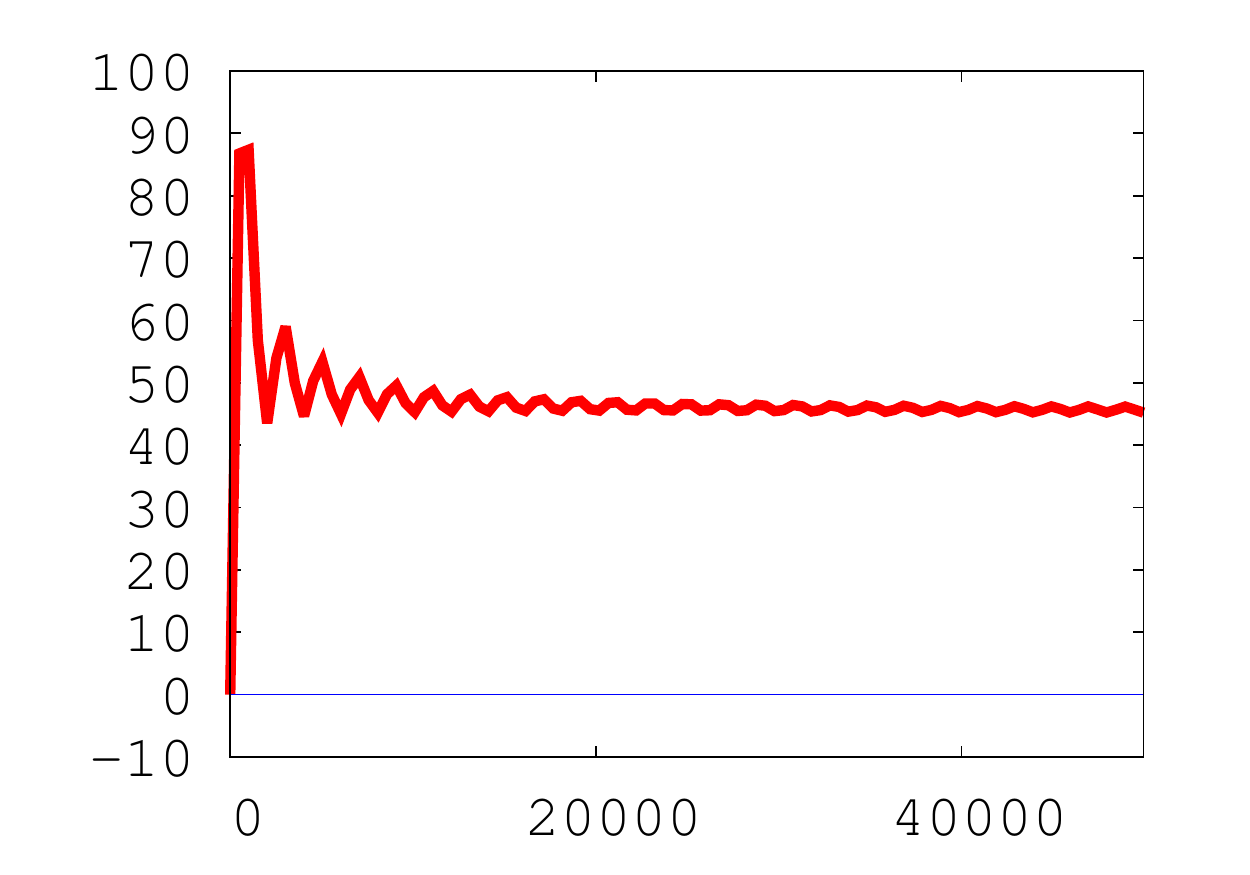}
&
\includegraphics[trim=10bp 25bp 30bp
10bp,clip,width=.15\linewidth]{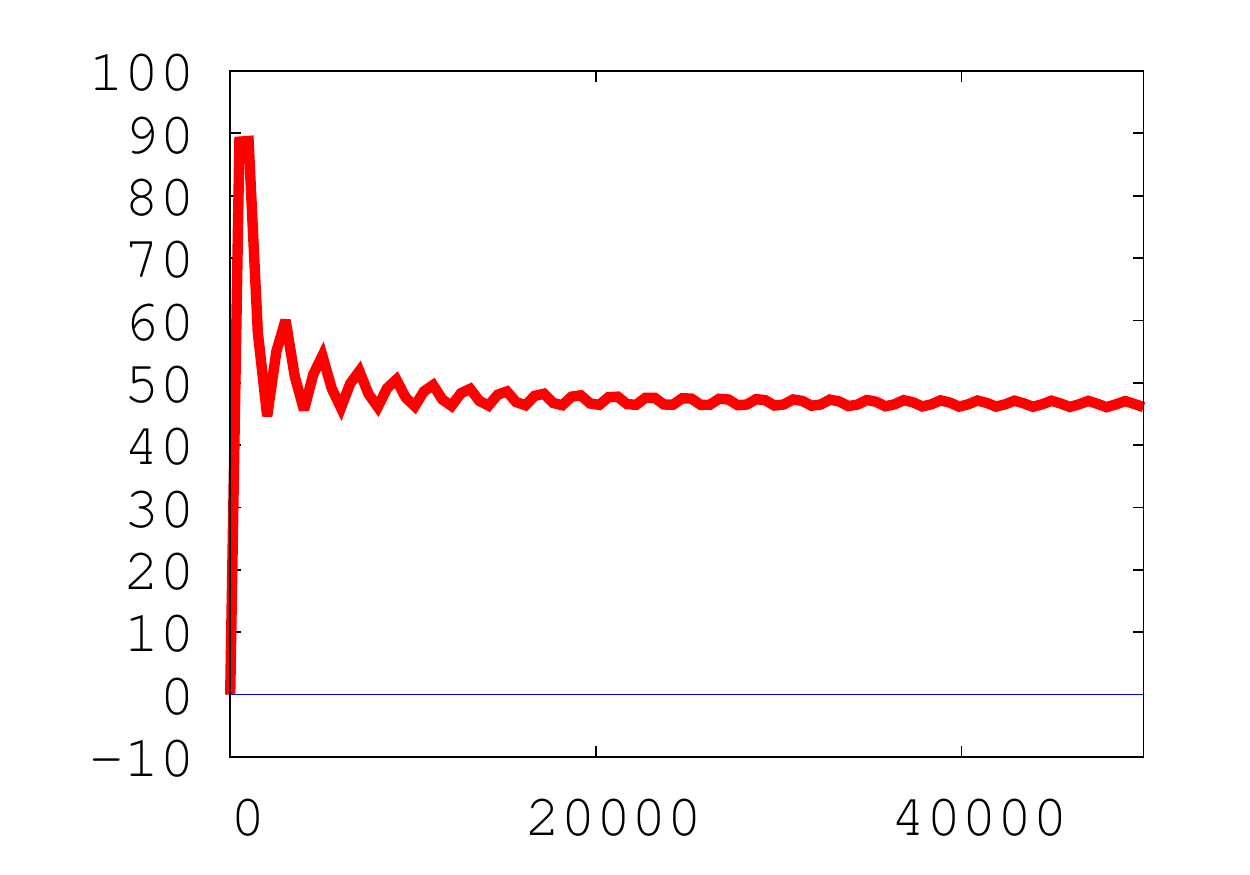}
&\\
$\rho =1.3$ & $\rho = 1.4$ & $\rho = 1.5$ & $\rho = 1.6$ & $\rho =
1.7$ &  \\\hline \hline
\end{tabular}
}
\end{center}
\caption{Error($p$-LMS) - Error(DN-$p$-LMS) as a function of $t$ ($\in \{1, 2, ..., 50 000\}$), $\bm{u}$ = sparse, $(p,q) = (1.17, 6.9)$.}
  \label{tc1_supp_rr2}
\end{sidewaystable}

\begin{sidewaystable}[t]
\begin{center}
{\small
\begin{tabular}{cccccc}\hline \hline
\includegraphics[trim=10bp 30bp 30bp 10bp,clip,width=.15\linewidth]{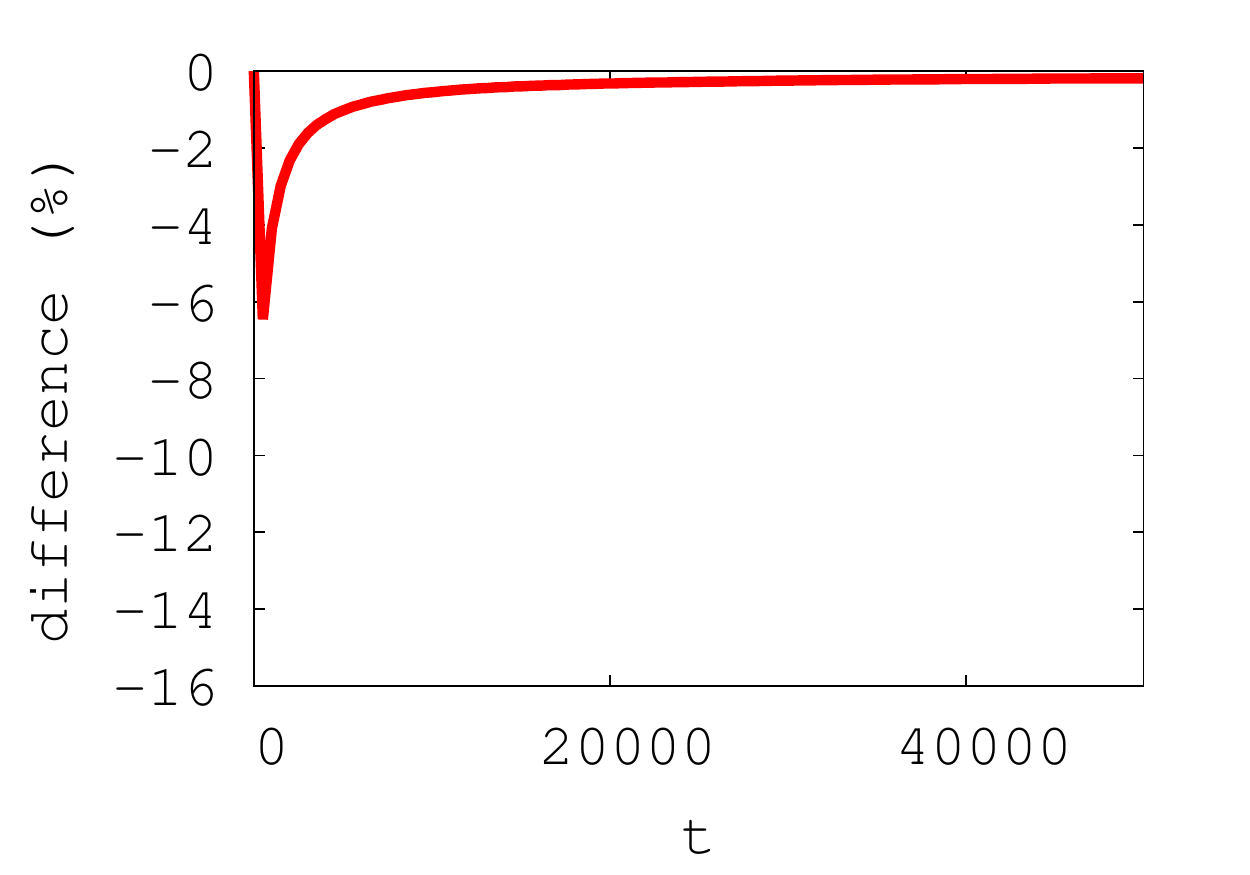}
&
\includegraphics[trim=10bp 25bp 30bp 10bp,clip,width=.15\linewidth]{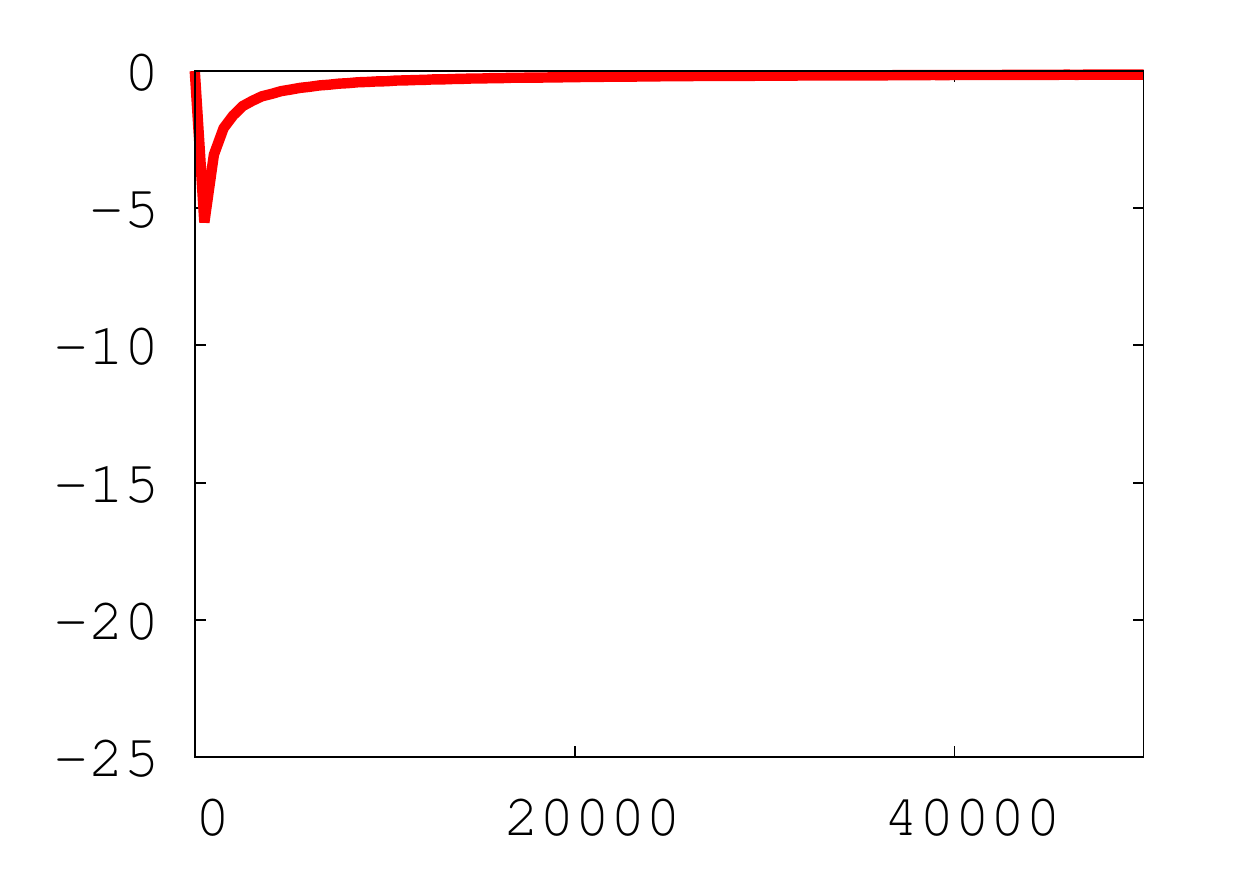}
&
\includegraphics[trim=10bp 25bp 30bp 10bp,clip,width=.15\linewidth]{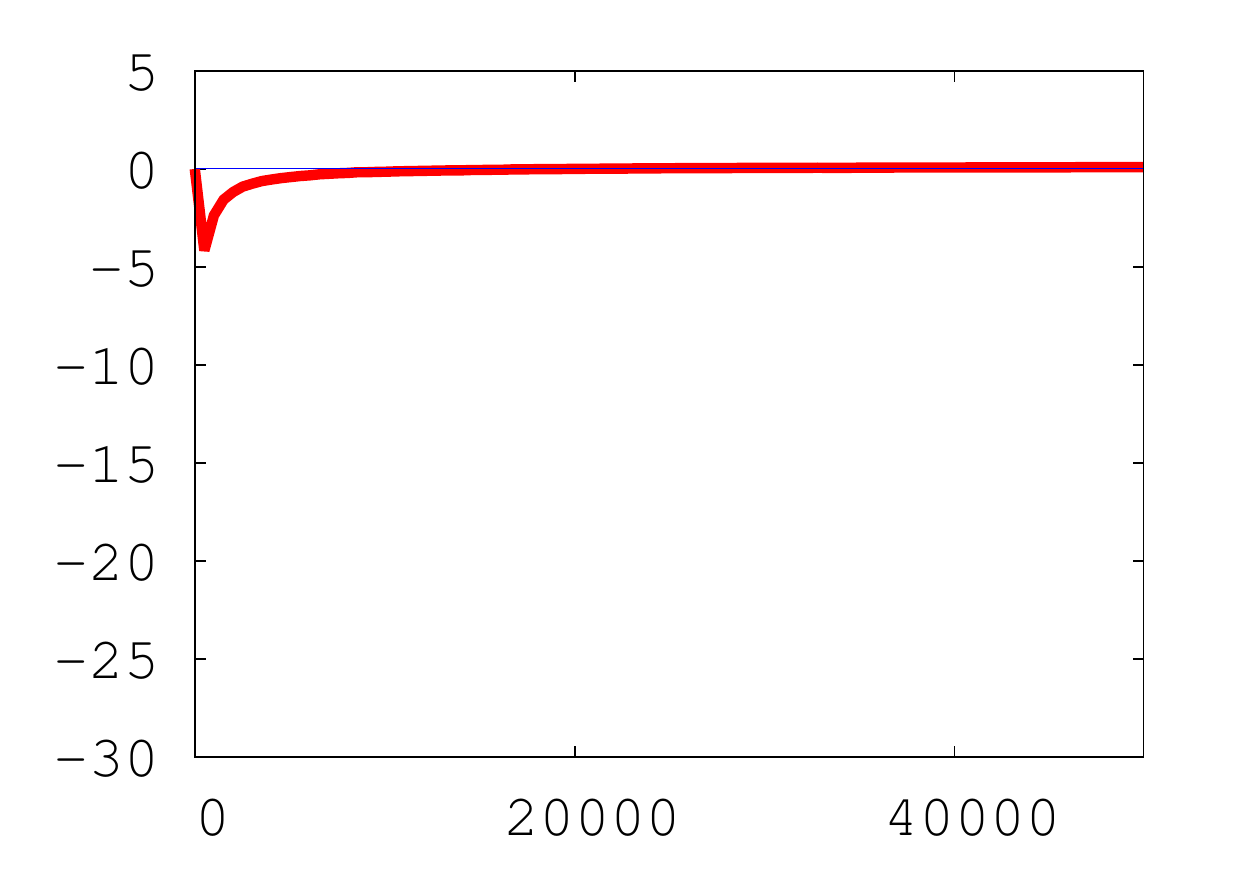}
&
\includegraphics[trim=10bp 25bp 30bp 10bp,clip,width=.15\linewidth]{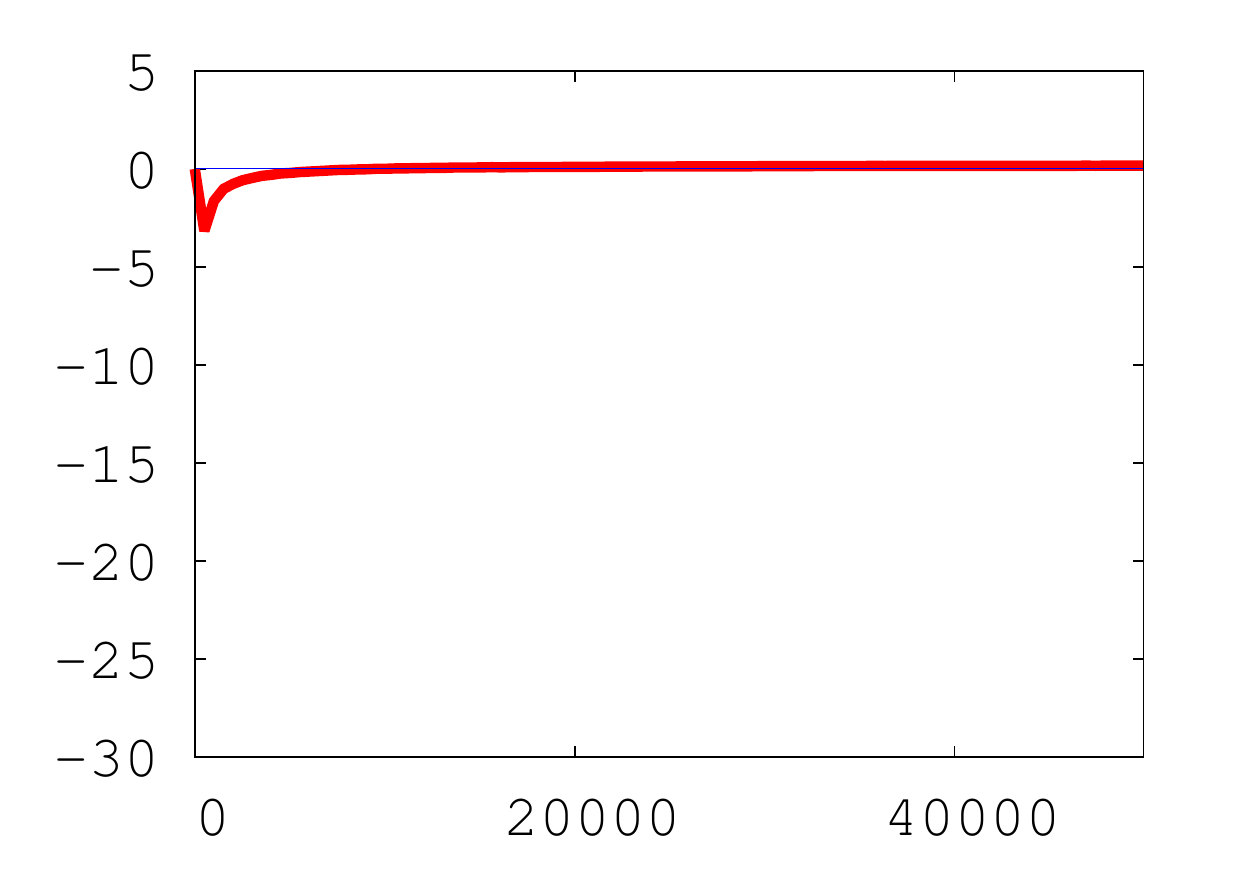}
&
\includegraphics[trim=10bp 25bp 30bp 10bp,clip,width=.15\linewidth]{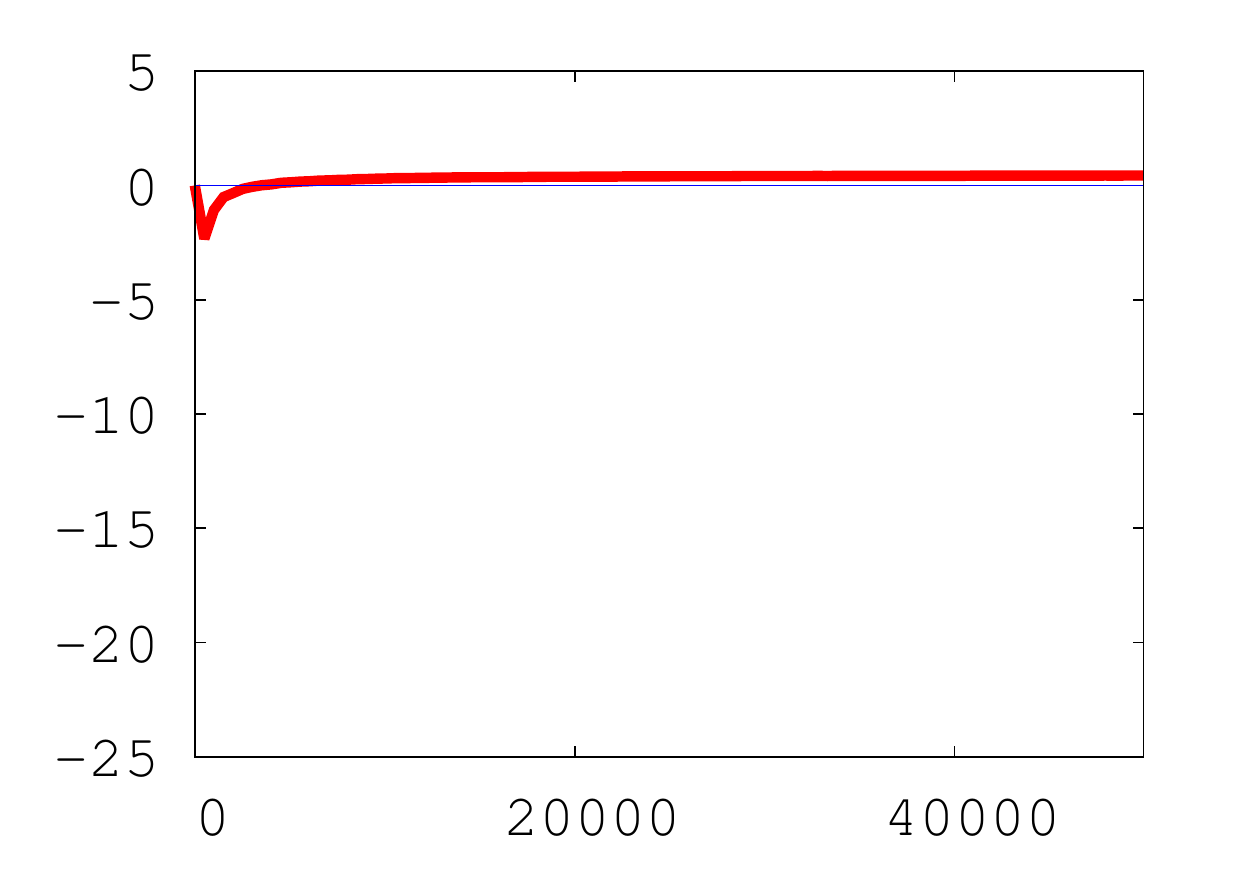}
&
\includegraphics[trim=10bp 25bp 30bp 10bp,clip,width=.15\linewidth]{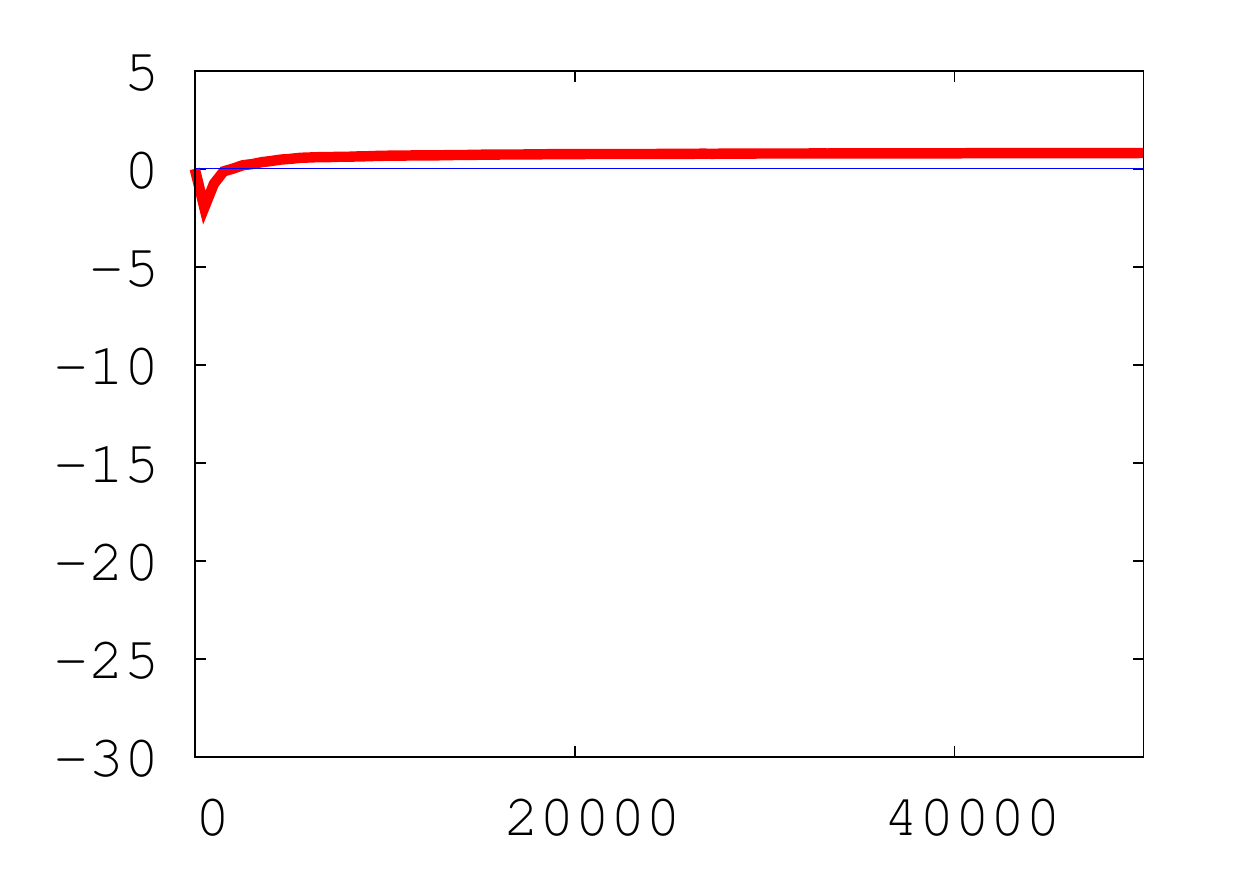}
\\
$\rho = 0.1$ & $\rho = 0.2$ & $\rho = 0.3$ & $\rho = 0.4$ & $\rho = 0.5$ & $\rho = 0.6$ \\\hline
\includegraphics[trim=10bp 25bp 30bp 10bp,clip,width=.15\linewidth]{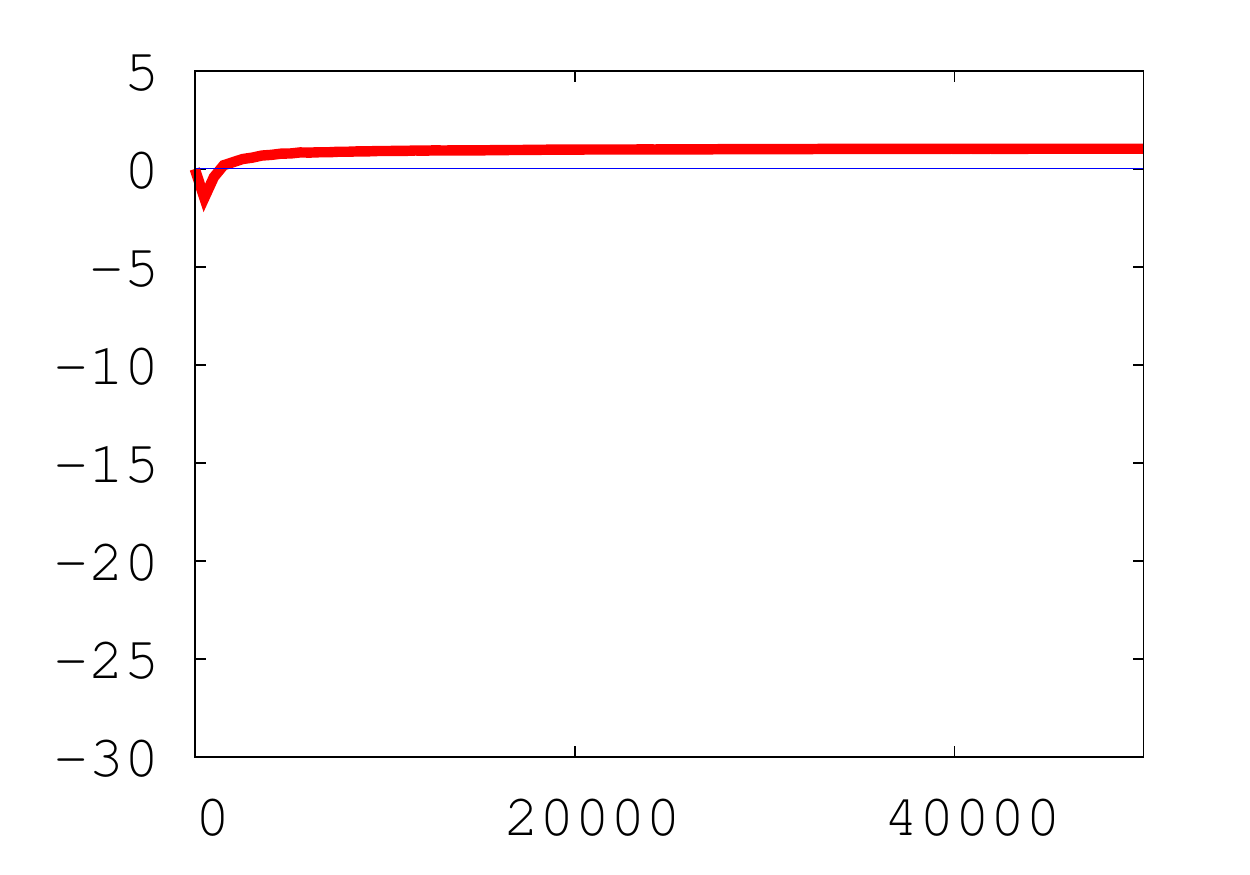}
&
\includegraphics[trim=10bp 25bp 30bp 10bp,clip,width=.15\linewidth]{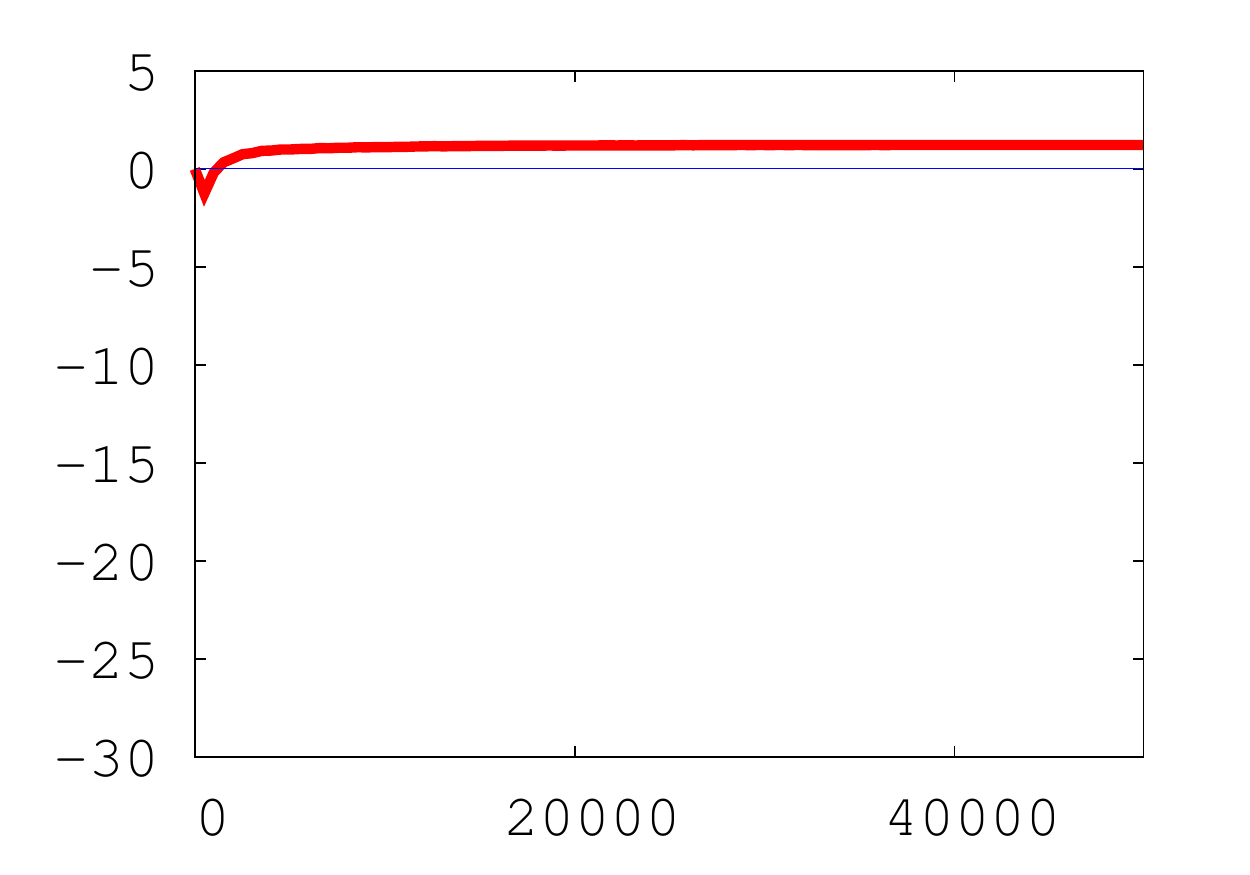}
&
\includegraphics[trim=10bp 25bp 30bp 10bp,clip,width=.15\linewidth]{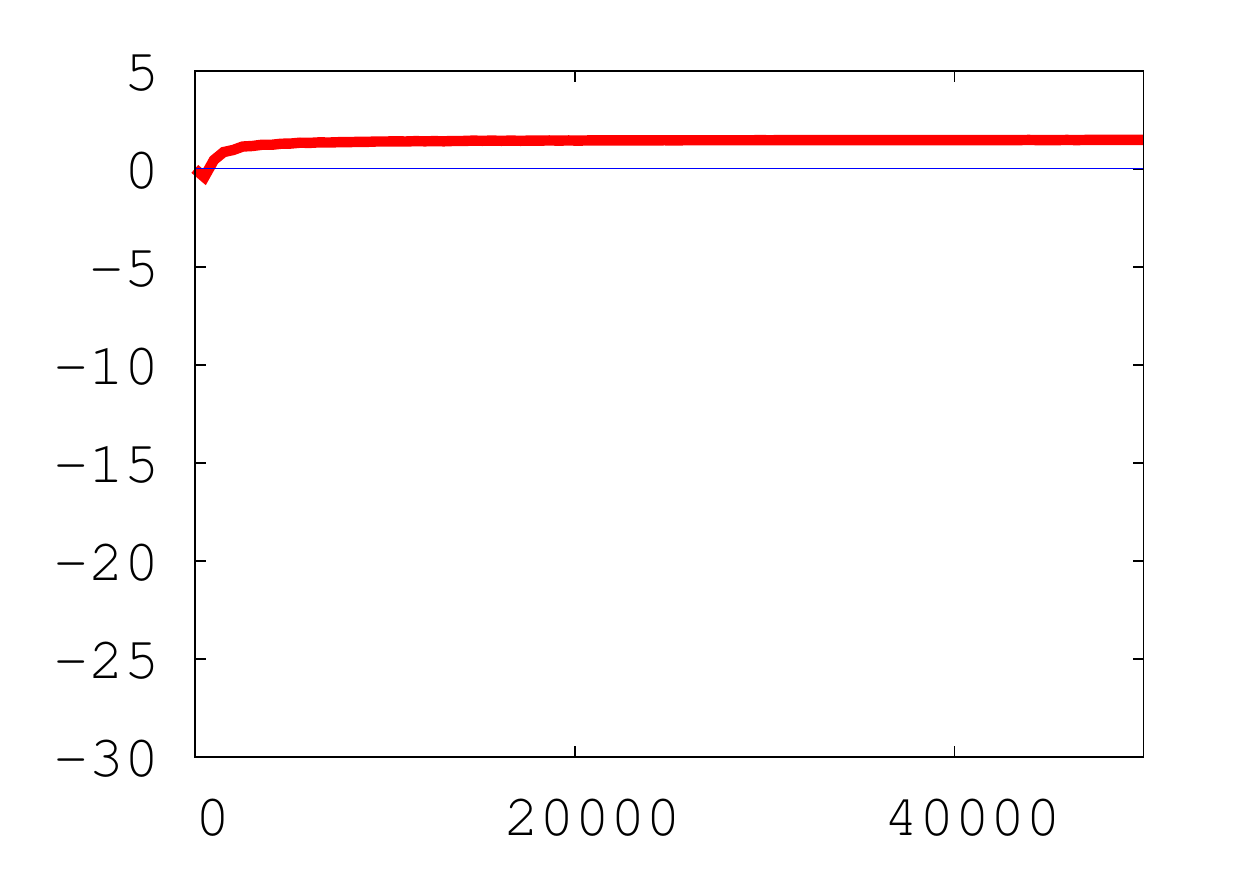}
&
\includegraphics[trim=10bp 25bp 30bp 10bp,clip,width=.15\linewidth]{results_P2_00_Q2_00_UCHOICE_GAUSS_UR_1_00_NOLABELS-eps-converted-to}
&
\includegraphics[trim=10bp 25bp 30bp 10bp,clip,width=.15\linewidth]{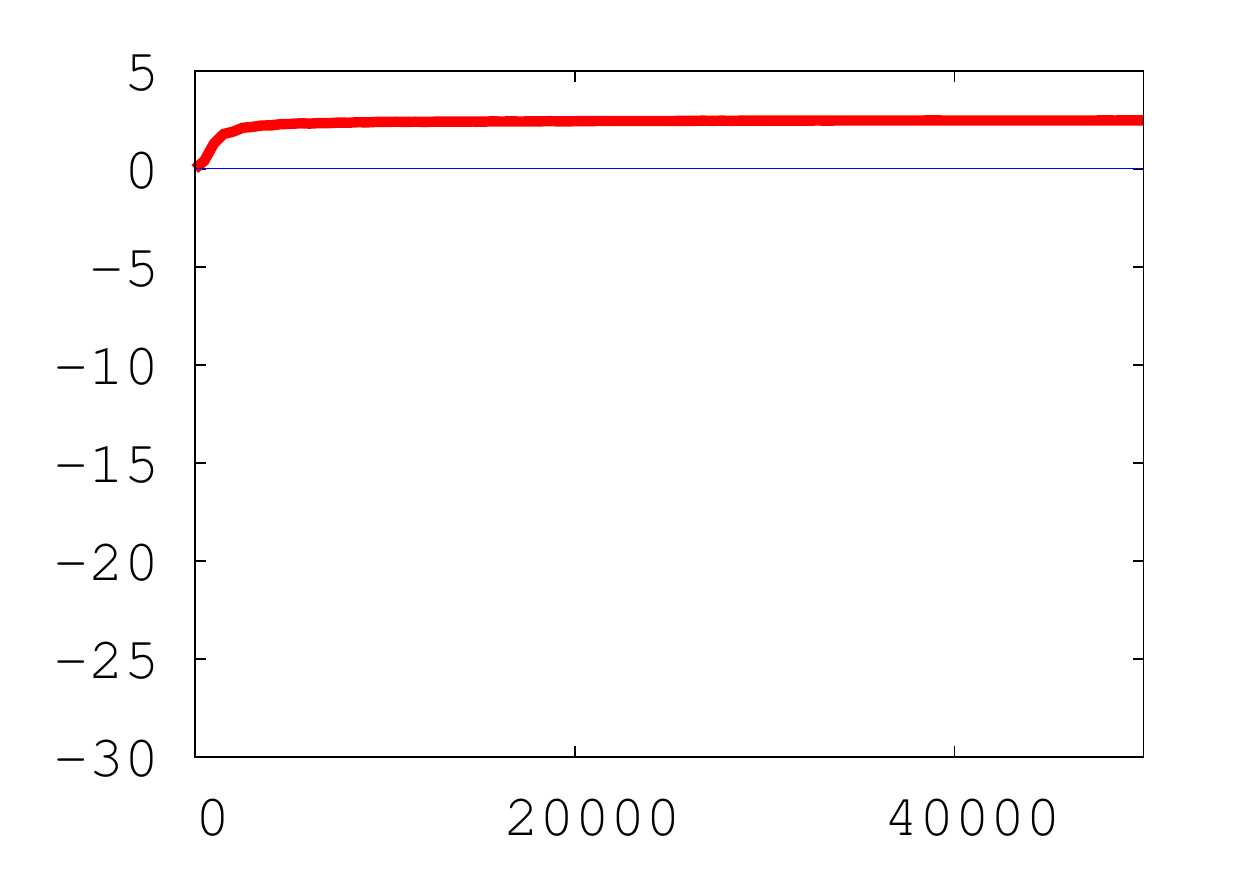}
&
\includegraphics[trim=10bp 25bp 30bp 10bp,clip,width=.15\linewidth]{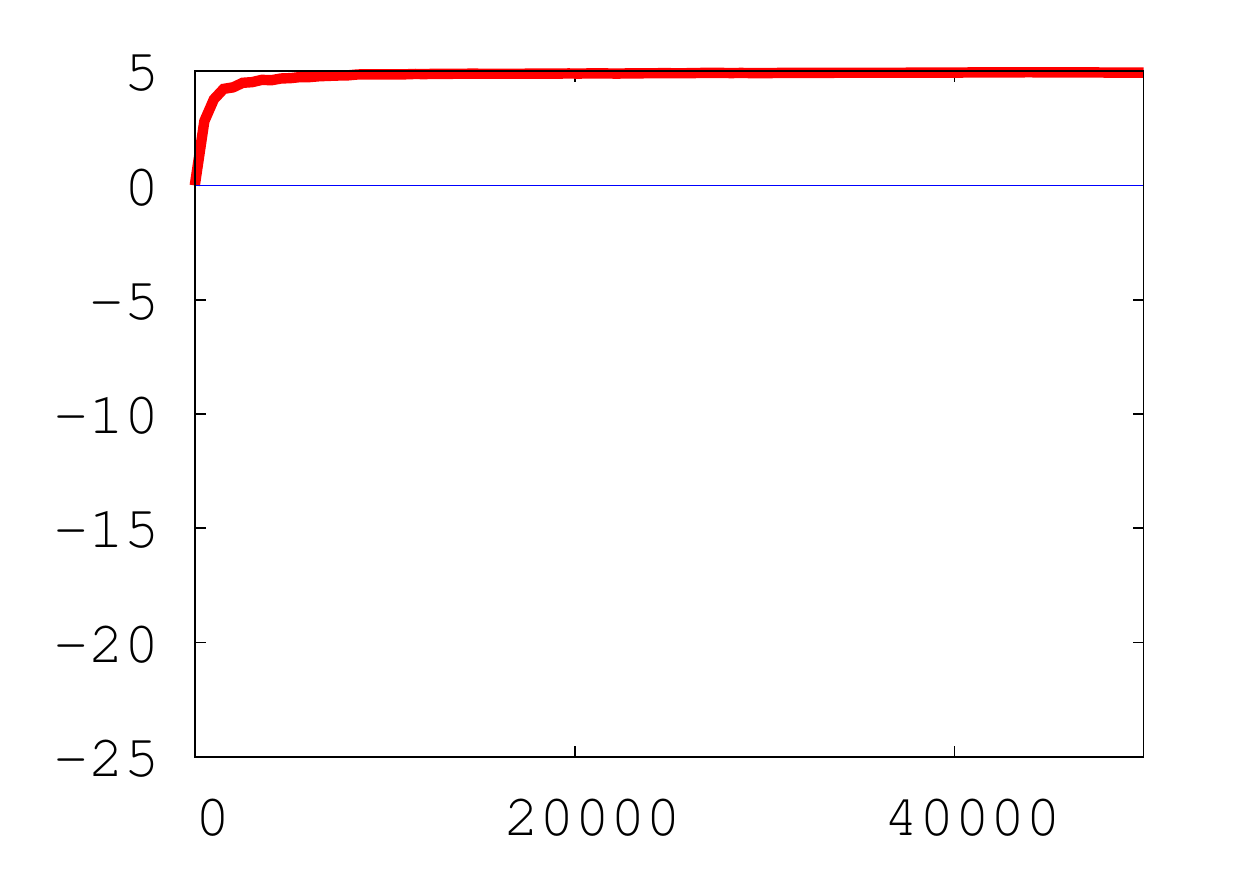}
\\
$\rho = 0.7$ & $\rho = 0.8$ & $\rho = 0.9$ & $\rho = \mbox{\textbf{1.0}}$ & $\rho = 1.1$ & $\rho = 1.2$ \\\hline
\includegraphics[trim=10bp 25bp 30bp 10bp,clip,width=.15\linewidth]{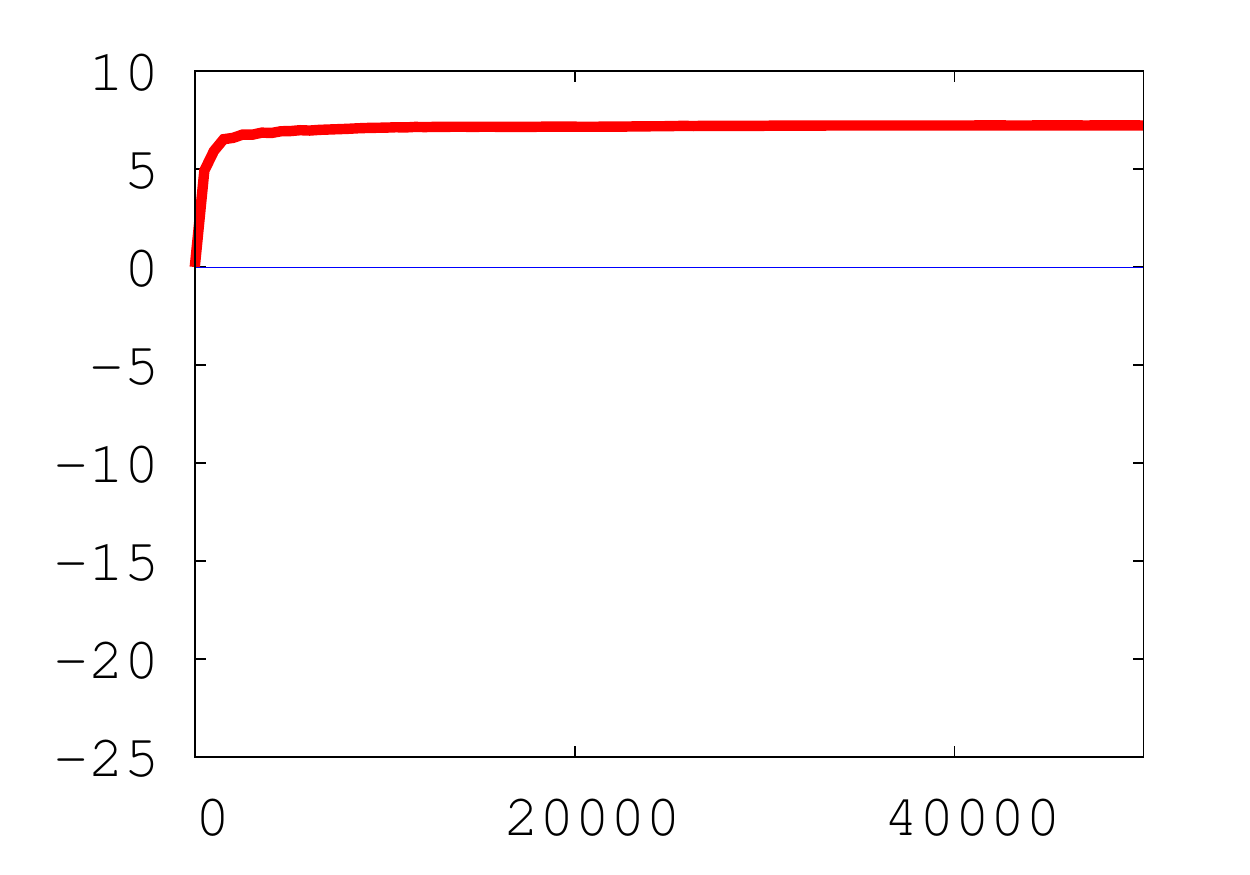}
&
\includegraphics[trim=10bp 25bp 30bp 10bp,clip,width=.15\linewidth]{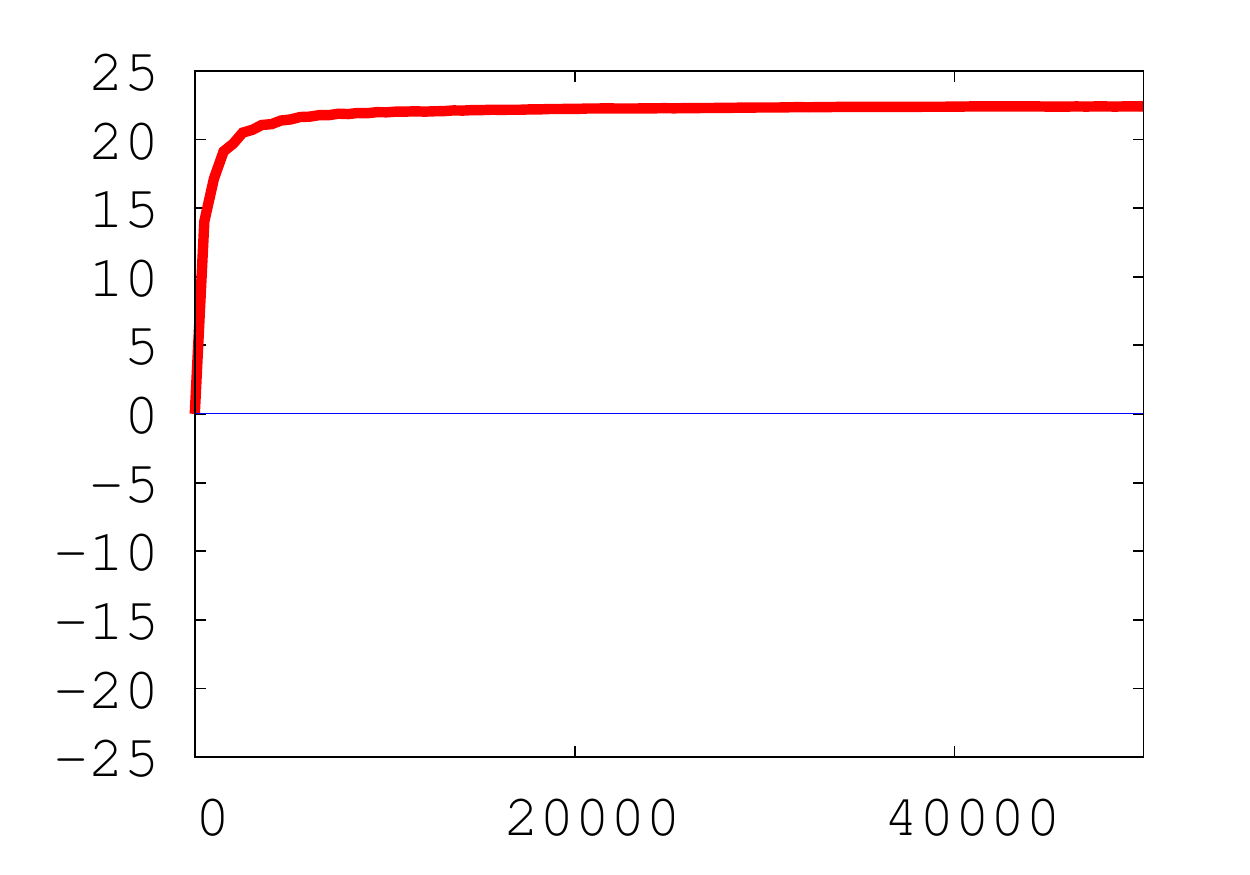}
&
\includegraphics[trim=10bp 25bp 30bp 10bp,clip,width=.15\linewidth]{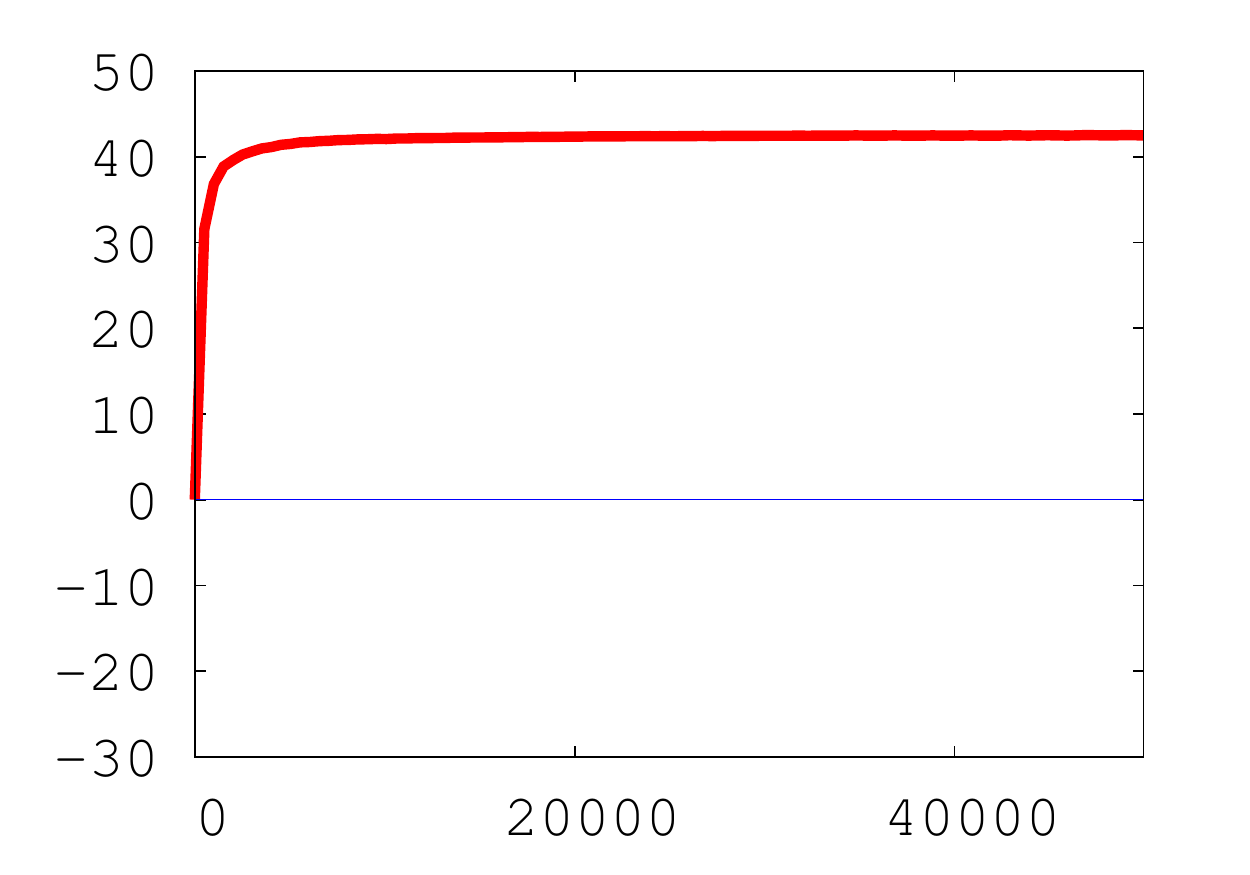}
&
\includegraphics[trim=10bp 25bp 30bp 10bp,clip,width=.15\linewidth]{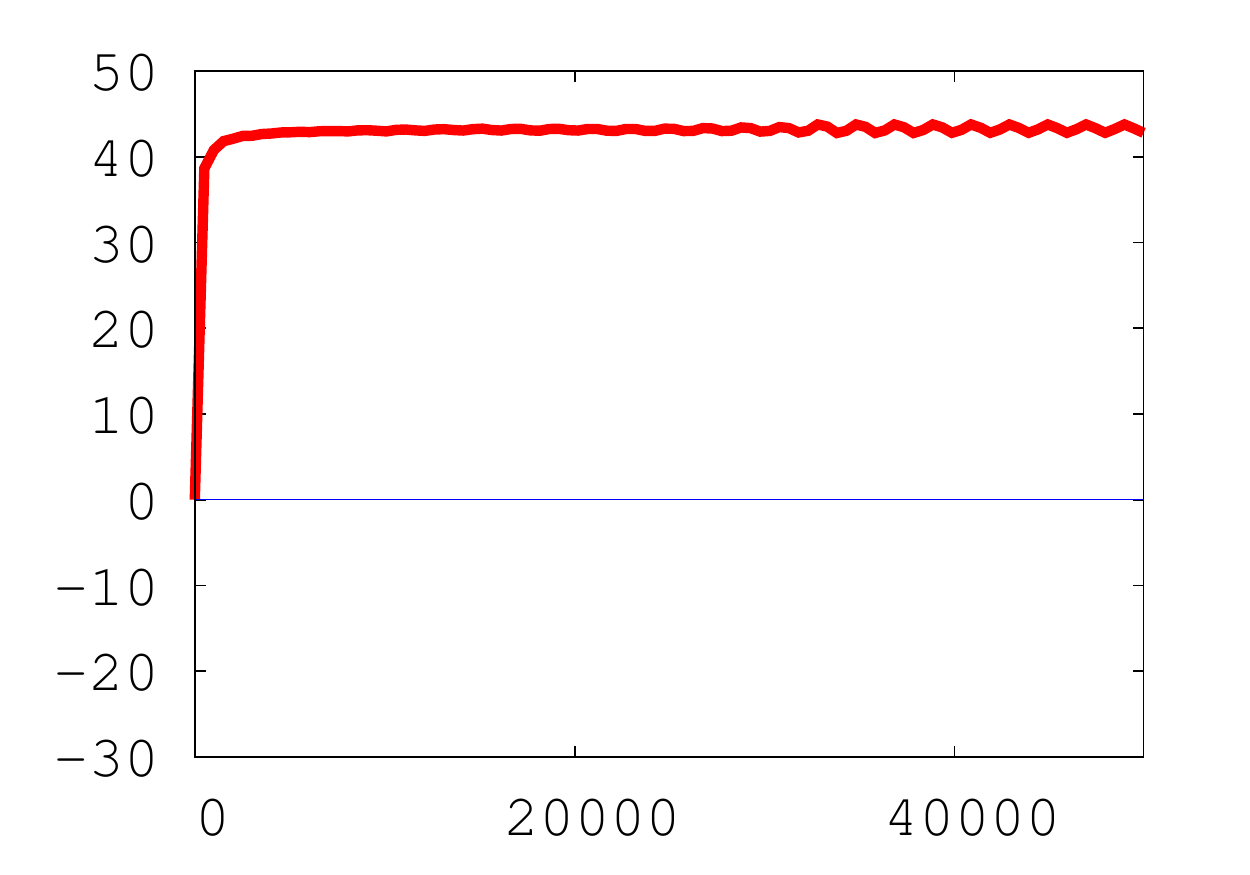}
&
\includegraphics[trim=10bp 25bp 30bp
10bp,clip,width=.15\linewidth]{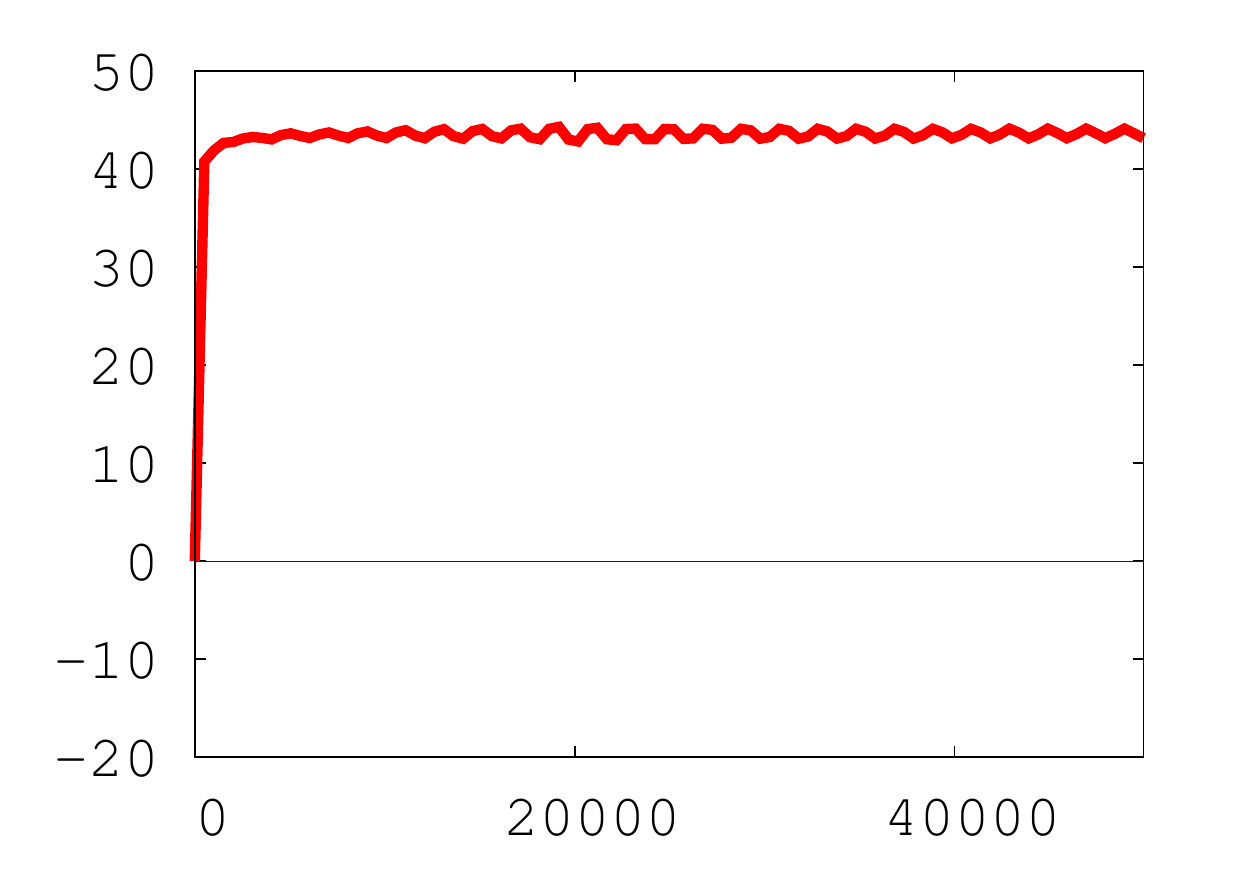}
&\\
$\rho =1.3$ & $\rho = 1.4$ & $\rho = 1.5$ & $\rho = 1.6$ & $\rho =
1.7$ &  \\\hline \hline
\end{tabular}
}
\end{center}
\caption{Error($p$-LMS) - Error(DN-$p$-LMS) as a function of $t$ ($\in \{1, 2, ..., 50 000\}$), $\bm{u}$ = dense, $(p,q) = (2.0, 2.0)$.}
  \label{tc1_supp_rr3}
\end{sidewaystable}

\begin{sidewaystable}[t]
\begin{center}
{\small
\begin{tabular}{cccccc}\hline \hline
\includegraphics[trim=10bp 30bp 30bp 10bp,clip,width=.15\linewidth]{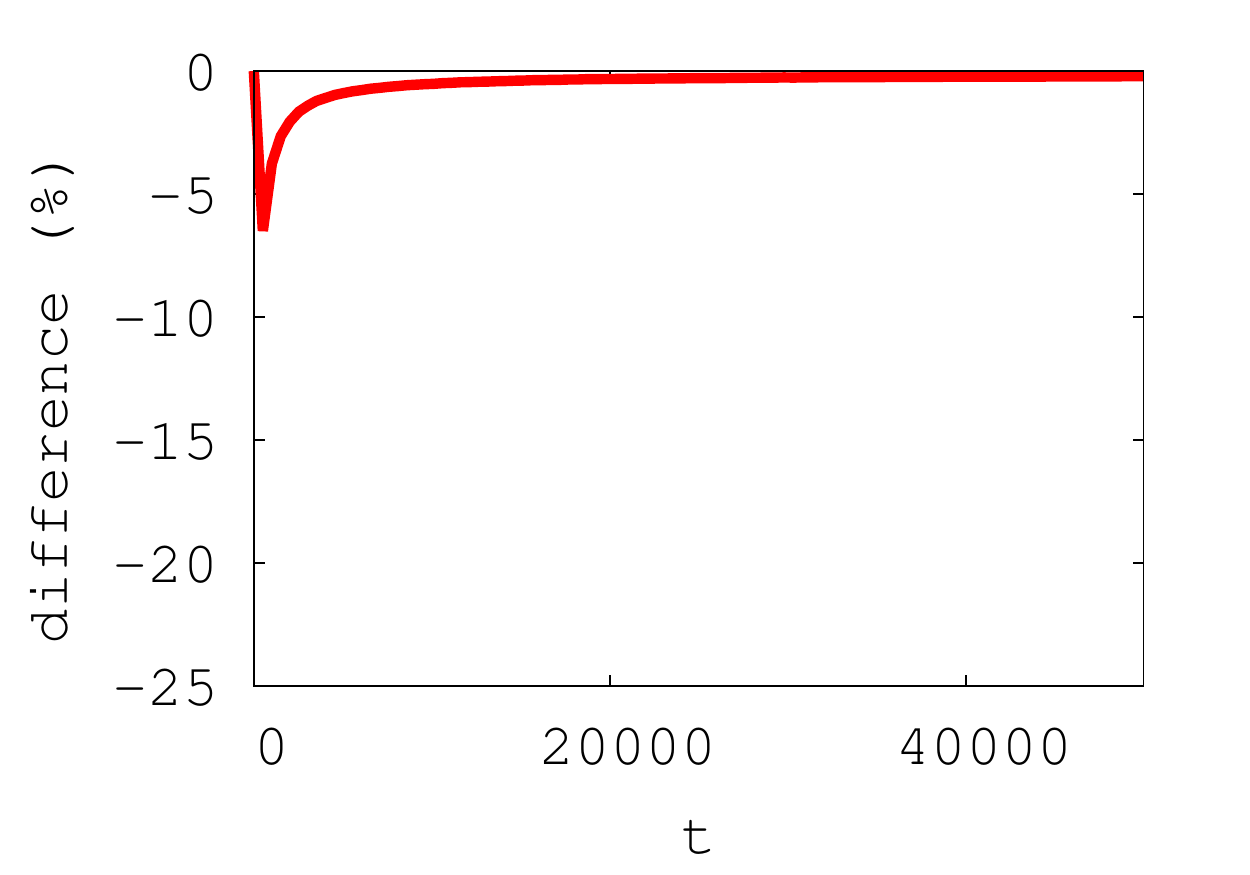}
&
\includegraphics[trim=10bp 25bp 30bp 10bp,clip,width=.15\linewidth]{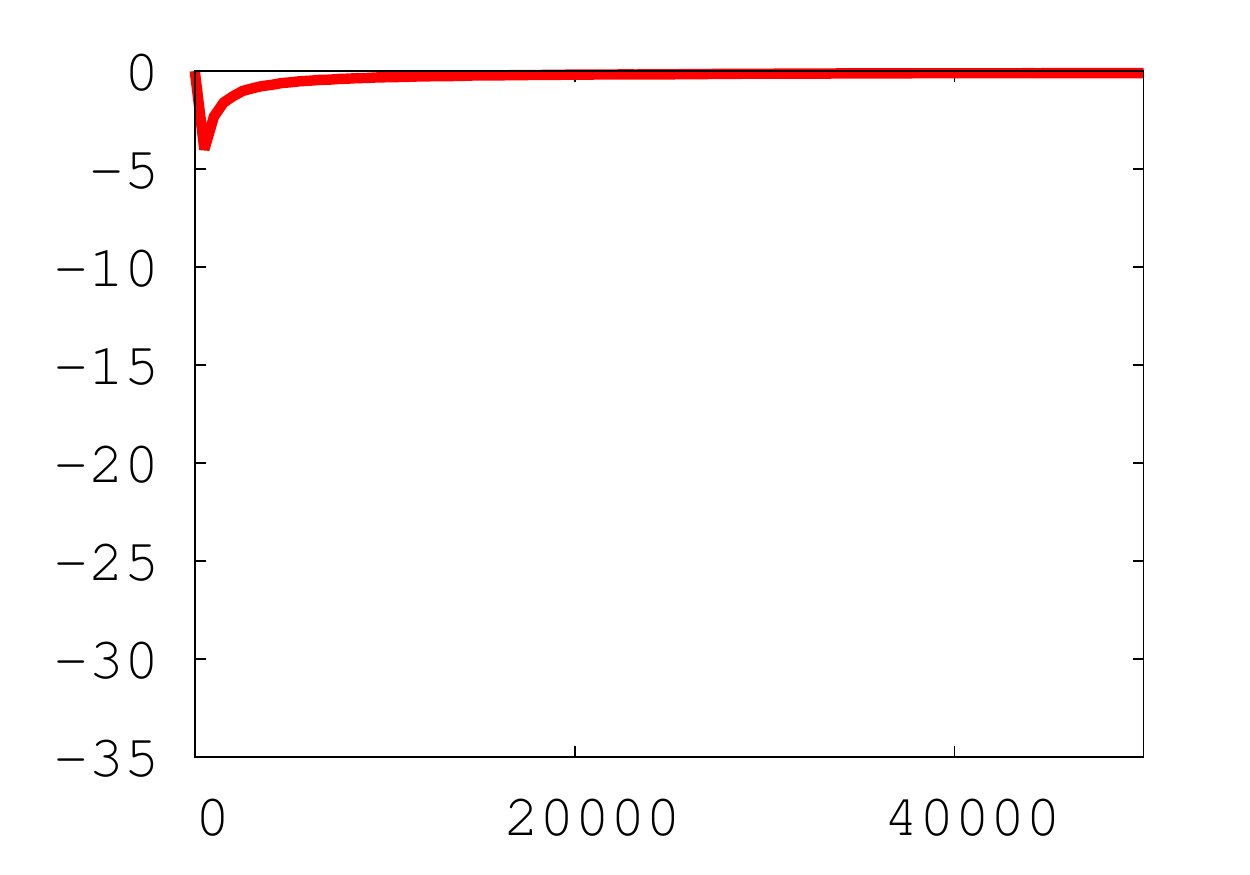}
&
\includegraphics[trim=10bp 25bp 30bp 10bp,clip,width=.15\linewidth]{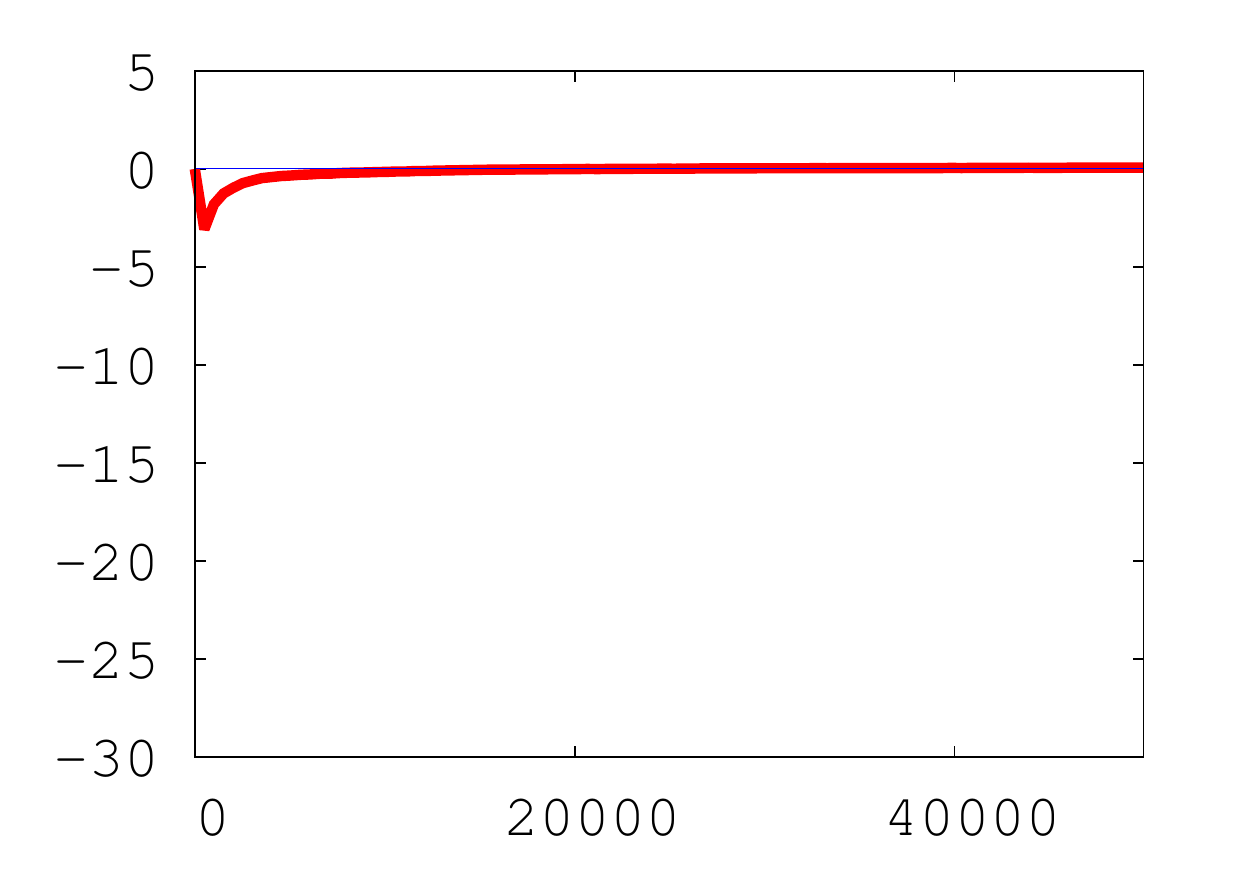}
&
\includegraphics[trim=10bp 25bp 30bp 10bp,clip,width=.15\linewidth]{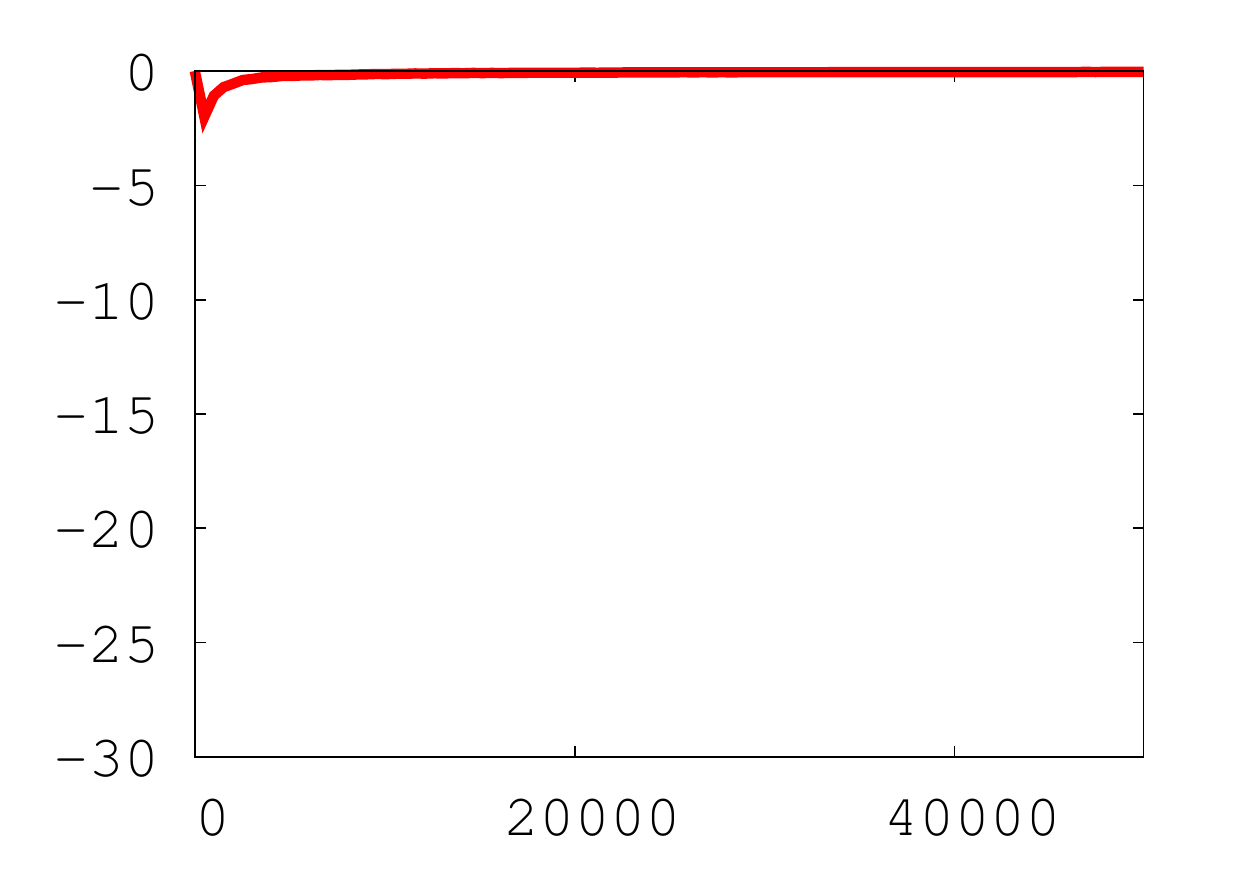}
&
\includegraphics[trim=10bp 25bp 30bp 10bp,clip,width=.15\linewidth]{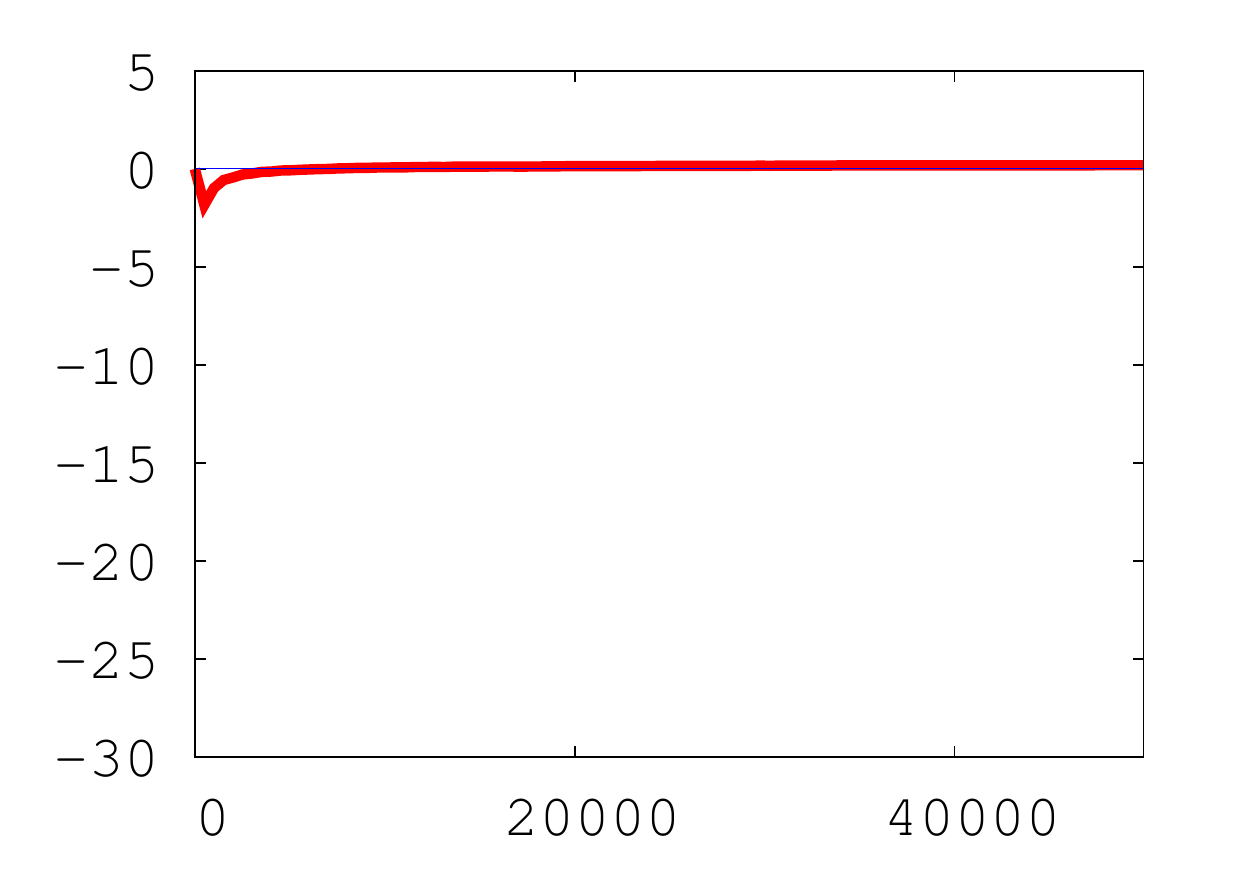}
&
\includegraphics[trim=10bp 25bp 30bp 10bp,clip,width=.15\linewidth]{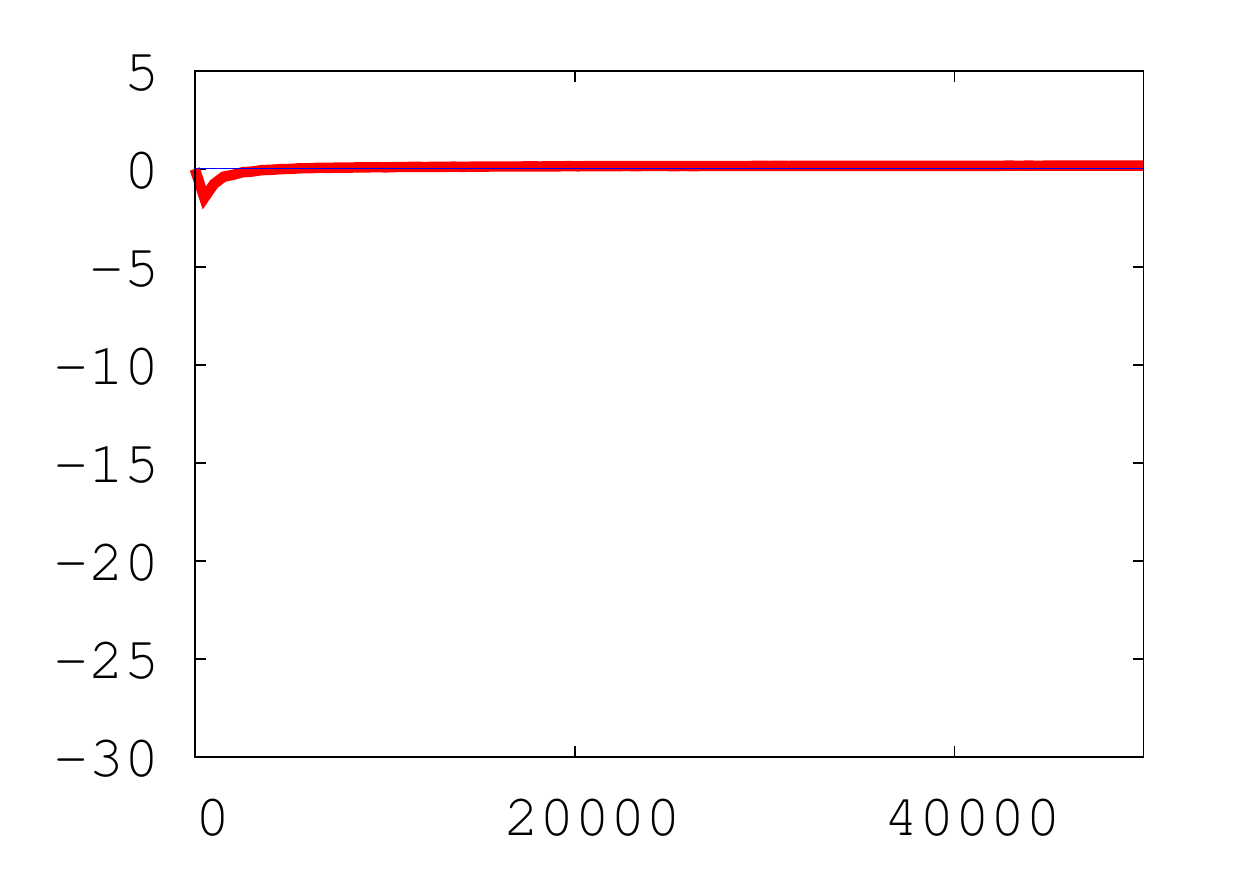}
\\
$\rho = 0.1$ & $\rho = 0.2$ & $\rho = 0.3$ & $\rho = 0.4$ & $\rho = 0.5$ & $\rho = 0.6$ \\\hline
\includegraphics[trim=10bp 25bp 30bp 10bp,clip,width=.15\linewidth]{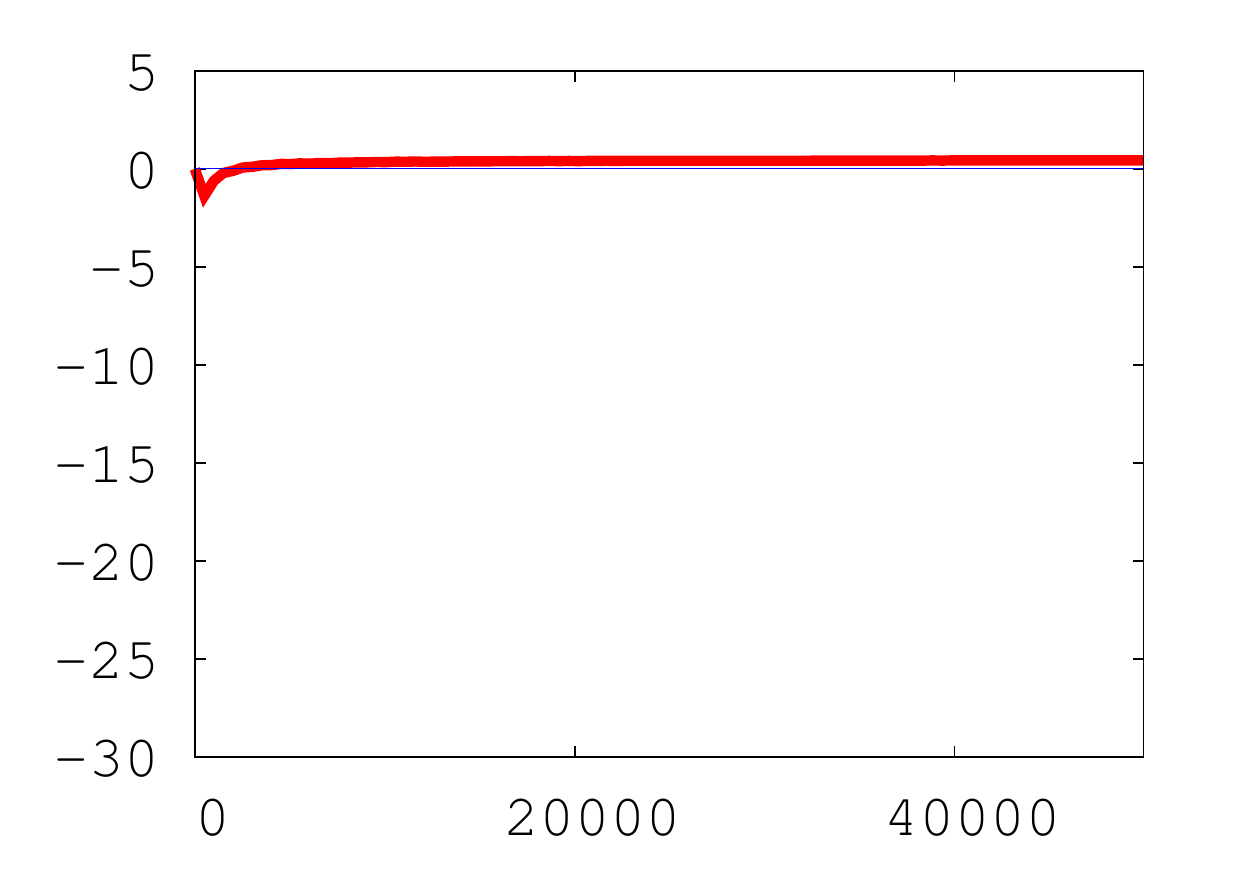}
&
\includegraphics[trim=10bp 25bp 30bp 10bp,clip,width=.15\linewidth]{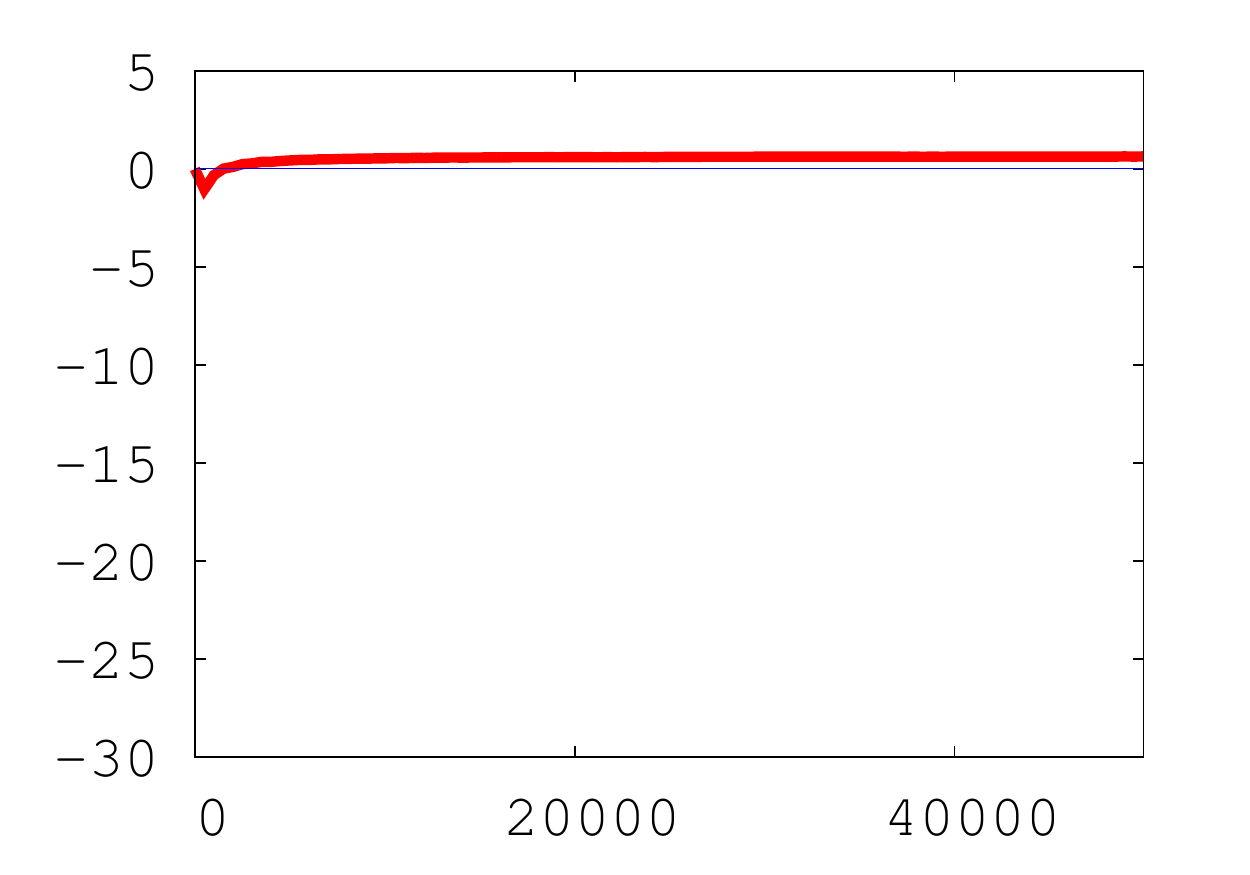}
&
\includegraphics[trim=10bp 25bp 30bp 10bp,clip,width=.15\linewidth]{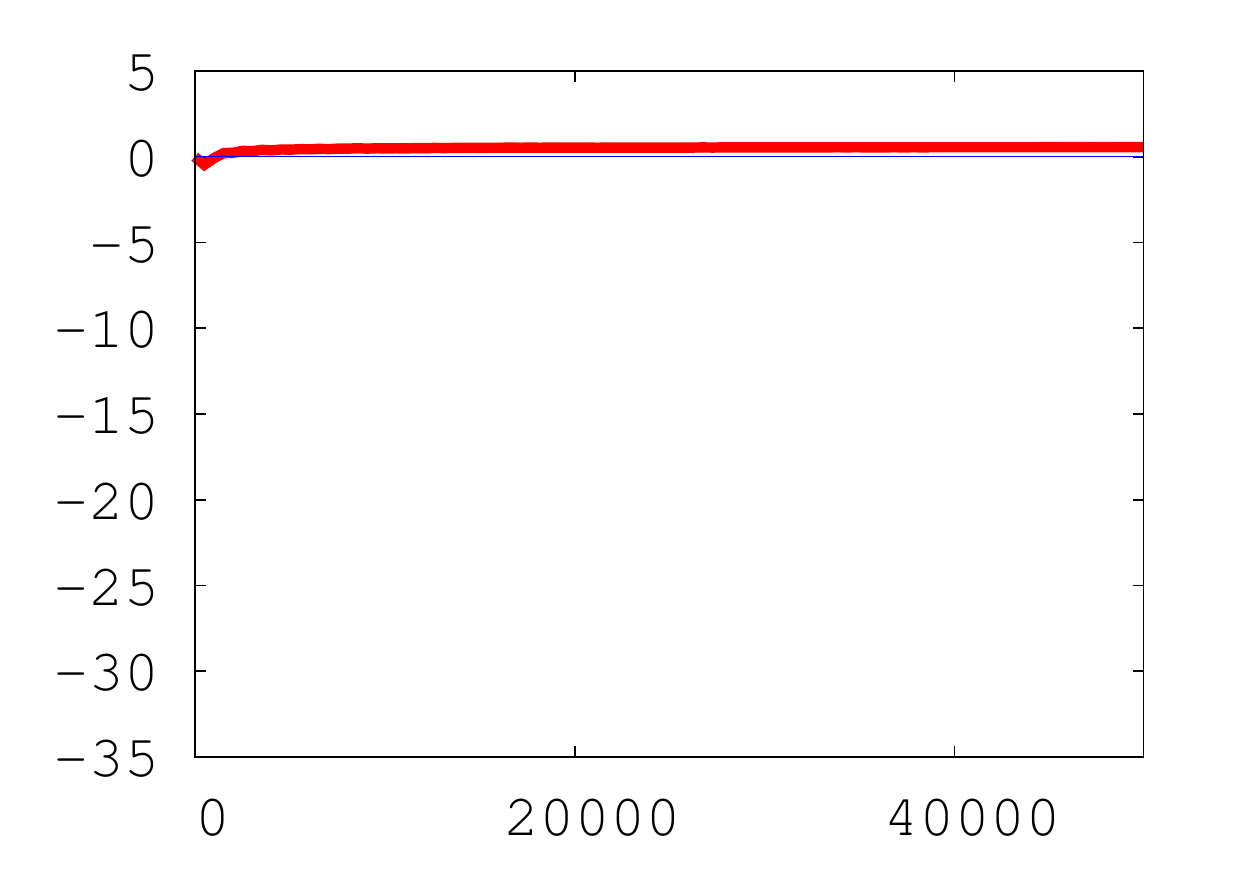}
&
\includegraphics[trim=10bp 25bp 30bp 10bp,clip,width=.15\linewidth]{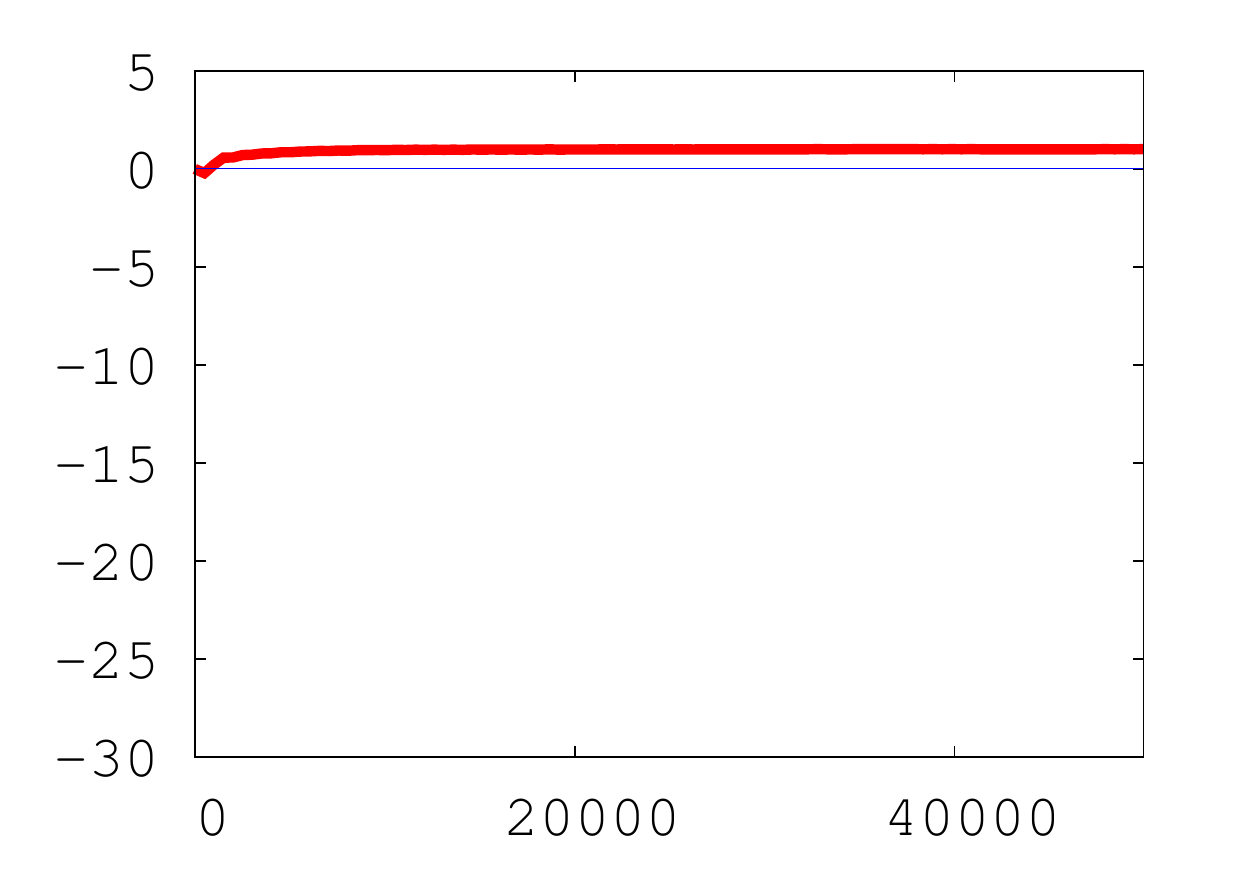}
&
\includegraphics[trim=10bp 25bp 30bp 10bp,clip,width=.15\linewidth]{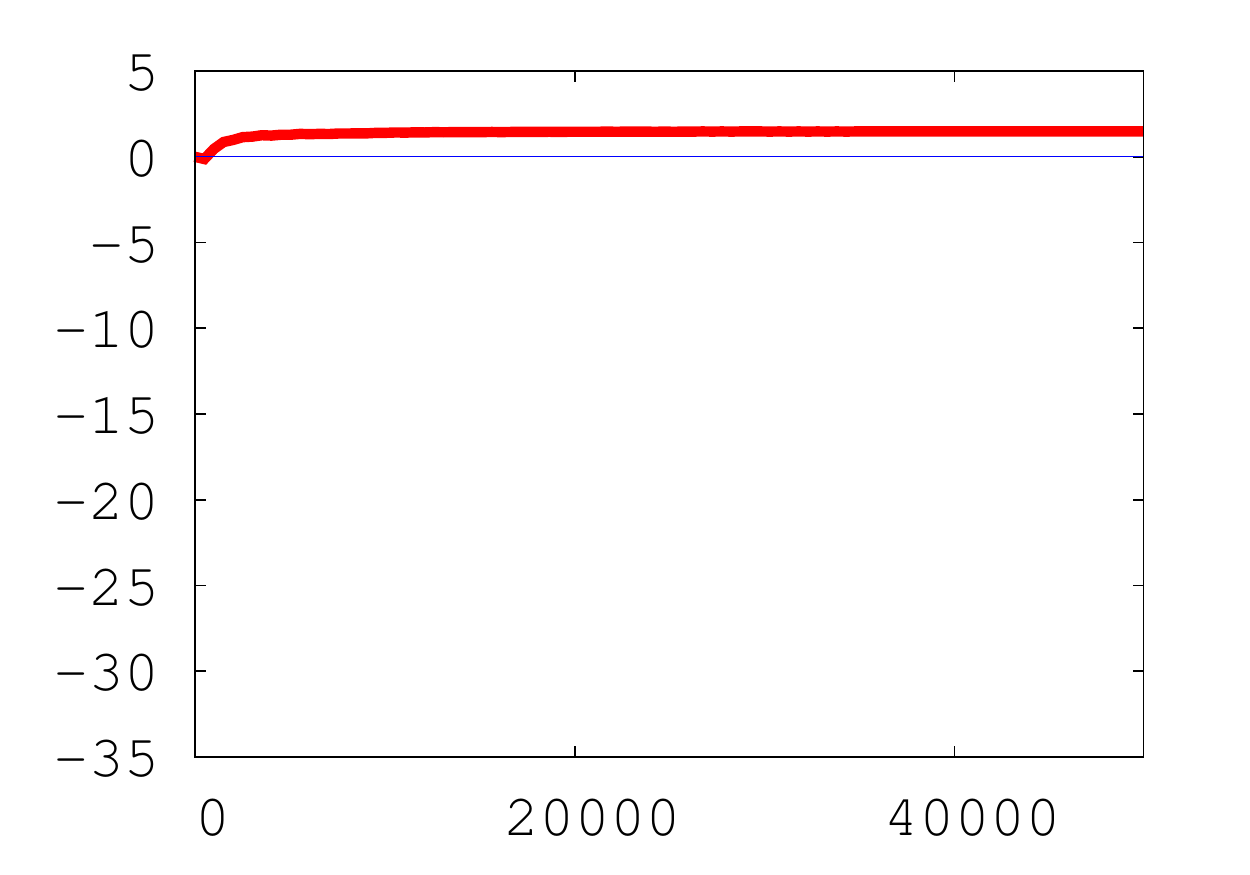}
&
\includegraphics[trim=10bp 25bp 30bp 10bp,clip,width=.15\linewidth]{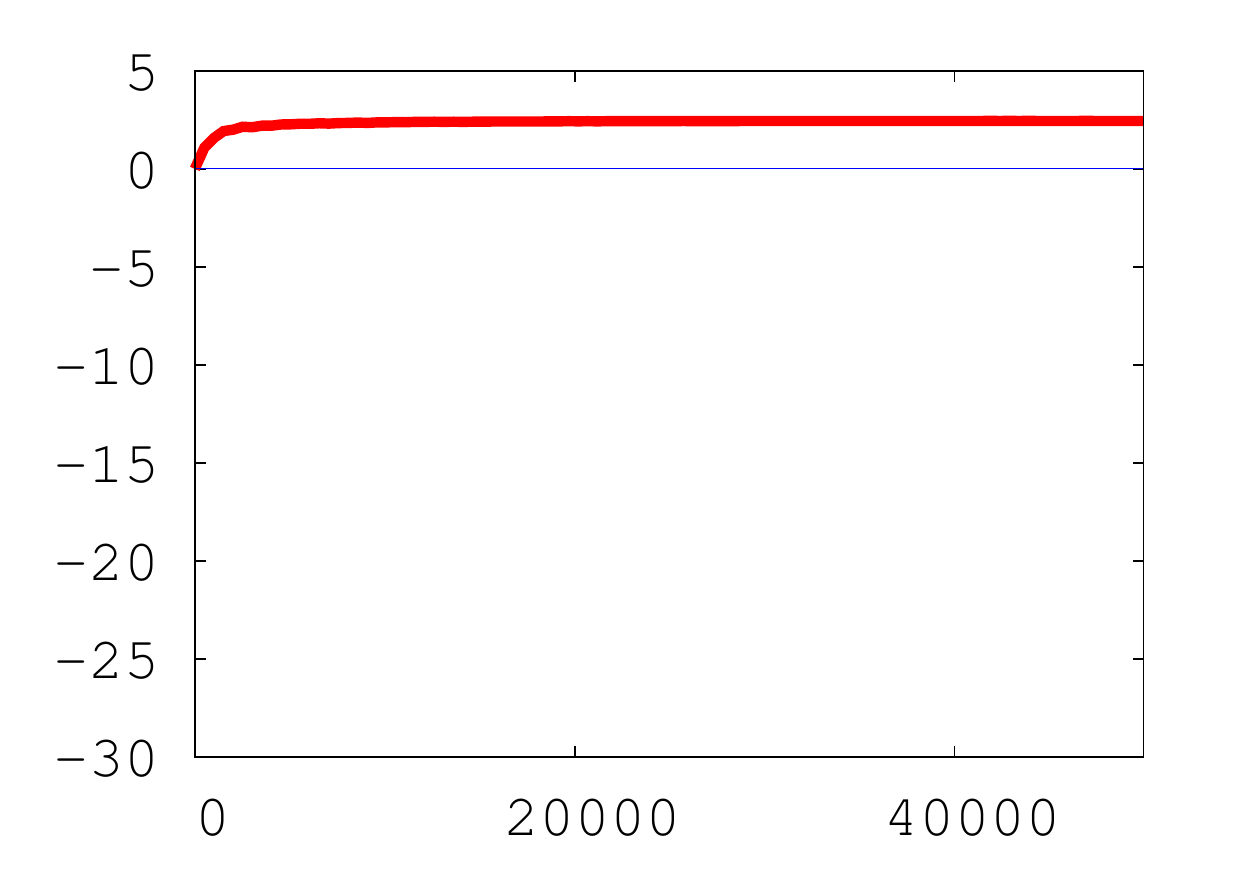}
\\
$\rho = 0.7$ & $\rho = 0.8$ & $\rho = 0.9$ & $\rho = \mbox{\textbf{1.0}}$ & $\rho = 1.1$ & $\rho = 1.2$ \\\hline
\includegraphics[trim=10bp 25bp 30bp 10bp,clip,width=.15\linewidth]{results_P2_00_Q2_00_UCHOICE_S_EXP_R_UR_1_30_NOLABELS-eps-converted-to}
&
\includegraphics[trim=10bp 25bp 30bp 10bp,clip,width=.15\linewidth]{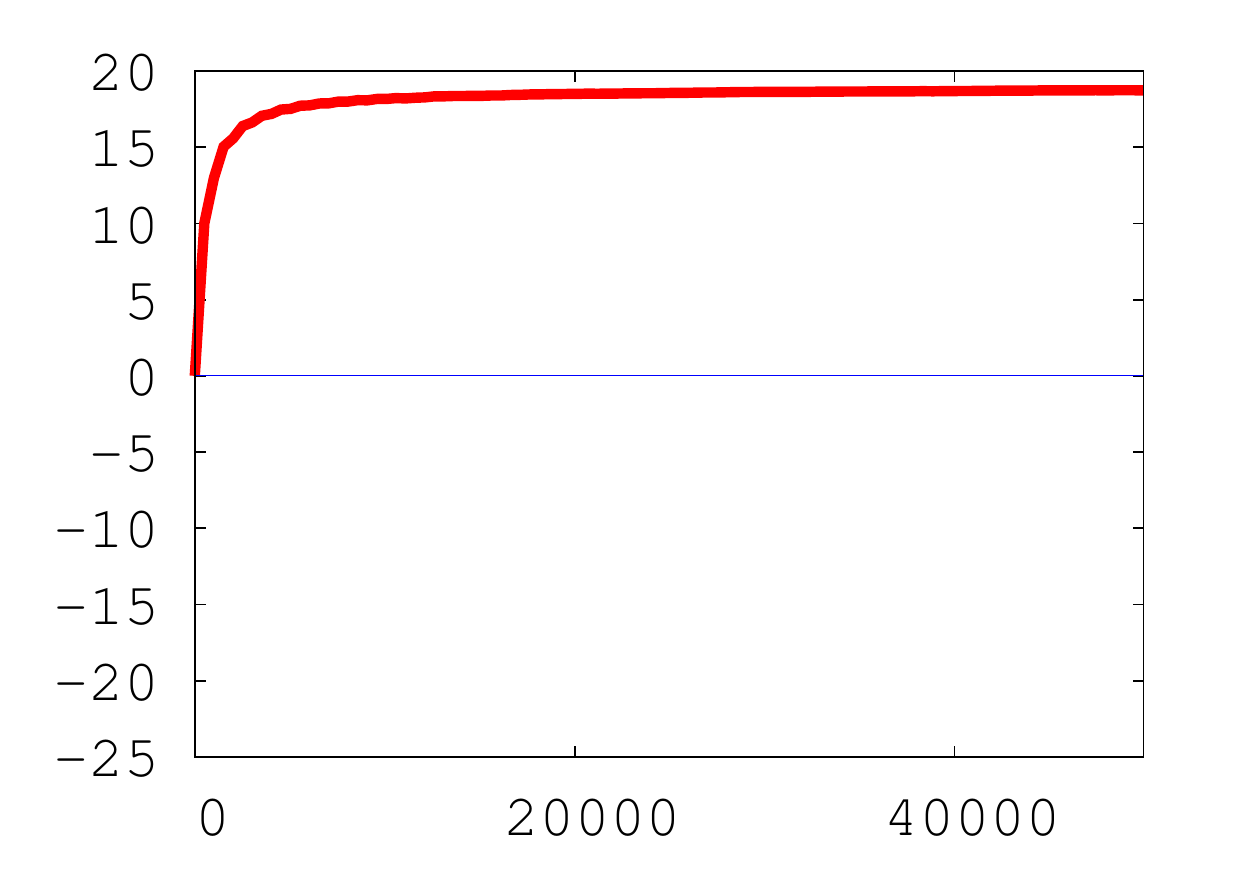}
&
\includegraphics[trim=10bp 25bp 30bp 10bp,clip,width=.15\linewidth]{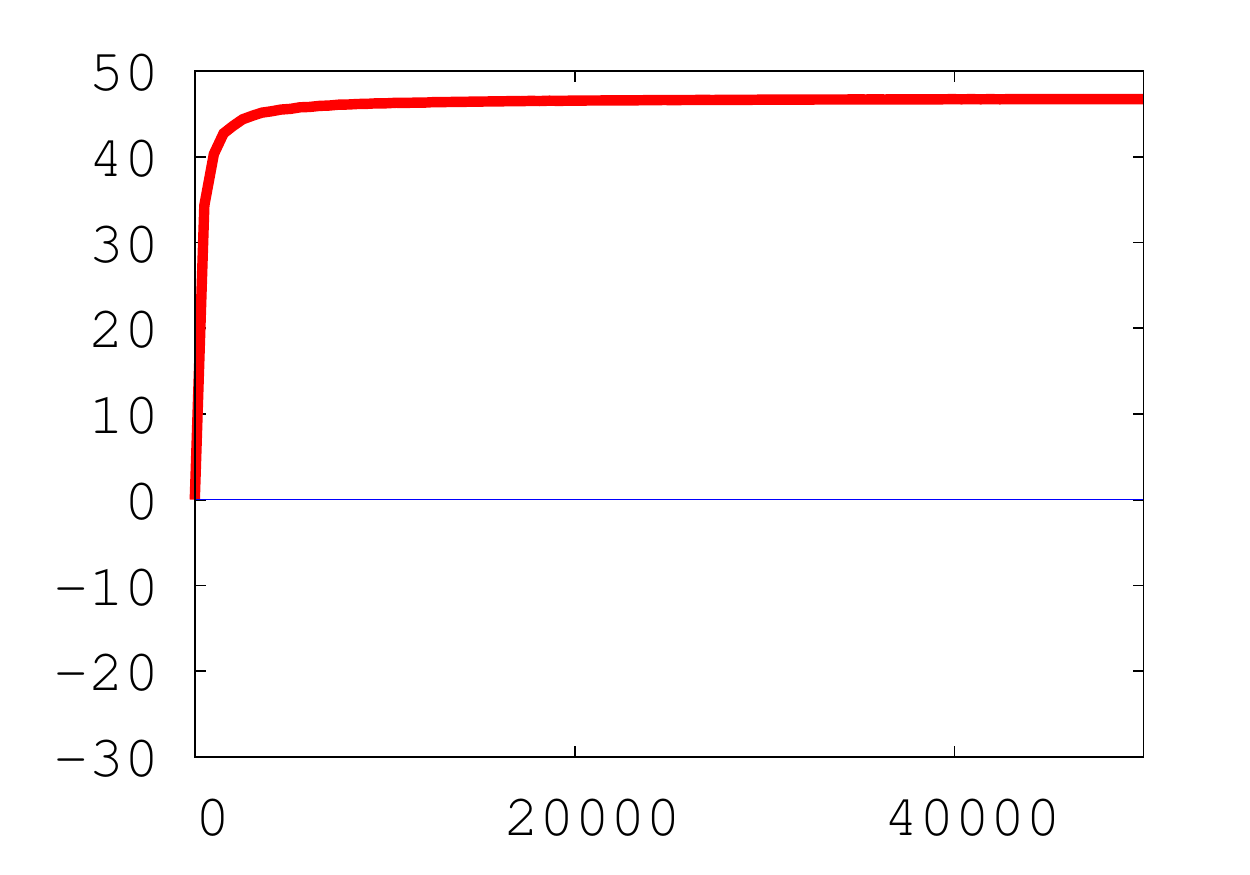}
&
\includegraphics[trim=10bp 25bp 30bp 10bp,clip,width=.15\linewidth]{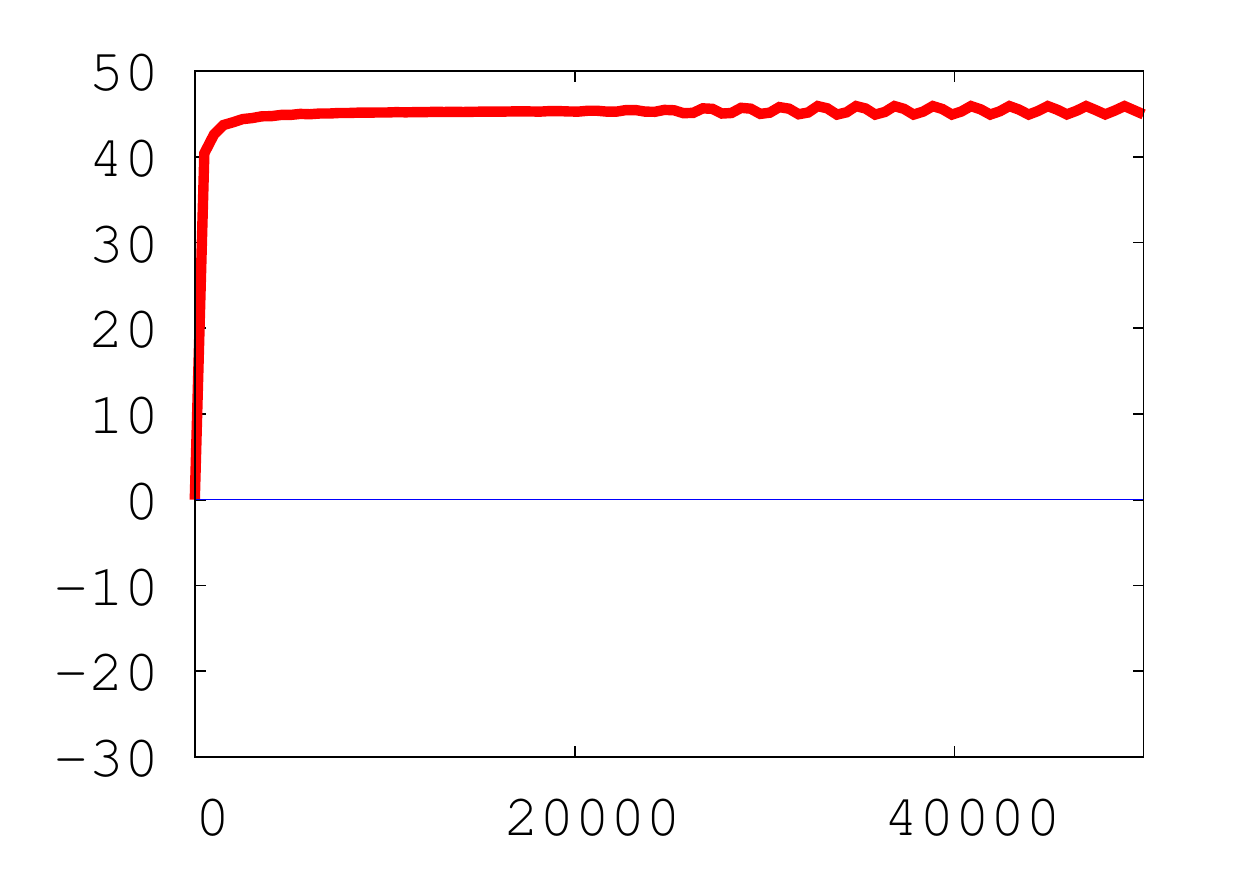}
&
\includegraphics[trim=10bp 25bp 30bp
10bp,clip,width=.15\linewidth]{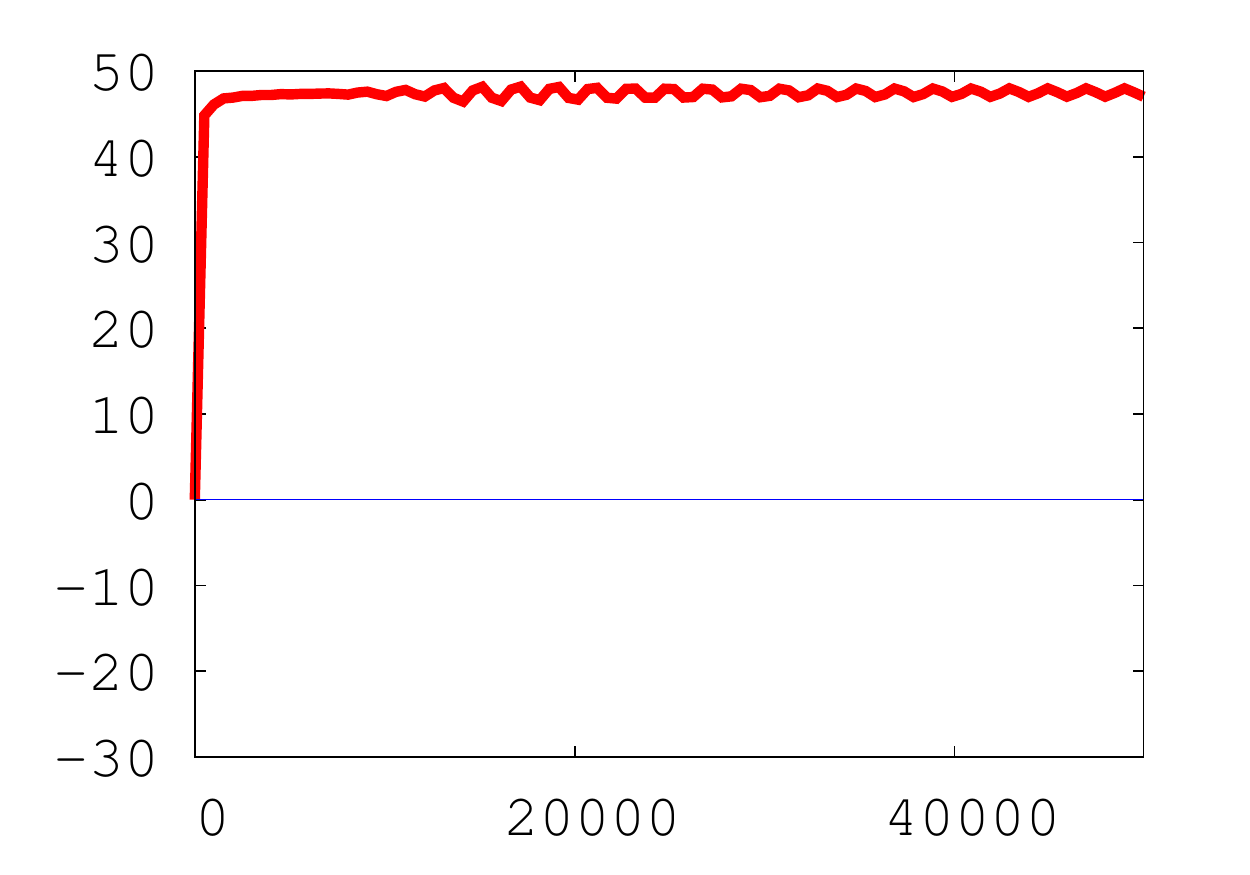}
&\\
$\rho =1.3$ & $\rho = 1.4$ & $\rho = 1.5$ & $\rho = 1.6$ & $\rho =
1.7$ &  \\\hline \hline
\end{tabular}
}
\end{center}
\caption{Error($p$-LMS) - Error(DN-$p$-LMS) as a function of $t$ ($\in \{1, 2, ..., 50 000\}$), $\bm{u}$ = sparse, $(p,q) = (2.0, 2.0)$.}
  \label{tc1_supp_rr4}
\end{sidewaystable}

\begin{sidewaystable}[t]
\begin{center}
{\small
\begin{tabular}{cccccc}\hline \hline
\includegraphics[trim=10bp 30bp 30bp 10bp,clip,width=.15\linewidth]{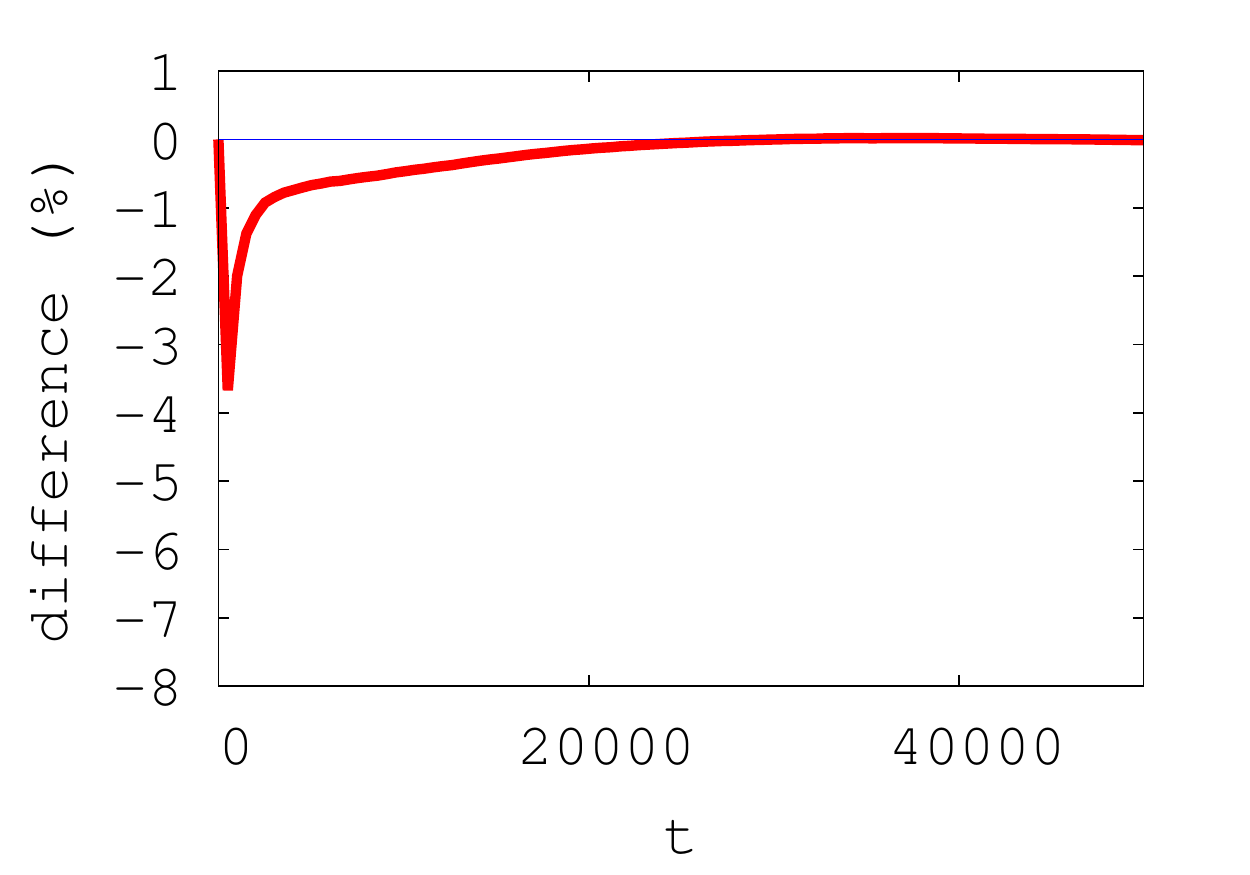}
&
\includegraphics[trim=10bp 25bp 30bp 10bp,clip,width=.15\linewidth]{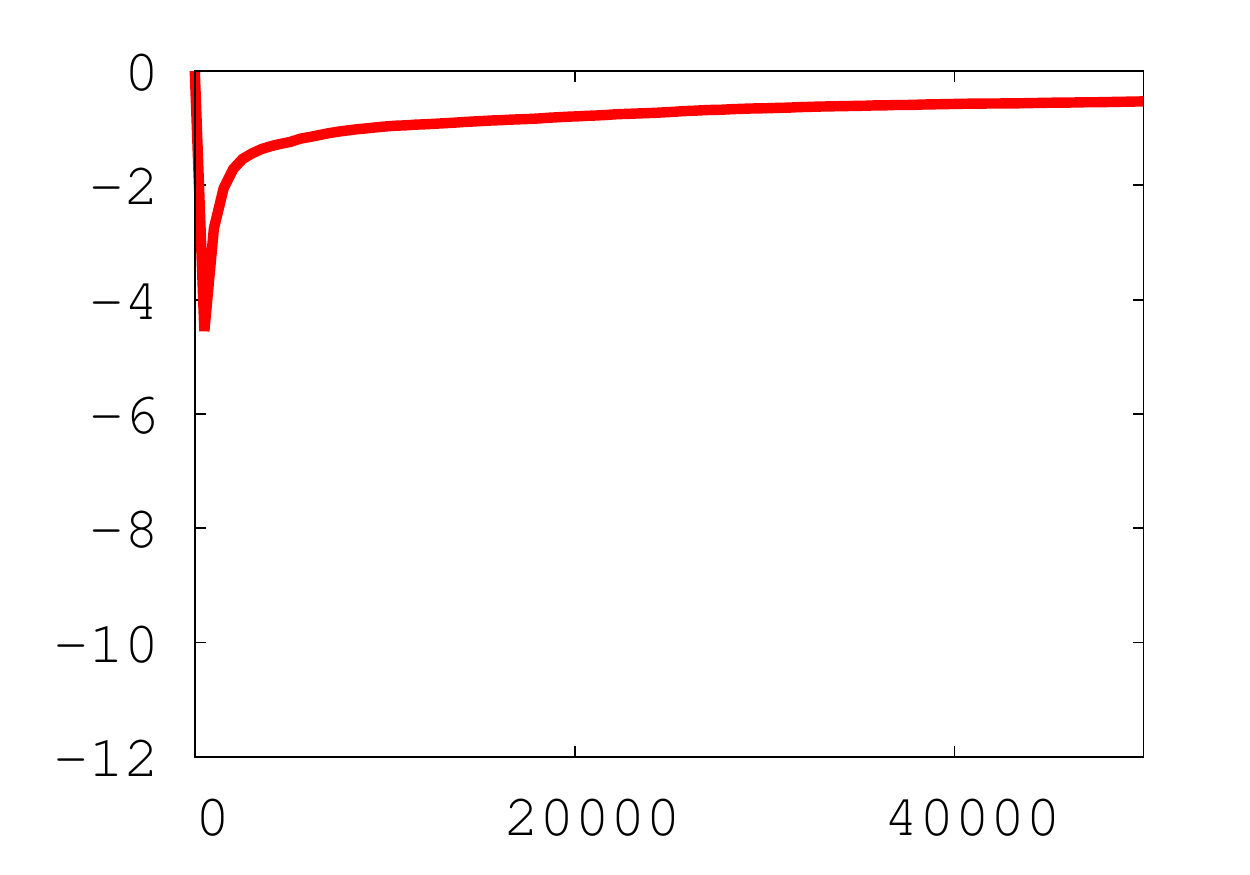}
&
\includegraphics[trim=10bp 25bp 30bp 10bp,clip,width=.15\linewidth]{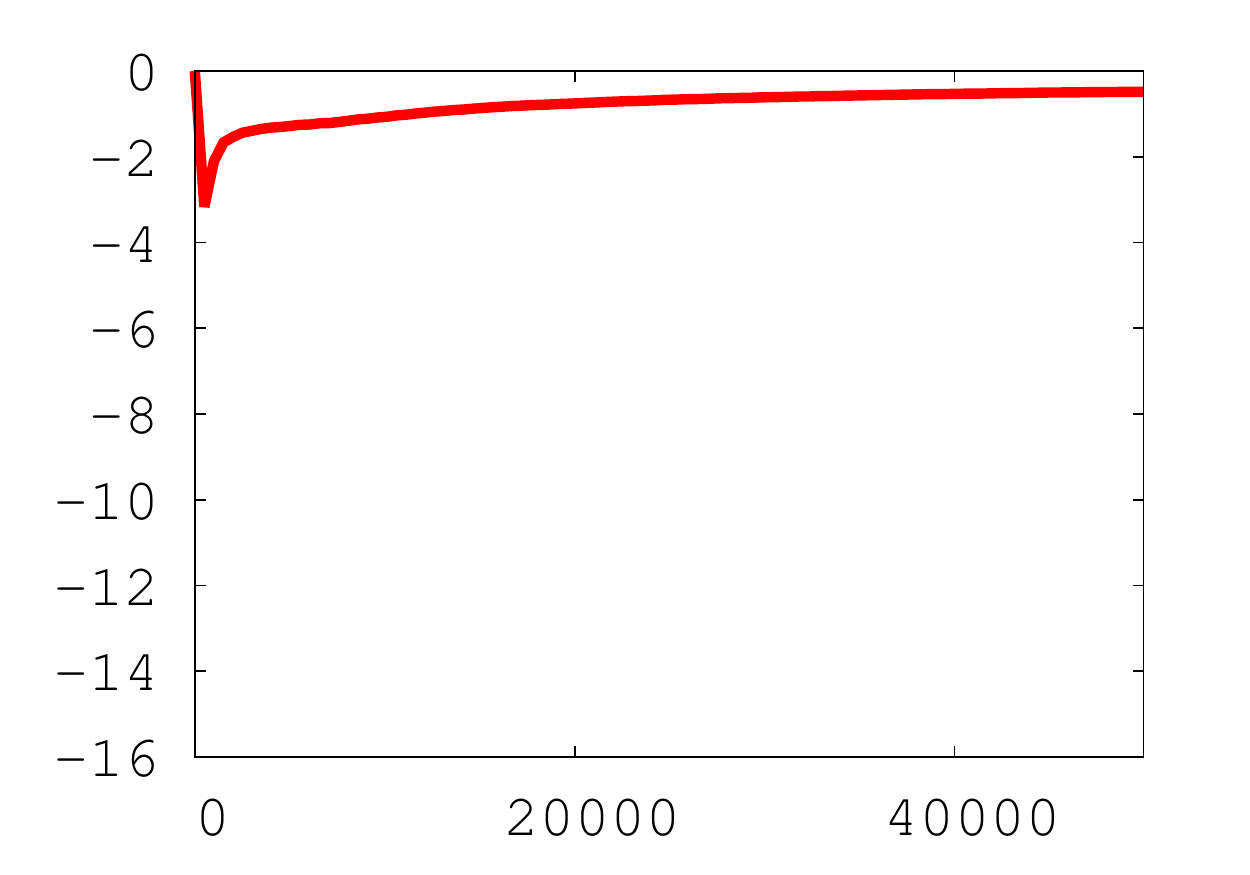}
&
\includegraphics[trim=10bp 25bp 30bp 10bp,clip,width=.15\linewidth]{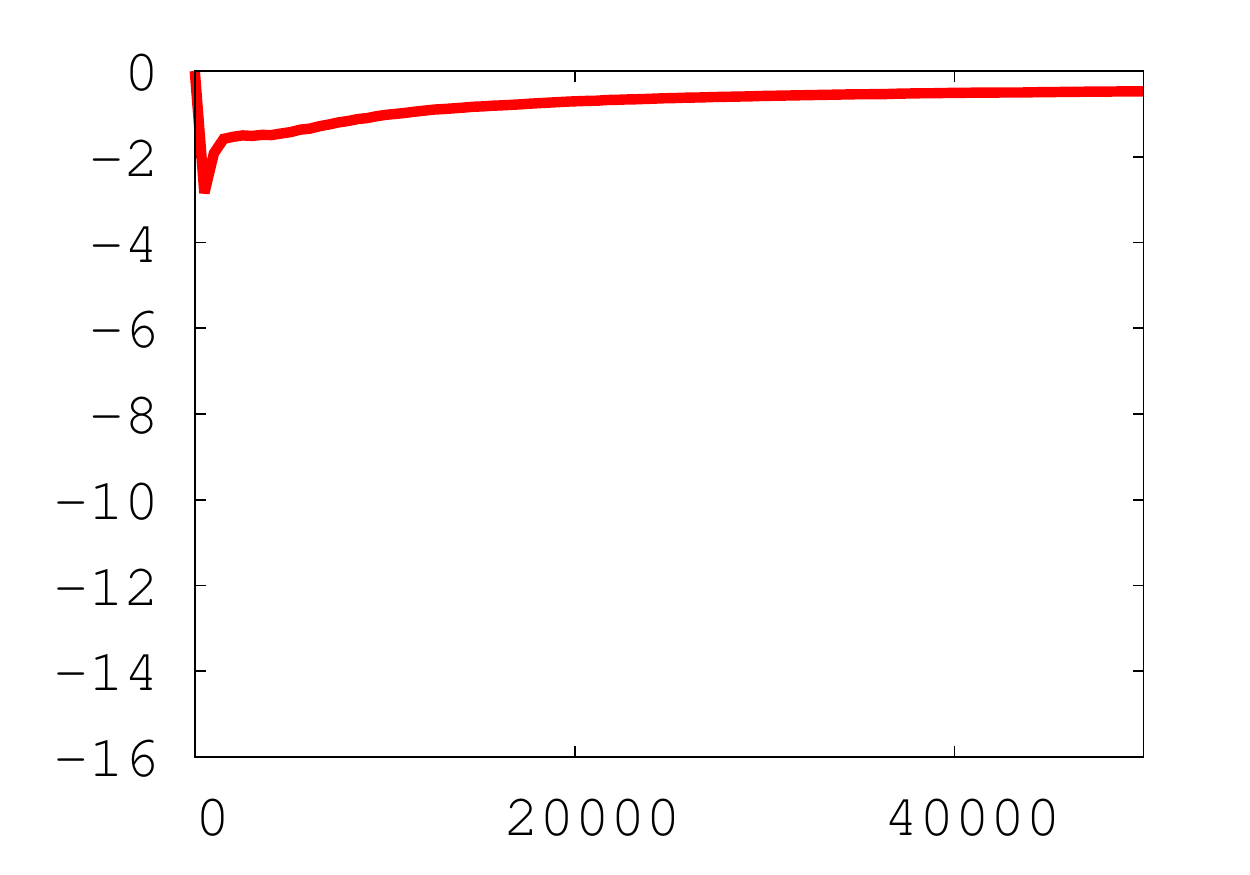}
&
\includegraphics[trim=10bp 25bp 30bp 10bp,clip,width=.15\linewidth]{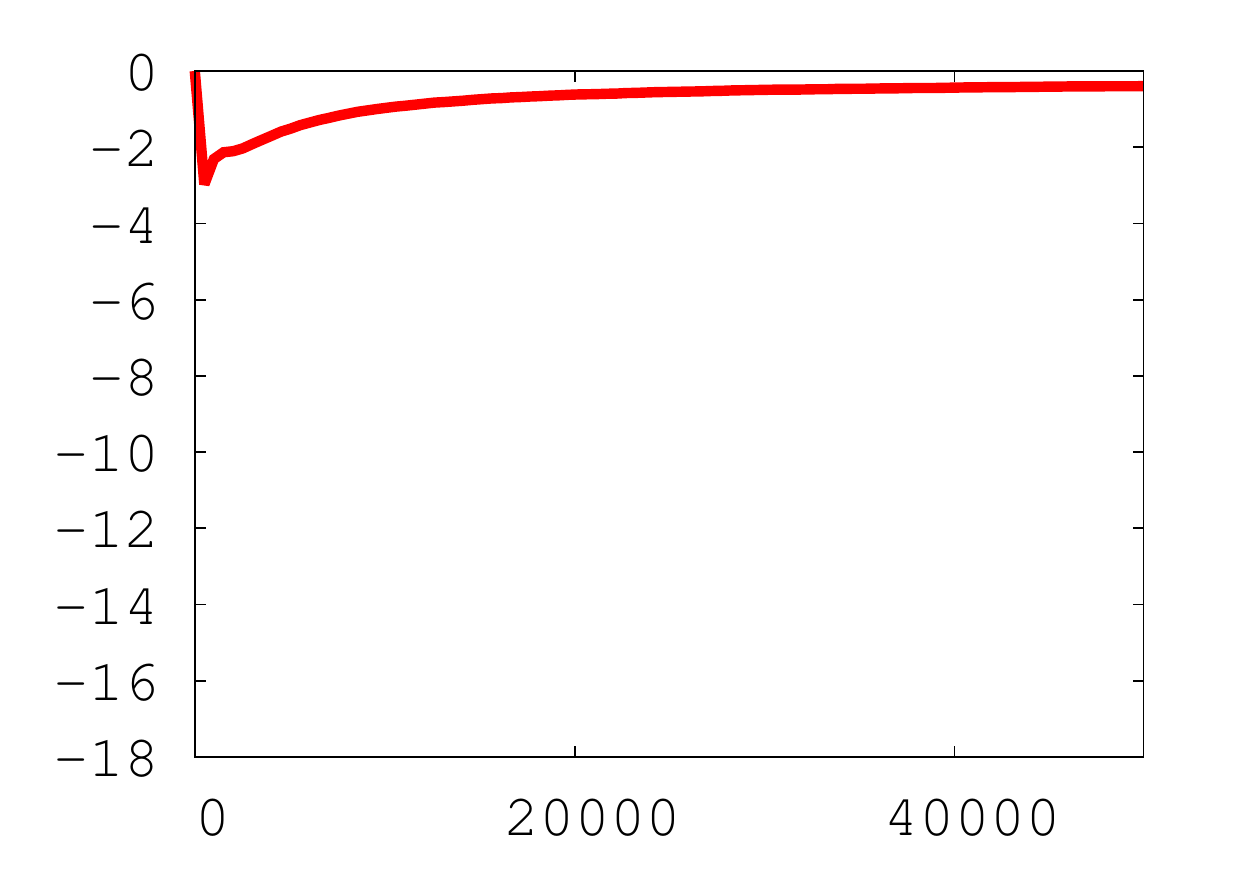}
&
\includegraphics[trim=10bp 25bp 30bp 10bp,clip,width=.15\linewidth]{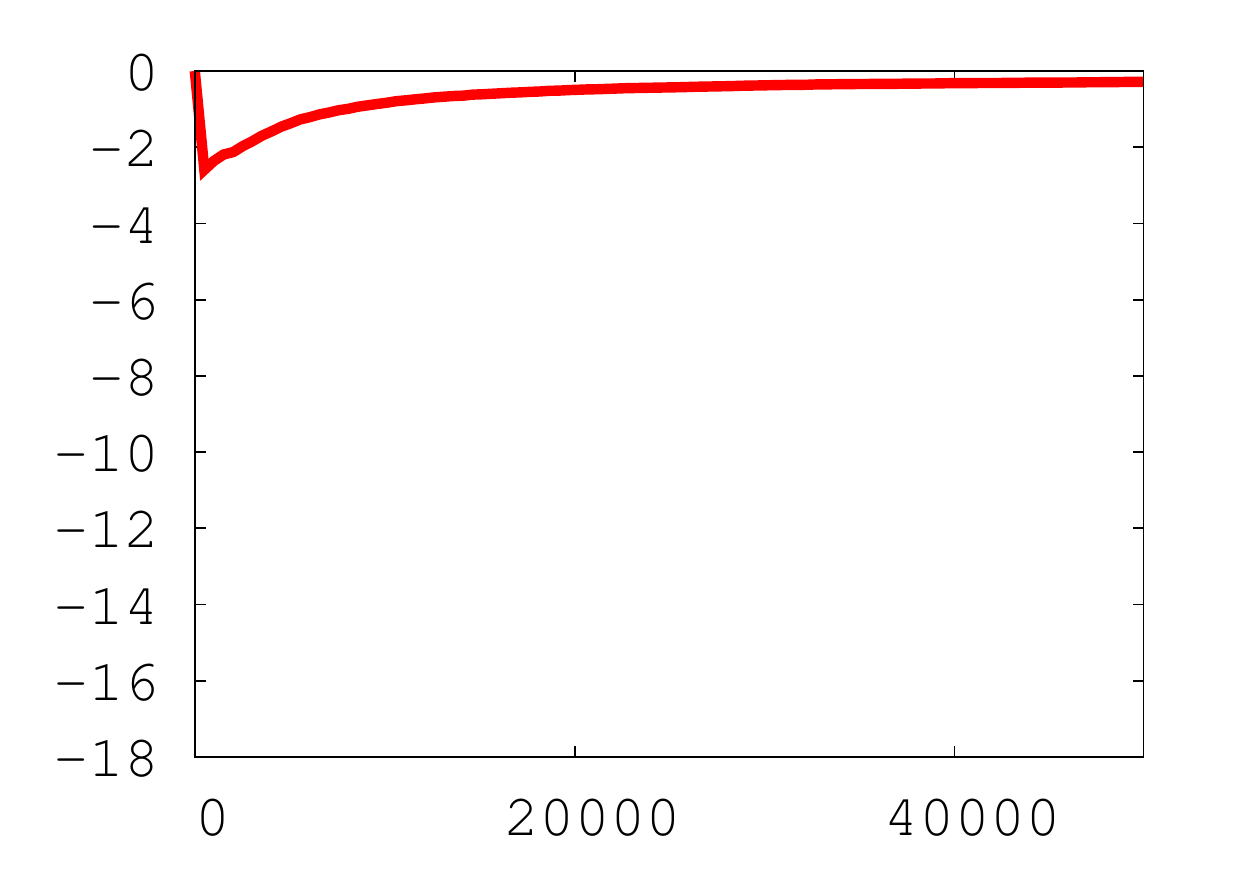}
\\
$\rho = 0.1$ & $\rho = 0.2$ & $\rho = 0.3$ & $\rho = 0.4$ & $\rho = 0.5$ & $\rho = 0.6$ \\\hline
\includegraphics[trim=10bp 25bp 30bp 10bp,clip,width=.15\linewidth]{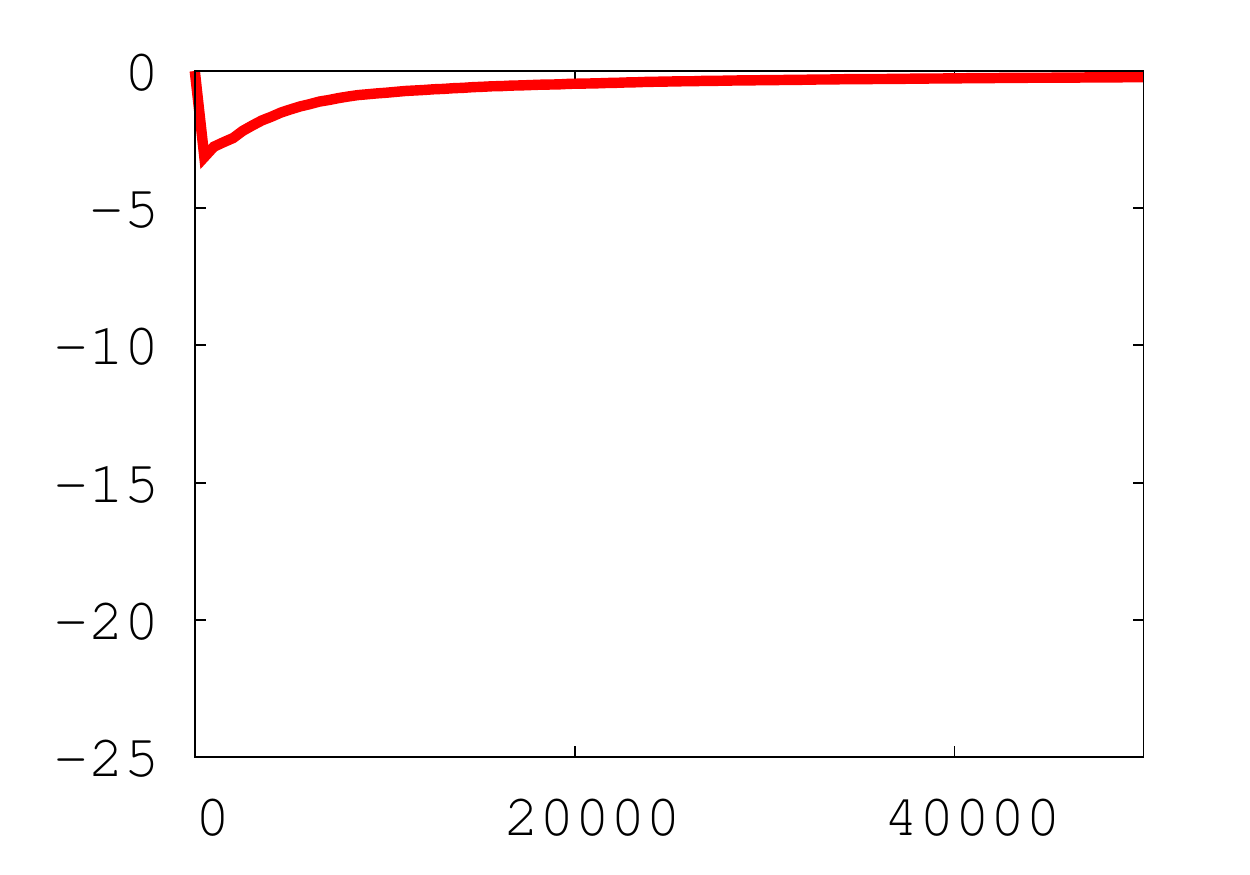}
&
\includegraphics[trim=10bp 25bp 30bp 10bp,clip,width=.15\linewidth]{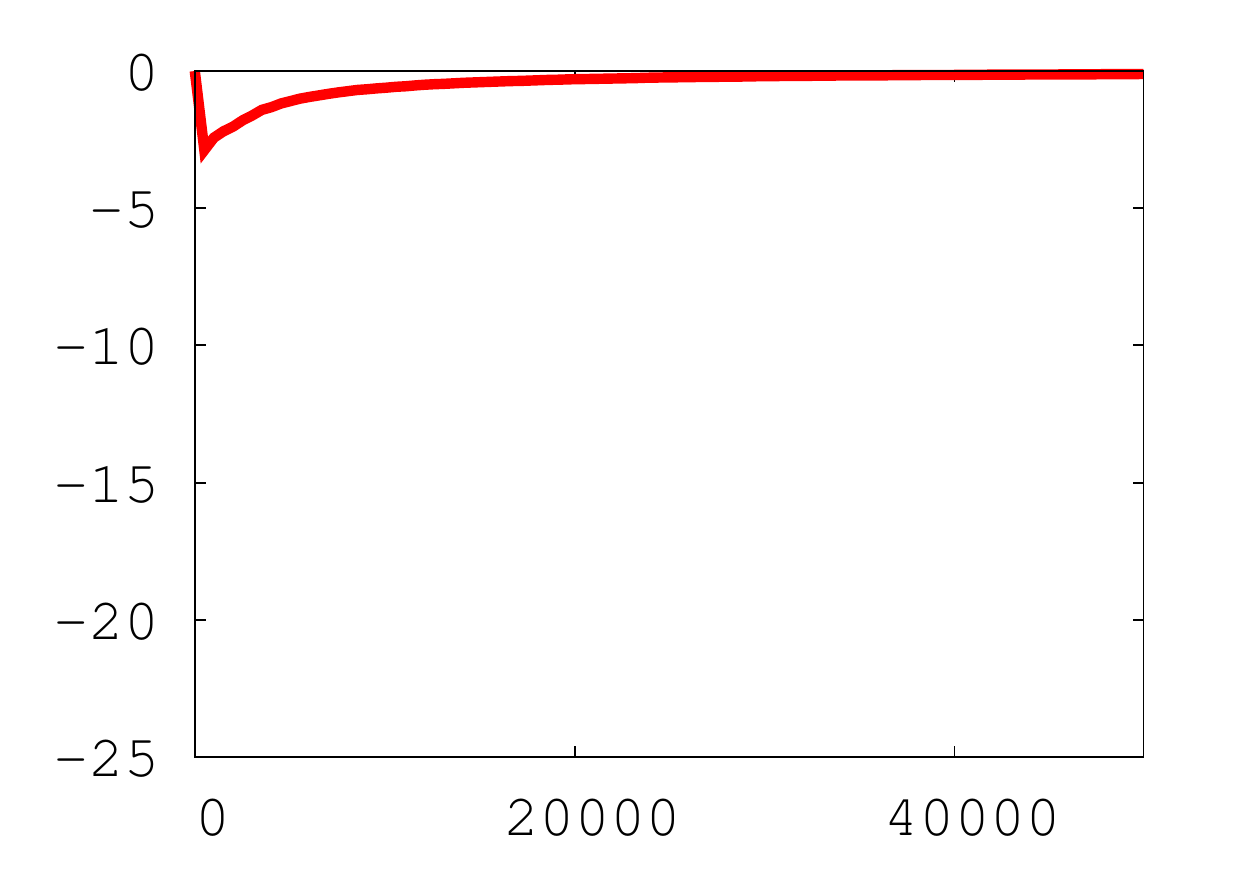}
&
\includegraphics[trim=10bp 25bp 30bp 10bp,clip,width=.15\linewidth]{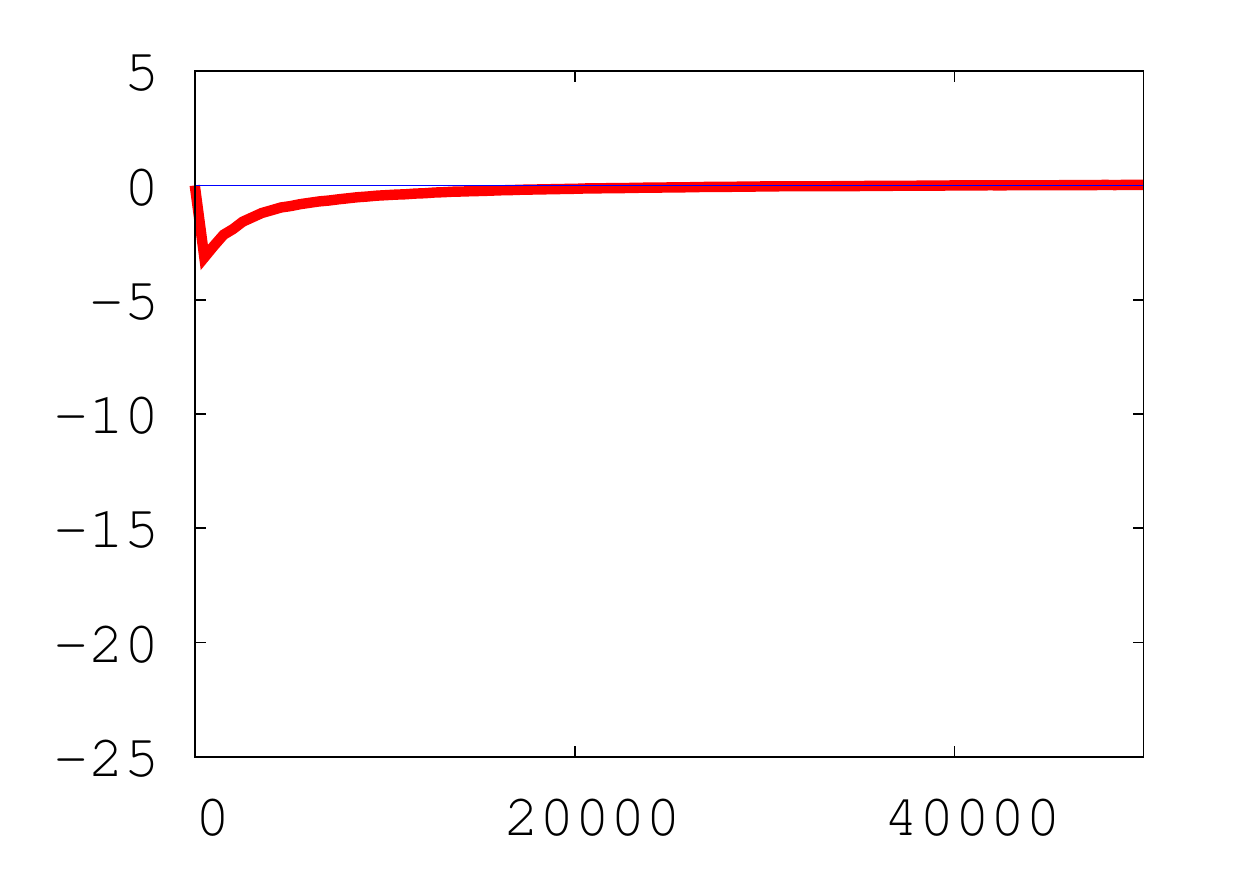}
&
\includegraphics[trim=10bp 25bp 30bp 10bp,clip,width=.15\linewidth]{results_P6_90_Q1_17_UCHOICE_GAUSS_UR_1_00_NOLABELS-eps-converted-to}
&
\includegraphics[trim=10bp 25bp 30bp 10bp,clip,width=.15\linewidth]{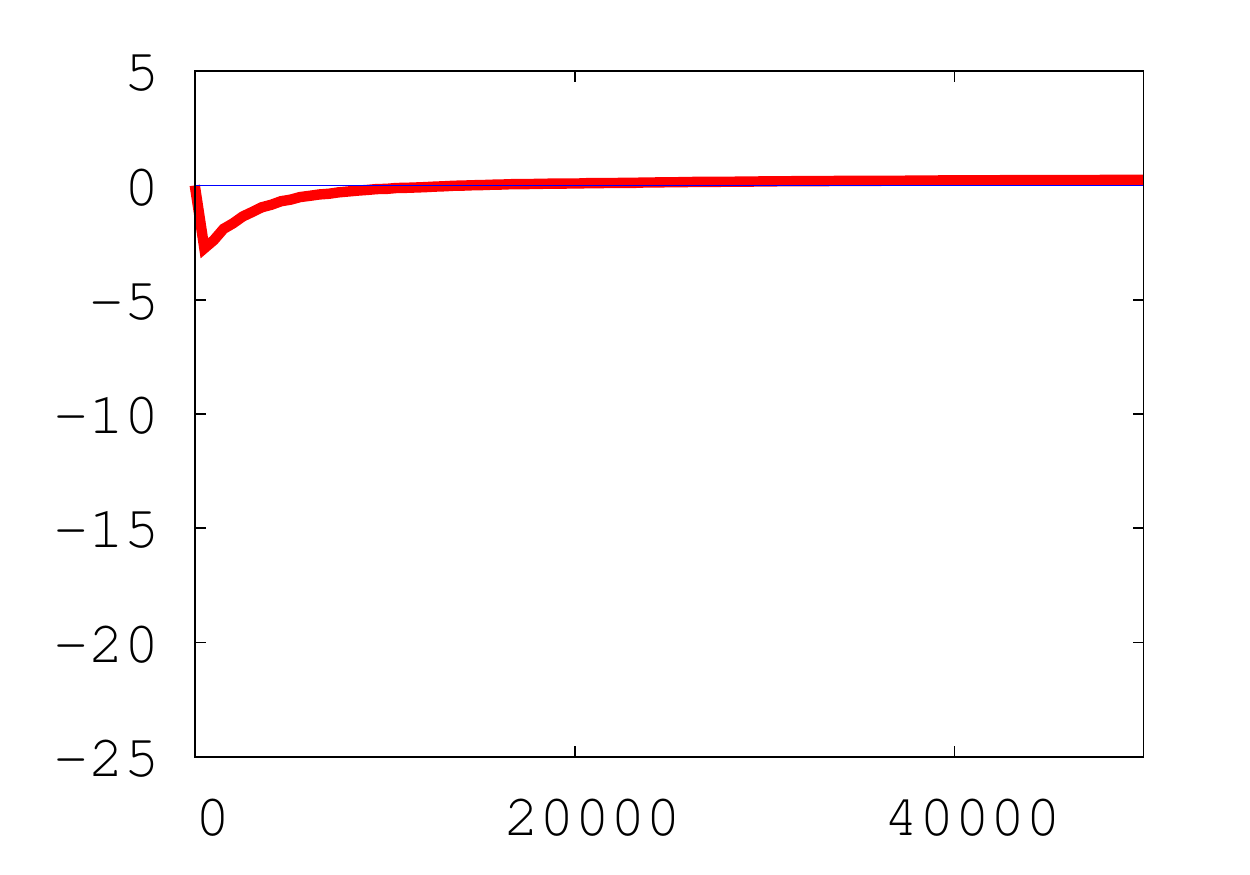}
&
\includegraphics[trim=10bp 25bp 30bp 10bp,clip,width=.15\linewidth]{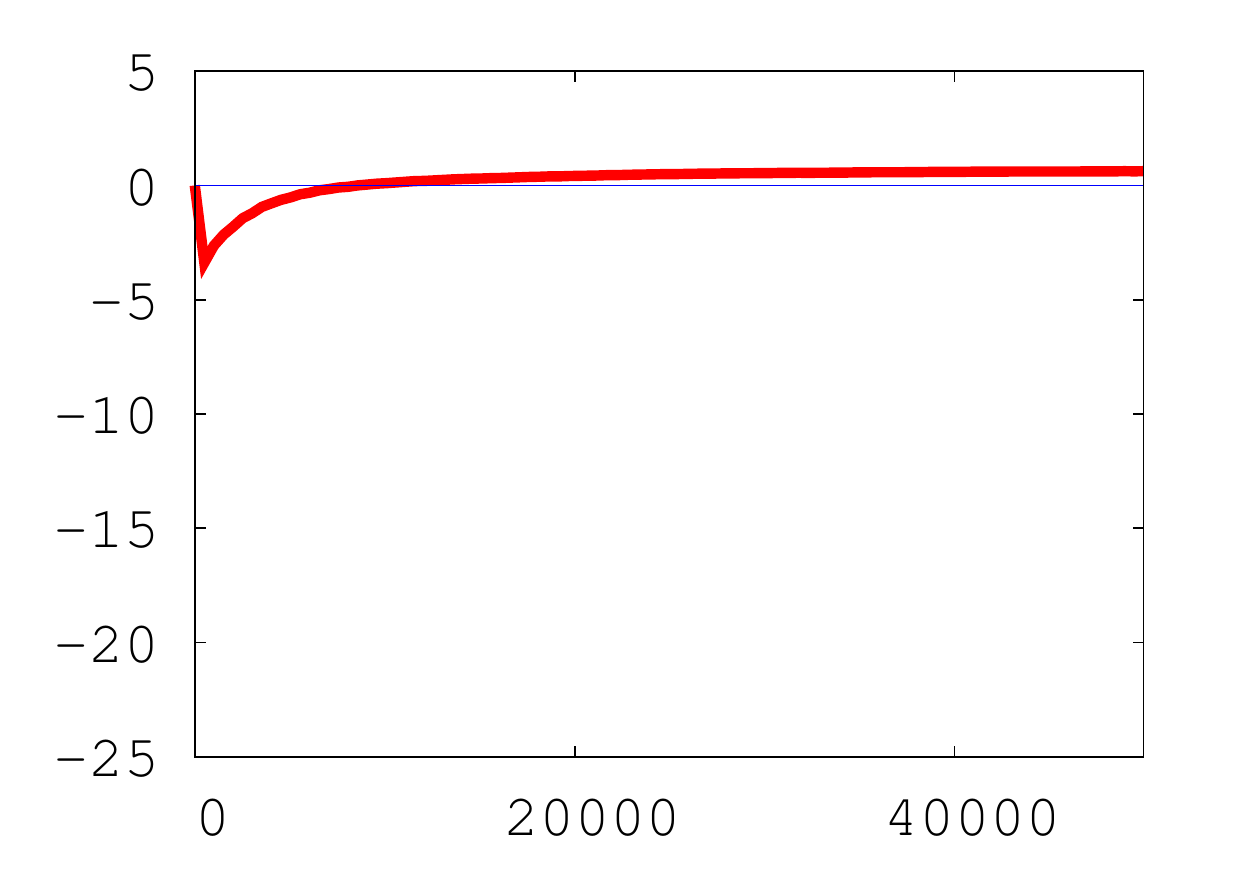}
\\
$\rho = 0.7$ & $\rho = 0.8$ & $\rho = 0.9$ & $\rho = \mbox{\textbf{1.0}}$ & $\rho = 1.1$ & $\rho = 1.2$ \\\hline
\includegraphics[trim=10bp 25bp 30bp 10bp,clip,width=.15\linewidth]{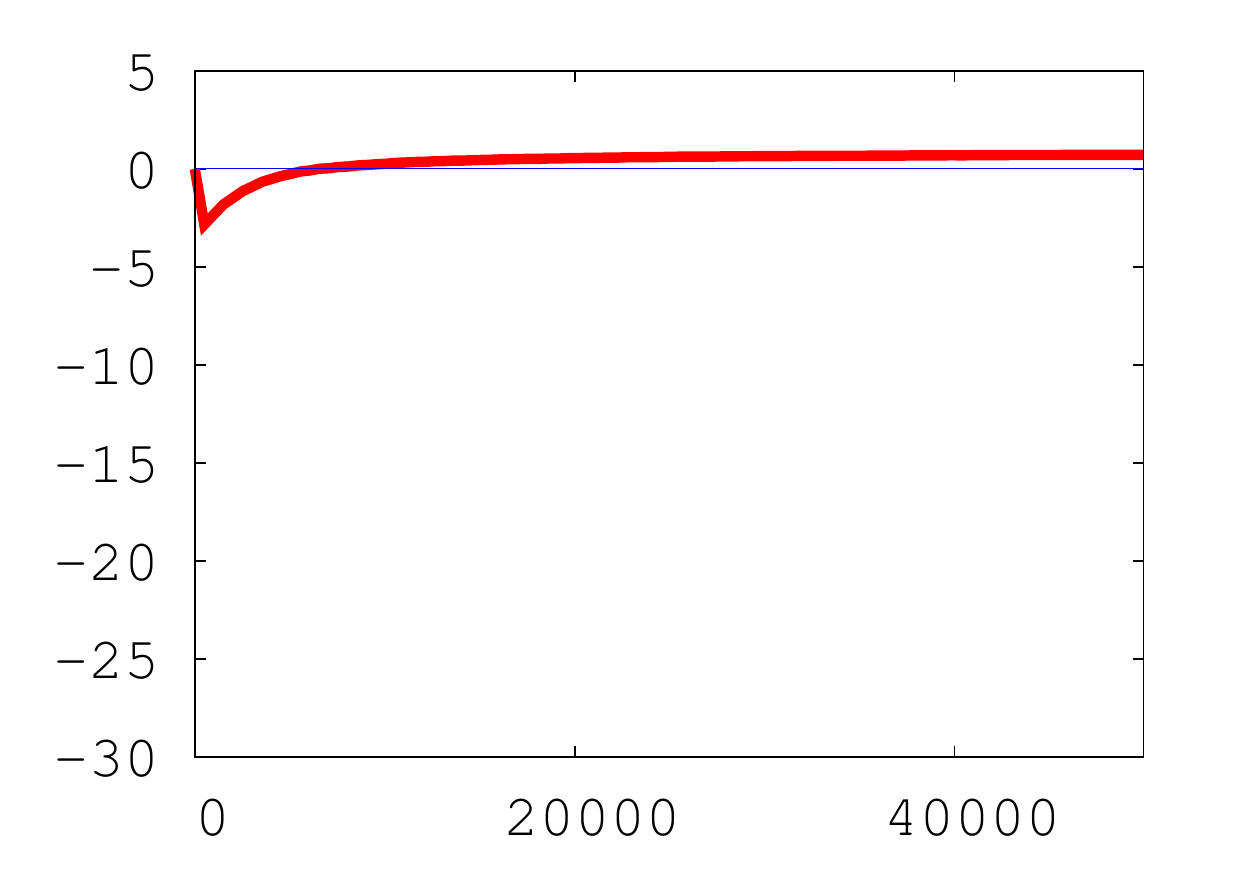}
&
\includegraphics[trim=10bp 25bp 30bp 10bp,clip,width=.15\linewidth]{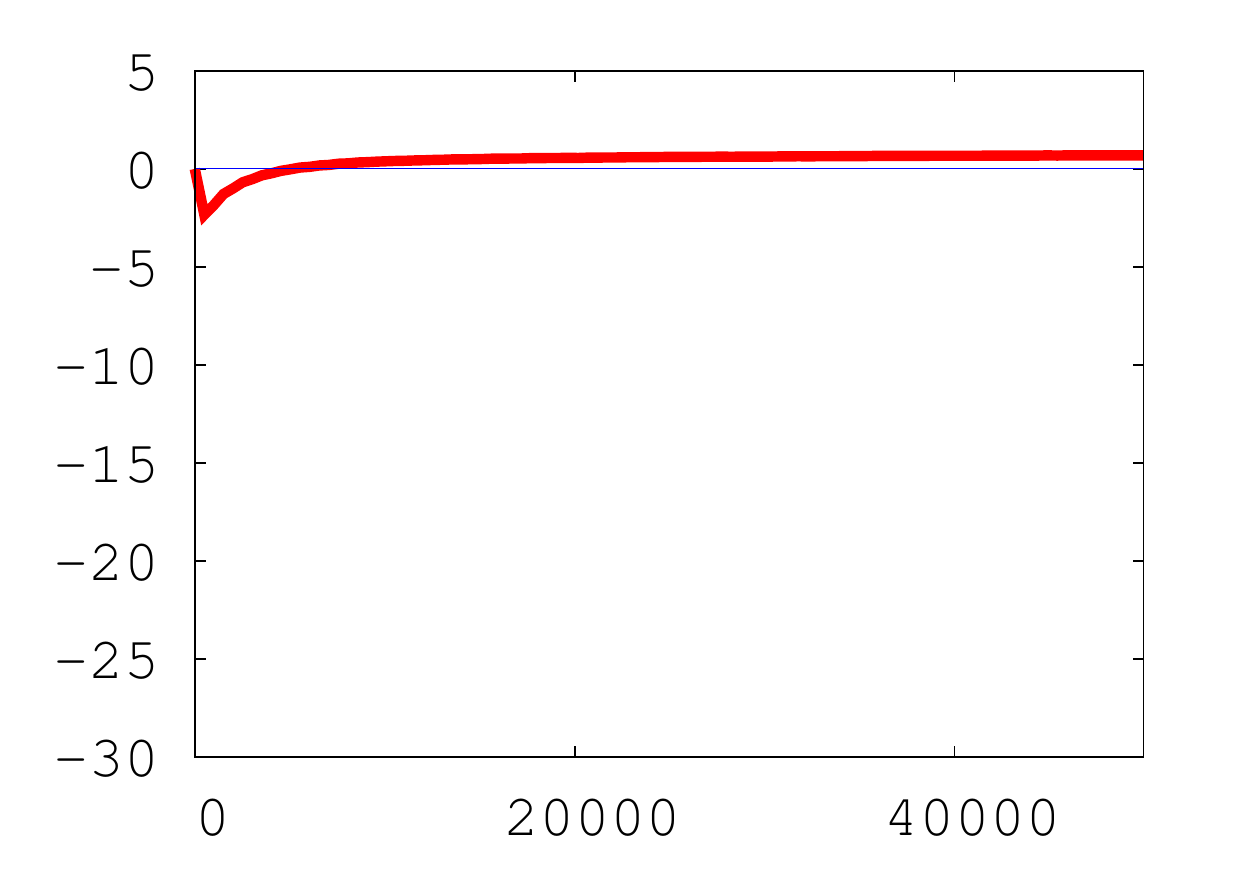}
&
\includegraphics[trim=10bp 25bp 30bp 10bp,clip,width=.15\linewidth]{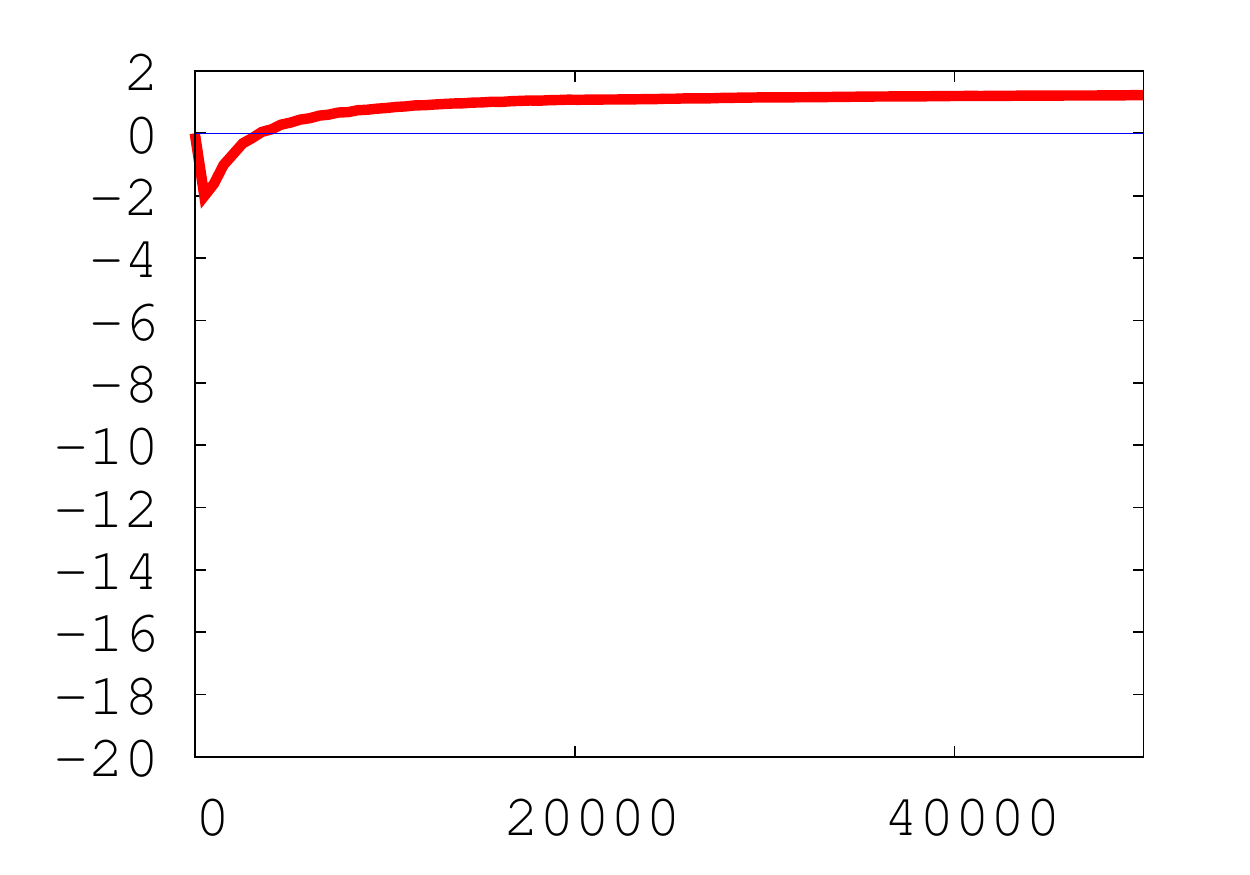}
&
\includegraphics[trim=10bp 25bp 30bp 10bp,clip,width=.15\linewidth]{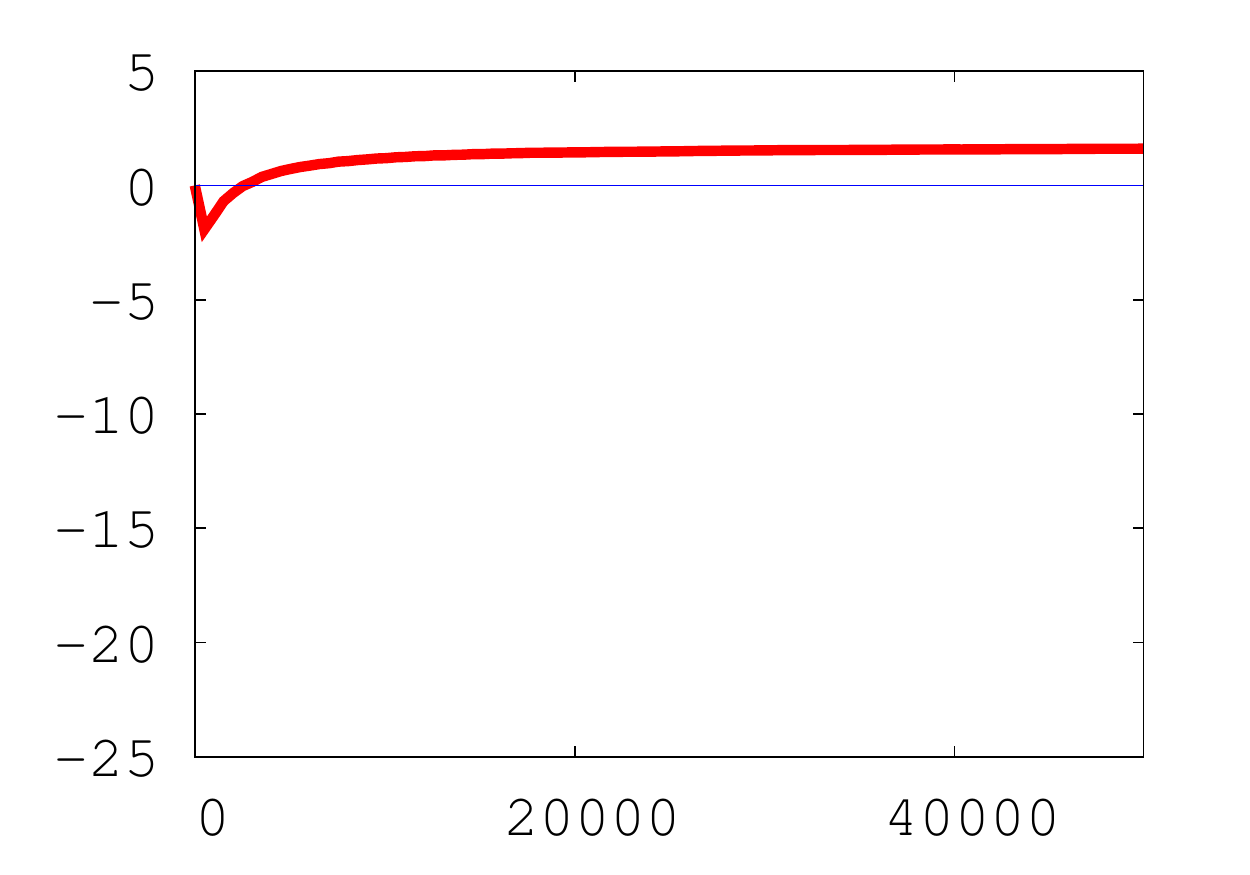}
&
\includegraphics[trim=10bp 25bp 30bp
10bp,clip,width=.15\linewidth]{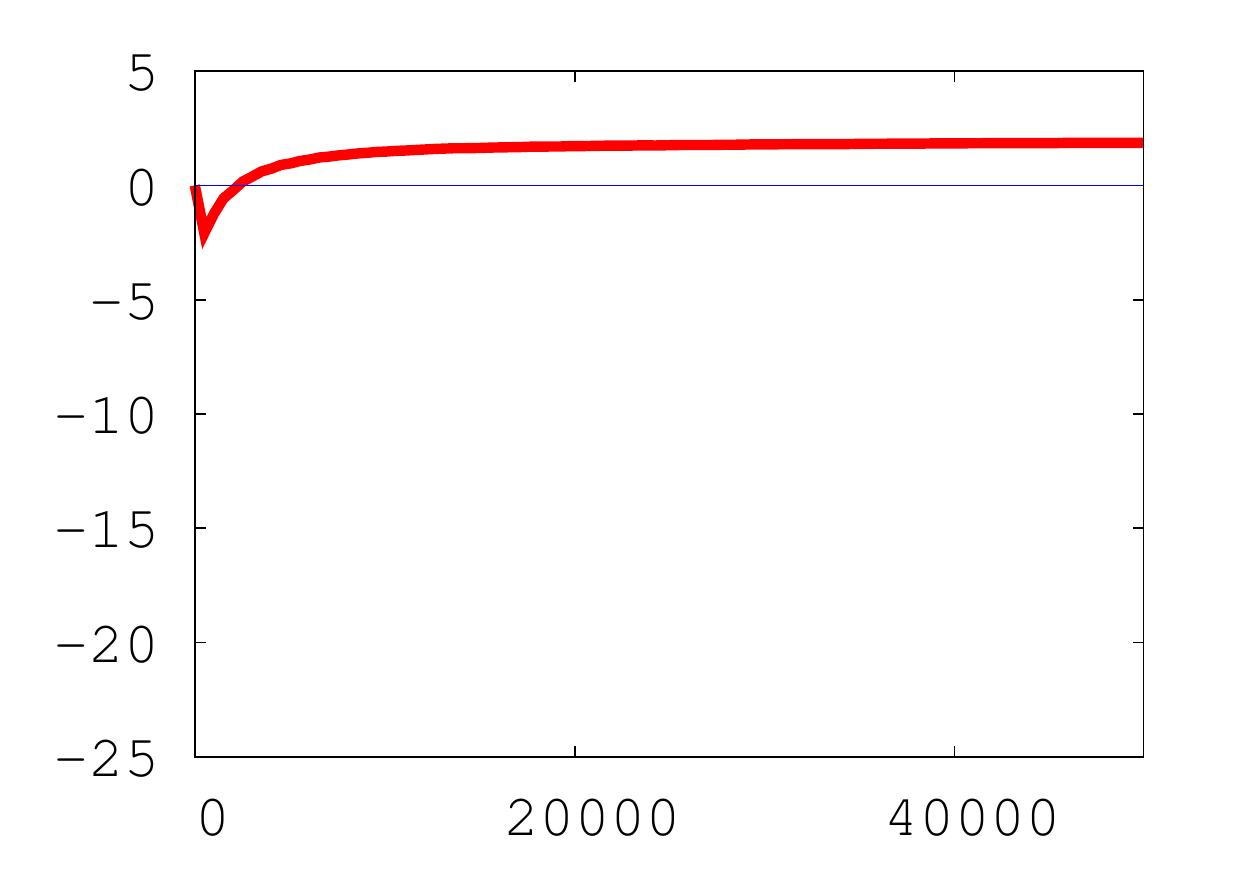}
&\\
$\rho =1.3$ & $\rho = 1.4$ & $\rho = 1.5$ & $\rho = 1.6$ & $\rho =
1.7$ &  \\\hline \hline
\end{tabular}
}
\end{center}
\caption{Error($p$-LMS) - Error(DN-$p$-LMS) as a function of $t$ ($\in \{1, 2, ..., 50 000\}$), $\bm{u}$ = dense, $(p,q) = (6.9, 1.17)$.}
  \label{tc1_supp_rr5}
\end{sidewaystable}

\begin{sidewaystable}[t]
\begin{center}
{\small
\begin{tabular}{cccccc}\hline \hline
\includegraphics[trim=10bp 30bp 30bp 10bp,clip,width=.15\linewidth]{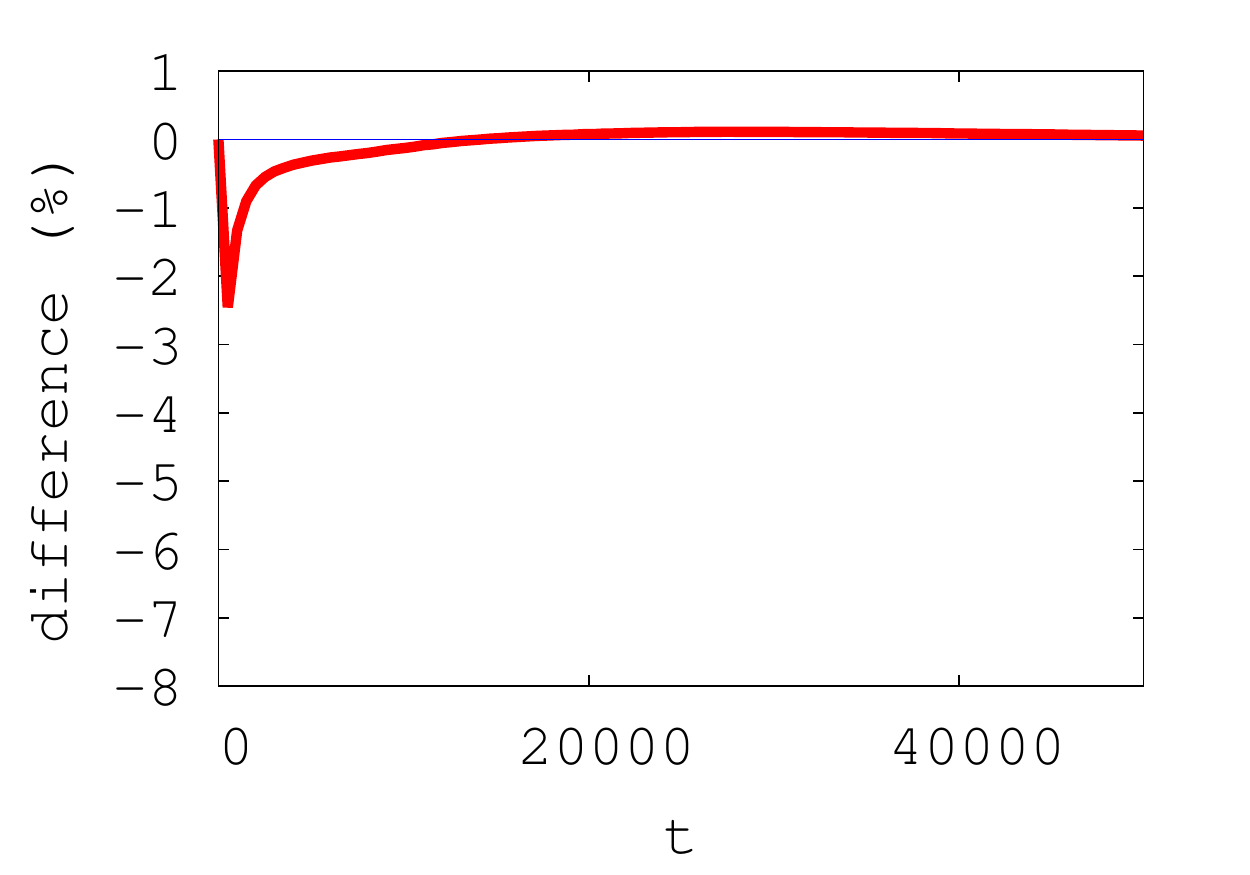}
&
\includegraphics[trim=10bp 25bp 30bp 10bp,clip,width=.15\linewidth]{results_P6_90_Q1_17_UCHOICE_S_EXP_R_UR_0_20_NOLABELS-eps-converted-to}
&
\includegraphics[trim=10bp 25bp 30bp 10bp,clip,width=.15\linewidth]{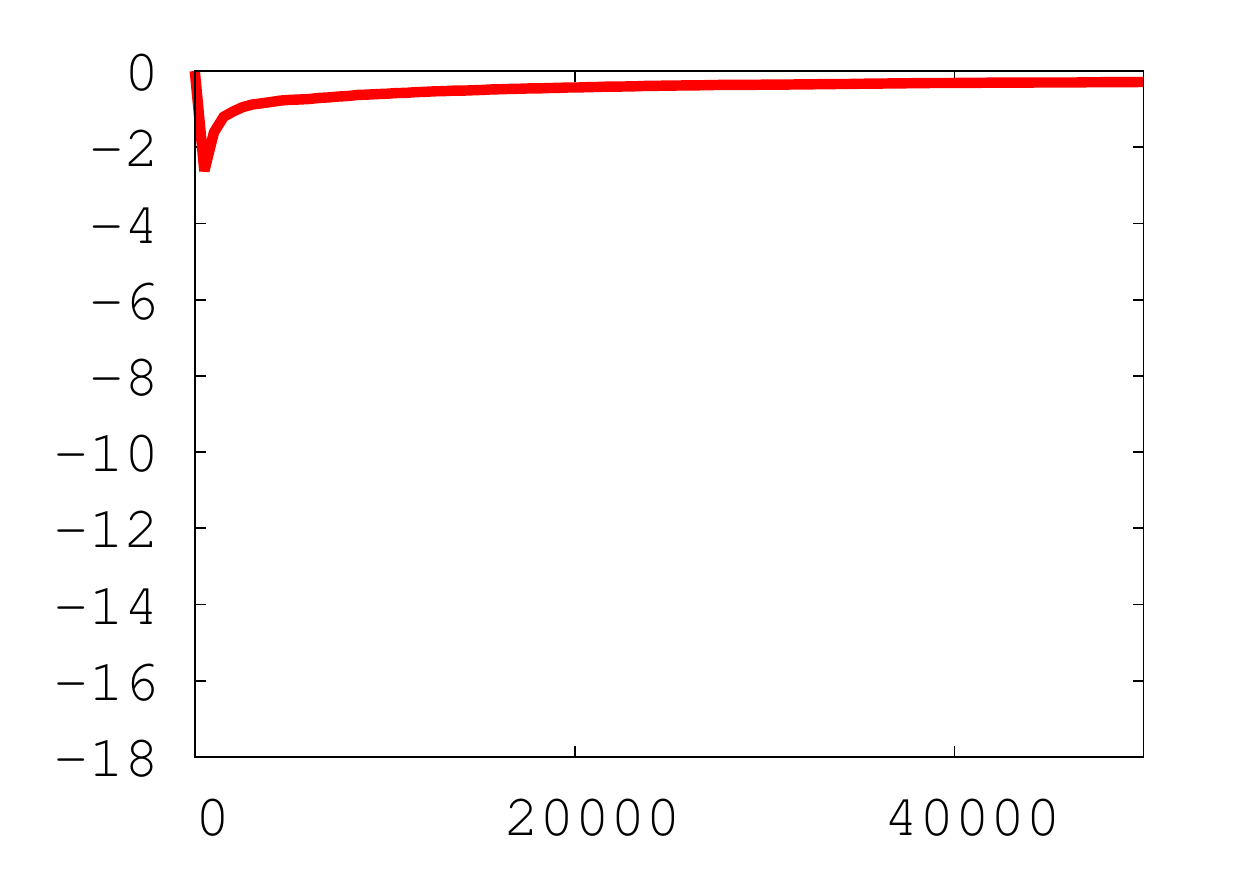}
&
\includegraphics[trim=10bp 25bp 30bp 10bp,clip,width=.15\linewidth]{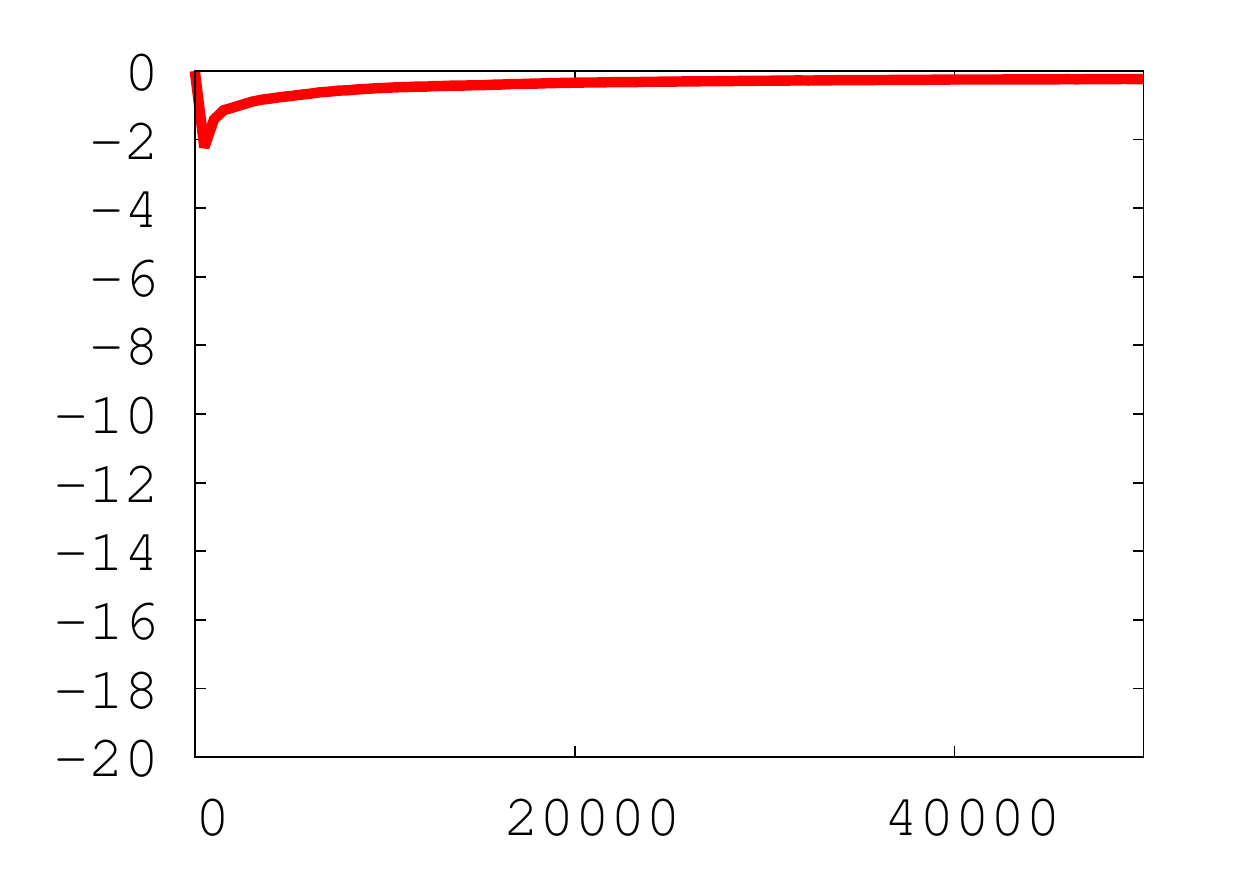}
&
\includegraphics[trim=10bp 25bp 30bp 10bp,clip,width=.15\linewidth]{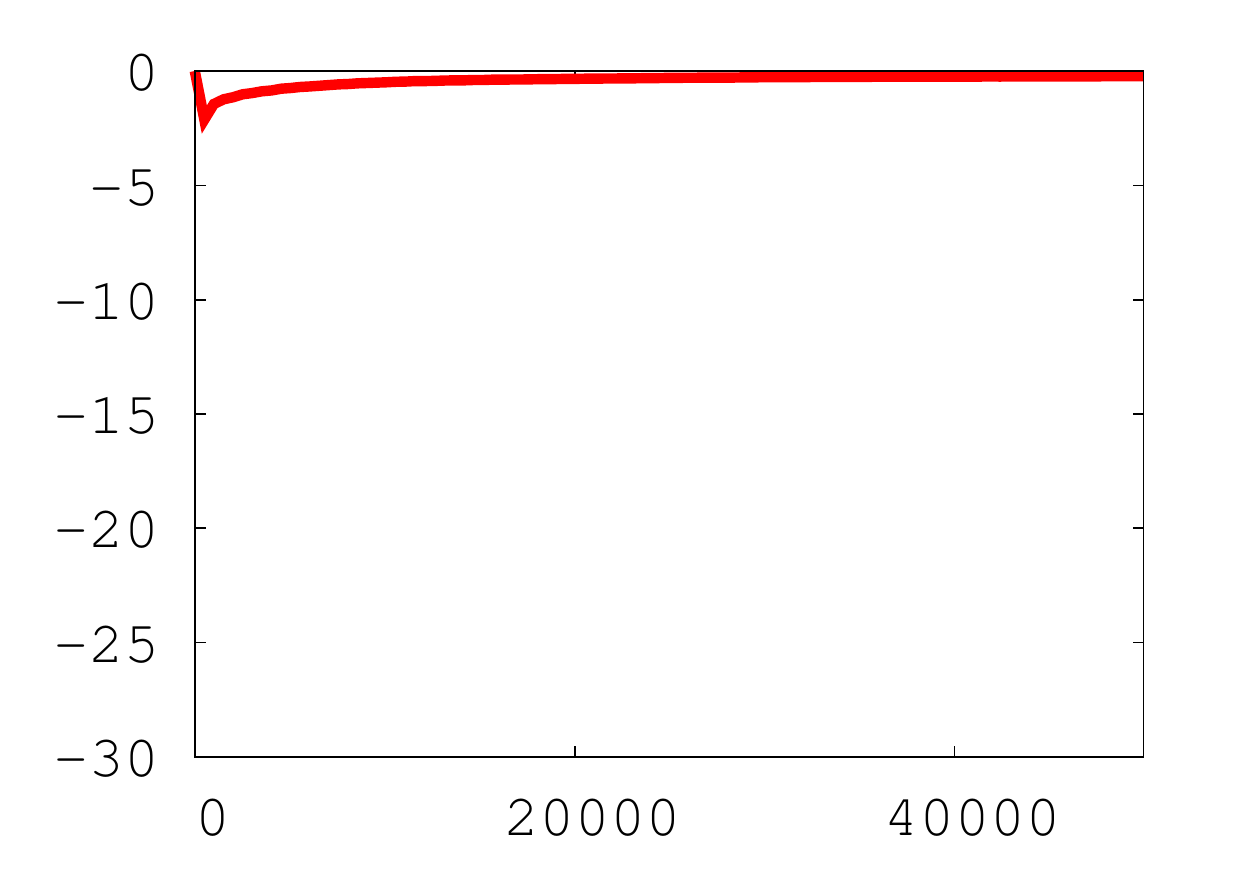}
&
\includegraphics[trim=10bp 25bp 30bp 10bp,clip,width=.15\linewidth]{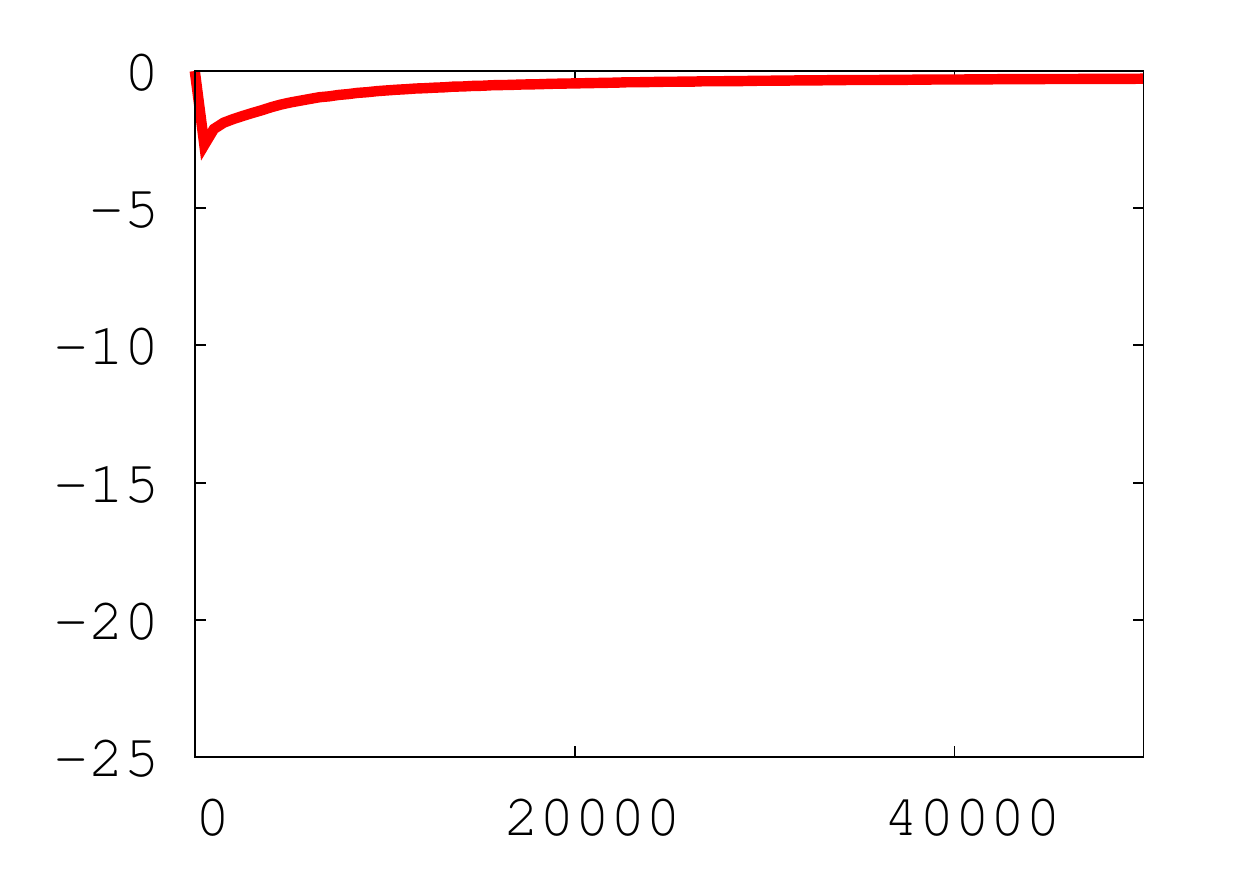}
\\
$\rho = 0.1$ & $\rho = 0.2$ & $\rho = 0.3$ & $\rho = 0.4$ & $\rho = 0.5$ & $\rho = 0.6$ \\\hline
\includegraphics[trim=10bp 25bp 30bp 10bp,clip,width=.15\linewidth]{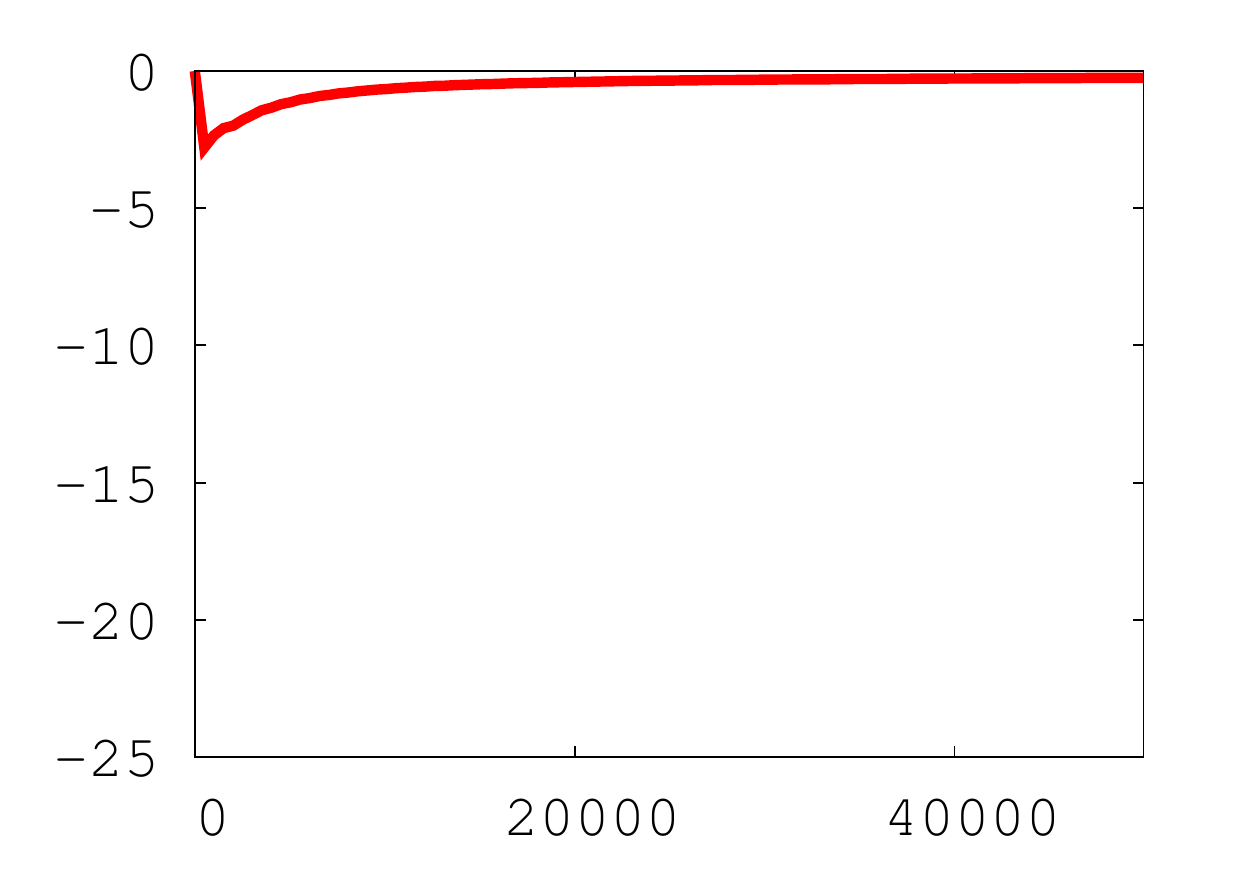}
&
\includegraphics[trim=10bp 25bp 30bp 10bp,clip,width=.15\linewidth]{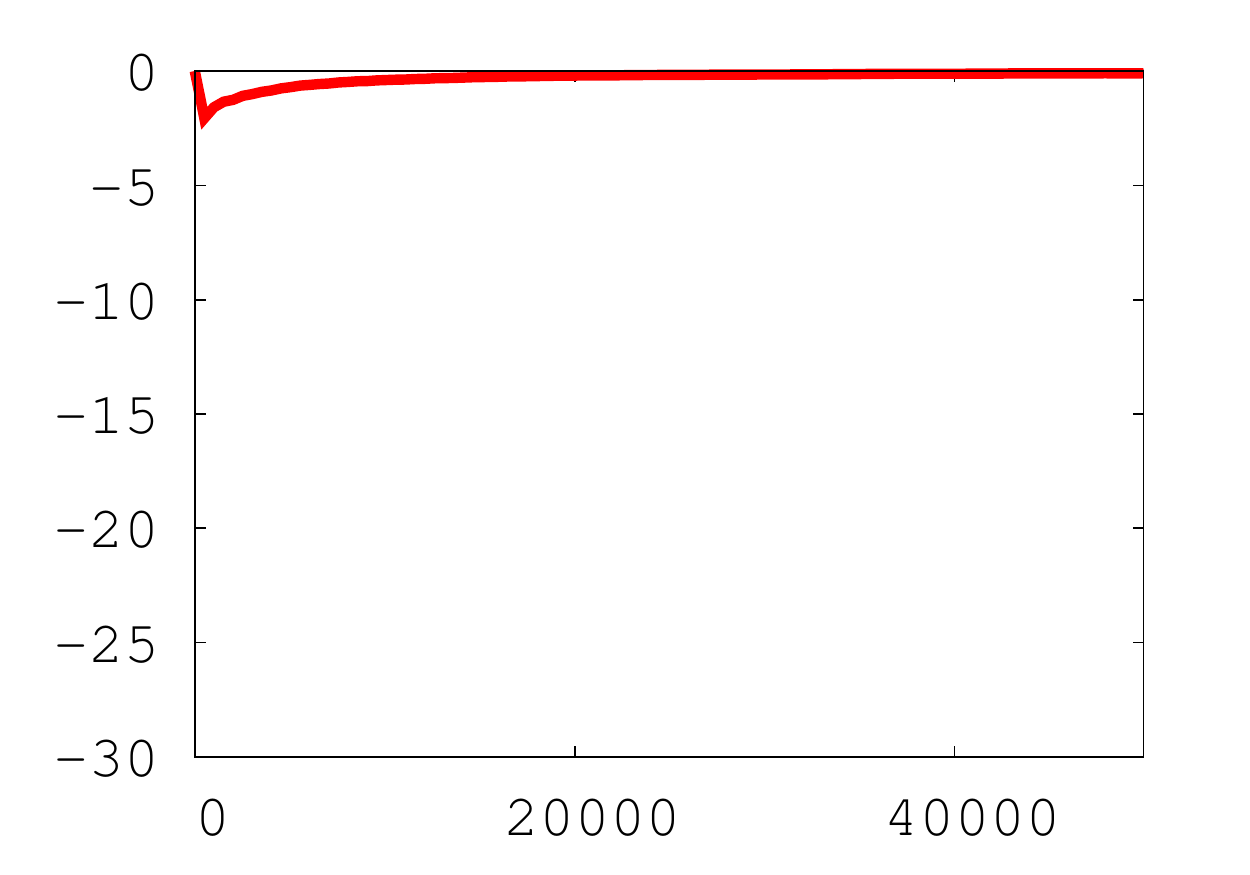}
&
\includegraphics[trim=10bp 25bp 30bp 10bp,clip,width=.15\linewidth]{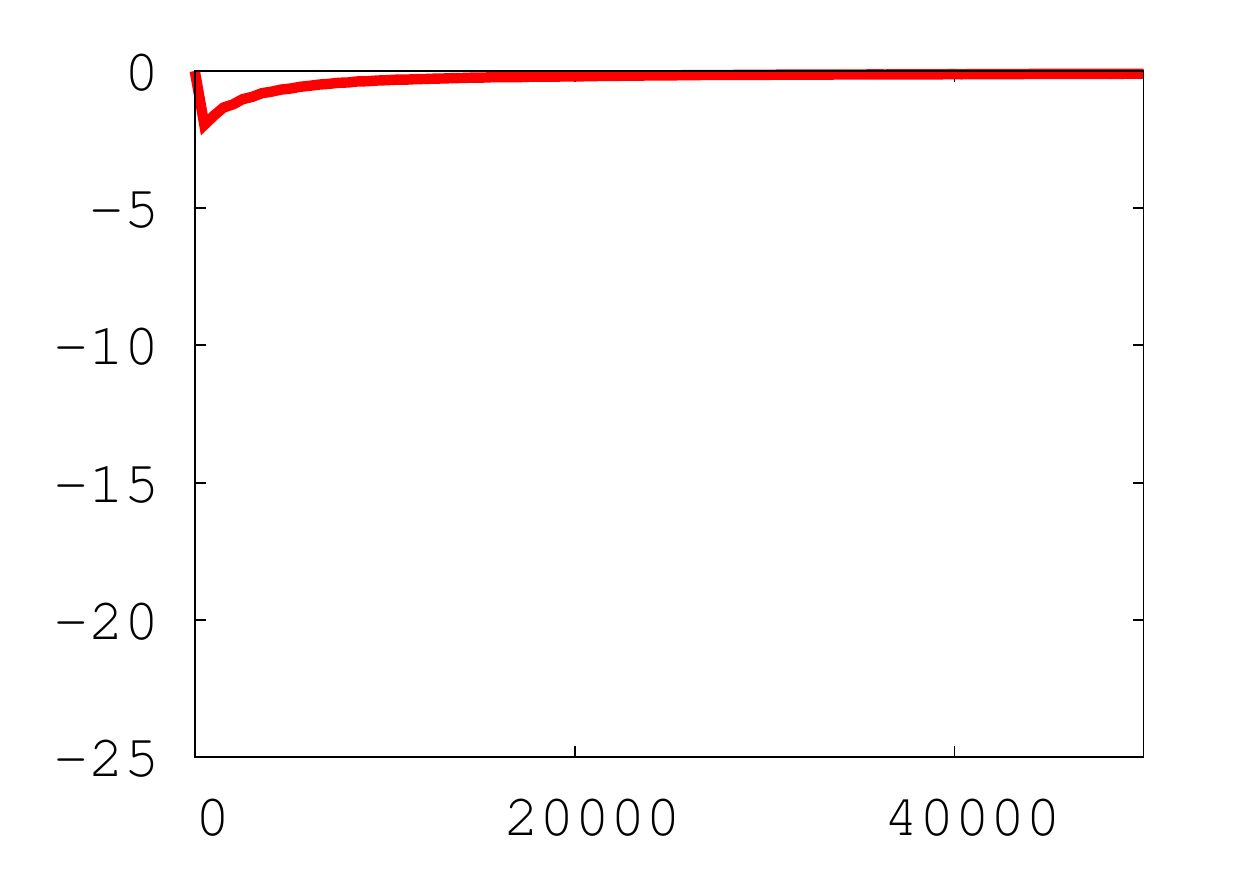}
&
\includegraphics[trim=10bp 25bp 30bp 10bp,clip,width=.15\linewidth]{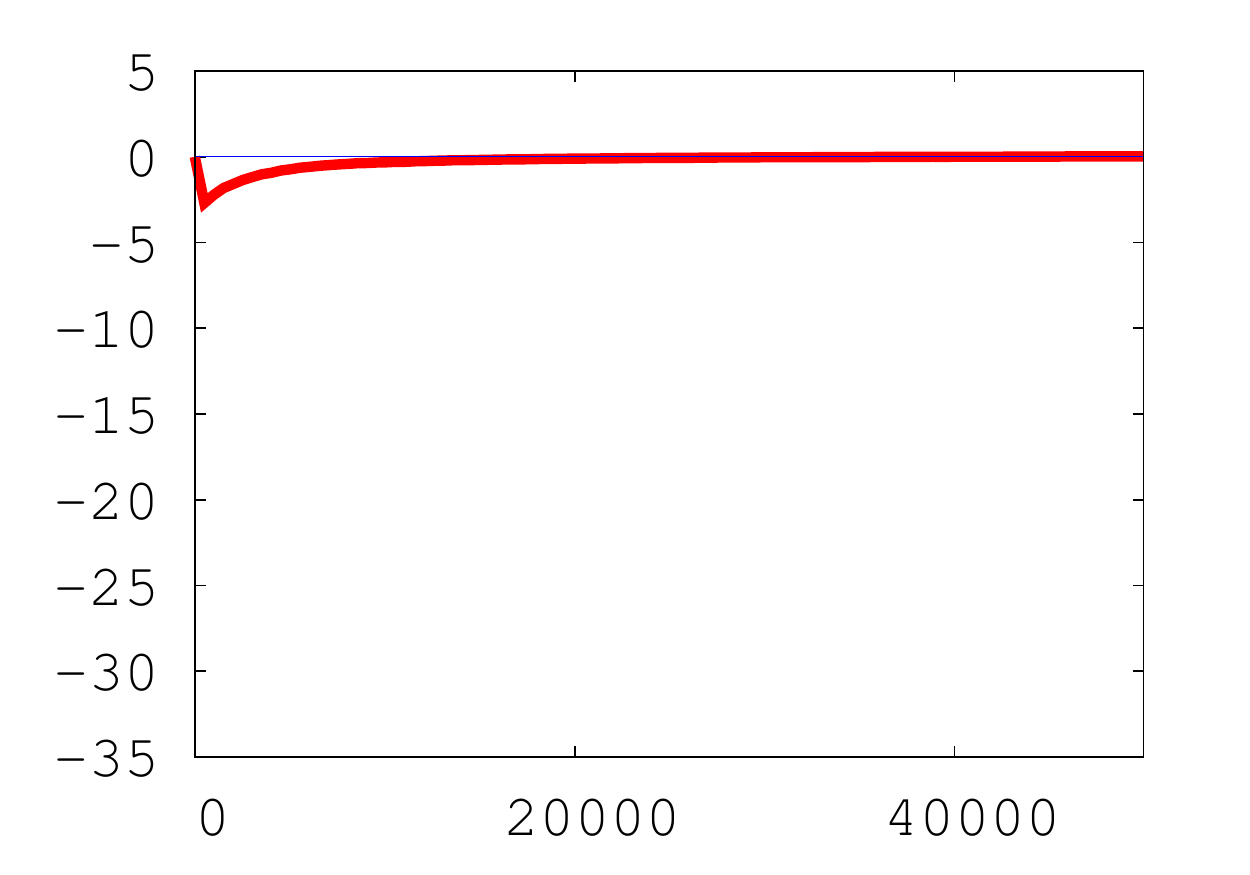}
&
\includegraphics[trim=10bp 25bp 30bp 10bp,clip,width=.15\linewidth]{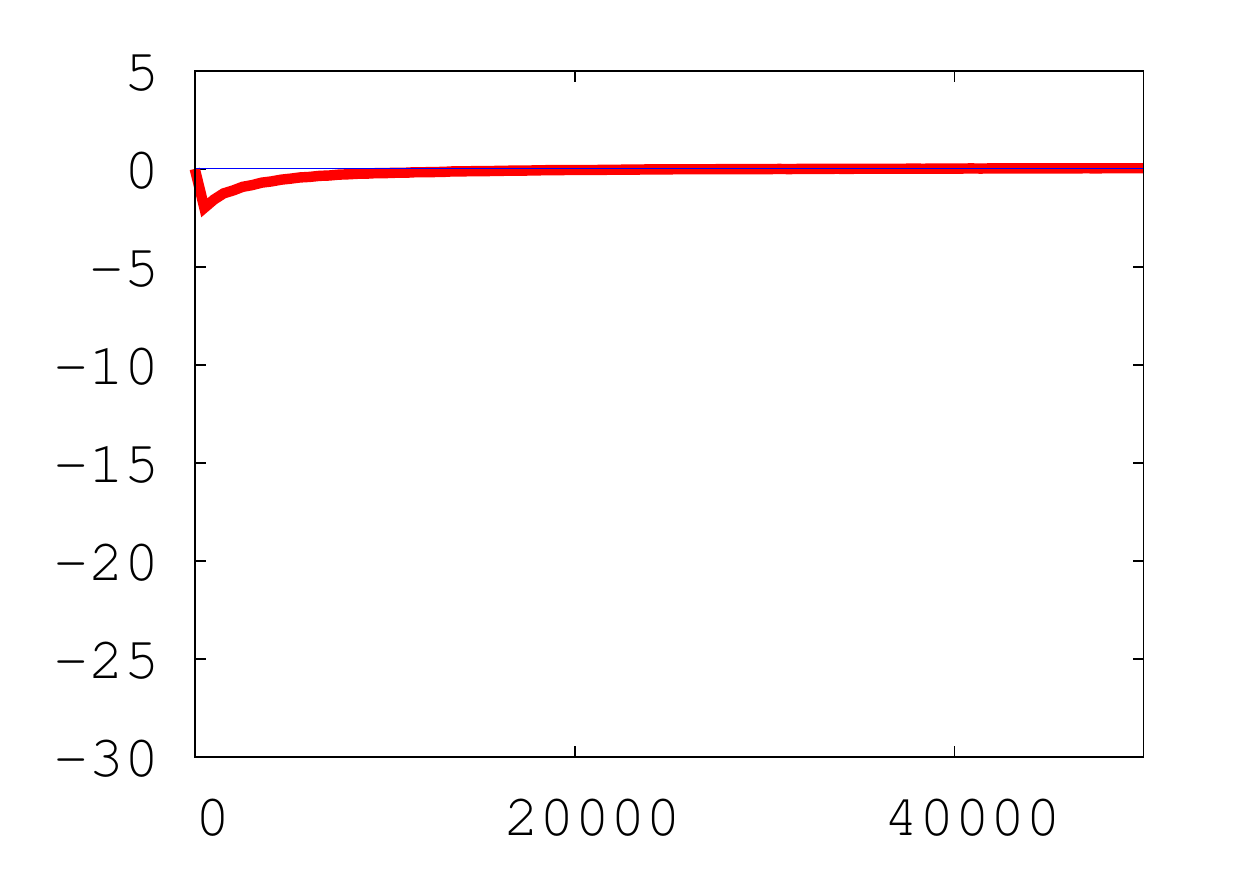}
&
\includegraphics[trim=10bp 25bp 30bp 10bp,clip,width=.15\linewidth]{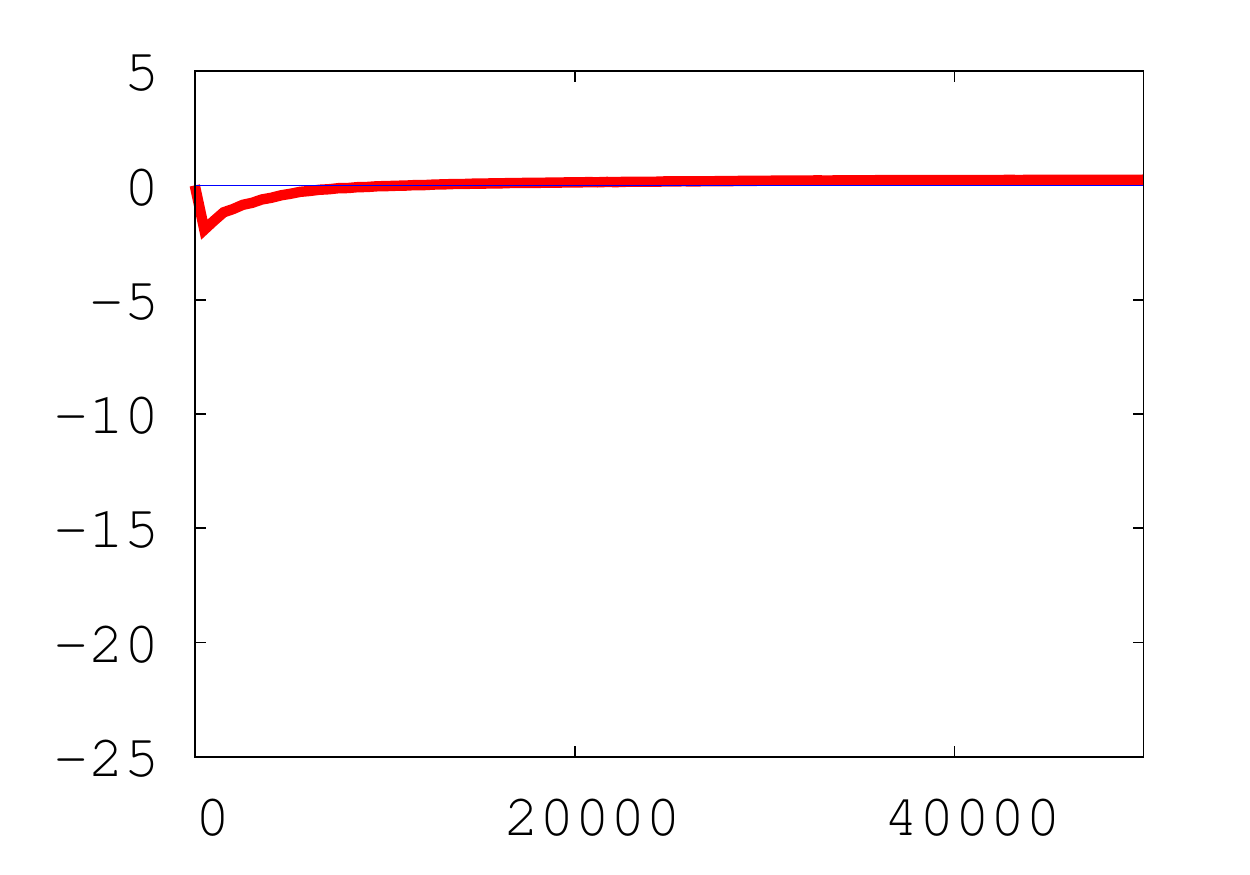}
\\
$\rho = 0.7$ & $\rho = 0.8$ & $\rho = 0.9$ & $\rho = \mbox{\textbf{1.0}}$ & $\rho = 1.1$ & $\rho = 1.2$ \\\hline
\includegraphics[trim=10bp 25bp 30bp 10bp,clip,width=.15\linewidth]{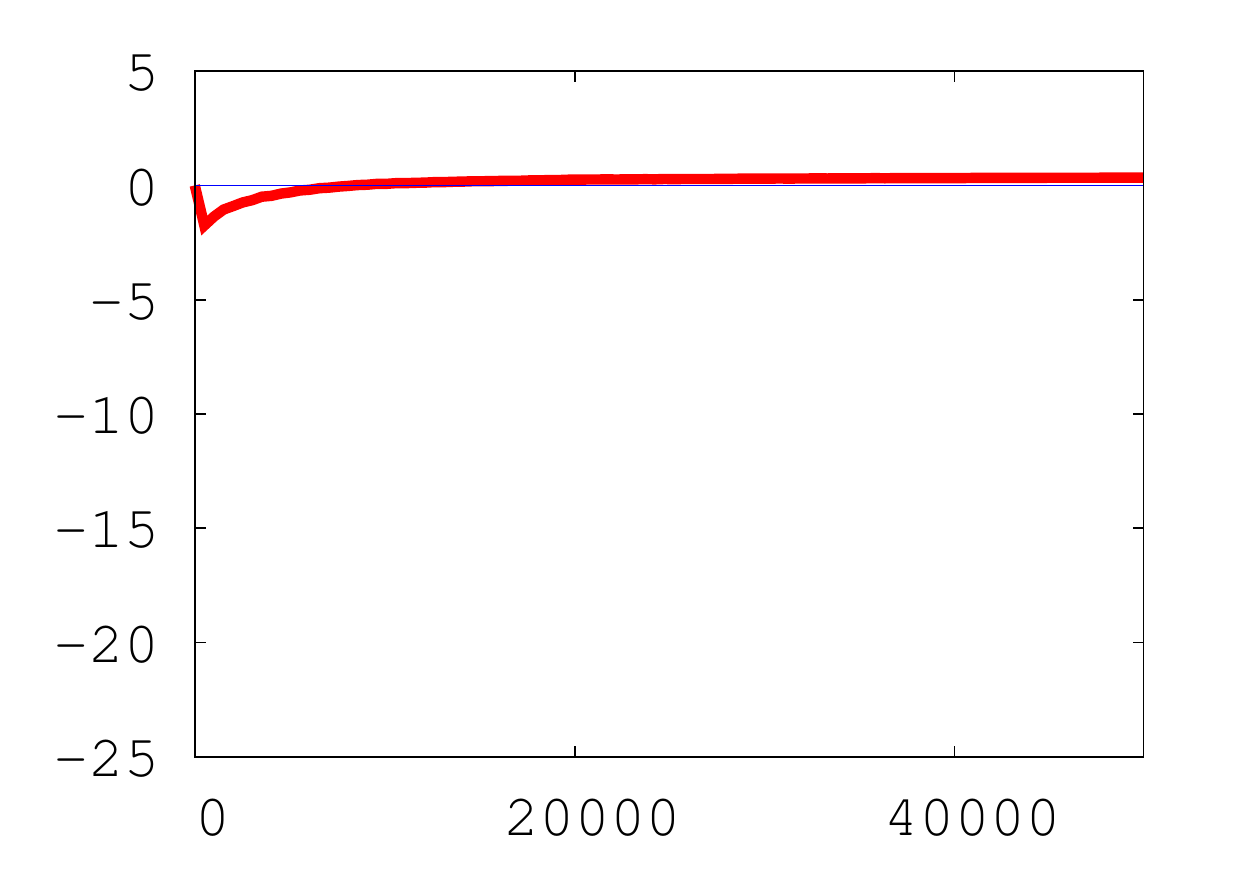}
&
\includegraphics[trim=10bp 25bp 30bp 10bp,clip,width=.15\linewidth]{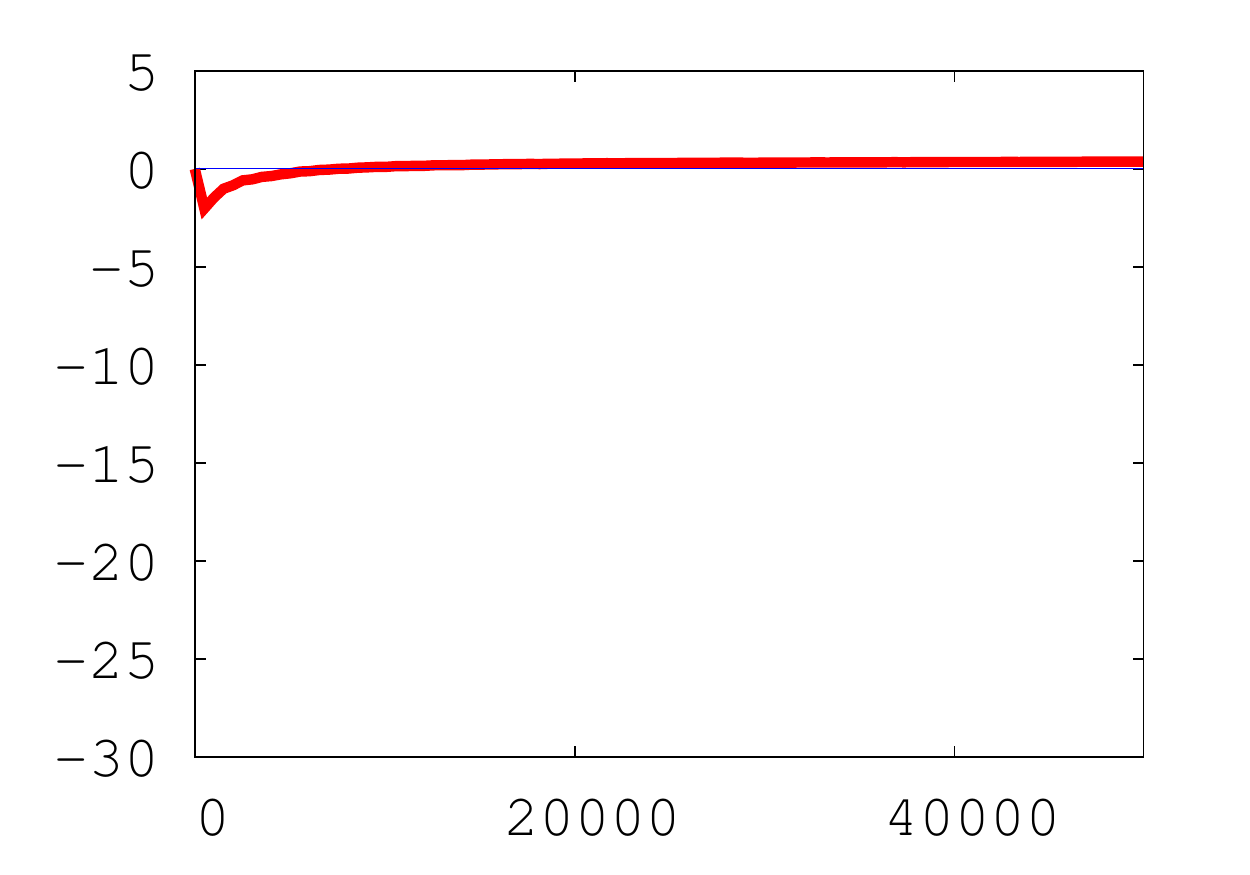}
&
\includegraphics[trim=10bp 25bp 30bp 10bp,clip,width=.15\linewidth]{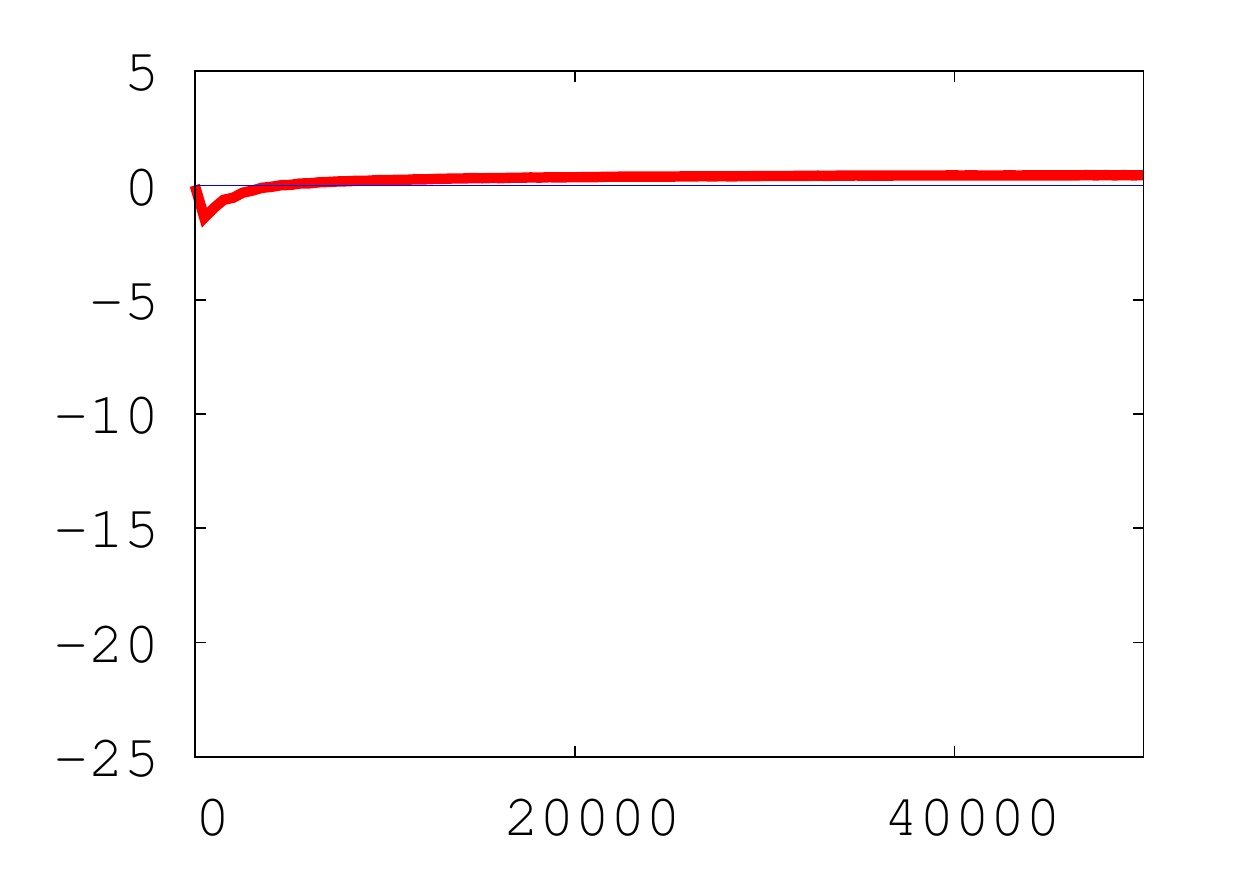}
&
\includegraphics[trim=10bp 25bp 30bp 10bp,clip,width=.15\linewidth]{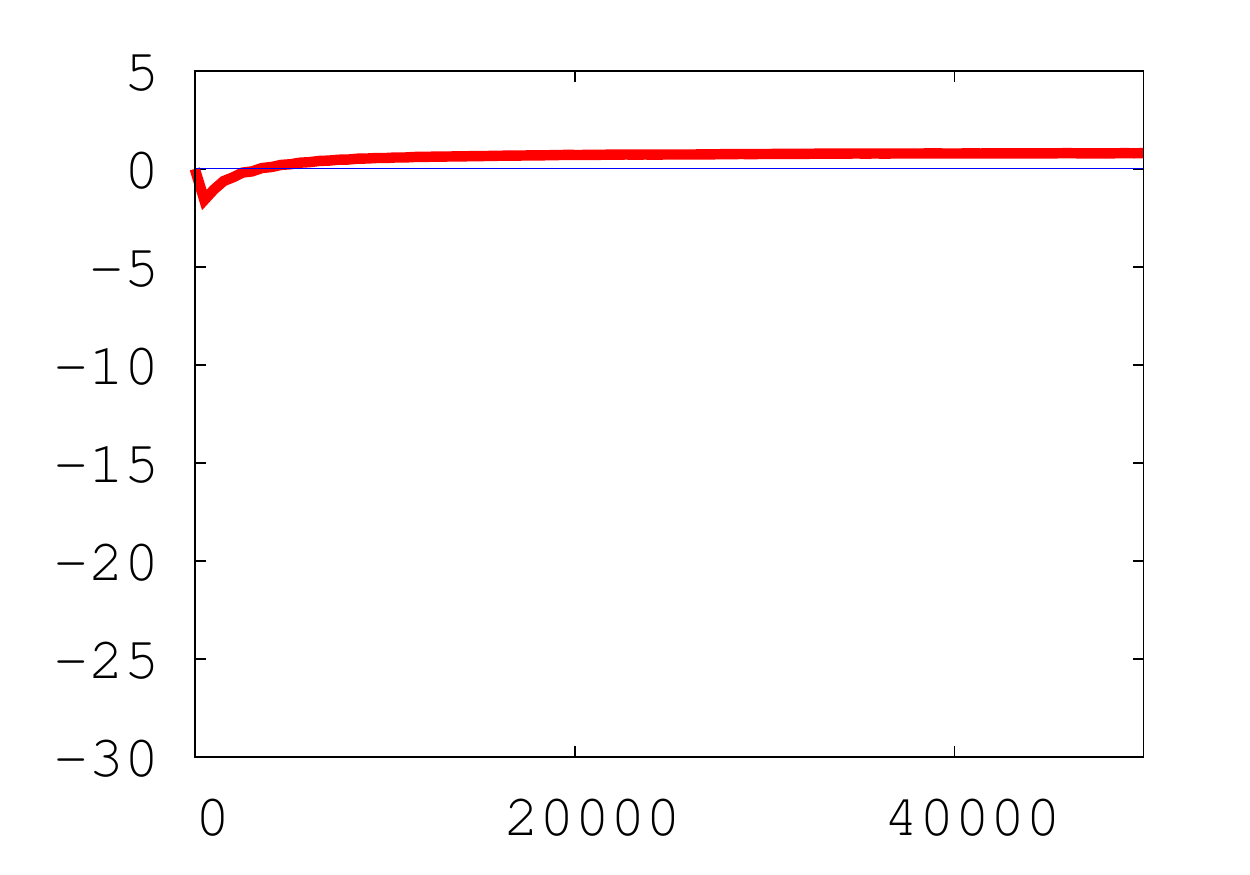}
&
\includegraphics[trim=10bp 25bp 30bp
10bp,clip,width=.15\linewidth]{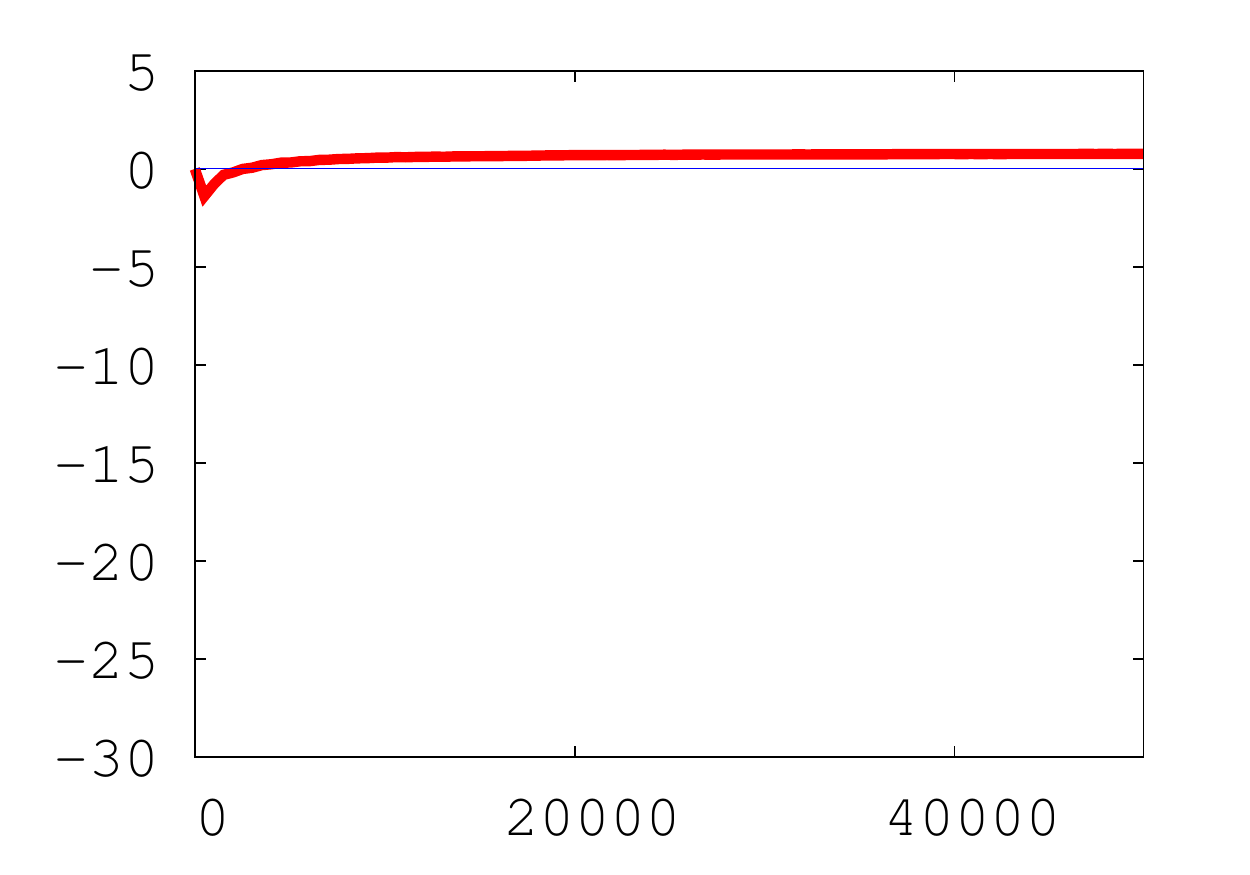}
&\\
$\rho =1.3$ & $\rho = 1.4$ & $\rho = 1.5$ & $\rho = 1.6$ & $\rho =
1.7$ &  \\\hline \hline
\end{tabular}
}
\end{center}
\caption{Error($p$-LMS) - Error(DN-$p$-LMS) as a function of $t$ ($\in \{1, 2, ..., 50 000\}$), $\bm{u}$ = sparse, $(p,q) = (6.9, 1.17)$.}
  \label{tc1_supp_rr6}
\end{sidewaystable}

\section{Comment: Theorem \ref{th00} is a scaled isometry in disguise (sometimes)}
\label{app:scaled-iso}

\begin{figure}[t]
\centering
\scalebox{0.925}{
\begin{tabular}{c}
\includegraphics[width=.8\linewidth]{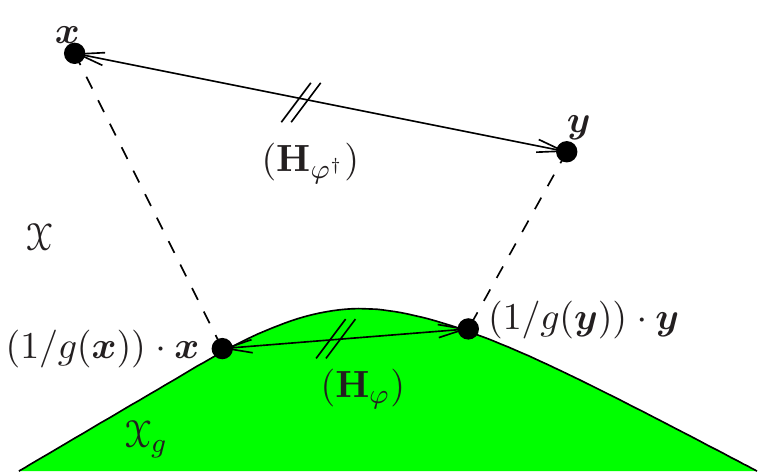}
\end{tabular}
}

\caption{A depiction of the adaptive isometry that Theorem \ref{th00} provides.
	\label{fig:iso1}} 
\vspace{-0.125in}  
\end{figure}

Theorem \ref{th00} states in fact an isometry under some conditions, but an adaptive one in
the sense that metrics involved rely on all parameters, and in
particular on the points involved in the divergences (See Figure \ref{fig:iso1}). Indeed, a simple
Taylor expansion of the equation (\ref{eqBreg}) (main file) shows that
any such Bregman distortion with a twice differentiable generator can
be expressed as:
\begin{eqnarray}
D_{\varphi}(
  \ve{x}\|\ve{y}) & = & \frac{1}{2}\cdot (\ve{x} - \ve{y})^\top \textbf{H}_{\varphi} (\ve{x} - \ve{y}) \:\:, \label{eqBreg2}
\end{eqnarray}
for \textit{some} value of the Hessian $\textbf{H}_{\varphi}$
depending on $\ve{x}, \ve{y}$ (see for example \cite[Appendix
I]{kwhTP}, \citep{anMO}). Hence,
under the constraint that both $\varphi$ and $\varphi^{\dagger}$ are
twice differentiable, eq. (\ref{eq11}) becomes
\begin{eqnarray}
g(\ve{x}) \cdot \left(\frac{1}{g(\ve{x}) } \cdot \ve{x} -
    \frac{1}{g(\ve{y}) } \cdot \ve{y} \right)^\top
\textbf{H}_{\varphi} \left(\frac{1}{g(\ve{x}) } \cdot \ve{x} -
    \frac{1}{g(\ve{y}) } \cdot \ve{y}\right)  & = & (\ve{x} -
  \ve{y})^\top \textbf{H}_{\varphi^\dagger} (\ve{x} - \ve{y}) \:\:.\label{eq11b}
\end{eqnarray}
Assuming $g$ non-negative (which, by the way, enforces the convexity
of $\varphi^{\dagger}$), we get by taking square roots,
\begin{eqnarray}
\sqrt{g(\ve{x})}\cdot \left\| \frac{1}{g(\ve{x}) } \cdot \ve{x} -
    \frac{1}{g(\ve{y}) } \cdot \ve{y}  \right\|_{\textbf{H}_{\varphi}}
  & = & \left\| \ve{x} - \ve{y}\right\|_{\textbf{H}_{\varphi^\dagger}}\:\:,\label{eqISO}
\end{eqnarray}
which is a scaled isometry relationship between $\XCal_g$ (left) and
$\XCal$ (right), but again the metrics involved depend on the
arguments. Nevertheless, eq. (\ref{eqISO}) displays a sophisticated
relationship between distances in $\XCal_g$ and in
$\XCal$ which may prove useful in itself.

\setlength{\bibsep}{0.5pt plus 0.5ex}
\bibliography{bibgen}
}{
\setlength{\bibsep}{0pt plus 0.1ex}
{\footnotesize
\bibliography{bibgen}
}
}

\end{document}